\documentclass[12pt,twoside]{report}

\usepackage{graphicx} % Required for inserting images
\usepackage{amsfonts}
\usepackage{amsmath}
\usepackage{amssymb}
\usepackage{amsthm}
\usepackage{stmaryrd}
\usepackage{bbold}
\usepackage{parskip}  % to \par instead of \\ without hbox errors
\usepackage{listings}
\usepackage{csquotes}
\usepackage{float}
\usepackage[sorting=none]{biblatex}
\newcommand{\reporttitle}{Combinatorial Complex Score-based Diffusion Modelling through Stochastic Differential Equations}

\newcommand{\reportauthor}{Adrien Carrel}
\newcommand{\supervisor}{Dr. Tolga Birdal}
\newcommand{\secondMarker}{Dr. Pedro Mediano}
\newcommand{\degreetype}{Advanced Computing}

\usepackage[a4paper,hmargin=2.8cm,vmargin=2.0cm,includeheadfoot]{geometry}
\usepackage{textpos}
\usepackage{tabularx,longtable,multirow,subfigure,caption}%hangcaption
\usepackage{fncylab} %formatting of labels
\usepackage{fancyhdr} % page layout
\usepackage{url} % URLs
\usepackage[english]{babel}
\usepackage{amsmath}
\usepackage{graphicx}
\usepackage{dsfont}
\usepackage{epstopdf} % automatically replace .eps with .pdf in graphics
\usepackage{array}
\usepackage{latexsym}
\usepackage[pdftex,hypertexnames=false,colorlinks]{hyperref} % provide links in pdf [pagebackref]

\hypersetup{pdftitle={},
  pdfsubject={}, 
  pdfauthor={},
  pdfkeywords={}, 
  pdfstartview=FitH,
  pdfpagemode={UseOutlines},% None, FullScreen, UseOutlines
  bookmarksnumbered=true, bookmarksopen=true, colorlinks,
    citecolor=black,%
    filecolor=black,%
    linkcolor=black,%
    urlcolor=black}

\usepackage[all]{hypcap}

\def\@makechapterhead#1{%
  \vspace*{10\p@}%
  {\parindent \z@ \raggedright \sffamily
    \interlinepenalty\@M
    \Huge\bfseries \thechapter \space\space #1\par\nobreak
    \vskip 30\p@
  }}

\def\@makeschapterhead#1{%
  \vspace*{10\p@}%
  {\parindent \z@ \raggedright
    \sffamily
    \interlinepenalty\@M
    \Huge \bfseries  #1\par\nobreak
    \vskip 30\p@
  }}

\allowdisplaybreaks

\newcommand{\E}[0]{\mathbb{E}}  % expectancy

\date{September 2023}

\theoremstyle{plain}

\newtheorem{theorem}{Theorem}  % [section]
\newtheorem{definition}{Definition}  % [section]
\newtheorem{layer}{Layer}
\newtheorem{model}{Model}
\newtheorem{proposition}{Proposition}
\newtheorem*{remark}{Remark}
\newtheorem{corollary}[theorem]{Corollary}

\begin{document}

\renewcommand\qedsymbol{$\blacksquare$}  % square end of proof

% load title page
% Last modification: 2015-08-17 (Marc Deisenroth)
\begin{titlepage}

\newcommand{\HRule}{\rule{\linewidth}{0.5mm}} % Defines a new command for the horizontal lines, change thickness here

%----------------------------------------------------------------------------------------
%	LOGO SECTION
%----------------------------------------------------------------------------------------

\includegraphics[width = 4cm]{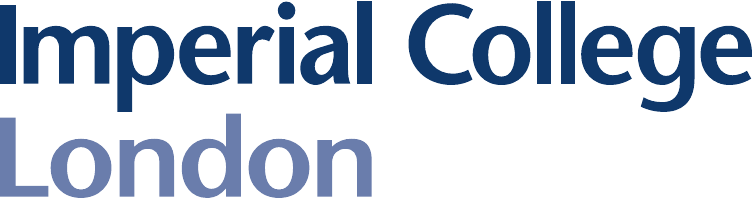}\\[0.5cm] 

\center % Center remainder of the page

%----------------------------------------------------------------------------------------
%	HEADING SECTIONS
%----------------------------------------------------------------------------------------

\textsc{\Large Imperial College London}\\[0.5cm] 
\textsc{\large Department of Computing}\\[0.5cm] 

%----------------------------------------------------------------------------------------
%	TITLE SECTION
%----------------------------------------------------------------------------------------

\HRule \\[0.4cm]
{ \huge \bfseries \reporttitle}\\ % Title of your document
\HRule \\[1.5cm]
 
%----------------------------------------------------------------------------------------
%	AUTHOR SECTION
%----------------------------------------------------------------------------------------

\begin{minipage}{0.4\textwidth}
\begin{flushleft} \large
\emph{Author:}\\
\reportauthor % Your name
\end{flushleft}
\end{minipage}
~
\begin{minipage}{0.4\textwidth}
\begin{flushright} \large
\emph{Supervisor:} \\
\supervisor\\ % Supervisor's Name
\end{flushright}
\end{minipage}\\[1cm]

\begin{minipage}{0.4\textwidth}
\begin{flushleft} \large
 
\end{flushleft}
\end{minipage}
~
\begin{minipage}{0.4\textwidth}
\begin{flushright} \large
\emph{Second Marker:} \\
\secondMarker % Second marker's Name
\end{flushright}
\end{minipage}\\[2cm]

\includegraphics[width = 5cm]{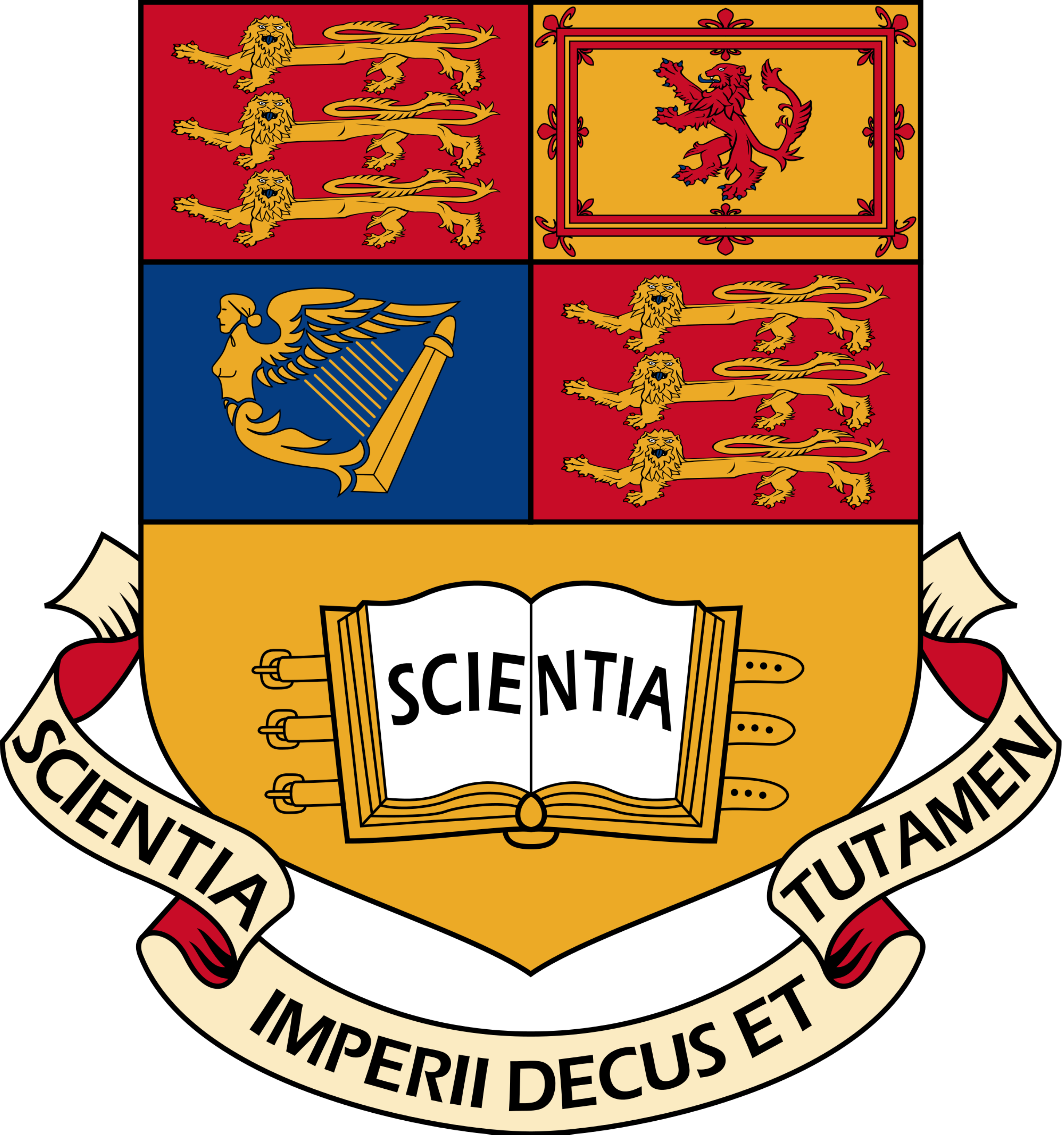}\\[0.5cm]

%----------------------------------------------------------------------------------------
%	FOOTER & DATE SECTION
%----------------------------------------------------------------------------------------
\vfill % Fill the rest of the page with whitespace
Submitted in partial fulfillment of the requirements for the MSc degree in
\degreetype~of Imperial College London\\[0.5cm]

\makeatletter
\@date 
\makeatother

\end{titlepage}

% page numbering etc.
\pagenumbering{roman}
\clearpage{\pagestyle{empty}\cleardoublepage}
\setcounter{page}{1}
\pagestyle{fancy}

%%%%%%%%%%%%%%%%%%%%%%%%%%%%%%%%%%%%
\begin{abstract}
Graph structures offer a versatile framework for representing diverse patterns in nature and complex systems, applicable across domains like molecular chemistry, social networks, and transportation systems. While diffusion models have excelled in generating various objects, generating graphs remains challenging. This thesis explores the potential of score-based generative models in generating such objects through a modelization as combinatorial complexes, which are powerful topological structures that encompass higher-order relationships.\par

In this thesis, we propose a unified framework by employing stochastic differential equations. We not only generalize the generation of complex objects such as graphs and hypergraphs, but we also unify existing generative modelling approaches such as Score Matching with Langevin dynamics and Denoising Diffusion Probabilistic Models. This innovation overcomes limitations in existing frameworks that focus solely on graph generation, opening up new possibilities in generative AI.\par

The experiment results showed that our framework could generate these complex objects, and could also compete against state-of-the-art approaches for mere graph and molecule generation tasks.
\end{abstract}

\cleardoublepage
%%%%%%%%%%%%%%%%%%%%%%%%%%%%%%%%%%%%
\section*{Acknowledgments}
\label{Acknowledgments}

I would like to express my heartfelt gratitude to the following individuals and entities who have played a significant role in the realization of this thesis:\par

First and foremost, I extend my sincere appreciation to Dr. Tolga Birdal for his exceptional guidance and support throughout this academic journey. His advice has been pivotal and his experience has been a constant source of inspiration. I am also grateful to Dr. Pedro Mediano for co-marking my thesis.\par

I am indebted to Leo Anthony Celi for opening the doors to numerous research opportunities and for facilitating my engagement with the Laboratory for Computational Physiology at the esteemed Massachusetts Institute of Technology, where I had the privilege of being partially hosted during the formulation of this thesis.\par

I would like to thank my professors, both during my time in Classes Préparatoires at Lycée Pierre Corneille and Lycée Hoche, and later at CentraleSupélec and Imperial College London, for imparting the foundational mathematical skills and fostering the sense of rigour that enabled me to undertake this thesis.\par

My heartfelt thanks go to my parents and my sister for their unending support and belief in me.\par

Last, but not least, I am deeply thankful to all my friends for the camaraderie and the joy over the years. Your presence has added a cherished dimension to my academic journey.\par

\clearpage{\pagestyle{empty}\cleardoublepage}

%%%%%%%%%%%%%%%%%%%%%%%%%%%%%%%%%%%%
%--- table of contents
\fancyhead[RE,LO]{\sffamily {Table of Contents}}
\tableofcontents

%--- list figures and tables
\listoffigures
\listoftables
\newpage

\clearpage{\pagestyle{empty}\cleardoublepage}
\pagenumbering{arabic}
\setcounter{page}{1}
\fancyhead[LE,RO]{\slshape \rightmark}
\fancyhead[LO,RE]{\slshape \leftmark}

%%%%%%%%%%%%%%%%%%%%%%%%%%%%%%%%%%%%
\chapter{Introduction}
\label{Introduction}

Graph structures have established themselves as an elegant and versatile language capable of encapsulating diverse patterns in natural and complex systems \cite{velickovic_everything_2023}. Their utility spans various domains, including molecular chemistry, social networks, and transportation systems, where we need to represent interconnected entities. Take molecules, for instance, where atoms and bonds translate seamlessly into nodes and edges within a graph, enabling the utilization of geometric deep learning techniques like graph neural networks \cite{velickovic_everything_2023, bronstein2021geometric}. These techniques empower researchers to grasp the structural and functional properties of molecules, a knowledge that can catalyze the design of novel compounds with specific attributes.\par

Lately, diffusion models, specifically score-based generative models and related methodologies \cite{NEURIPS2019_378a063b, bordes2017learning, goyal2017variational}, have risen to prominence in the realm of generative artificial intelligence (generative AI). They have delivered remarkable results in a wide array of fields, from text-to-image generation models such as Imagen \cite{saharia2022photorealistic}, Dall-E \cite{ramesh2021zeroshot}, and Stable Diffusion \cite{rombach2022highresolution}, to audio \cite{chen2020wavegrad, kong2021diffwave, segre}, shapes \cite{cai2020learning}, as well as tabular data generation \cite{kotelnikov2022tabddpm}. However, despite this impressive repertoire, the generation of graphs continues to present a challenge \cite{zhang_survey_2023}.\par

The advent of diffusion models within the field of geometric deep learning has ushered in new prospects and advancements. Graph generation carries immense potential in diverse domains, encompassing molecular conformation generation, drug discovery, protein design (both structure and sequence), motif-conditioned structure design, antibody design, and materials design \cite{zhang_survey_2023}. Nevertheless, the complexity of graphs, their varying sizes, properties, and potential higher-order relationships, have posed formidable hurdles in the quest to develop models capable of effective generalization across different graph types.\par

To address these obstacles and broaden the horizons of generative AI models, we explore mathematical structures that can adeptly represent higher-order relationships between entities. While graphs excel at encoding binary relations, alternative structures such as simplicial and cell complexes, hypergraphs, and combinatorial complexes excel at capturing hierarchical and more complex relationships \cite{papillon2023architectures, hajij_topological_2023}. Combinatorial complexes (CCs), in particular, stand out as a potent framework that generalizes the aforementioned structures (see Figure \ref{fig:topo}). Armed with both set-type relations and a hierarchy among these relations, CCs may hold the key to preserving complex higher-order relations - an essential facet in learning representations of intricate objects and consequently, constructing robust generative models.\par

In this thesis, we extend the capabilities of score-based generative models by introducing a unified framework. Our model exhibits generality on two distinct fronts: a generalized perspective with our score-based modelling approach and a generalization due to the generation of combinatorial complexes. Firstly, we approach diffusion through the prism of stochastic differential equations (SDEs), thus offering a generalized perspective of the main score-based generative modelling frameworks. As elucidated in \cite{jo2022scorebased}, Score matching with Langevin dynamics (SMLD) \cite{sohldickstein2015deep, ho2020denoising} and Denoising Diffusion Probabilistic Models (DDPM) \cite{NEURIPS2020_4c5bcfec} harmonize into our framework as discretizations of two separate SDEs. Secondly, our model generates combinatorial complexes (CCs), allowing us to not only generate graphs but also more intricate and higher-dimensional topological entities that serve as special cases of combinatorial complexes, including hypergraphs and simplicial complexes. Tasks like molecule generation, previously intertwined with graph generation, can now be more naturally generalized within our proposed framework. By applying transformation procedures to molecules, graphs can be seamlessly converted into CCs, preserving intrinsic geometric attributes like rings. This capacity to incorporate higher-order relations circumvents limitations observed in existing frameworks that exclusively generate graphs, thereby charting a novel path for generative AI.\par

We hope our framework could kickstart a new era of generative models by enabling the synthesis of a rich spectrum of topological structures through the generation of combinatorial complexes. Based on this baseline work, we hope researchers from different domains can work together to push the boundaries of generative AI.

\section{Objectives}
\label{Objectives}

Our objective is to develop a score-based diffusion model capable of generating combinatorial complexes, addressing the limitations of existing graph generation methods by preserving complex higher-order relationships. We aim to assess the model's performance, scalability, and efficiency across different datasets and explore applications in domains such as drug discovery where we need to generate molecules, with the hope of potentially revolutionizing these fields.\par

% \section{Challenges}
% \label{Challenges}

% Our work faces several challenges inherent to the generation of combinatorial complexes and diffusion models, as elaborated in the Limitations section (Section \ref{limitation}). Firstly, the irregular structures of topological objects, including graphs and combinatorial complexes, with varying node and edge counts, present a substantial obstacle. These structures also exhibit diverse and heterogeneous properties \cite{hamilton2018representation}, making generalized model development a difficult task. The second challenge pertains to the scalability of combinatorial complexes. As their complexity and size increase, the computational resources required for generation can become limiting. Lastly, the high computational complexity intrinsic to diffusion models adds another layer of difficulty. Successfully addressing these challenges is paramount for the creation of an accurate and efficient generative model.\\

\section{Contributions}
\label{Contributions}

The work presented in the thesis is driven by the need for a unified framework capable of generating diverse topological structures with high fidelity. Our contributions could be summarized as follows:\par

\begin{itemize}
    \item We introduce CCSD (Combinatorial Complex Score-based Diffusion), a pioneering framework for generating combinatorial complexes. This framework extends beyond previous graph-centric approaches and incorporates score-based generative model techniques using stochastic differential equations (SDEs). We provide formal proofs of its convergence and practical implementability.
    \item We introduce novel mathematical objects that position combinatorial complexes within the broader context of generative AI.
    \item We design and redefine operators to facilitate neural network architectures capable of handling higher-order topological structures like CCs.
    \item We propose new layers and neural network architectures tailored for learning partial score functions.
    \item We develop a procedure to transform (lift) objects, such as molecules typically represented as lower-dimensional graphs, to combinatorial complexes. This includes a modified version of the path-based lifting procedure.
    \item Pioneering the generation of objects with higher dimensionality than graphs, we devise new metrics to assess the quality of generated combinatorial complexes in comparison to the original object distribution.
    \item We offer a Python library, CCSD \cite{Carrel_CCSD_-_Combinatorial_2023}, facilitating model training on diverse datasets and enabling sampling from our models. This library is well-documented and incorporates best practices as well as an extensive suite of unit test functions.
    \item We comprehensively evaluate the framework on diverse datasets by benchmarking our framework against state-of-the-art models in graph and molecule generation tasks.
\end{itemize}

In essence, our work stands at the forefront of what we can call \textbf{Generative Topological Deep Learning}, a pioneering field introduced by this thesis, and that is dedicated to the development of deep learning models proficient in generating diverse topological structures.\par

\section{Outline}
\label{Outline}

We structure this thesis as follows. First, we provide a foundational exploration of relevant literature in Chapter \ref{Background} (Background). This chapter delves into the realms of graph diffusion models and the integration of topology within deep learning and graph generation. Following this, Chapter \ref{Preliminaries} (Preliminaries) delves into the mathematics behind diffusion models and topological deep learning, providing readers with a foundational understanding of these models' origins and their capability. We believe that a deep understanding of these concepts is crucial to grasp the specificities of our approach. In Chapter \ref{Theory} (Theoretical Contributions), we present our theoretical contributions, introducing CCSD (Combinatorial Complex Score-based Diffusion), our innovative diffusion framework for generating combinatorial complexes. Chapter \ref{Implementation} (Implementation) transitions into the practical aspects of our work, explaining the implementation of our framework detailing the experiments conducted. Chapter \ref{Results_eval} (Experiments Evaluation \& Results) showcases our results after evaluating our model's performance using a variety of metrics and comparing them with other approaches. Finally, Chapter \ref{Conclusion_chap} (Conclusion) serves as the thesis' culmination, summarizing findings, explaining the limitations of our work, and charting potential paths for future research.\par

%%%%%%%%%%%%%%%%%%%%%%%%%%%%%%%%%%%%
\chapter{Background}
\label{Background}

Various techniques have been developed for graph generation, including auto-regressive models (AR), variational autoencoders (VAE), normalizing flows, and generative adversarial networks (GAN) \cite{you2018graphrnn, popova2019molecularrnn, simonovsky2018graphvae, decao2022molgan}. Variational autoencoders offer the advantage of providing lower-dimensional latent representations, yet they struggle with scalability, particularly for large graphs \cite{zhang_survey_2023}. On the contrary, auto-regressive models excel in generating high-quality samples \cite{huang2022graphgdp}, but face challenges in capturing permutation-invariant properties inherent in structures like graphs. In response to these limitations, diffusion models have emerged as a promising solution.\par

Diffusion models have demonstrated success in tasks like molecular conformer generation \cite{xu2022geodiff, jing_torsional_2023}. Torsional Diffusion \cite{jing_torsional_2023}, for instance, employs topology to diffuse on a torus, reducing the search space, improving model performance, and accelerating inference. Diffusion models enhance sample quality and diversity compared to other methods \cite{vignac_digress_2023} while naturally accommodating global information and conditional dependencies, leading to precise generation. However, the unique properties of graph data sometimes necessitate adaptations to the standard diffusion process \cite{NEURIPS2019_f3a4ff48}. Continuous Gaussian noise processes often result in fully connected, noisy graphs lacking clear structural information. Methods like DiGress \cite{vignac_digress_2023} introduce a discrete denoising diffusion model, incorporating a noise model for independent noise addition to the nodes and the edges. Similarly, GDSS \cite{jo2022scorebased} considers the joint probability density of nodes and edges and trains neural networks to predict partial score functions. These two methods achieve state-of-the-art performance across various graph datasets, molecular or not.\par

While discrete diffusion methods demonstrate improved sampling quality and computational efficiency \cite{haefeli2022diffusion}, stochastic differential equations, as in GDSS, offer a more general and flexible framework with similar performance \cite{song2021scorebased}, serving as a milestone in our framework's development. Alternative approaches like the Graph Spectral Diffusion Model (GSDM) \cite{luo2022fast} insert Gaussian noise into the eigenvalues of the adjacency matrix, offering an alternative to full-rank diffusion on the entire adjacency matrix.\par

In specific tasks such as molecular and protein generation and analysis \cite{walters_applications_2021, atz2021geometric}, model properties play a crucial role. Ensuring invariance to rotation and translation, particularly for generating 3D molecular data, requires specialized approaches like roto-translation equivariant score networks for 3D  and permutation invariant and equivariant score networks for graphs dada. Equivariant neural networks, such as the graph convolution network in GeoDiff \cite{xu2022geodiff}, enforce invariance and have significantly contributed to advancing drug discovery. As a result, deep learning holds promises in identifying novel drug candidates with improved effectiveness and reduced side effects \cite{acsmolpharmaceut8b00930}.\par

Furthermore, the application of structured methods to lift graphs to higher-order structures and the incorporation of topological deep learning models can enhance predictive performance in graph learning tasks. This approach introduces an inductive bias, allowing algorithms to prioritize solutions based on factors beyond observed data \cite{hajij_topological_2023}, as demonstrated in improving classification predictions on molecular data \cite{bodnar2022weisfeiler}.\par

Despite the recent advancements in topological deep learning and diffusion models, scant attention, if any, has been devoted to models capable of generating high-order topological structures, notably combinatorial complexes. The lone model developed for hypergraph generation relies on empirical observations and multi-level decomposition \cite{Do_2020}, raising questions about generalizability and optimality, which sheds light on the need for further research in this area.\par

%%%%%%%%%%%%%%%%%%%%%%%%%%%%%%%%%%%%
\chapter{Preliminaries}
\label{Preliminaries}

In this chapter, we lay the foundational groundwork necessary for comprehending the mathematical framework and subsequent implementation. We start by presenting some notations (Section \ref{Notations}) and the evolution of diffusion models (Section \ref{Diffusion}), tracing their roots from variational autoencoders to their cutting-edge form as score-based generative models. Then, we introduce key descriptions of topological structures, encompassing graphs, hypergraphs, simplicial and cell complexes, and combinatorial complexes (Section \ref{Topological_Deep_Learning}). To facilitate a deeper understanding of molecule and graph generation, central to the forthcoming results and evaluation section (Section \ref{Results_eval}), we finally introduce essential terminologies and metrics (Section \ref{Graphs_molecule_generation}).\par

\section{Notations}
\label{Notations}

To enhance conciseness and clarity, we employ specific notations throughout this thesis. Multiple random variables or observations, denoted as $x_{1},\ldots,x_{t}$, are succinctly represented as $x_{1:t}$. Similarly, when integrating over multiple variables, we adopt the notation $dx_{1:t}$ instead of $dx_{1}\ldots dx_{t}$. The domain of integration of the latent variable(s) $z$ or the variable/observations $x$ is denoted as $\mathcal{Z}$ or $\mathcal{X}$, respectively. The use of the $+$ sign (resp. ) in expressions like $\mathbb{R}_{+}^{*}$ signifies the consideration of exclusively non-negative numbers (resp. all numbers excluding $0$). Lastly, $\llbracket a, b\rrbracket$ represents an interval encompassing all integers between $a$ and $b$, including $a$ and $b$.

\section{Diffusion}
\label{Diffusion}

This section draws inspiration from the comprehensive perspective on denoising diffusion models offered by C. Luo \cite{luo_understanding_2022}, along with insights from D. McAllester \cite{mcallester_mathematics_2023} and T. Segré \cite{segre}. The objective of this part is to equip the reader with a sound understanding of the mathematics behind diffusion models, their operational mechanisms, as well as the notations and equations crucial for comprehending the proposed framework in the context of complex object generation.\\

In this work, we operate within the framework of Score-based Generative Modelling through Stochastic Differential Equations, as initially introduced in \cite{song2021scorebased}. Convergence is assured under a minimal set of assumptions \cite{chen2023sampling}, which we consider verified. Similarly, when introducing SDEs in Subsection \ref{score_sde}, we assume that the coefficients are globally Lipschitz in both state and time to ensure that we have a unique strong solution \cite{oksendal_stochastic_2003}. This section on Diffusion Models ultimately presents Variational Diffusion Models (VDMs) and Score-based Generative Models (SBMs). The primary distinction between SBMs and VDMs resides in their optimization objectives. SBMs, exemplified by models like Denoising Diffusion Probabilistic Models (DDPMs), explicitly optimize the denoising process by minimizing the difference between denoised samples and the original data. In contrast, VDMs optimize the parameters of the diffusion process to directly match the target distribution, bypassing denoising as a separate step. Theoretically, these two approaches are equivalent \cite{luo_understanding_2022} (also see Subsection \ref{towards}), although empirical evidence suggests that employing DDPMs yields better performance \cite{NEURIPS2020_4c5bcfec, saharia2022photorealistic}.\par

\subsection{Generative modelling}

A generative model, at its core, seeks to learn an underlying data distribution, symbolized as $p(x)$, based on observed data samples $x$. By acquiring knowledge of this distribution, the generative model becomes capable of generating new samples at will through a procedure called sampling. As we will see below, in certain formulations, the learned model can also be utilized to assess the likelihood of both observed and sampled data.\par

Presently, a spectrum of well-established approaches exists in the literature, and we offer a high-level overview of these conceptual paradigms. Generative Adversarial Networks (GANs) represent one approach, employing an adversarial learning framework to model the sampling procedure of complex distributions. Another category, Likelihood-Based Generative Models, focuses on learning models that assign high probabilities to observed data samples. Within this category, one encounters auto-regressive models (AR), normalizing flows, and Variational AutoEncoders (VAE). Energy-Based Modelling (EBM) constitutes a related approach, involving the learning of a highly flexible energy function that is subsequently normalized to produce a distribution.\par

Score-Based Generative Models (SBMs) share similarities with EBMs but diverge in their approaches. Instead of learning the energy function directly, SBMs train a neural network to estimate the score\footnotemark{} of a data distribution. Diffusion models, which we delve into in detail below, can be interpreted from both likelihood-based and score-based perspectives. In this section, we aim to cultivate a mathematical comprehension of diffusion models through the lens of the likelihood-based approach, providing a foundation for understanding the fundamental principles underpinning these models. We then also briefly show the equivalence of this approach to the score-based one.\par

\footnotetext{Let $X$ be a random variable, $X\sim p_{\theta^{*}}$, with $\theta^{*}$ unknown and $p_{\theta^{*}}\in\mathcal{M}_{\Theta}$ where $\mathcal{M}_{\Theta}$ is a family of conditional probability laws given $X$. Let $x\in X$ be an observation of $X$. The map $\mathcal{L}(\theta;x):=\theta\mapsto p_{\theta}(x)$ defined from $\Theta$ to $[0,1]$ is called likelihood of the parameter $\theta$. If $\Theta$ is an open set of dimension $p$, if all laws of $\mathcal{M}_{\Theta}$ have the same support $S$, and if for almost all $x\in S$, $\theta\mapsto \ln\left (\mathcal{L}(\theta;x) \right )$ is differentiable, then the \textbf{score} $S_{\theta}$ is the random vector that is the gradient of the log-likelihood with respect to the parameter vector:\\

$S_{\theta}(X)=\nabla_{\theta}\ln\left ( p_{\theta} (X)\right )=\left (\frac{\partial\ln\left ( p_{\theta}(X)\right )}{\partial\theta_{1}},\ldots,\frac{\partial\ln\left ( p_{\theta}(X)\right )}{\partial\theta_{p}} \right )^{T}$}

\subsection{Likelihood-based approach}

For many types of data, it is possible to consider the observed data as being generated or represented by a corresponding latent variable that is \textit{a priori} unknown. This latent variable can be denoted as a random variable $z$.\par

From a probabilistic perspective, we can envision the latent variables and the observed data as governed by a joint distribution $p(x, z)$. One avenue in generative modelling, known as the likelihood-based approach, involves learning a model that maximizes the likelihood $p(x)$ for all observed samples $x$. To extract the likelihood of the observed data $p(x)$ from this joint distribution, two methods are commonly employed:\par

\begin{itemize}
    \item Explicit marginalization of the latent variable $z$: $p(x)=\int p(x,z)dz$
    \item Chain rule of probability: $p(x)=\frac{p(x,z)}{p(z|x)}$
\end{itemize}

However, directly computing and maximizing the likelihood $p(x)$ presents challenges. It demands either the integration of all latent variables $z$ in the first scenario, a task infeasible for complex models, or access to an accurate latent encoder $p(z|x)$ in the second scenario.\par

Nonetheless, an approximation known as the Evidence Lower Bound (ELBO) can be derived, serving as a lower bound for the evidence $p(x)$. Hereafter, $q_{\phi}(z|x)$ denotes a flexible approximate variational distribution parameterized by $\phi$, which we aim to optimize. This distribution operates as a learnable model that estimates the true distribution of latent variables given observations $x$, effectively approximating the true posterior $p(z|x)$. In practice, this learnable model often takes the form of a neural network. Two methods for deriving the ELBO are as follows:\par

\begin{align*}
 & \log (p(x)) \\
 &= \log\left ( \int p(x,z)dz \right ) & \left ( \text{Marginalize }p(x)\right )\\
 &= \log\left ( \int_{\mathcal{Z}} \frac{p(x,z)q_{\phi}(z|x)}{q_{\phi}(z|x)}dz \right ) & \left ( \text{Multiply by }\frac{q_{\phi}(z|x)}{q_{\phi}(z|x)}=1\right )\\
 &= \log\left ( \E_{q_{\phi}(z|x)} \left [ \frac{p(x,z)}{q_{\phi}(z|x)}\right ] \right ) & \left ( \text{Expectation}\right )\\
 &\geq \E_{q_{\phi}(z|x)} \left [ \log\left ( \frac{p(x,z)}{q_{\phi}(z|x)}\right ) \right ] & \left ( \text{Apply Jensen's inequality}\right )
\end{align*}

However, the gap between the ELBO and the log of the evidence $\log (p(x))$ is unknown using the derivation above. Another derivation that uses the chain rule of probability is the following:

\begin{align*}
 & \log (p(x)) \\
 &= \log (p(x)) \int_{\mathcal{Z}}  q_{\phi}(z|x)dz & \left ( \text{Multiply by }\int_{\mathcal{Z}}  q_{\phi}(z|x)dz=1\right )\\
 &= \int_{\mathcal{Z}}  q_{\phi}(z|x)\log (p(x))dz & \left ( \text{Swap scalar and integral sign}\right )\\
 &= \E_{q_{\phi}(z|x)}\left [ \log (p(x)) \right ] & \left ( \text{Expectation}\right )\\
 &= \E_{q_{\phi}(z|x)}\left [ \log \left (\frac{p(x,z)}{p(z|x)}\right ) \right ] & \left ( \text{Chain rule of probability}\right )\\
 &= \E_{q_{\phi}(z|x)}\left [ \log \left (\frac{p(x,z)q_{\phi}(z|x)}{p(z|x)q_{\phi}(z|x)}\right ) \right ] & \left ( \text{Multiply by }\frac{q_{\phi}(z|x)}{q_{\phi}(z|x)}=1\right )\\
 &= \E_{q_{\phi}(z|x)}\left [ \log \left (\frac{p(x,z)}{q_{\phi}(z|x)}\right ) \right ] + \E_{q_{\phi}(z|x)}\left [ \log \left (\frac{q_{\phi}(z|x)}{p(z|x)}\right ) \right ] & \left ( \text{Split the expectation}\right )\\
 &= \E_{q_{\phi}(z|x)}\left [ \log \left (\frac{p(x,z)}{q_{\phi}(z|x)}\right ) \right ] + D_{KL}\left ( q_{\phi}(z|x) || p(z|x) \right ) & \left ( \text{Definition of KL divergence\footnotemark{}}\right )\\
 &\geq \E_{q_{\phi}(z|x)}\left [ \log \left (\frac{p(x,z)}{q_{\phi}(z|x)}\right ) \right ] & \left ( \text{KL divergence always non-negative}\right )
\end{align*}

\footnotetext{The Kullback-Leibler (KL) divergence is a measure of dissimilarity between two distributions. Mathematically, let $P$ and $Q$ be two probability measures on a measurable space $\mathcal{X}$ such that $P$ is absolutely continuous with respect to $Q$. The relative entropy from $Q$ to $P$ is defined by $D_{KL}\left ( P||Q \right )=\int_{\mathcal{X}}\log\left ( \frac{P(dx)}{Q(dx)}\right )P(dx)$. $\frac{P(dx)}{Q(dx)}$ is the Radon–Nikodym derivative of $P$ with respect to $Q$. It corresponds to the ratio of their density if they are dominated.}

Minimizing the KL divergence term: the difference between the approximate posterior distribution $q_{\phi}(z|x)$ and the true yet \textit{a priori} unknown posterior distribution $p(z|x)$, is unfeasible. However, the ELBO and KL divergence terms together sum to a constant, which is the evidence term $\log (p(x))$. Consequently, maximizing the ELBO can function as a surrogate objective for effectively modelling the true latent posterior distribution. Through ELBO optimization, we can progressively approach the true posterior. Thus, the ELBO serves as an objective function for this purpose. Furthermore, post-training, the ELBO facilitates the estimation of the likelihood of observed or generated data, as it is trained to approximate the model evidence $\log (p(x))$. For Variational AutoEncoders (VAE), as we will see below, maximizing the ELBO yields two components: one enabling the representation of the true data distribution in a latent space (encoder), and another allowing generation via sampling from the latent space (decoder).\par

\subsection{Variational Autoencoder}

In the classic formulation of the Variational Autoencoder (VAE) \cite{kingma2022autoencoding}, the previously introduced ELBO is maximized, and the input data is trained to predict itself following an intermediate bottleneck representation step. This method falls under the variational category as it consists of optimizing the most suitable choice for $q_{\phi}(z|x)$ from a range of potential posterior distributions parametrized by $\phi$. This term effectively functions as an encoder, transforming input data into a distribution spanning potential latent variables. Concurrently, a deterministic function $p_{\theta}(x|z)$ is learned to map a given latent vector $z$ to an observation $x$, which can be interpreted as a decoder. The nomenclature \textit{autoencoder} stems from the resemblance VAEs bear to traditional autoencoder models. To underscore this connection explicitly, the ELBO term can be further deconstructed as follows:\par

\begin{align*}
 & \E_{q_{\phi}(z|x)} \left [ \log\left ( \frac{p(x,z)}{q_{\phi}(z|x)}\right ) \right ] \\
 &= \E_{q_{\phi}(z|x)} \left [ \log\left ( \frac{p_{\theta}(x|z)p(z)}{q_{\phi}(z|x)}\right ) \right ] & \left ( \text{Chain rule of probability}\right )\\
 &= \E_{q_{\phi}(z|x)} \left [ \log\left ( p_{\theta}(x|z)\right ) \right ] + \E_{q_{\phi}(z|x)} \left [ \log\left ( \frac{p(z)}{q_{\phi}(z|x)}\right ) \right ] & \left ( \text{Split the expectation}\right )\\
 &= \underbrace{\E_{q_{\phi}(z|x)} \left [ \log\left ( p_{\theta}(x|z)\right ) \right ]}_{\text{Reconstruction term}} - \underbrace{D_{KL}\left ( q_{\phi}(z|x) || p(z)\right )}_{\text{Prior matching term}} & \left ( \text{Definition of KL divergence} \right )
\end{align*}

\subsection{Markovian Hierarchical Variational Autoencoder}

A Hierarchical Variational Autoencoder (HVAE) \cite{NIPS2016_ddeebdee, NIPS2016_6ae07dcb} represents an extension of the VAE model that introduces multiple hierarchies of latent variables. In contrast to VAEs, these latent variables are themselves generated from higher-level, and thus more abstract, latent variables.\par

Markovian Hierarchical Variational Autoencoders (MHVAE) constitute a specific subset of HVAE models wherein the generative process follows a Markov chain structure. In this structure, each transition down the hierarchy follows a Markovian principle, meaning that the decoding of each latent variable $z_{t}$ relies solely on the preceding latent variable $z_{t+1}$. This configuration can be envisioned as a cascade of VAEs stacked atop one another. Mathematically, it yields the following expressions for the joint distribution and posterior of an MHVAE:\par

Joint distribution: $p(x,z_{1:T})=p(z_{T})p_{\theta}(x|z_{1})\prod_{t=2}^{T}p_{\theta}(z_{t-1}|z_{t})$\par

Posterior: $q_{\phi}(z_{1:T}|x)=q_{\phi}(z_{1}|x)\prod_{t=2}^{T}q_{\phi}(z_{t}|z_{t-1})$\par

Employing similar derivations as in the preceding subsections, we can derive the following ELBO term:

\begin{align*}
\log (p(x)) &= \log\left ( \int_{\mathcal{Z}} p(x,z_{1:T})dz_{1:T} \right ) & \left ( \text{Chain rule of probability}\right )\\
    &= \log\left ( \int_{\mathcal{Z}} \frac{p(x,z_{1:T})q_{\phi}(z_{1:T}|x)}{q_{\phi}(z_{1:T}|x)}dz_{1:T} \right ) & \left ( \text{Multiply by }\frac{q_{\phi}(z_{1:T}|x)}{q_{\phi}(z_{1:T}|x)}=1\right )\\
    &= \log\left ( \E_{q_{\phi}(z_{1:T}|x)} \left [ \frac{p(x,z_{1:T})}{q_{\phi}(z_{1:T}|x)}\right ] \right ) & \left ( \text{Expectation}\right )\\
    &\geq \E_{q_{\phi}(z_{1:T}|x)} \left [ \log\left ( \frac{p(x,z_{1:T})}{q_{\phi}(z_{1:T}|x)}\right ) \right ] & \left ( \text{Apply Jensen's inequality}\right )
\end{align*}

An alternative form of this ELBO can be obtained by replacing the joint distribution and the posterior previously established:\par

$\log (p(x)) \geq \E_{q_{\phi}(z_{1:T}|x)} \left [ \log\left ( \frac{p(z_{T})p_{\theta}(x|z_{1})\prod_{t=2}^{T}p_{\theta}(z_{t-1}|z_{t})}{q_{\phi}(z_{1}|x)\prod_{t=2}^{T}q_{\phi}(z_{t}|z_{t-1})}\right ) \right ]$

\subsection{Variational Diffusion Models}

Variational Diffusion Models (VDMs) can be framed as a specific instance of the Markovian Hierarchical Variational Autoencoder framework, with the inclusion of three key characteristics. Firstly, the latent dimension matches exactly with the data dimension. Secondly, the structure of the latent encoder at each time step is not learned; instead, it is predetermined as a linear Gaussian model. This means that the latent encoder distribution is centred around the output of the previous time step. Finally, the Gaussian parameters of the latent encoders evolve over time, ensuring that the distribution of the latent variable at the final time step, denoted as $T$, follows a standard Gaussian distribution.\par

As the dimensions are preserved throughout encoding and decoding, we represent the latent representation, previously denoted as $z$, in the same manner as the generated or source data, denoted as $x$. For instance, following this widely used convention, we express the posterior of an MHVAE as follows: $q(x_{1:T}|x_{0})=\prod_{t=1}^{T}q(x_{t}|x_{t-1})$.\par

Regarding the linear Gaussian model structure of the encoder at each time step, two approaches exist. The mean and standard deviation of the Gaussian encoder can either be learned as parameters \cite{NEURIPS2021_b578f2a5}, or set as hyperparameters \cite{NEURIPS2020_4c5bcfec}. Specifically, the Gaussian encoder is parameterized with a mean $\mu_{t}(x_{t}) = \sqrt{\alpha_{t}}x_{t-1}$ and a variance $\Sigma_{t}(x_{t}) = (1 - \alpha_{t})I$. This formulation is said to be variance-preserving as it preserves the variance of the latent variables throughout the encoding process and allows flexibility in adding noise via the coefficient $\alpha_{t}$, which can vary with the hierarchical depth $t$. It's important to note that alternative Gaussian parameterizations are possible and yield similar results. Mathematically, the encoder transitions are expressed as $q(x_{t}|x_{t-1})\sim \mathcal{N}(x_{t};\sqrt{\alpha_{t}}x_{t-1},(1-\alpha_{t})I)$. Notably, in VDMs, the encoder distributions $q(x_{t}|x_{t-1})$ are no longer parameterized by $\phi$; they are fully modelled as Gaussians with predefined mean and variance parameters at each time step.\par

Regarding the third assumption, $\alpha_{t}$ evolves over time based on a fixed or learnable schedule, ensuring that the distribution of the final latent variable $p(x_{T})$ conforms to a standard Gaussian distribution. This simplifies decoder training and the sampling process. Consequently, the joint distribution of an MHVAE can be reformulated as $p(x_{0:T})=p(x_{T})\prod_{t=1}^{T}p_{\theta}(x_{t-1}|x_{t})$, where $p(x_{T})\sim \mathcal{N}(x_{T};0,I)$.\par

In summary, these three assumptions describe the gradual introduction of noise into an input object (e.g., graph, image, etc.) over time. The object is progressively corrupted by the addition of Gaussian noise until it eventually becomes similar to a Gaussian noise. In a VDM, the primary focus lies on learning the conditionals $p_{\theta}(x_{t-1}|x_{t})$ to enable sampling new data points. Once the VDM is optimized, the sampling procedure involves sampling Gaussian noise from $p(x_{T})$ and then iteratively applying the denoising transitions $p_{\theta}(x_{t-1}|x_{t})$ for $T$ steps to generate a new object $x_{0}$ that follows the learned original distribution $p_{0}$.\par

By adhering to the same principles, one can derive an ELBO term for VDMs. Although the expectations in the ELBO derivation can be approximated using Monte Carlo estimates, they involve two variables ($x_{t-1}$ and $x_{t+1}$), resulting in high variance for large $T$ due to the summation of $T-1$ consistency terms. To obtain a more reliable evidence lower bound, we need to reformulate the encoder transitions as $q(x_{t}|x_{t-1}) = q(x_{t}|x_{t-1}, x_{0})$. The additional conditioning term on the original data point $x_{0}$ has been added as it is redundant due to the Markov property. By applying Bayes' rule, we can express each transition as $q(x_{t}|x_{t-1},x_{0})=\frac{q(x_{t-1}|x_{t},x_{0})q(x_{t}|x_{0})}{q(x_{t-1}|x_{0})}$. This leads to the derivation of the ELBO outlined below:\par

\begin{align*}
\log (p(x_{0})) &= \log\left ( \int_{\mathcal{X}} p(x_{0:T})dx_{1:T} \right )\\
    &= \log\left ( \int_{\mathcal{X}} \frac{p(x_{0:T})q(x_{1:T}|x_{0})}{q(x_{1:T}|x_{0})}dx_{1:T} \right )\\
    &= \log\left ( \E_{q(x_{1:T}|x_{0})} \left [ \frac{p(x_{0:T})}{q(x_{1:T}|x_{0})}\right ] \right )\\
    &\geq \E_{q(x_{1:T}|x_{0})} \left [ \log\left ( \frac{p(x_{0:T})}{q(x_{1:T}|x_{0})}\right ) \right ]\\
    &\geq \E_{q(x_{1:T}|x_{0})} \left [ \log\left ( \frac{p(x_{T})\prod_{t=1}^{T}p_{\theta}(x_{t-1}|x_{t})}{\prod_{t=1}^{T}q(x_{t}|x_{t-1})}\right ) \right ]\\
    &\geq \E_{q(x_{1:T}|x_{0})} \left [ \log\left ( \frac{p(x_{T})p(x_{0}|x_{1})\prod_{t=2}^{T}p_{\theta}(x_{t-1}|x_{t})}{q(x_{T}|x_{T-1})\prod_{t=1}^{T-1}q(x_{t}|x_{t-1})}\right ) \right ]\\
    &\geq \E_{q(x_{1:T}|x_{0})} \left [ \log\left ( \frac{p(x_{T})p(x_{0}|x_{1})\prod_{t=1}^{T-1}p_{\theta}(x_{t}|x_{t+1})}{q(x_{T}|x_{T-1})\prod_{t=1}^{T-1}q(x_{t}|x_{t-1})}\right ) \right ]\\
    &\geq \E_{q(x_{1:T}|x_{0})} \left [ \log\left ( \frac{p(x_{T})p(x_{0}|x_{1})}{q(x_{T}|x_{T-1})}\right ) \right ] + \E_{q(x_{1:T}|x_{0})} \left [ \log\left ( \prod_{t=1}^{T-1}\frac{p_{\theta}(x_{t}|x_{t+1})}{q(x_{t}|x_{t-1})}\right ) \right ]\\
    &\geq \E_{q(x_{1:T}|x_{0})} \left [ \log\left ( p(x_{0}|x_{1})\right ) \right ] + \E_{q(x_{1:T}|x_{0})} \left [ \log\left ( \frac{p(x_{T})}{q(x_{T}|x_{T-1})}\right ) \right ]\\
    & + \E_{q(x_{1:T}|x_{0})} \left [ \sum_{t=1}^{T-1}\log\left (\frac{p_{\theta}(x_{t}|x_{t+1})}{q(x_{t}|x_{t-1})}\right ) \right ]\\
    &\geq \E_{q(x_{1:T}|x_{0})} \left [ \log\left ( p(x_{0}|x_{1})\right ) \right ] + \E_{q(x_{1:T}|x_{0})} \left [ \log\left ( \frac{p(x_{T})}{q(x_{T}|x_{T-1})}\right ) \right ]\\
    & + \sum_{t=1}^{T-1}\E_{q(x_{1:T}|x_{0})} \left [\log\left (\frac{p_{\theta}(x_{t}|x_{t+1})}{q(x_{t}|x_{t-1})}\right ) \right ]\\
    &\geq \E_{q(x_{1}|x_{0})} \left [ \log\left ( p(x_{0}|x_{1})\right ) \right ] + \E_{q(x_{T-1},x_{T}|x_{0})} \left [ \log\left ( \frac{p(x_{T})}{q(x_{T}|x_{T-1})}\right ) \right ]\\
    & + \sum_{t=1}^{T-1}\E_{q(x_{t-1},x_{t},x_{t+1}|x_{0})} \left [\log\left (\frac{p_{\theta}(x_{t}|x_{t+1})}{q(x_{t}|x_{t-1})}\right ) \right ]\\
    &\geq \E_{q(x_{1}|x_{0})} \left [ \log\left ( p(x_{0}|x_{1})\right ) \right ] - \E_{q(x_{T-1}|x_{0})} \left [ D_{KL}\left ( q(x_{T}|q_{T-1}) || p(x_{T}) \right ) \right ]\\
    & - \sum_{t=1}^{T-1}\E_{q(x_{t-1},x_{t+1}|x_{0})} \left [ D_{KL}\left ( q(x_{t}|x_{t-1}) || p_{\theta}(x_{t}|x_{t+1}) \right ) \right ]
\end{align*}

We can notice that the ELBO can be decomposed into three terms:

\begin{itemize}
    \item $\E_{q(x_{1}|x_{0})}\left [ \log (p_{\theta}(x_{0}|x_{1}))\right ]$ can be interpreted as a reconstruction term, similar to the one found in the ELBO of a VAE. It quantifies how well the model reconstructs the original data from a noisy version of it. This term can be approximated and optimized using Monte-Carlo estimation techniques.\\
    \item $D_{KL}\left ( q(x_{T}|x_{0}) || p(x_{T})\right )$ measures the divergence between the distribution of the final noisy input and the standard Gaussian prior. It does not involve any trainable parameters and, based on the model's assumptions, is typically close to or equal to zero.
    \item $\E_{q(x_{t}|x_{0})}\left [ D_{KL}\left ( q(x_{t-1}|x_{t}, x_{0}) || p_{\theta}(x_{t-1}|x_{t}) \right ) \right ]$ serves as a denoising matching term, analogous to its counterpart in the ELBO of a VAE. Here, the goal is to learn an approximate denoising transition step $p_{\theta}(x_{t-1}|x_{t})$ that approximates the tractable, ground-truth denoising transition step $q(x_{t-1}|x_{t}, x_{0})$, which can be considered as a ground-truth signal since it defines the denoising process for a noisy object $x_{t}$ while having access to the completely denoised object $x_{0}$. Minimizing this term aims to align the two denoising steps as closely as possible, as measured by the Kullback-Leibler (KL) divergence.
\end{itemize}

As per \cite[equations.~59-99]{luo_understanding_2022}, by applying the reparametrization trick, leveraging properties of Gaussian variables, and computing tractable KL divergences between Gaussian distributions, the optimization problem simplifies to:\par

$\underset{\theta}{\text{arg min }}D_{KL}\left ( q(x_{t-1}|x_{t},x_{0}) || p_{\theta}(x_{t-1}|x_{t}) \right )=\underset{\theta}{\text{arg min }}\frac{1}{2\sigma_{q}^{2}(t)}\frac{\overline{\alpha}_{t-1}(1-\alpha_{t})^{2}}{(1-\overline{\alpha}_{t})^{2}}\left [ \left \| \hat{x}_{\theta}(x_{t},t)-x_{0} \right \|_{2}^{2} \right ]$\par

where for all $t$ in $\llbracket 1, T\rrbracket$, $\overline{\alpha}_{t}=\prod_{i=1}^{t}\alpha_{i}$ and $\sigma_{q}^{2}(t)=\frac{(1-\alpha_{t})(1-\overline{\alpha}_{t})}{1-\overline{\alpha}_{t}}$ is defined after all the derivations: $\Sigma_{q}(t)=\sigma_{q}^{2}(t)I$, with $q(x_{t-1}|x_{t},x_{0})\propto \mathcal{N}\left ( x_{t-1}; \mu_{q}(x_{t},x_{0}), \Sigma_{q}(t) \right )$.\par

Therefore, optimizing a VDM involves training a neural network to estimate the original ground truth object from a noisy version of it \cite{NEURIPS2020_4c5bcfec}. This can be interpreted in two other equivalent ways after some derivations, as outlined in \cite{luo_understanding_2022}. First, the neural network can be trained to predict the source noise $\epsilon_{0}$, drawn from a standard Gaussian distribution $\mathcal{N}(\epsilon; 0, I)$, that generates the noisy object $x_{t}$ given the initial object $x_{0}$. Second, the neural network can be trained to predict the score function $s_{t}=\nabla x_{t} \log (p(x_{t}))$, which represents here the gradient of $x_{t}$ in the data space. This interpretation aligns with the principles of score-based models that can also be derived from energy-based models \cite{928a56b7d6f1473e930f282a0c4b534e, song2021train}.\par

\subsection{Towards Score-based Diffusion Models}
\label{towards}

Score-based diffusion models are a category of diffusion models that rely on approximating the score function using neural networks and subsequently generate objects through a time-reversal process. These models exhibit connections with variational diffusion models, as explained below, as well as with energy-based models, which broaden their applicability within the wider context of generative models. Consequently, there has been a growing interest in this approach.\par

There are two primary classes of score-based generative models:\par

\begin{itemize}
\item \textbf{Denoising Diffusion Probabilistic Modelling (DDPM):} In DDPM \cite{NEURIPS2019_3001ef25}, a sequence of probabilistic models is trained to reverse each step of the noise corruption process. The training process often involves making approximations, leveraging domain knowledge, or utilizing knowledge of the functional form of the reverse distributions. DDPMs have found applications in graph generation, such as in models like DiGress \cite{vignac_digress_2023}, GRAPHARM \cite{kong2023autoregressive}, and SGGM \cite{yang2023scorebased}. These models are termed \textit{score-based} because, in cases where the state space is continuous, the training implicitly computes scores at each noise scale.

\item \textbf{Score Matching with Langevin Dynamics (SMLD):} SMLD \cite{yang2023scorebased} directly estimates the score at each noise scale and then employs Langevin dynamics to sample from a sequence of decreasing noise scales during the generation process. This approach has also been applied in the context of graph and molecule conformer generation, seen in models like EDP-GNN \cite{niu2020permutation} and ConfGF \cite{shi2021learning}.
\end{itemize}

Now, let's explore the equivalence between variational diffusion models and score-based models, specifically the score matching with Langevin dynamics approach. To do so, we start with Tweedie's formula: \cite{tweedie}:

\begin{theorem}[Tweedie's formula]
Let $z\sim \mathcal{N}(\mu_{z}, \Sigma_{z})$ be a Gaussian variable. Then, we have:

$\E \left [ \mu_{z}|z\right ]=z+\Sigma_{z} \nabla_{z}\log(p(z))$
\end{theorem}

As derived in C. Luo \cite{luo_understanding_2022}, the samples $x_{t}$ are drawn from the distribution $q(x_{t}|x_{0})=\mathcal{N}(\mu_{x_{t}},\Sigma_{x_{t}})$ where $\mu_{x_{t}}=\sqrt{\overline{\alpha}_{t}}x_{0}$ and $\Sigma_{x_{t}}=(1-\overline{\alpha}_{t}) I$. Applying Tweedie's formula to this distribution yields:\par

$\sqrt{\overline{\alpha}_{t}}x_{0}=\E \left [ \mu_{x_{t}}|x_{t}\right ]=x_{t}+(1-\overline{\alpha}_{t}) \nabla_{x_{t}}\log(p(x_{t}))\implies x_{0}=\frac{x_{t}+(1-\overline{\alpha}_{t}) \nabla_{x_{t}}\log(p(x_{t}))}{\sqrt{\overline{\alpha}_{t}}}$. Therefore,\par

$\mu_{q}(x_{t},x_{0})=\frac{\sqrt{\alpha_{t}}(1-\overline{\alpha}_{t-1})x_{t}+\sqrt{\overline{\alpha}_{t-1}}(1-\alpha_{t})x_{0}}{1-\overline{\alpha}_{t}}=\frac{1}{\sqrt{\alpha_{t}}}x_{t}+\frac{1-\alpha_{t}}{\sqrt{\alpha_{t}}}\nabla_{x_{t}}\log\left (p(x_{t})\right )$\par

More information, including detailed derivation with all the steps, could be found in the literature \cite{luo_understanding_2022}.\par

We can train a neural network, denoted as $\mu_{\theta}(x_{t},t)$, to approximate the mean of the denoising transition. This approximation is achieved by estimating the score function using a Noise Conditional Score Network (NCSN), represented as $s_{\theta}(x_{t},t)$, and can be expressed as:\par

$\mu_{\theta}(x_{t},t)=\frac{1}{\sqrt{\alpha_{t}}}x_{t}+\frac{1-\alpha_{t}}{\sqrt{\alpha_{t}}}s_{\theta}(x_{t},t)$\par

The neural network $s_{\theta}$ is designed to predict the gradient of the score of $p(x_{t})$ in the data space where $x_{t}$ is a point subjected to a specific level of injected noise, denoted as $t$. Consequently, the optimization problems take the following form:\par

\begin{align*}
 & \underset{\theta}{\text{arg min }}D_{KL}\left ( q(x_{t-1}|x_{t},x_{0}) || p_{\theta}(x_{t-1}|x_{t}) \right )\\
 &= \underset{\theta}{\text{arg min }}D_{KL}\left ( \mathcal{N}(\mu_{q}, \Sigma_{q}(t)) || \mathcal{N}(\mu_{\theta}, \Sigma_{q}(t)) \right )\\
 &= \underset{\theta}{\text{arg min }}\frac{1}{2\sigma_{q}^{2}(t)} \left [ \left \| \frac{1}{\sqrt{\alpha_{t}}}x_{t}+\frac{1-\alpha_{t}}{\sqrt{\alpha_{t}}}\nabla_{x_{t}}\log(p(x_{t}) - \frac{1}{\sqrt{\alpha_{t}}}x_{t}-\frac{1-\alpha_{t}}{\sqrt{\alpha_{t}}}s_{\theta}(x_{t},t) \right \|_{2}^{2} \right ]\\
 &= \underset{\theta}{\text{arg min }}\frac{(1-\alpha_{t})^{2}}{2\sigma_{q}^{2}(t)\alpha_{t}} \left [ \left \| \nabla_{x_{t}}\log\left ( p(x_{t})\right ) - s_{\theta}(x_{t},t) \right \|_{2}^{2} \right ]
\end{align*}

Training our model consists of predicting the score. Thus, this concludes the analogy to score-based generative models.

\subsection{Score-based Generative Modelling through Stochastic Differential Equations}
\label{score_sde}

Recent works have showcased the integration of DDPM and SMLD under a unified framework known as score-based generative modelling through stochastic differential equations (SDE) \cite{song2021scorebased}. This framework has been applied to graph generation, as exemplified by GDSS (Graph Diffusion via the System of Stochastic Differential Equations) \cite{jo2022scorebased}. The core concept involves transforming data from its original distribution to a noise distribution, effectively treated as the prior, through the use of a Stochastic Differential Equation (SDE). Subsequently, the generation process is done by reversing the same SDE or reversing the associated probability flow Ordinary Differential Equation (ODE). The reverse-time SDE (RSDE) and the probability flow ODE are obtained by estimating the score. Before introducing the mathematics behind this idea, we introduce a definition of a diffusion process, as well as a reminder of what is a standard Wiener process (or standard Brownian movement).

\begin{definition}[Diffusion process]
A diffusion process of length $T\in\mathbb{R}_{+}$, on a data distribution $p_{0}$ (or $p_{\text{data}}$) with a prior $p_{T}$ is a stochastic process $(x_{t})_{t\in [0, T]}$ where $t$ is a continuous time variable, such that $p_{0}$ is made of independent and identically distributed samples, $p_{T}$ is tractable, $x(0)\sim p_{0}$, and $x(T)\sim p_{T}$.
\end{definition}

\begin{definition}[Standard Wiener Process]
A standard Wiener process $W=(W_{t})_{t\in\mathbb{R}_{+}}$ \cite{Revuz1999}, or standard Brownian movement, is a stochastic process such that:\par

\begin{itemize}
    \item $W_{0}=0$ almost surely,
    \item $W$ has \textit{independent increments}, which means that for every $m\in\mathbb{N}$, for every $t_{0},\ldots,t_{m}$ such that $t_{0}<\ldots <t_{m}$, the random variables $\left ( W_{t_{i+1}}-W_{t_{i}}\right )_{i=0}^{m-1}$ are stochastically independent,
    \item For all $s, t\in\mathbb{R}_{+}$, $W_{t+s}-W_{t}\sim\mathcal{N}(0,s)$,
    \item $t\mapsto W_{t}$ is almost surely continuous.
\end{itemize}
\end{definition}

Let $(x_{t})_{t\in [0, T]}$ be a diffusion process. We can model it as the solution to an Itô stochastic differential equation:\par

$dx=f_{t}(x)dt+g_{t}dW$, where $W$ is the standard Wiener process (or Brownian motion), $f_{t}(\cdot):\mathbb{R}^{d}\rightarrow \mathbb{R}^{d}$ is referred to as the \textit{drift coefficient}\footnotemark{}, and $g_{(\cdot)}:\mathbb{R}\rightarrow \mathbb{R}$ is a scalar function known as the \textit{diffusion coefficient} of $x(t)$.\par

\footnotetext{The notation $f_{t}(\cdot):=f(\cdot,t)$ is used to write a function of space and time.}

For the generation of samples, we begin by sampling a noisy object from our prior, $x(T)\sim p_{T}$. Then, we employ a result from Anderson \cite{anderson}. We reverse the diffusion process in time by following the reverse-time SDE associated with the forward SDE mentioned earlier:\par

$dx=\left [f_{t}(x)-g_{t}^{2}\nabla_{x}\log\left ( p_{t}(x) \right ) \right ]d\tilde{t} + g_{t}d\tilde{W}$, where $\tilde{W}$ is a standard Wiener process when time flows backwards, from $T$ to $0$, and where $d\tilde{t}$ is an infinitesimal \textit{negative} timestep. To implement this generative modelling method, we only require a way of estimating the score of all marginal distributions $\nabla_{x}\log \left ( p_{t}(x) \right )$ for all $t\in [0,T]$.\par

For further elaboration on this method, refer to Song et al. \cite{song2021scorebased} and our framework, CCSD, introduced in Section \ref{CCSD}. The training objective relies on score-matching \cite{JMLR:v6:hyvarinen05a,pmlr-v115-song20a,song2021scorebased}, with the derivation extensively detailed in previous work \cite[Section 3.3]{song2021scorebased}. Different types of SDEs explored in our framework, namely VE, VP, and sub-VP SDE, are presented in Song et al. along with their derivations \cite[Appendix B.]{song2021scorebased}.\par

This concludes the diffusion preliminaries section.\par

\section{Topological Deep Learning}
\label{Topological_Deep_Learning}

The notations and object definitions presented in this chapter are primarily drawn from the works of Papillon et al. \cite{papillon2023architectures} and Hajij et al. \cite{hajij_topological_2023}. Our initial focus is on establishing the concept of a topological space through the lens of neighbourhoods, rather than relying on open sets. Subsequently, we delve into the core objects employed in Topological Deep Learning, starting with graphs and extending to the more abstract and versatile combinatorial complexes that we aim to generate in our work. Lastly, we present the notion of \textit{lifting} within the context of topological deep learning, complemented by two illustrative examples of lifting procedures utilized in our experiments.\par

\subsection{From the Graphs to the Combinatorial Complexes}

Here, we provide definitions of a topological space to foster a comprehensive understanding of topology in the context of this work. In our context, topology pertains to the structure and connectivity of the objects manipulated in our experiments. An overview of these different objects that will be presented below is presented in Figure \ref{fig:topo}.\par

\begin{definition}[Neighborhood function]
Let $S$ be a non-empty set. A neighborhood function on $S$ is a function $\mathcal{N}: S \rightarrow \mathcal{P}(S)$ that assigns to each point $x$ in $S$ a non-empty subset $\mathcal{N}(x)$ of the powerset $\mathcal{P}(S)$ of $S$. The elements of $\mathcal{N}(x)$ are called neighbourhoods of $x$ with respect to $\mathcal{N}$.
\end{definition}

\begin{definition}[Neighborhood topology]
Let $\mathcal{N}$ be a neighbourhood function on a set $S$. $\mathcal{N}$ is called a neighborhood topology on $S$ if it satisfies the following axioms:

\begin{itemize}
    \item If $N$ is a neighborhood of $x$, then $x\in N$.
    \item If $N$ is a subset of $S$ containing a neighborhood of $x$, then $N$ is a neighborhood of $x$.
    \item The intersection of two neighborhoods of a point $x$ in $S$ is a neighborhood of $x$.
    \item Any neighborhood $N$ of a point $x$ in $S$ contains a neighborhood $M$ of $x$ such
that $N$ is a neighborhood of each point of $M$.
\end{itemize}
\end{definition}

\begin{definition}[Topological space]
Let $S$ be a non-empty set. A topological space is a pair $(S, \mathcal{N})$ where $\mathcal{N}$ is a neighbourhood topology on $S$.
\end{definition}

\begin{definition}[Undirected Graph]
Let $S$ be a non-empty set. A graph on $S$ is a pair $(S, E)$ where $E$ is a set of non-empty subsets of \textbf{size 2} of the powerset $\mathcal{P}(S)$ of $S$, which are called edges. Elements of $S$ are called vertices.\par

A graph is said to be undirected if for all $(u, v)\in E$, $(v, u)\in E$.
\end{definition}

In this thesis, we will refer to undirected graphs when mentioning graphs. Moving beyond this, we introduce more abstract topological structures that offer generalizations of graphs. Hypergraphs provide the advantage of providing relations between nodes or entities that extend beyond pairwise interactions. Simplicial complexes and regular cell complexes introduce order among sets of nodes. Combinatorial complexes encompass both of these properties, making them more versatile.

\begin{definition}[Hypergraph]
Let $S$ be a non-empty set. A hypergraph on $S$ is a pair $(S, X)$, where $X$ is a set of non-empty subsets of the powerset $\mathcal{P}(S)$ of $S$, which are called hyperedges. Elements of $S$ are called vertices.
\end{definition}

\begin{definition}[Simplicial complex]
An abstract simplicial complex on a nonempty set $S$ is a pair $(S, X)$, where $X$ is a subset of $\mathcal{P}(S) \backslash \{\emptyset \}$ such that, for all $x\in X$, for all $y \in \mathcal{P}(S)$, $y\subseteq x$ implies $y\in X$. Elements of $X$ are called simplices.
\end{definition}

\begin{remark}
This can be interpreted as a generalization of triangles in a more abstract space.
\end{remark}

\begin{definition}[Regular cell complex (CW complex)]
A regular cell complex is a topological space $(S, \mathcal{T})$ with a partition into sub-spaces (cells) $(x_{\alpha})_{\alpha\in P_{S}}$, where $P_{S}$ is an index set, satisfying the following conditions:

\begin{itemize}
\item $S = \cup_{\alpha\in P_{S}} int(x_{\alpha})$, where $int(x)$ denotes the interior of cell $x$
\item For each $\alpha\in P_{S}$, there exists a homeomorphism $\phi$ from $x_{alpha}$ to $\mathbb{R}^{n_{\alpha}}$ for some $n_{\alpha}\in\mathbb{N}$, called the dimension $n_{\alpha}$ of cell $x_{\alpha}$
\item For each cell $x_{\alpha}$, the boundary $\partial x_{\alpha}$ is a union of finitely many cells, each having a dimension less than that of $x_{\alpha}$.
\end{itemize}
\end{definition}

\begin{remark}A graph is a 1-dimensional CW complex in which the 0-cells are the vertices and the 1-cells are the edges.
\end{remark}

Last but not least, below is the definition of a combinatorial complex.

\begin{definition}[Combinatorial Complex]
A combinatorial complex (CC) is a triple $(S, \mathcal{X}, rk)$ consisting of a set $S$, a subset $\mathcal{X}$ of $\mathcal{P}(S) \backslash \{\emptyset \}$, and a function $rk: \mathcal{X} \rightarrow \mathbb{N}$ with the following properties:
\begin{itemize}
    \item $\forall s\in S, \{s\}\in \mathcal{X}$
    \item the function $rk$ is order-preserving, which means that if $x, y\in \mathcal{X}$ satisfy $x \subseteq y$, then $rk(x) \leq rk(y)$.
\end{itemize}

The elements of $S$ are called entities or vertices, the elements of $\mathcal{X}$ are called relations or cells, and $rk$ is called the rank function of the CC. The dimension of a CC is $\text{dim}(CC)=\text{max}(rk(\mathcal{X}))$ and, for all $r\in\llbracket 0, \text{dim}(CC)\rrbracket$, we note $\mathcal{X}_{r}$ the set of all cells or rank $r$ ($\mathcal{X}_{r}=rk^{-1}(r)$). In this thesis, we will often denote $R=\text{dim}(CC)$.
\end{definition}

\begin{figure}
\centering
\includegraphics[width=0.9\linewidth]{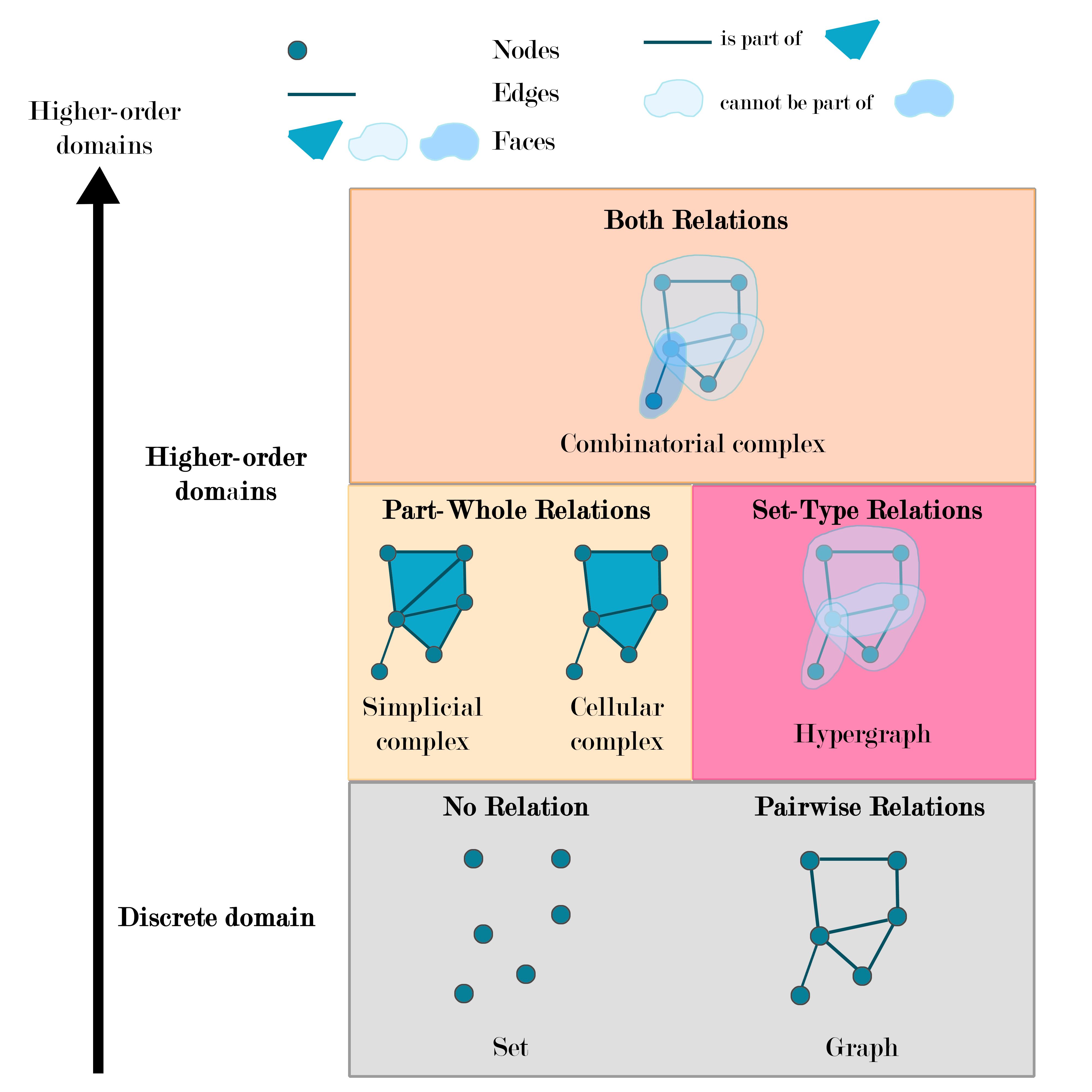}
\captionof{figure}{\textbf{Overview of different topological structures.} From the sets and graphs to the combinatorial complex, this figure presents the hierarchy of some topological structures in function of how they incorporate higher-order relations as part of their definitions. Combinatorial complexes generalize all these objects as they both have part-whole relations and set-type relations \cite{hajij_topological_2023}. The figure has been adapted from Papillon et al. \cite{papillon2023architectures} et Hajij et al. \cite{hajij_topological_2023}}
% \captionof{figure}{Overview of different topological structures. From the sets and graphs to the combinatorial complex, this figure presents the hierarchy of some topological structures in function of how they incorporate higher-order relations as part of their definitions. Combinatorial complexes generalize all these objects as they both have part-whole relations and set-type relations \cite{hajij_topological_2023}. The figure has been adapted from Papillon et al. \cite{papillon2023architectures} et Hajij et al. \cite{hajij_topological_2023}}
\label{fig:topo}
\end{figure}

With these foundational objects introduced, we can now delve into lifting procedures that consist of transforming lower-dimensional objects into higher-order ones. For instance, we will use later lifting procedures to convert graph datasets into combinatorial complex datasets to apply our framework.\par

\subsection{Lifting}

Lifting denotes the process of mapping a featured domain to another featured domain through a well-defined procedure \cite{papillon2023architectures, hajij_topological_2023}. For example, the incorporation of rank-2 cells onto a graph, transforming it into a combinatorial complex, represents a lifting procedure. In this work, we employ two specific lifting procedures outlined in \cite{hajij_topological_2023}: the loop-based and the path-based methods. We have slightly modified the path-based procedure to accommodate multiple paths and multiple source nodes. Below, Figure \ref{fig:molecule_lift} (resp. Figure \ref{fig:path_lift}) illustrates the loop-based (resp. path-based) lifting procedure applied to a molecule.\par

\begin{definition}[Loop-based CC of a graph]
Let $G = (S, E)$ be a graph. We associate a CC structure with $G$ that considers loops in $G$. We define a loop-based CC of $G$ \cite{hajij_topological_2023}, denoted by $CC_{loop}(G)$, to be a CC consisting of 0-cells, 1-cells and 2-cells specified as follows. First, we set $\mathcal{X}_{0}$ and $\mathcal{X}_{1}$ in $CC_{loop}(G)$ to be the nodes and edges of $G$, respectively. We now explain how to construct a 2-cell in $CC_{loop}(G)$. A 2-cell in $CC_{loop}(G)$ is a set $C = \{x_{0}^{1},\ldots,x_{0}^{k}\} \subset\mathcal{X}_{0}$ such that for all $i\in\llbracket 1, k-1\rrbracket$, $\{x_{0}^{i},x_{0}^{i+1}\}$ and $\{x_{0}^{k},x_{0}^{1}\}$ are the only edges in $\mathcal{X}_{1}\cap C$. The set $\mathcal{X}_{2}$ in $CC_{loop}(G)$ is a nonempty collection of elements $C$. It is easy to verify that $CC_{loop}(G)$ is a CC with $\text{dim}(CC_{loop}(G)) = 2$. Note that the sequence $(x_{0}^{1},\ldots,x_{0}^{k})$ defines a loop in $G$. This loop is called the loop that characterizes the 2-cell $C = \{x_{0}^{1},\ldots,x_{0}^{k}\}$.
\end{definition}

\begin{remark}
When the graph is extracted from a molecule, the loops or cycles will refer to the rings of the molecule. Therefore, we will also refer to this method as a \textbf{ring-based lifting procedure}.
\end{remark}

\begin{figure}[H]
\centering
\includegraphics[width=0.8\linewidth]{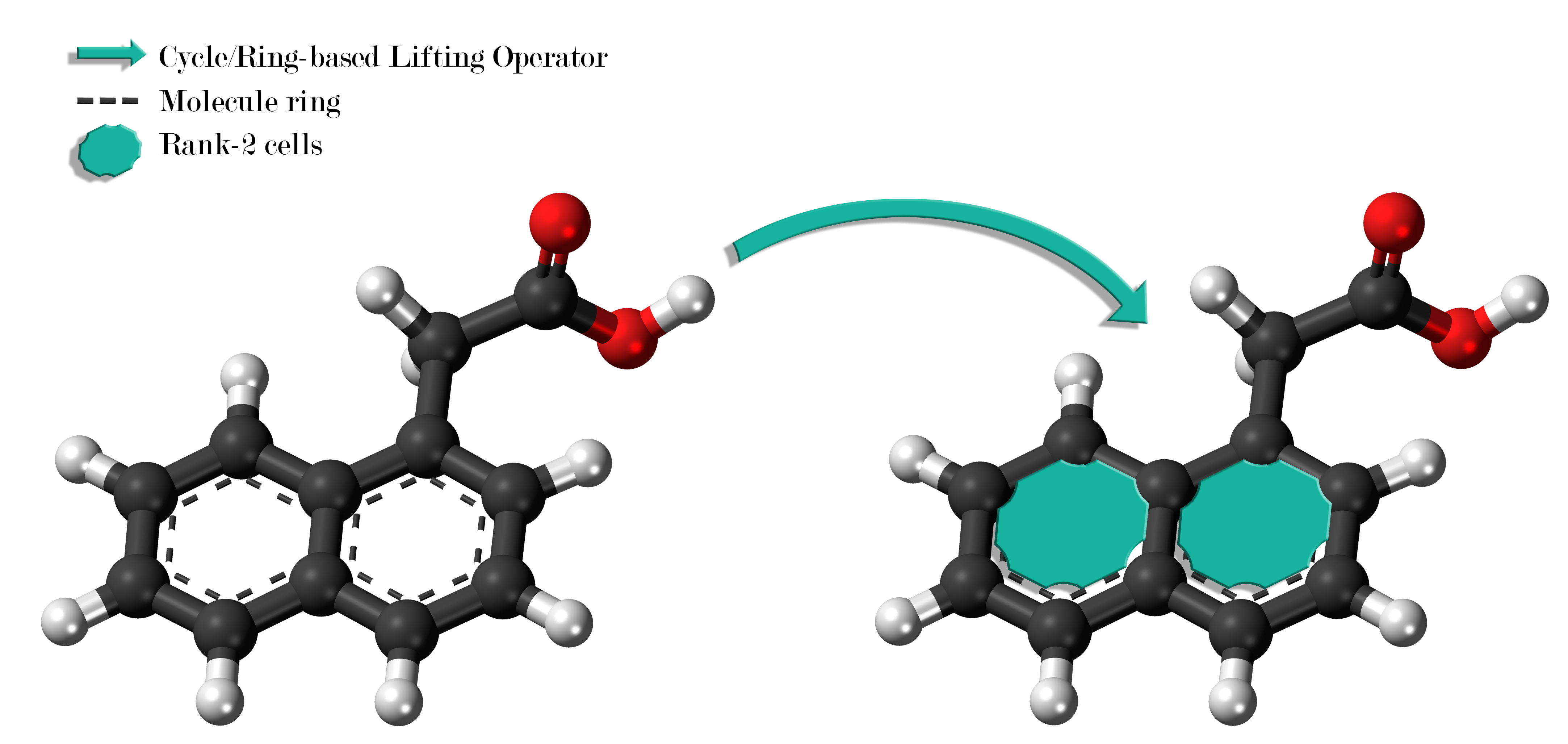}
\captionof{figure}{\textbf{Overview on the ring-based lifting procedure.} We start from the graph representation of a molecule, here a 1-naphthaleneacetic acid. Once the nodes belonging to a ring are identified, we group them to form a rank-2 cell that is added to create a combinatorial complex.}
\label{fig:molecule_lift}
\end{figure}

\begin{definition}[Path-based CC of a graph]
Let $G = (S, E)$ be a graph. We associate a CC structure with $G$ that considers paths in $G$. We define a loop-based CC of $G$, denoted by $CC_{P}(G)$, to be a CC consisting of 0-cells, 1-cells and 2-cells specified as follows. First, we set $\mathcal{X}_{0}$ and $\mathcal{X}_{1}$ in $CC_{P}(G)$ to be the nodes and edges of $G$, respectively. We now explain how to construct a 2-cell in $CC_{P}(G)$. Let $\mathcal{S}$ be a set of nodes that we will call sources nodes and $k\geq 1$ be a path length. Both objects are parameters. Let $\mathcal{P}$ be the set of all paths in $G$ starting from a node that belongs to $\mathcal{S}$ and that has exactly $k$ different nodes. A 2-cell in $CC_{P}(G)$ is a set $C = \{x_{0}^{1},\ldots,x_{0}^{k}\} \subset\mathcal{X}_{0}$ such that for all $x=(x_{0}^{1},\ldots,x_{0}^{k})\in C$, it exists a permutation $\pi_{k}\in P_{k}$ such that $\pi_{k}(x)\in\mathcal{P}$ and such that for all $i\in\llbracket 1, k\rrbracket$, $(\pi_{k}(x)_{i}, \pi_{k}(x)_{(i+1) \% k})\in\mathcal{X}_{1}$. It is easy to verify too that $CC_{P}(G)$ is a CC with $\text{dim}(CC_{P}(G)) = 2$.
\end{definition}

\begin{figure}[H]
  \centering
  \includegraphics[width=0.9\linewidth]{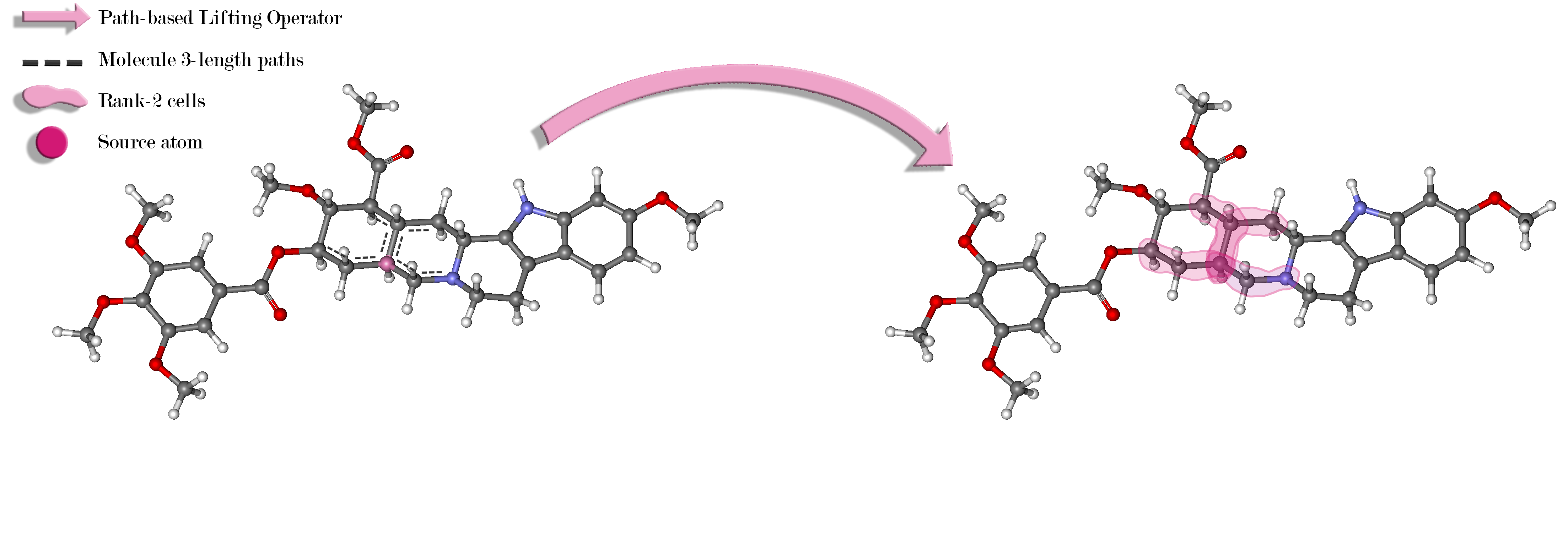}
  \captionof{figure}{\textbf{Overview on the path-based lifting procedure.} We start from the graph representation of a molecule, here an adelphan acid (more precisely, Reserpine). We start with one or many source node(s) and a path length $k\geq 1$. We identify the nodes belonging to the same paths of length $k$ in the graphs and that start with a node that belongs to the set of source nodes. We group them together to form a rank-2 cell that is added to create a combinatorial complex.}
  \label{fig:path_lift}
\end{figure}

\section{Graphs and Molecule Generation Metrics}
\label{Graphs_molecule_generation}

To effectively evaluate generative models, it is crucial to establish a quantitative measure of the proximity between generated samples and the original data distribution. However, evaluating complex structures such as graphs, molecules, and combinatorial complexes presents a challenge due to the absence of a well-defined distribution for these objects, as they are not mere numerical values. A viable approach involves assessing distributions of specific metrics for a given sample of these objects and a reference sample, subsequently evaluating the dissimilarity between these distributions. To do so, numerous metrics have been developed in previous works and are commonly employed to assess methods that generate graphs and molecules. In this section, we present the metrics used for benchmarking our models.\par

Later on, in section \ref{Metrics}, we will introduce novel metrics tailored to evaluate the quality of higher-order topological structures like combinatorial complexes. To the best of our knowledge, such an endeavour has not been previously undertaken.\par

\subsection{Evaluate distributions}

Maximum Mean Discrepancy (MMD) is used to quantify the difference or dissimilarity between two probability distributions $\mathcal{P}$ and $\mathcal{Q}$. Below, we present the definition of a kernel, of the discrepancy function, before introducing the definition of the MMD.\par

\begin{definition}[Kernel]
A kernel is a function $K:\mathbb{R}^{d}\mapsto \mathbb{R}$ integrable on $\mathbb{R}^{d}$ such that $\int_{\mathbb{R}^{d}}K(x)dx=1$, where $d\in\mathbb{N}^{*}$.
\end{definition}

\begin{remark}
In practice, the kernels are often chosen to be positive and symmetric.
\end{remark}

\begin{definition}[Discrepancy]
Let $X=\left ( X_{i}\right )_{1\leq i\leq M}\in\left (\mathbb{R}^{d}\right )^{M}$, $Y=\left ( Y_{j}\right )_{1\leq j\leq N}\in\left (\mathbb{R}^{d}\right )^{N}$, be two sets of histograms of size $d$ where $d,M,N\in\mathbb{N}^{*}$, and $K$ be a kernel function. The discrepancy associated with the kernel $K$, denoted $\mathcal{D}_{K}$, is defined by:\par

$\mathcal{D}_{K}(X,Y)=\sum_{i=1}^{M}\sum_{j=1}^{N}K(X_{i},Y_{j})$.
\end{definition}

\begin{remark}
The discrepancy is a symmetric function.
\end{remark}

\begin{definition}[Maximum Mean Discrepancy]
The Maximum Mean Discrepancy (MMD) between histograms $P$ and $Q$ is defined as:\par

$\text{MMD}(P, Q) = \sup_{f \in \mathcal{F}} \left | \E_{x \sim P}\left [f(x)\right ] - \E_{y \sim Q}\left [ f(y) \right ] \right |$, where:\par

\begin{itemize}
    \item $x$ (resp. $y$) is a random variable drawn from $P$ (resp. $Q$),
    \item $f$ is a function from a function space $\mathcal{F}$ that returns values in $\mathbb{R}$. In practice, it is often a kernel function,
    \item $\E_{x \sim P}\left [ f(x) \right ]$ represents the expected value of $f(x)$ with $x$ being drawn from $P$,
    \item $|\cdot |$ is the absolute value.
\end{itemize}

In practice, we equip the MMD with a kernel function $K(\cdot, \cdot)$, and compute:\par

$\text{MMD}(P, Q)=\mathcal{D}_{K}(P, P)+\mathcal{D}_{K}(Q,Q)-2 \mathcal{D}_{K}(P,Q)$, where $\mathcal{D}_{K}$ is the discrepancy function associated with the kernel $K$.
\end{definition}

Now that we have defined how to compare two distributions, let's define the two kernels, \textit{Gaussian} and \textit{Gaussian EMD} that have been implemented to evaluate the MMD for the different metrics.\par

\begin{definition}[Gaussian Kernel]
Let $\sigma>0$ be a standard deviation. The Gaussian kernel $\mathcal{G}$ between two distributions $x$ and $y$ is defined by:\par

$\mathcal{G}(x,y)=e^{-\frac{\left \|x-y \right \|_{2}^{2}}{2\sigma^{2}}}$.
\end{definition}

\begin{definition}[Earth Mover's Distance (EMD)]
The Earth Mover's Distance (EMD) \cite{monge_memoire_1781,Carrel_Optimal_transport_applied_2019} between two probability distributions $P$ and $Q$ associated with the distance $d$ is defined by:\par

$\text{EMD}(P,Q)=\underset{\gamma\in\Pi (P,Q)}{\text{inf }}\E_{(x,y)\sim\gamma}\left [ d(x,y)\right ]$, where $\Pi (P, Q)$ is the set of all joint distributions whose marginals are $P$ and $Q$.
\end{definition}

\begin{definition}[Gaussian EMD Kernel]
Let $\sigma>0$ be a standard deviation. The Gaussian EMD kernel $\mathcal{G}_{\text{EMD}}$ between two distributions $x$ and $y$ is the Gaussian kernel where the squared distance in the exponential term is replaced by the Earth Mover's distance. More precisely,\par

$\mathcal{G}_{\text{EMD}}(x,y)=e^{-\frac{\text{EMD}(x,y)}{2\sigma^{2}}}$.
\end{definition}

\subsection{Graphs}

For graph generation tasks, we will mainly look at the degree, clustering coefficient and orbit distribution. We define these three metrics below:\par

\begin{definition}[Degree Distribution]
Let $G$ be an undirected graph with $N$ nodes. We note its adjacency matrix $A=(A_{i,j})_{1\leq i,j\leq N}$, without self-loops. We index the nodes from 1 to $N$, same for the rows and columns of $A$.\par

The degree of a node $i\in\llbracket 1,N\rrbracket$ is $\text{deg}_{G}(i)=\sum_{j=1}^{N}A_{i,j}$.\par

Let's note $\text{deg}_{\text{min}}(G)=\underset{1\leq i\leq N}{\text{min }}\text{deg}_{G}(i)$ and $\text{deg}_{\text{max}}(G)=\underset{1\leq i\leq N}{\text{max }}\text{deg}_{G}(i)$. The degree distribution associated with the graph $G$ is the vector $\left ( d_{j} \right )_{\text{deg}_{\text{min}}(G)\leq j\leq \text{deg}_{\text{max}}(G)}$ where for all $j\in\llbracket \text{deg}_{\text{min}}(G),\text{deg}_{\text{max}}(G) \rrbracket$, $d_{j}=\sum_{i=1}^{N}\mathbb{1}_{\{ \text{deg}_{G}(i)=j \}}$.
\end{definition}

\begin{definition}[Clustering Coefficient Distribution]
Let $G$ be an undirected graph with $N$ nodes. We note its adjacency matrix $A=(A_{i,j})_{1\leq i,j\leq N}$, without self-loops. We index the nodes from 1 to $N$, same for the rows and columns of $A$.\par

The clustering coefficient of a node $i\in\llbracket 1,N\rrbracket$ is $C_{G}(i)=\frac{\lambda_{G}(i)}{\tau_{G}(i)}$,\par

where $\lambda_{G}(i)=2 \sum_{j=1}^{N}\sum_{k=1}^{N}\mathbb{1}_{\{ A_{i,j}=1\}\cap\{ A_{i,k}=1\}\cap\{ A_{j,k}=1\}}$ and $\tau_{G}(i)=\text{deg}_{G}(i)\left ( \text{deg}_{G}(i)-1\right )$.\par

$\lambda_{G}(i)$ represents the number of triangles that we can construct with the neighbours of the node $i$, whereas $\tau_{G}(i)$ represents the number of links that could exist among the vertices within the neighbourhood of $i$. A high clustering coefficient thus means that a node is highly connected to its neighbourhood.\par

Let's note $C_{\text{min}}(G)=\underset{1\leq i\leq N}{\text{min }}C_{G}(i)$ and $C_{\text{max}}(G)=\underset{1\leq i\leq N}{\text{max }}C_{G}(i)$. The clustering coefficient distribution associated with the graph $G$ is the vector $\left ( c_{j} \right )_{C_{\text{min}}(G)\leq j\leq C_{\text{max}}(G)}$ where for all $j\in\llbracket C_{\text{min}}(G),C_{\text{max}}(G) \rrbracket$, $c_{j}=\sum_{i=1}^{N}\mathbb{1}_{\{ C_{G}(i)=j \}}$.
\end{definition}

\begin{definition}[Orbit Distribution]
Let $G=(V,E)$ be an undirected graph with $N=|V|$ nodes. The orbit of a node $i\in\llbracket 1,N\rrbracket$ is defined by $\text{Orb}(G,i)=|\{ w\in V | \exists\sigma\in\text{Aut}(G):\sigma(v)=w \}|$ \cite{10.1093/bioinformatics/btx758}, where $\text{Aut}(G)$ is the group group of automorphisms of a graphlet $G$. Graphlets, as introduced by Przulig et al. \cite{prvzulj2004modeling}, are subgraphs that are motifs. More intuitively, $\text{Aut}(G)$ is the group of permutations of the nodes that leaves the edge set unchanged. The orbit usually defines the set of nodes but we will consider the orbit as the cardinal of this set as defined above.\par

Let's note $\text{Orb}_{\text{min}}(G)=\underset{1\leq i\leq N}{\text{min }}\text{Orb}(G,i)$ and $\text{Orb}_{\text{max}}(G)=\underset{1\leq i\leq N}{\text{max }}\text{Orb}(G,i)$. The orbit distribution associated with the graph $G$ is the vector $\left ( o_{j} \right )_{\text{Orb}_{\text{min}}(G)\leq j\leq \text{Orb}_{\text{max}}(G)}$,\par

where for all $j\in\llbracket \text{Orb}_{\text{min}}(G),\text{Orb}_{\text{max}}(G) \rrbracket$, $o_{j}=\sum_{i=1}^{N}\mathbb{1}_{\{ \text{Orb}(G,i)=j \}}$.
\end{definition}

\begin{remark}
We used the tool ORbit Counting Algorithm (ORCA) 
 \cite{hovcevar2014combinatorial} developed in C++ to compute the Orbit distribution efficiently.
\end{remark}

To compute our MMDs, we used the following kernels for each metrics:\par

\begin{itemize}
    \item \textbf{Degree:} Gaussian EMD,
    \item \textbf{Cluster:} Gaussian EMD,
    \item \textbf{Orbit:} Gaussian.
\end{itemize}

\subsection{Molecules}

For the molecule generation task, we will compare several metrics, including the Fréchet ChemNet Distance (FCD) \cite{preuer2018frechet}, the Neighborhood subgraph pairwise distance kernel (NSPDK) MMD \cite{Costa2010FastNS}, Validity (with and without correction), Novelty, and Uniqueness. Additionally, we will compare the inference time required to generate 10,000 molecules. Lastly, we will also compare the average Tanimoto similarity, as detailed in Subsection \ref{Tanimoto}.\par

\begin{definition}[Fréchet ChemNet Distance (FCD)]
The Fréchet ChemNet Distance \cite{preuer2018frechet} is a metric used to compare the similarity between two chemical molecules based on their structural features. It is defined as follows:\par

Let $M_1$ and $M_2$ be two molecular graphs representing the chemical structures of two molecules. The Fréchet ChemNet Distance between $M_1$ and $M_2$, denoted as $FCD(M_1, M_2)$, is defined as the minimum continuous assignment of two continuous functions $f:[0,1]\rightarrow V(M_1)$ and $g:[0,1]\rightarrow V(M_2)$, such that:\par

\begin{align*}
    &f(0) = g(0) = \text{start node}, \\
    &f(1) = g(1) = \text{end node}, \\
    &\text{For all } t\in[0,1], d(f(t),g(t)) \leq \text{radius}(f(t)) + \text{radius}(g(t)),
\end{align*}

where $d(v_1, v_2)$ denotes the Euclidean distance between the coordinates of nodes $v_1$ and $v_2$, and $\text{radius}(v)$ represents the radius associated with a node $v$ in the molecular graph. In summary, the Fréchet ChemNet Distance measures the similarity between the two molecules by finding the minimum \textit{continuous path} between them while taking into account the spatial arrangement of atoms in their structures.
\end{definition}

\begin{definition}[Validity]
Validity is the fraction of the generated molecules that do not violate the chemical valency rule.
\end{definition}

\begin{definition}[Uniqueness]
Uniqueness is the fraction of the generated valid molecules that are unique.
\end{definition}

\begin{definition}[Novelty]
Novelty is the fraction of the valid molecules that are not included in the training set.
\end{definition}

\begin{definition}[Validity w/o correction]
Validity w/o correction is the fraction of valid molecules without valency correction or edge resampling. In this thesis, we allowed atoms to have formal charges when checking their valency following the methodology of Zang \& Wang \cite{Zang_2020} and Jo et al. \cite{jo2022scorebased}. The metric is thus different from the metric used in Shi et al. \cite{shi2020graphaf} and Luo et al. \cite{luo2021graphdf}. We have chosen to also implement this approach as it seems to be more reasonable due to the existence of formal charges in the training molecules.
\end{definition}

\begin{definition}[Sampling Time]
Sampling time measures the time for generating 10,000 molecules in the form of RDKit molecules.
\end{definition}

The \textit{Preliminaries} notions being introduced, we will now in the following chapter of this thesis delve into our theoretical contributions.\par

%%%%%%%%%%%%%%%%%%%%%%%%%%%%%%%%%%%%
\chapter{Theoretical Contributions}
\label{Theory}

With the essential foundational material now covered, we can proceed to present our theoretical contributions, encompassing mathematical constructs, theorems, machine learning architectures, algorithms, metrics, and our overarching framework. Our theoretical framework revolves around core concepts and objects that we created and first need to introduce in the \textit{Preamble} section below (Section \ref{Preamble}).\par

\section{Preamble}
\label{Preamble}

As part of the mathematical framework, we introduce novel objects that extend the domain of combinatorial complexes, designed to serve as the fundamental entities in our generative modelling context. Recognizing the potential large search space when diffusing along higher-order dimensions within a combinatorial complex, we introduce Dimension-Constrained Combinatorial Complexes (DCCC). For instance, consider a lifted molecule with up to 15 atoms. This could result in up to $2^{15}=32768$ distinct rank-2 cells. However, if we constrain our focus to rings (cycles in a molecular graph) containing between 3 and 9 atoms, the search space narrows down to $\sum_{k=3}^{9}\binom{15}{k}=27703$ cells. To illustrate the complexity of the lifting procedure for certain datasets, we present in Figure \ref{fig:longest_mol_24} an example from the ZINC250k dataset \cite{irwin_zinc_2012} featuring the longest ring of the dataset with 24 atoms. For such datasets, with a native approach, the search space would be too large. Consequently, we seek to construct a diffusion model capable of generating combinatorial complexes with specific attributes characterizing the size of the higher-order cells. To achieve this, we introduce Dimension-Constrained Combinatorial Complexes (DCCC).\par

\begin{figure}[H]
\centering
\includegraphics[width=0.5\linewidth]{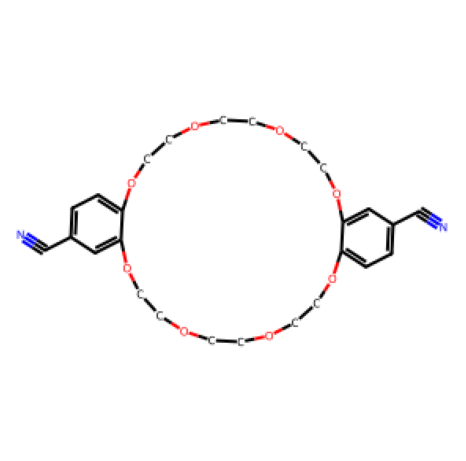}
\captionof{figure}{Molecule with the longest ring in the ZINC250k dataset \cite{irwin_zinc_2012}. The molecule has a ring made of 24 atoms.}
\label{fig:longest_mol_24}
\end{figure}

Moreover, our goal is for these combinatorial complexes to exhibit characteristics that align with the underlying objects they represent. To achieve this, we introduce Featured Combinatorial Complexes (FCC), which involve the attachment of features -or cochains- to the cells within these CCs.\par

Given that the hyperparameters associated with our combinatorial complexes are task-specific, we intend for them to be provided as input parameters for our model.\par

\begin{definition}[Dimension-Constrained Combinatorial Complex]

\par
A Dimension-Constrained Combinatorial Complex (DCCC) is a tuple $(CC, D)$ where $CC=(S, \mathcal{X}, rk)$ is a combinatorial complex and $D=(d^{r}_{min}, d^{r}_{max})_{0\leq r \leq R}$ is a collection of tuples where $R=\text{dim}(CC)$ is the dimension of the combinatorial complex CC, and such that, for all $r\in\llbracket 0, R\rrbracket$, for all $x\in\mathcal{X}_{r}$, $d^{r}_{min}\leq |x| \leq d^{r}_{max}$. $|x|$ represents the cardinal (number of nodes) of the $r$-rank cell $x$.\par

Without loss of generality, we assume that the nodes are rank-0 cells and that the edges are rank-1 cells. Using the notation above, this can be rewritten as $d^{0}_{min}=d^{0}_{max}=1$ and $d^{1}_{min}=d^{1}_{max}=2$.
\end{definition}

\begin{definition}[Featured Combinatorial Complex (FCC)]
A Featured Combinatorial Complex (FCC) is a tuple $(CC, \mathcal{F})$ where $CC=(S, \mathcal{X}, rk)$ is a combinatorial complex and $\mathcal{F}:\mathcal{X}\rightarrow \mathfrak{G}$ is a mapping function that assigns to every cell a \textit{feature} that belongs to a set $\mathfrak{G}$. We assume that it exists an underlying ring $\mathbb{K}$ such that $(\mathfrak{G}, \cdot, *)$ verifies that for all $r\in\llbracket 0, R\rrbracket$ where $R=\text{dim}(CC)$, $(\mathfrak{G}_{|\mathcal{X}_{r}}, \cdot, *)$ is a ring made of one or multiple elements of $\mathbb{K}$. $\mathfrak{G}_{|\mathcal{X}_{r}}$ is the restriction of $\mathfrak{G}$ to the set of rank-$r$ cells $\mathcal{X}_{r}$.\par

This means that we can do operations between features of cells with the same rank and, with some adjustments in terms of the size of the elements, some operations between the features of all the cells of the combinatorial complex.\par

We also define for all $r\in\llbracket 0, R\rrbracket$, $f_{r}=\underset{x\in\mathcal{X}_{r}}{\text{max} }|\mathcal{F}(x)|\in\mathbb{N}^{*}$ where $|x|$ is the cardinal or size of the object $x$. By convention, if for a given $r$ $|\mathcal{X}_{r}|=0$, $f_{r}=1$.
\end{definition}

\begin{remark}
We can map every combinatorial complex as a featured combinatorial complex with $\mathcal{F}=1$ (constant function equal to 1).
\end{remark}

\begin{remark}
In practice, the underlying ring/field $\mathbb{K}$ will often be $\mathbb{Z}$ or $\mathbb{R}$, and for all $r\in\llbracket 0, R\rrbracket$, $\mathfrak{G}_{|\mathcal{X}_{r}}$ will be isomorph to a vector space $\mathbb{K}^{r'}$ where $r'\in\mathbb{N}^{*}$.
\end{remark}

\begin{definition}[Dimension-Constrained Featured Combinatorial Complex (DCFCC)]
A Dimension-Constrained Featured Combinatorial Complex (DCFCC) is a tuple

$(CC, D, \mathcal{F})$ such that $(CC, D)$ is a DCCC and $(CC, \mathcal{F})$ is a FCC.
\end{definition}

\begin{remark}
From now on, we will assimilate a DCFCC as a CC as the difference is essentially from a computational and modelling perspective.
\end{remark}

\begin{definition}[CC structure class, Representation]
As stated in Hajij et al. \cite{hajij_topological_2023}, a \textit{CC structure class} of a combinatorial complex is a set of objects that allows to represent the combinatorial complex up to an isomorphism, according to the definition \cite[Definition.~10]{hajij_topological_2023}.\par

We will also call a CC structure class of a combinatorial complex or a set of combinatorial complexes a \textbf{representation}.\par

This notion of representation can be extended to the generalization or subclass of combinatorial complexes such as the ones introduced before.
\end{definition}

As we want to efficiently and numerically generate combinatorial complexes for all different types of domains, we need to represent them coherently. Hence the theorem below:\par

\begin{theorem}[Representation of Dimension-Constrained Featured Combinatorial Complexes of dimension 2]
\label{theorem:representation}

Every Dimension-Constrained Featured Combinatorial Complexes of dimension 2 $(CC, D, \mathcal{F})$ is entirely defined by three tensors

$(X, A, F)\in\mathcal{M}_{n,f_{0}}(\mathbb{K})\times \mathcal{M}_{n,n,f_{1}}(\mathbb{K})\times \mathcal{M}_{\binom{n}{2},\mathcal{K},f_{2}}(\mathbb{K})$ and the tuple $(d_{min}^{2},d_{max}^{2})$.\par

\begin{itemize}
\item $n=|\mathcal{X}_{0}|$ is the number of rank-0 cells (or nodes),
\item $\binom{n}{2}=\frac{n(n-1)}{2}$ is the maximum number of rank-1 cells,
\item $\forall i,j\in\llbracket 1, N\rrbracket$, $A_{i,j,:}=A_{j,i,:}$ (\textit{i.e.} $A$ must be symmetric along its first two axes),
\item $\mathcal{K}=\sum_{k=d_{min}^{2}}^{d_{max}^{2}}\binom{n}{k}$ is the maximum number of rank-2 cells,
\item $\forall j\in\llbracket 1,\mathcal{K}\rrbracket$, $\exists C_{j}\in\mathbb{K}^{f_{2}}$, $\forall i\in\llbracket 1, \binom{N}{2}\rrbracket$, $F_{i,j,:}\in\left \{ 0_{\mathbb{K}}^{f_{2}},C_{j} \right \}$ (\textit{i.e.} for every columns of $F$, the elements are either null or, if not for some rows, they share the same value/cochain).
\end{itemize}
\end{theorem}

\begin{proof}
Let $(CC, D, \mathcal{F})$ be a Dimension-Constrained Featured Combinatorial Complex where $CC=(S,\mathcal{X}, rk)$. For all $x\in\mathcal{X}_{0}$, $|x|=1$ and $|\mathcal{F}(x)|=k_{x}\leq f_{0}$. By isomorphism, we can represent $\mathcal{F}(x)$ as a vector $(x_{1},\ldots ,x_{k_{x}}, 0,\ldots ,0)\in\mathbb{K}^{f_{0}}$. Without loss of generality, let's order these nodes $x_{1},\ldots , x_{n}$.\par

For all $e\in\mathcal{X}_{1}$, $|e|=2$ and $|\mathcal{F}(e)|=k_{e}\leq f_{1}$. By isomorphism, we can represent $\mathcal{F}(e)$ as a vector $(e_{1},\ldots ,e_{k_{e}}, 0,\ldots ,0)\in\mathbb{K}^{f_{1}}$. By definition, it exists $x,y\in\mathcal{X}_{0}$ such that $e=\{ x, y\}=\{ y, x\}$ as it is a set. They then share the same image by $\mathcal{F}$. We order the edges, even those that are not in our combinatorial complex, in ascending order in function of the node indexes as follows:

$\left (e_{1}=(x_{1},x_{2}), e_{2}=(x_{1},x_{3}), \ldots, e_{n-1}=(x_{1},x_{n}), e_{n}=(x_{2},x_{3}), \ldots, e_{\binom{n}{2}}=(x_{n-1},x_{n})\right )$.\\

Finally, for all $h\in\mathcal{X}_{2}$, $|\mathcal{F}(h)|=k_{h}\leq f_{2}$. By isomorphism, we can represent $\mathcal{F}(h)$ as a vector $(l_{1},\ldots ,l_{k_{h}}, 0,\ldots ,0)\in\mathbb{K}^{f_{2}}$. We know that there is at most $\binom{n}{2}$ edges (or rank-1 cells as in our case $d_{max}^{1}=2$), so $0\leq |X_{1}|\leq \binom{n}{2}$. We also know that the dimension of the CC is two so there is at least one rank-2 cell with a cardinal between $d_{min}^{2}$ and $d_{max}^{2}$. Therefore, $1\leq |\mathcal{X}_{2}|\leq K=\sum_{k=d_{min}^{2}}^{d_{max}^{2}}\binom{n}{k}$. All the groups of nodes forming a rank-2 cell are part of a set and thus share the same image by $\mathcal{F}$. For a given $k\in\llbracket d_{min}^{2}, d_{max}^{2}\rrbracket$, we can also order the rank-2 cells of size $k$, if they exist, in ascending order in function of the node indexes. Then, we order the rank-2 cells by concatenating the existing ordered cells of size $d_{min}^{2}$, then $d_{min}^{2}+1$, etc, until $d_{max}^{2}$. The ordered rank-2 cells, including the one not in our combinatorial complex, will be denoted $h_{1}, \ldots, h_{K}$.\par

Below, we extend $\mathcal{F}$ such that, for $x\in\mathcal{P}(S)\backslash\emptyset$, if $x\notin\mathcal{X}$, $\mathcal{F}(x)=(0,\ldots,0)\in\mathbb{K}^{m}$ where $m$ is equal to $f_{1}$ if $|x|=2$, or $f_{2}$ otherwise. $|x|\neq 1$ as the combinatorial complex definition implies that $\forall x\in S, x\in\mathcal{X}$. Below, $0$ represents the additive identity (or zero) of $\mathbb{K}$.\par

We then construct the three tensors:\\

$X=\begin{pmatrix}
\mathcal{F}(x_{1})\\ 
\vdots\\ 
\mathcal{F}(x_{n})
\end{pmatrix}\in\mathcal{M}_{n,f_{0}}(\mathbb{K})$,\par

$A=\begin{pmatrix}
0 & \mathcal{F}((x_{1},x_{2})) & \ldots & \mathcal{F}((x_{1},x_{n}))\\ 
\mathcal{F}((x_{2},x_{1})) & 0 &  & \vdots\\ 
\vdots &  & 0 & \mathcal{F}((x_{n-1},x_{n}))\\ 
\mathcal{F}((x_{n},x_{1})) &  & \mathcal{F}((x_{n},x_{n-1})) & 0
\end{pmatrix}\in\mathcal{M}_{n,n,f_{1}}(\mathbb{K})$,\par

and $F=(m_{i,j})_{i,j\in\llbracket 1,\binom{n}{2}\rrbracket\times\llbracket 1,K\rrbracket}\in\mathcal{M}_{\binom{n}{2},\mathcal{K},f_{2}}(\mathbb{K})$ defined by for all $i\in\llbracket 1,\binom{n}{2}\rrbracket$, for all $j\in\llbracket 1,K\rrbracket$, $m_{i,j}=\begin{cases}
\mathcal{F}(h_{j}) & \text{ if } e_{i}\in h_{j} \\ 
(0,\ldots, 0) & \text{ else }
\end{cases}\in\mathbb{K}^{f_{2}}$.\par

The construction is well-defined. Now let's verify that we can build our original $DCFCC$ from our original tensors and tuple.\par

From $X$, we immediately have our set $S=(0,\ldots, n-1)=\mathcal{X}_{0}$, $f_{0}$, and $\mathcal{F}_{|\mathcal{X}_{0}}$, where $(n,f_{0})$ is the dimension of $X$.\par

$A$ is symmetric along its first two axes, its shape is $(n,n,f_{1})$. For all $i,j\in\llbracket 1, n\rrbracket$ with $i\neq j$, we add $(i,j)$ to the set of rank-1 cells $\mathcal{X}_{1}$ if $A_{i,j}\neq 0$ and we define $\mathcal{F}_{|\mathcal{X}_{1}}:=(i,j)\mapsto A_{i,j}$. $\mathcal{F}_{|\mathcal{X}_{1}}$ is 0 elsewhere.\par

From $X$ and $A$, we can deduce the incidence matrix $B_{0,1}\in\mathcal{M}_{n,\binom{n}{2}}(\mathbb{K})$ that maps the edges to their corresponding nodes and thus, preserves the rank.\par

$F$ verifies the column value property and the shape criteria. We deduce from it $f_{2}$, the rank-2 cells and their nodes by looking at the indexes of the non-zeros column of $F$, $\mathcal{F}_{|\mathcal{X}_{2}}$ by looking at the first non-zero coefficient if it exists, and by the construction of $F$, we also have the rank-2 incidence matrix $B_{1,2}\in\mathcal{M}_{\binom{n}{2},K}(\mathbb{K})$ that maps the rank-2 cells or faces to their corresponding edges and thus, preserves the rank.\par

By applying \cite[Proposition 8.1]{hajij_topological_2023}, $\{ B_{r,r+1}\}_{r=0}^{\text{dim}(CC)-1}=\{ B_{0,1}, B_{1,2}\}$, we are guaranteed that we have the CC structure class of a combinatorial complex $CC=(S, \mathcal{X}, rk)$.\par

From the dimensions constrained on the cells, we also have a DCCC. For all $r\in\llbracket 0, 2\rrbracket$, we verified that $\mathcal{F}_{|\mathcal{X}_{r}}$ is a linear vector subspace of $\mathbb{K}^{f_{r}}\sim\mathfrak{G}_{|\mathcal{X}_{r}}$ and as our constructed function $\mathcal{F}$ is defined on the entire cell domain $\mathcal{X}$, we also have a FCC. Hence the proof.
\end{proof}

\begin{remark}
In our implementation, $d_{min}^{2}$ and $d_{max}^{2}$ are simply referred as $d_{min}$ and $d_{max}$.
\end{remark}

\begin{remark}
\textbf{Crucial Note:} Moving forward in our modelling discussions, we will consider for the rest of this work combinatorial complexes as dimension-constrained featured combinatorial complexes. Moreover, when $f_{r}=1$ for a given $r$, we will simplify our notation and treat the corresponding tensors (such as $A$, $F$, or other incidence matrices) as 2D matrices. For brevity, we may refer to these matrices to mention a CC.
\end{remark}

\begin{corollary}[Representation of Dimension-Constrained Featured Combinatorial Complexes of dimension greater or equal than 2]
Every Dimension-Constrained Featured Combinatorial Complexes of dimension $R\geq 2$ $(CC, D, \mathcal{F})$ is entirely defined by $R+1$ tensors:\par

$(X, A, F, B_{2,3}, \ldots, B_{R-1,R})\in\mathcal{M}_{n,f_{0}}(\mathbb{K})\times \mathcal{M}_{n,n,f_{1}}(\mathbb{K})\times \mathcal{M}_{\binom{n}{2},\mathcal{K}_{2},f_{2}}(\mathbb{K}) \times \mathcal{M}_{\mathcal{K}_{2},\mathcal{K}_{3},f_{3}}(\mathbb{K})\times \ldots \times \mathcal{M}_{\mathcal{K}_{R-1},\mathcal{K}_{R},f_{R}}(\mathbb{K})$ and the tuples $\left \{ (d_{min}^{r},d_{max}^{r}) \right \}_{r\in\llbracket 2, R\rrbracket}$.\par

\begin{itemize}
\item $n=|\mathcal{X}_{0}|$ is the number of rank-0 cells (or nodes),
\item $\binom{n}{2}=\frac{n(n-1)}{2}$ is the maximum number of rank-1 cells,
\item $\forall i,j\in\llbracket 1, N\rrbracket$, $A_{i,j,:}=A_{j,i,:}$ (\textit{i.e.} $A$ must be symmetric along its first two axes),
\item For all $r\in\llbracket 2, R\rrbracket$, $\mathcal{K}_{r}=\sum_{k=d_{min}^{r}}^{d_{max}^{r}}\binom{n}{k}$ is the maximum number of rank-$r$ cells,
\item $\forall j\in\llbracket 1,\mathcal{K}\rrbracket$, $\exists C_{j}\in\mathbb{K}^{f_{2}}$, $\forall i\in\llbracket 1, \binom{N}{2}\rrbracket$, $F_{i,j,:}\in\left \{ 0_{\mathbb{K}}^{f_{2}},C_{j} \right \}$ (\textit{i.e.} for every column of $F$, the elements are either null or, if not for some rows, they share the same value/cochain).
\item $\forall r\in\llbracket 3, R\rrbracket$, $\forall j\in\llbracket 1,\mathcal{K}_{r}\rrbracket$, $\exists C_{r,j}\in\mathbb{K}^{f_{r}}$, $\forall i\in\llbracket 1, \mathcal{K}_{r-1}\rrbracket$, $B_{r-1,r,i,j,:}\in\left \{ 0_{\mathbb{K}}^{f_{r}},C_{r,j} \right \}$ (\textit{i.e.} for every column of $B_{r-1,r}$, the elements are either null or, if not for some rows, they share the same value/cochain).
\end{itemize}
\end{corollary}

\begin{proof}
We prove this result by induction. $R=2$ is immediate. Let $CC$ be a dimension-constrained featured combinatorial complex of dimension $R>2$. We suppose that our induction hypothesis is verified for all $r<R$. If we restrict $CC$ to its first $R-1$ dimensions, we obtain $CC'$ that has a representation $(X,A,F,B_{2,3},\ldots,B_{R-2,R-1})$.\par

By applying again \cite[Proposition 8.1]{hajij_topological_2023}, we only need to incorporate $B_{R-1,R}$ to obtain a CC structure class as this last matrix will encode our rank-$R$ cells. The shape $\mathcal{K}_{R-1}\times\mathcal{K}_{R}$ is imposed by the dimension contraints. The third dimension of the tensor, $f_{R}$, is also imposed by the features (or cochains) attached to our rank-$R$ cells. This concludes the proof.
\end{proof}

\begin{definition}[Hodge Dual Operator]
Let $n\in\mathbb{N}^{*}$ be a positive integer and $A=\left (A_{i,j} \right )_{1\leq i,j\leq n} \in\mathcal{S}_{n}(\mathbb{R})$ be an adjacency matrix with no self-loops i.e. $A$ verifies:
\begin{itemize}
\item $\forall i,j\in\llbracket 1,n \rrbracket, A_{i,j}\in\mathbb{R}$
\item $\forall i,j\in\llbracket 1,n \rrbracket, A_{i,j}=A_{j,i}$
\item $\forall i\in\llbracket 1,n \rrbracket, A_{i,i}=0$ (no self-loops)
\end{itemize}

We define the Hodge Dual Operator $\mathfrak{H}_{n}$ as the following bijection:\par

$\mathfrak{H}_{n}:= \underset{A=\begin{pmatrix}
a_{1,1} & \ldots & a_{1,j} & \ldots & a_{1,n} \\ 
 &  \ddots &  &  & \vdots \\ 
 &  &  a_{j,j} &  & a_{j,n} \\ 
 &  \ldots &  &  \ddots & \vdots \\ 
 &  &  &  & a_{n,n}
\end{pmatrix}}{S_{n}(\mathbb{R})}\underset{\mapsto}{\longrightarrow}\underset{\begin{pmatrix}
a_{1,2} &  &  &  &  & \\ 
 &  \ddots &  &  & (0) & \\ 
 &  & a_{1,n} &  &  & \\ 
 &  &  & a_{2,3} &  & \\ 
 & (0) &  &  & \ddots & \\
 &  &  &  &  & a_{n-1,n}\\
\end{pmatrix}}{S_{\binom{n}{2}}(\mathbb{R})}$\\

where $\binom{n}{2}=\frac{n(n-1)}{2}$ represents the number the number of possible unique undirected edges in a graph with no self-loops.\par

The resulting matrix will be called \textbf{Hodge dual} or \textbf{Hodge adjacency} matrix with respect to the adjacency matrix $A$.
\end{definition}

\begin{remark}
The bijection of $\mathfrak{H}_{n}$ (resp. $\mathfrak{H}$ if there is no ambiguity) will be written $\mathfrak{H}_{n}^{-1}$ (resp. $\mathfrak{H}^{-1}$).
\end{remark}

\begin{remark}
We can extend the application to a group of batched and channelled adjacency matrices $A=\{ A_{b,c}\}_{b=1,c=1}^{B,C}\in\left (\mathcal{M}_{n}(\mathbb{R})\right )^{B\times C}$ where $B,C\in\mathbb{N}^{*}$ with the following convention: $\mathfrak{H}(A)=\{ \mathfrak{H}(A_{b,c})\}_{b=1,c=1}^{B,C}\in\left (\mathcal{M}_{\binom{n}{2}}(\mathbb{R})\right )^{B\times C}$.
\end{remark}

\begin{definition}[Higher-order adjacency matrix]
Let $A\in\mathcal{S}_{n}(\mathbb{R})$ be an adjacency matrix (symmetric with zeros on the diagonal) and $p\in\mathbb{N}^{*}$. The higher-order adjacency matrix of $A$ of order $p$ is a channelled matrix defined by $A^{p}=\{ (A)^{i}\}_{i=1}^{p}$ where $(A)^{i}$ is $A$ to the power $i$ with $i$ being an integer.\par

These matrices encode paths of certain lengths between nodes through the edges within a given topological object represented by an adjacency matrix.
\end{definition}

In the same way, we can define such an encoding of paths between edges through the rank-2 cells of a combinatorial complex. The resulting channelled matrix will be called a higher-order rank-2 incidence matrix. First, we need to define the Hodge Laplacian. In practice and as detailed below, we adopt a slightly lighter version than the one used in the literature \cite{Schaub_2021,Barbarossa_2020}.

\begin{definition}[Hodge Laplacian\footnotemark{}]
Let $\mathfrak{F}$ be the space of all the possible rank-2 incidence matrices associated with a combinatorial complex of dimension 2. For an element $F\in\mathfrak{F}$, its Hodge Laplacian matrix $\mathcal{H}$ is the result of the Hodge Laplacian operator $\mathcal{L}$ defined by:\par

$\mathcal{L}(F)=F F^{T}\in S_{\binom{n}{2}}(\mathbb{K})$,\par

where $n$ is the number of nodes of the underlying combinatorial complex.\par

We can extend the definition to a combinatorial complex of dimension 2 by applying the Hodge Laplacian operator its rank-2 incidence matrix obtained through its representation.
\end{definition}
\footnotetext{This definition is a simpler version of the original Hodge $k$-Laplacian as presented in Lek-Heng Lim \cite[Proposition 4.3]{doi:10.1137/18M1223101}}

\begin{proposition}
Let $(CC,D,\mathcal{F})$ be a combinatorial complex such that $f_{2}=1$ and $\mathcal{F}_{|\mathcal{X}_{2}}=1$. Then, its Hodge Laplacian corresponds to the number of different paths of length 2 going from one edge to another by going through a common rank-2 cell.
\end{proposition}

\begin{proof}
Let $(CC,D,\mathcal{F})$ be a combinatorial complex with the hypothesis mentioned above verified. Let $i,j\in\llbracket 1, \binom{n}{2}\rrbracket$ be the two indexes of two edges that could exist in the combinatorial complex ($i$ and $j$ could be the same).\par

With $\mathcal{H}=\mathcal{L}(F)$, by following the notation from the representation theorem, we have:\par

$\mathcal{H}_{i,j}=\sum_{k=1}^{\mathcal{K}}\mathcal{F}(h_{k}) \mathbb{1}_{\{ h_{k}\}}(e_{i})  \mathcal{F}(h_{k}) \mathbb{1}_{\{ h_{k}\}}(e_{j})$ where $\mathbb{1}$ is the indicator function\footnotemark{} and $(h_{i})_{1\leq i\leq\mathcal{K}}$ are the ordered potential rank-2 cells.

\footnotetext{Let $A$ be a set and $a$ a mathematical object. We define the indicator function on $A$ by: $\mathbb{1}_{A}:=\begin{cases}
1 & \text{ if } a\in A \\
0 & \text{ else }
\end{cases}$. By convention, if $A=\emptyset$, $\mathbb{1}_{A}=0$.}

If one of the edges and/or the rank-2 cells $h_{i}$ is not in the combinatorial complex, the indicator functions are equal to 0. Thus, because $\mathcal{F}=1$, for a given $k\in\llbracket 1, \mathcal{K}\rrbracket$,\par

$\mathcal{F}(h_{k}) \mathbb{1}_{\{ h_{k}\}}(e_{i}) \mathcal{F}(h_{k}) \mathbb{1}_{\{ h_{k}\}}(e_{j})=\begin{cases}
1 & \text{ if } e_{i}\in h_{k} \text{ and } e_{j}\in h_{k} \\ 
0 & \text{ else }
\end{cases}$.\\

Therefore, the coefficient $\mathcal{H}_{i,j}$ of $\mathcal{H}$ is the number of common rank-2 cells that belong to the CC between the edges $i$ and $j$ if they are in the CC too.
\end{proof}

\begin{definition}[Higher-order rank-2 incidence matrix]
Let $F\in\mathcal{M}_{\binom{n}{2},\mathcal{K}}(\mathbb{R})$ be an rank-2 incidence matrix and $p\in\mathbb{N}^{*}$. The higher-order rank-2 incidence matrix of $F$ of order $p$ is a channelled matrix defined by $F^{p}=\{ (\mathcal{H})^{i}F\}_{i=0}^{p-1}$ where $(\mathcal{H})^{i}$ is $\mathcal{H}=\mathcal{L}(F)$ to the power $i$ with $i$ being an integer.
\end{definition}

This concludes the \textit{Preamble} section.

\section{CCSD - Proposed framework}
\label{CCSD}

We introduce \textbf{CCSD}, a \textit{Combinatorial Complex Score-based Diffusion model through Stochastic Differential Equations}. We present below the general formulation that leverages concepts introduced in GDSS \cite{jo2022scorebased} for graph generation.

\subsection{CCSD Framework}

\subsubsection{Forward process, reverse-time system of SDEs, training objectives, etc}

A combinatorial complex $CC$ of dimension $R\in\mathbb{N}^{*}$ with $n\in\mathbb{N}^{*}$ nodes will be represented using theorem ... as $R+1$ matrices $(X, A, F, B_{2,3}, \ldots, B_{R-1,R})\in\mathcal{M}_{n,f_{0}}(\mathbb{K})\times \mathcal{M}_{n,n,f_{1}}(\mathbb{K})\times \mathcal{M}_{\binom{n}{2},\mathcal{K}_{2},f_{2}}(\mathbb{K}) \times \mathcal{M}_{\mathcal{K}_{2},\mathcal{K}_{3},f_{3}}(\mathbb{K})\times \ldots \times \mathcal{M}_{\mathcal{K}_{R-1},\mathcal{K}_{R},f_{R}}(\mathbb{K})$ that will be written $\left ( \Omega_{r}\right )_{r\in\llbracket 0, R\rrbracket}$. The set $\mathcal{CC}$ will be the set of all the combinatorial complexes of dimension $R$.\par

Let $T\in\mathbb{R}_{+}$ and let $p_{T}$ be a Gaussian prior, which is tractable, and $p_{0}$ be an original distribution of combinatorial complexes. The diffusion process of a combinatorial complex of length $T$ will be denoted as: $CC_{t}=((\Omega_{r,t})_{0\leq r < R})_{t\in [0, T]}$\par

The diffusion process can be modelled by the same Itô stochastic differential equation presented in Subsection \ref{score_sde}, and adapted to our problem:\par

$dCC_{t}=f_{t}(CC_{t})dt + g_{t}(CC_{t})dW$, where $CC_{0}\sim p_{data}$, $f_{t}:\mathcal{CC}\rightarrow\mathcal{CC}$ is the linear drift coefficient, $g_{t}:\mathcal{CC}\rightarrow\mathbb{R}$ is the diffusion coefficient, and $W$ is the standard Wiener process (or standard Brownian motion). The coefficients or functions $f_{t}$\footnotemark{} and $g_{t}$ need to be chosen such that $CC_{T}\sim p_{T}$. We have chosen $g_{t}$ to be a scalar function similarly to \cite{song2021scorebased,jo2022scorebased}.\par

\footnotetext{Similarly to Subsection \ref{score_sde}, the notation $f_{t}(\cdot):=f(\cdot,t)$ is used to write a function of space and time.}

To generate a new combinatorial complex, the process involves sampling a noisy combinatorial complex from the distribution $p_{T}$ and then following the diffusion process backwards in time. This reverse-time diffusion process is referred to as the reverse-time Stochastic Differential Equation (RSDE) and is described in detail by Anderson and Song \cite{anderson, song2021scorebased}.\par

$dCC_{t}=\left [f_{t}(CC_{t})-g_{t}^{2}\nabla_{CC_{t}}\log\left ( p_{t}(CC_{t}) \right ) \right ]d\tilde{t}+g_{t}d\tilde{W}$ where $p_{t}$ denotes the marginal distribution under the forward diffusion process at time $t$, $\tilde{W}$ is a reverse-time standard Wiener process, and $d\tilde{t}$ is an infinitesimal \textit{negative} time step. However, as mentioned in Jo et al. \cite{jo2022scorebased}, solving this reverse-time SDE requires us to compute $\nabla_{CC_{t}}\log \left ( p_{t}(CC_{t}) \right )\in\mathcal{M}_{n,f_{0}}(\mathbb{K})\times \mathcal{M}_{n,n,f_{1}}(\mathbb{K})\times \mathcal{M}_{\binom{n}{2},\mathcal{K}_{2},f_{2}}(\mathbb{K}) \times \mathcal{M}_{\mathcal{K}_{2},\mathcal{K}_{3},f_{3}}(\mathbb{K})\times \ldots \times \mathcal{M}_{\mathcal{K}_{R-1},\mathcal{K}_{R},f_{R}}(\mathbb{K})$ which is computationnaly expensive to the point that the method could be considered untractable.\par

To bypass that, we generalize \cite[Eq 3.]{jo2022scorebased} using the following method.\par

Let $f_{0,t},\ldots,f_{R,t}$ be linear drift coefficients such that $f_{t}(CC)=(f_{0,t}(\Omega_{0}),\ldots, f_{R,t}(\Omega_{R}))$, $g_{0,t},\ldots, g_{R,t}$ be scalar diffusion coefficients, and $\tilde{w}_{0},\ldots, \tilde{w}_{R}$ be reverse-time standard Wiener processes. Then, the reverse-time diffusion process is given by the following system of equations:\par

$\begin{cases}
d\Omega_{0,t}=\left (f_{0,t}(\Omega_{0,t})-g_{0,t}^{2}\nabla_{\Omega_{0,t}}\log \left ( p_{t}(CC_{t})\right ) \right )d\tilde{t} + g_{0,t}d\tilde{W}_{0} \\ 
\ldots \\
d\Omega_{R,t}=\left (f_{R,t}(\Omega_{R,t})-g_{R,t}^{2}\nabla_{\Omega_{R,t}}\log \left ( p_{t}(CC_{t})\right ) \right )d\tilde{t} + g_{R,t}d\tilde{W}_{R}
\end{cases}$\\

These $R+1$ diffusion processes are related to each other through the partial score functions $\left ( \nabla_{\Omega_{r,t}}\log \left (p_{t}(CC_{t}) \right ) \right )_{0\leq r\leq R}$\par

Regarding the training objectives, we extend the objectives defined in \cite[Eq. 7]{jo2022scorebased}. In addition to minimizing the Euclidean distance between the partial score functions and the approximations predicted by neural networks in a tractable manner, we also extend the new objectives that generalize the score matching \cite{JMLR:v6:hyvarinen05a,pmlr-v115-song20a,song2021scorebased} and incorporate the concept of denoising score matching \cite{Vincent2011ACB, song2021scorebased}. For a detailed derivation of the objective function, we invite the reader to follow \cite[Appendix A.1., A.2.]{jo2022scorebased}.\par

$\begin{cases}
\underset{\theta_{0}}{\text{min }}\E_{t}\left [ \lambda_{0}(t)\E_{CC_{0}}\left [\E_{CC_{t}|CC_{0}}\left [ \left \| s_{\theta_{0},t}(CC_{t}) - \nabla_{\Omega_{0,t}}\log \left ( p_{0t}(CC_{t}|CC_{0}) \right )\right \|_{2}^{2} \right ] \right ]\right ]\\
\ldots\\
\underset{\theta_{R}}{\text{min }}\E_{t}\left [ \lambda_{R}(t)\E_{CC_{0}}\left [\E_{CC_{t}|CC_{0}}\left [ \left \| s_{\theta_{R},t}(CC_{t}) - \nabla_{\Omega_{R,t}}\log \left ( p_{0t}(CC_{t}|CC_{0}) \right )\right \|_{2}^{2} \right ] \right ]\right ]
\end{cases}$\par

where for all $r\in\llbracket 0, R\rrbracket$, $\lambda_{r} : [0,T]\rightarrow \mathbb{R}_{+}$ is a positive weighting function and $t$ is uniformly sampled from $[0, T]$.\par

The expectations are taken over $CC_{0} \sim p_{0}$ and $CC_{t} \sim p_{0t}(CC_{t}|CC_{0})$, where $p_{0t}(CC_{t}|CC_{0})$ represents the transition distribution from $p_{0}$ to $p_{t}$ induced by the forward diffusion process. Given our choice of linear drift coefficients, this transition distribution $p_{0t}(CC_{t}|CC_{0})$ can be separated as follows:\par

$p_{0t}(CC_{t}|CC_{0})=\prod_{r=0}^{R}p_{0t}(\Omega_{r,t}|\Omega_{r,0})$.\par

With sufficient data and model capacity, score matching ensures that the optimal solution to the training objectives, denoted by $\left\{ s_{\theta_{r},t}(CC_{t}) \right\}_{0\leq r\leq R}$, is equivalent to $\left\{ \nabla_{\Omega_{r,t}}\log \left( p_{0t}(CC_{t}|CC_{0}) \right) \right\}_{0\leq r\leq R}$ for all $CC\in\mathcal{CC}$ and $t$.\par

This defines an easier sampling procedure, as we can sample each components of $\Omega$ separately. The coefficients of the forward diffusion process are tractable, given that $\left( p_{0t}(\Omega_{r,t}|\Omega_{r,0}) \right)_{0\leq r\leq R}$ follow Gaussian distributions \cite{särkkä2019applied}.\par

While we primarily employ denoising score matching, it's worth noting that other score-matching objectives, such as sliced score matching \cite{pmlr-v115-song20a} and finite-difference score matching \cite{pang2020efficient}, can also be applied in our framework.\par

Typically, for all $r\in\llbracket 1, R\rrbracket$, we choose $\lambda_{r}: \underset{t}{[0,T]}\underset{\longrightarrow}{\mapsto}\underset{\frac{C}{\E\left [ \left \| \nabla_{\Omega_{r,t}}\log \left ( p_{0t}(CC_{t}|CC_{0}) \right ) \right \|_{2}^{2} \right ]}}{\mathbb{R}_{+}}$ where $C$ is a positive constant.\par

Figure \ref{fig:diffusion_ccsd} presents a visual representation of our approach, with the reverse-time process and the partial score functions.\par

\begin{figure}[H]
\centering
\includegraphics[width=1\linewidth]{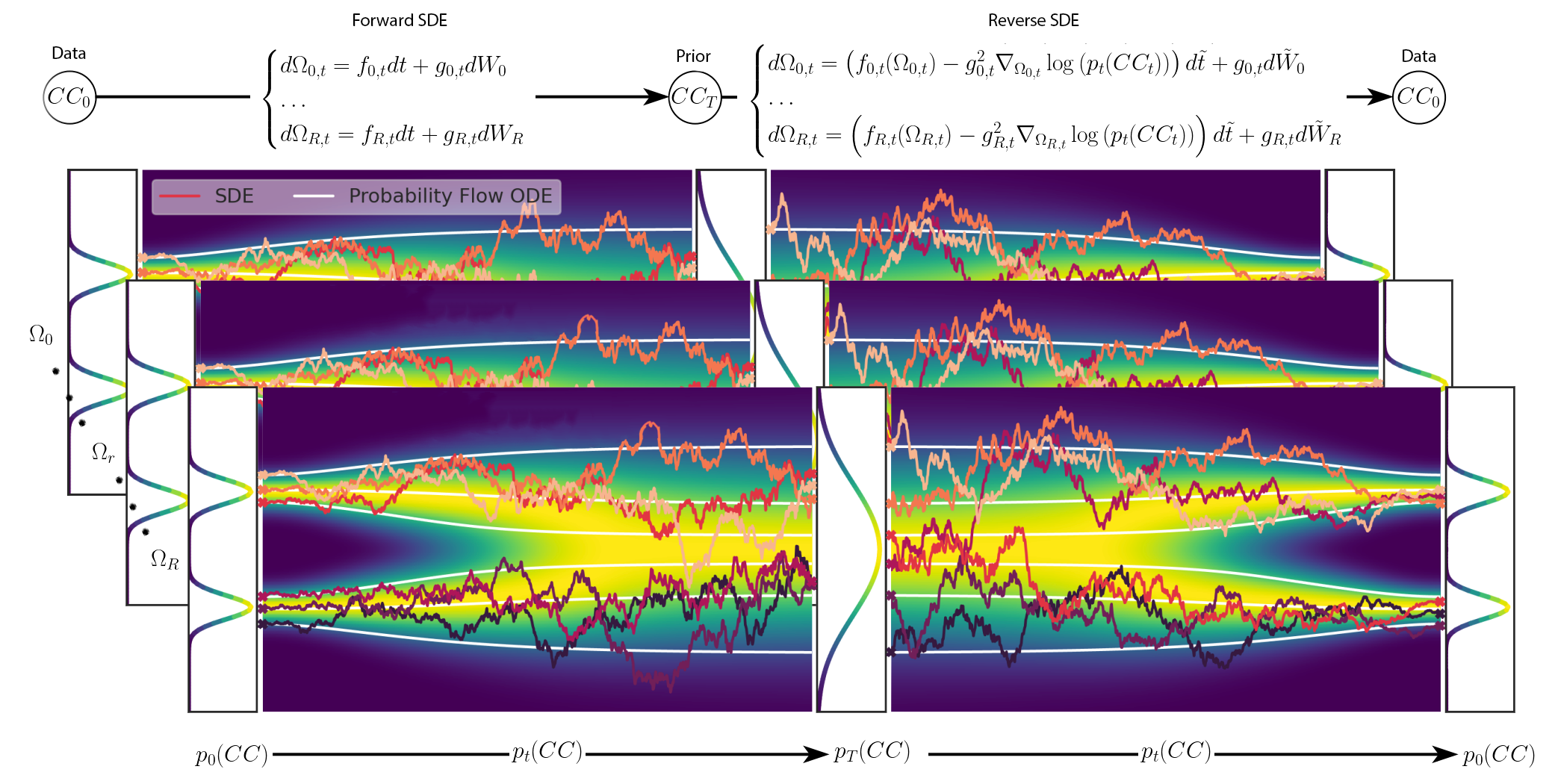}
\captionof{figure}{\textbf{Overview of CCSD.} We can map an original combinatorial complex to a noise distribution (the prior) with an SDE, and reverse this SDE for generative modelling. We can also reverse the associated probability flow ODE, which yields a deterministic process that samples from the same distribution as the SDE. Both the reverse-time SDE and probability flow ODE can be obtained by estimating the partial score functions $\left ( \nabla_{\Omega_{r,t}} \log \left ( p_{t}(CC_{t})\right )\right )_{0\leq r\leq R}$. The image of a diffusion background has been adapted from \cite[Figure 2.]{song2021scorebased}.}
% \captionof{figure}{Overview of CCSD. We can map an original combinatorial complex to a noise distribution (the prior) with an SDE, and reverse this SDE for generative modelling. We can also reverse the associated probability flow ODE, which yields a deterministic process that samples from the same distribution as the SDE. Both the reverse-time SDE and probability flow ODE can be obtained by estimating the partial score functions $\left ( \nabla_{\Omega_{r,t}} \log \left ( p_{t}(CC_{t})\right )\right )_{0\leq r\leq R}$. The image of a diffusion background has been adapted from \cite[Figure 2.]{song2021scorebased}.}
\label{fig:diffusion_ccsd}
\end{figure}

Regarding the loss, in practice, we follow the derivation in C. Luo \cite[Equation 151]{luo_understanding_2022}, so that learning to model the partial score function is equivalent to modelling the negative of the source noise injected (up to a scaling factor that scales with time). Mathematically, we leverage during the training the fact that $\nabla_{CC_{t}}\log\left (p(CC_{t}) \right )=-\frac{1}{\sqrt{1-\overline{\alpha}_{t}}}\epsilon_{0}$.\par

Now, the next step is to determine the models that will enable us to predict the partial score functions of the joint probability across time. For detailed information on these models and the layers used, please refer to the Models and Layers section (Section \ref{Models_and_layers}) below.\par

\subsubsection{Generating samples from the reverse diffusion process}

Generating samples from the reverse diffusion process consists of following the procedure below:

\begin{itemize}
\item First, we sample the number of nodes, denoted as $N$, from the empirical distribution representing the number of nodes in the training dataset. This approach aligns with the methods employed by Li et al. \cite{li2018multiobjective} and Niu et al. \cite{niu2020permutation}. We also retrieve the size of the combinatorial complexes that we want to generate, $R$, as well as the dimension constraints tuples that define $\left (\mathcal{K}_{r} \right )_{2\leq r\leq R}$.
\item Then, we sample the noise with a batch size of $B$ from the prior distribution $p_{T}\left ( \Omega_{0,T},\ldots, \Omega_{R,T}\right )$. Here, $X_{T}=\Omega_{0,T}$ has dimensions $B\times N\times f_{0}$, $A_{T}=\Omega_{1,T}$ has dimensions $B\times N\times N\times f_{1}$, $F_{T}=\Omega_{2,T}$ has dimensions $B\times \binom{N}{2}\times \mathcal{K}_{2}\times f_{2}$, and for all $r\in\llbracket 3, R\rrbracket$, $B_{r-1,r,T}=\Omega_{r,T}$ has dimensions $B\times \mathcal{K}_{r-1}\times \mathcal{K}_{r}\times f_{r}$.
\item Using this sampled noise, we simulate the reverse-time system of stochastic differential equations to obtain the solution $CC_{0}=\left ( \Omega_{0,0},\ldots, \Omega_{R,0}\right )$. More information about the solver to solve the system of SDEs below in Subsubsection \ref{solver_sde}.
\item Finally, we apply quantization operations to $CC_{0}$ based on the specific requirements of the underlying generation tasks and the conditions the tensors must satisfy to form a valid combinatorial complex (mainly, we want to preserve the property of the rank function $rk$).
\end{itemize}

\subsubsection{Solving the reverse-time system of stochastic differential equations}
\label{solver_sde}

To complete this sampling procedure, one needs to simulate the reverse-time system of SDEs. An intractable solution can be derived as follows:\par

If we define two operators,

$F=\left ( f_{r,t}(\Omega_{r,t})d\tilde{t} + g_{r,t}d\tilde{W}_{r} \right )_{0\leq r\leq R}^{T}$ and $S=\left ( -g_{r,t}^{2}s_{\theta_{r},t}(CC_{t})d\tilde{t} \right )_{0\leq r\leq R}^{T}$, then the system of reverse-time SDEs can be rewritten:\par

$dCC_{t}=F-S$.\par

By denoting the marginal joint distribution of the equation above at time $t$ as $\tilde{p}_{t}(CC_{t})$, we obtain a partial differential equation, more precisely a \textit{Fokker-Planck equation}, that rules the evolution of $\tilde{p}_{t}$ through time:\par

$\frac{\partial \tilde{p}_{t}(CC_{t})}{\partial t}=-\nabla_{CC_{t}}\left [ f_{t}(CC_{t})\tilde{p}_{t}(CC_{t})-\frac{1}{2}g_{t}^{2}\tilde{p}_{t}(CC_{t})\nabla_{CC_{t}}\log \left ( \tilde{p}_{t}(CC_{t})\right )-g_{t}^{2}s_{t}(CC_{t})\tilde{p}_{t}(CC_{t}) \right ]$,\par

where $s_{t}(CC_{t})=\left ( s_{\theta_{r},t}(CC_{t}) \right )_{0\leq r\leq R}$ is a vector made of the partial score functions. The Fokker-Planck equation can be rewritten using the Fokker-Planck operators as follows:\par

$\frac{\partial \tilde{p}_{t}(CC_{t})}{\partial t}=\left ( \hat{\mathcal{L}}_{F}^{*}+ \hat{\mathcal{L}}_{S}^{*}\right )\tilde{p}_{t}(CC_{t})$. The actions of the Fokker-Plank operators on a function $\mathcal{J}$ that takes as an input CCs are defined by:

$\bullet$ $\hat{\mathcal{L}}_{F}^{*}\left ( \mathcal{J} \right ):= CC_{t}\mapsto -\nabla_{CC_{t}}\left ( f_{t}(CC_{t})\mathcal{J}(CC_{t})-\frac{1}{2}g_{t}^{2}\mathcal{J}(CC_{t})\nabla_{CC_{t}}\log \left ( \mathcal{J}(CC_{t})\right ) \right )$\par

$\bullet$ $\hat{\mathcal{L}}_{S}^{*}\left ( \mathcal{J} \right ):= CC_{t}\mapsto -\nabla_{CC_{t}}\left (-g_{t}^{2}s_{t}(CC_{t})\mathcal{J}(CC_{t}) \right )$\par

We can then derive from the Fokker-Planck equation an intractable solution to the system of reverse-time SDEs:\par

$\overline{CC}_{t}=e^{t\left (\hat{\mathcal{L}}_{F}^{*} + \hat{\mathcal{L}}_{S}^{*} \right )}\overline{CC}_{0}$,\par

where for all $t\in [0, T]$, $\overline{CC}_{t}=CC_{T-t}$. This intractable solution is called the \textbf{classical propagator} as it propagates the actions of our two operators. Now that we have this form, we can apply approximation solvers of these differential equations to sample from our model. In this thesis, we adapted the PC Sampler and S4 solvers detailed and derived in Jo et al. \cite{jo2022scorebased} and Song et al. \cite{song2021scorebased}. Later on, in our implementation, we will provide details on which solver has been used for each dataset, as reported in Table \ref{tabular:hyperparameters_CCSD}.\par

This concludes the presentation of our framework. Below, we introduce some variations and other applications possible with our framework.

\subsection{Conditional sampling}
\label{conditional_sampling}

In this subsection, we present conditional sampling within the CCSD framework presented earlier. We remind that the forward SDE diffusion process is defined as:\par

$dCC_{t}=f_{t}(CC_{t})dt+g_{t}(CC_{t})dW$.\par

Let's assume that the initial state distribution is conditioned by a variable $y$, denoted as $p_{0}(CC_{0}|y)$. Consequently, the probability density at time $t$ is also conditioned on $y$, expressed as $p_{t}(CC_{t}|y)$. Employing Anderson \cite{anderson} and Song et al. \cite[Appendix I]{song2021scorebased}, we can derive the reverse-time SDE as follows:\par

$\begin{cases}
  d\Omega_{0,t}=\left (f_{0,t}(\Omega_{0,t})-g_{0,t}^{2}\nabla_{\Omega_{0,t}}\log \left ( p_{t}(CC_{t}|y)\right ) \right )d\tilde{t} + g_{0,t}d\tilde{W}_{0} \\
  \ldots \\
  d\Omega_{R,t}=\left (f_{R,t}(\Omega_{R,t})-g_{R,t}^{2}\nabla_{\Omega_{R,t}}\log \left ( p_{t}(CC_{t}|y)\right ) \right )d\tilde{t} + g_{R,t}d\tilde{W}_{R}
\end{cases}$\par

By applying Bayes' rule, we can express $p_{t}(CC_{t}|y)$ as proportional to $p_{t}(CC_{t})p_{t}(y|CC_{t})$. Consequently, for all $r\in\llbracket 0, R\rrbracket$, we have:\par

$\nabla_{\Omega_{r,t}}\log \left ( p_{t}(CC_{t}|y)\right )=\nabla_{\Omega_{r,t}}\log \left ( p_{t}(CC_{t})\right )+\nabla_{\Omega_{r,t}}\log \left ( p_{t}(y|CC_{t})\right )$.\par

To train the second term, the $\nabla_{\Omega_{r,t}}\log \left ( p_{t}(y|CC_{t})\right )$ part of this model, we can employ a time-dependent classifier $p_{t}(y|CC_{t})$ with distinct cross-entropy loss functions for different timesteps.\par

\subsection{Imputation}
\label{imputation}

Imputation, a specialized form of conditional sampling, is a well-known technique in computer vision. However, it can find applications in various domains, such as completing missing connections in existing social networks or filling in missing nodes within mesh structures. We thus derive it for our framework based on the formulation made by Song et al. \cite{song2021scorebased} as follows:\par

We denote $\Delta(CC)$ and $\overline{\Delta}(CC)$ the known and unknown dimensions of the combinatorial complex $CC$, respectively, and also define $f_{\overline{\Delta},t}$ and $g_{\overline{\Delta},t}$ the restrictions of $f_{t}$ and $g_{t}$ to the unknown dimensions. If a function is element-wise, it is applied exclusively to the unknown dimension. If it is a diagonal matrix, we restrict the sub-matrix to the unknown dimensions. Our objective is to sample solely along the unknown dimensions, which can be formulated as:\par

$p\left ( \overline{\Delta}(CC_{0})|\Delta(CC_{0})=y\right )$. Introducing $z(t)=\overline{\Delta}(CC_{t})$, we can derive the following SDE:\par

$dz=f_{\overline{\Delta},t}(z(t))dt+g_{\overline{\Delta},t}(z(t))dW$\par

This allows us to apply the same principles as in the previous \textit{Conditional Sampling} subsection (Subsection \ref{conditional_sampling}), where we conditioned on the unknown dimensions for controllable generation purposes. An approximation and a reparametrization trick, as detailed in Song et al. \cite[Appendix I.1.2]{song2021scorebased}, provide a generally tractable form for the score:\par

$p_{t}\left ( z(t)|\Delta(CC_{0})=y \right )\approx p_{t}\left ( z(t)|\hat{\Delta}(CC_{t}) \right )$ where $\hat{\Delta}(CC_{t})$ is a random sample drawn from the generally tractable distribution $p_{t}\left ( \Delta(CC_{t})|\Delta(CC_{0})=y \right )$.

\subsection{Penalization of higher-order cell generation}
\label{penalization}

In most applications, higher-dimensional cells within a combinatorial complex are relatively rare. Conversely, the number of possible rank-$r$ cells for $r\geq 2$ grows exponentially with the number of nodes, resulting in sparse incidence matrices and correspondingly sparse partial score functions. To address this sparsity issue, we propose an alternative objective function that includes an $L_{2}$ penalization term. This term helps restrict the number of non-zero entries in the partial score functions associated with higher-order incidence matrices. The objective function is structured as follows:\par

$\begin{cases}
\underset{\theta_{0}}{\text{min }}\E_{t}\left [ \lambda_{0}(t)\E_{CC_{0}}\left [\E_{CC_{t}|CC_{0}}\left [ \left \| s_{\theta_{0},t}(CC_{t}) - \nabla_{\Omega_{0,t}}\log \left ( p_{0t}(CC_{t}|CC_{0}) \right )\right \|_{2}^{2} \right ] \right ]\right ]+\gamma_{0}\left \| s_{\theta_{0},t}(CC_{t}) \right \|_{2}^{2} \\
\ldots\\
\underset{\theta_{R}}{\text{min }}\E_{t}\left [ \lambda_{R}(t)\E_{CC_{0}}\left [\E_{CC_{t}|CC_{0}}\left [ \left \| s_{\theta_{R},t}(CC_{t}) - \nabla_{\Omega_{R,t}}\log \left ( p_{0t}(CC_{t}|CC_{0}) \right )\right \|_{2}^{2} \right ] \right ]\right ]+\gamma_{R}\left \| s_{\theta_{R},t}(CC_{t}) \right \|_{2}^{2}
\end{cases}$\par

In this context, $\gamma_{0},\ldots,\gamma_{R}\in\mathbb{R}_{+}$ represent penalization or regularization hyperparameters that should be tailored to the specific data distribution being learned. These hyperparameters should be inversely proportional to the sparsity of higher-order cells.

\subsection{Probability flow and neural ODEs}

Similar to Song et al. \cite{song2021scorebased}, our score-based model provides an alternative numerical approach for solving the reverse-time Stochastic Differential Equation (SDE). For all diffusion processes, there exists a corresponding deterministic process whose trajectories share the same marginal probability densities $\left ( p_{t}(CC_{t})\right )_{0\leq t\leq T}$ as those of the SDE. This deterministic process is governed by the following Ordinary Differential Equation (ODE):\par

$dCC_{t}=\left (d\Omega_{0,t},\ldots,d\Omega_{R,t} \right )^{T}$\par

where $\forall r\in\llbracket 0, R\rrbracket, \forall t\in [0, T], d\Omega_{r,t}=\left (f_{r,t}-\frac{1}{2}g_{r,t}^{2}\nabla_{\Omega_{r,t}}\log\left (p_{t}(CC_{t}) \right ) \right )dt$\par

This ODE can be determined from the SDE once the scores are known. It is commonly referred to as the \textit{probability flow ODE}. score-based model, which is typically implemented as a neural network, this corresponds to a neural ODE \cite{chen2019neural}.\par

The connection to the probability flow ODE offers several advantages, including the ability to compute the exact likelihood for any input data, create latent representations of data points from the initial distribution $p_{0}$ to the prior distribution $p_{T}$, implement an efficient sampling procedure, and establish a unique identifiable encoding procedure provided there is sufficient training data \cite{roeder2020linear}.

\section{Models and layers}
\label{Models_and_layers}

This section defines the layers and models implemented to predict our partial score functions. We adopted some notations from Jo et al. \cite{jo2022scorebased} and we provided a more detailed explanation regarding the components of each layer and model. Due to the permutation-invariant nature of graphs, we would like ideally to build a permutation-equivariant score-based model. This has not yet been demonstrated -or proved wrong- for the models ScoreNetworkA\_CC, ScoreNetworkA\_Base\_CC and ScoreNetworkF.

\begin{layer}[MLP]
MLP is the Multi-Layer Perceptron \cite{Hinton1989ConnectionistLP,Sanger1958ThePA}. We added optional Batch Normalization \cite{ioffe2015batch} layers between the layers.
\end{layer}

\begin{layer}[GNN]
GNN stands for Graph Neural Network. Here, we follow the Graph Convolutional Network layer architecture presented by Thomas N. Kipf and Max Welling \cite{kipf2017semisupervised}. For a tuple $(X, A)$ of a node feature matrix $X$ and an adjacency matrix $A$, both representing a graph (resp. a dimension $\geq 1$ CC) with $n$ nodes (resp. rank-0 cells), we have:\par

$\text{GNN}(X,A)=X'=\hat{D}^{-\frac{1}{2}}\hat{A}\hat{D}^{-\frac{1}{2}}X\Theta$ where $\hat{A}=A+I$ is the adjacency matrix with inserted self-loops, $\hat{D}$ is a diagonal degree matrix defined by $\forall i\in\llbracket 1, n\rrbracket, \hat{D}_{i,i}=\sum_{j=0}\hat{A}_{i,j}$, and $\Theta$ are learnable parameters.
\end{layer}

\begin{model}[ScoreNetworkX]

$\text{ScoreNetworkX}(CC_{t})=s_{\theta_{0},t}(CC_{t})\approx \nabla_{X_{t}}\log (p(CC_{t}))$ has been introduced in \cite{jo2022scorebased} and is defined by:\par

$\text{ScoreNetworkX}(CC_{t})=\text{MLP}\left ( \left [ \{ H_{i} \}_{i=0}^{L} \right ] \right )$ where $H_{0}=X_{t}$ and $H_{i+1}=\text{GNN}\left (H_{i}, A_{t} \right )$.\par

The brackets $\{ .\}$ represent the concatenation operation along a \textit{channel} axis.
\end{model}

\begin{layer}[GMH]
GMH is a Graph Multi-Head Attention layer adapted from Baek et al. \cite{baek2021accurate} and used in Jo et al. \cite{jo2022scorebased}. For matrices $(X,A)$, the attention blocks \textit{GMH} are defined by:\par

Value = $\text{GNN}(X, A)$.\par

Attention = $\frac{1}{\sqrt{\text{dim}_{\text{out}}}}\text{Query}\times \text{Key}^{T}$ where Query and Key are the output of two GNNs applied on $X$ and $A$ split along a channel dimension. The Attention matrix is then symmetrized.\par
\end{layer}

\begin{layer}[Attention Layer (Att)]
The attention layer comes from Jo et al. \cite{jo2022scorebased}. It consists of combining $L\in\mathbb{N}^{*}$ GMH layers: for all $i\in\llbracket 1, L\rrbracket$, $(\text{Value}_{i}, \text{Attention}_{i})=\text{GMH}(X,A^{i})$ where $\{ (A)^{i}\}_{i=1}^{p}=A^{p}$ are the higher-order adjacency matrices of the adjacency matrix $A$.\par

Then, the attention layer is defined by:\par

$\text{Att}(X,A^{p})=(X',A')$ where $X'=\tanh \left (\text{MLP}\left (\left \{\text{Value}_{i} \right \}_{i=1}^{L} \right )\right )$ and $A'=\left (\text{MLP}\left (\left \{\text{Attention}_{i} \right \}_{i=1}^{L} \right )\right )$
\end{layer}

\begin{layer}[HCN]
We adapted the aforementioned Graph Neural Network to take as an input a tuple $(H,F)$ where $H$ is the Hodge dual of an adjacency matrix $A$ and $F$ is a rank-2 incidence matrix. We thus present Hodge Convolutional Networks (HCN). For a tuple $(H, F)$ we have:

$\text{HCN}(H,F)=F'=\hat{D}^{-\frac{1}{2}}H\hat{D}^{-\frac{1}{2}}F\Theta$ where $\hat{D}$ is a diagonal degree matrix defined by $\forall i\in\llbracket 1, \binom{n}{2}\rrbracket, \hat{D}_{i,i}=\sum_{j=0}\hat{H}_{i,j}$, and $\Theta$ are learnable parameters.
\end{layer}

\begin{layer}[HCCMH]
HCCMH stands for Hodge Combinatorial Complexes Multi-Head Attention layer and is a layer designed for higher-order objects. For matrices $(H,F)$, the Hodge attention blocks \textit{HCCMH} are defined by:\par

Value = $H\times F$.\par

Hodge attention = $\frac{1}{\sqrt{\text{dim}_{\text{out}}}}\text{Query}\times \text{Key}^{T}$ where Query and Key are the output of two GNNs applied on $H$ and $F$ split along a channel dimension. The Hodge attention matrix is then symmetrized.\par
\end{layer}

\begin{layer}[Hodge Attention Layer (HodgeAtt)]
The HodgeAtt layer that we developed consists of $L\in\mathbb{N}^{*}$ HCCMH layers: for all $i\in\llbracket 1, L\rrbracket$, $(\text{Value}_{i}, \text{Hodge attention}_{i})=\text{HCCMH}(H^{i},F)$ where $\{ (H)^{i}\}_{i=1}^{p}=H^{p}=\mathfrak{H}(A^{p})$ is the Hodge dual of the higher-order adjacency matrices of the adjacency matrix $A$.\par

Then, the Hodge attention layer is defined by:\par

$\text{HodgeAtt}(H^{p},F)=(H',F')$ where $H'=\tanh \left (\text{MLP}\left (\left \{\text{Hodge attention}_{i} \right \}_{i=1}^{L} \right )\right )$ and \par

$F'=\left (\text{MLP}\left (\left \{\text{Value}_{i} \right \}_{i=1}^{L} \right )\right )$.
\end{layer}

These layers allow us to define the new model ScoreNetworkA\_CC that can compute the partial score function with respect to the adjacency matrix $A$ by taking into account the higher-order structure of the CCs.\par

\begin{model}[ScoreNetworkA\_CC]

$s_{\theta_{0},t}(CC_{t})\approx \nabla_{X_{t}}\log (p(CC_{t}))=\text{ScoreNetworkA\_CC}(CC_{t})$ such that:\par

$\text{ScoreNetworkA\_CC}(CC_{t})=\text{MLP}\left ( \left [ \left \{ G_{t,i}[1] \right \}_{i=0}^{L_{att}}, \left \{ \mathfrak{H}^{-1}\left ( W_{t,i}[1]\right ) \right \}_{i=0}^{L_{Hodge att}} \right ] \right )$.\par

$G_{t,0}=(X_{t}, A_{t}^{p})$ and for all $i\in\llbracket 0, L_{att}-1\rrbracket$, $G_{t,i+1}=\text{Att}(G_{t,i})$. $G_{t,i}[1]$ means that we only access the second element which is the modified higher-order adjacency matrix. $W_{t,0}=(H_{t}^{p}, F_{t})$ where $H_{t}^{p}=\mathfrak{H}\left (A_{t}^{p}\right )$ and for all $i\in\llbracket 0, L_{Hodge att}-1\rrbracket$, $W_{t,i+1}=\text{HodgeAtt}(W_{t,i})$.
\end{model}

\begin{remark}
By setting $L_{Hodge att}=0$ to remove higher-order dependencies, we find the ScoreNetworkA architecture that is found in the model GDSS \cite{jo2022scorebased}.
\end{remark}

\begin{model}[ScoreNetworkA\_Base\_CC (Baseline)]

\par
$\text{ScoreNetworkA\_Base\_CC}(CC_{t})=s_{\theta_{1},t}(CC_{t})\approx \nabla_{A_{t}}\log (p(CC_{t}))=$ is derived from ScoreNetworkA\_CC and is used for our ablation study and as a potential alternative to our attention-based model ScoreNetworkA\_CC. We replace the HodgeAtt layers with \textit{HodgeBaselineLayer} layers made of \textit{BaselineBlock} that consist of replacing the attention mechanisms through HCNs by MLP layers.
\end{model}

\begin{layer}[HodgeNetwork]
A HodgeNetwork is a simple neural network architecture that consists of a MLP applied on the channels axis of a higher-order rank-2 incidence matrix. Mathematically,\par

$\text{HodgeHodgeNetwork}(F_{t}^{p})=\text{MLP}(\left [ \left \{ F_{t}^{i} \right \}_{i=1}^{p} \right ])$.
\end{layer}

\begin{model}[ScoreNetworkF]

\par
$\text{ScoreNetworkF}(CC_{t})=s_{\theta_{2},t}(CC_{t})\approx \nabla_{F_{t}}\log (p_{t}(CC_{t}))=\text{MLP}\left ( \left [ \left \{ K_{i} \right \}_{i=0}^{L-1} \right ] \right )$ where $K_{i+1}=\text{HodgeNetwork}(K_{i})$ and $K_{0}=\left \{ F_{t}^{i} \right \}_{i\geq 1}$ where $\left \{ F_{t}^{i} \right \}_{i\geq 1}$ are the higher-order rank-2 incidences matrices defined by $F_{t}^{1}=F_{t}$, $\mathcal{H}_{t}=\mathcal{L}(F_{t})$, and $\forall i\geq 1$, $F_{t}^{i+1}=\mathcal{H}_{t} F_{t}^{i}$.
\end{model}

\begin{remark}
Since GNNs' message-passing operations and GMH's attention function are permutation equivariant \cite{keriven2019universal}, the proposed score-based models ScoreNetworkX and ScoreNetworkA inherently exhibit equivariance. Therefore, based on the findings of Niu et al. \cite{niu2020permutation}, the log-likelihood implicitly defined by these models is also guaranteed to be permutation-invariant. However, as highlighted in the section Future Work (Section \ref{future}), more work needs to be done to assess the permutation equivariance of the other score network models that we proposed, or to create new models that encompass this property.
\end{remark}

\section{Evaluation Metrics}
\label{Metrics}

In this section, we present pioneering metrics specially designed to evaluate the quality of generated combinatorial complexes. These metrics offer a comprehensive assessment of various facets of the generated complexes, shedding light on their fidelity to the original distribution. To our knowledge, this work represents the first exploration of generative AI for objects beyond graphs, making these metrics groundbreaking in the realm of generative AI for higher-order topological structures.\par

\begin{definition}[Hodge Laplacians Distance]
Let $\left (H_{r}\right )_{1\leq r\leq R}$ (resp, $\left (\hat{H}_{r}\right )_{1\leq r\leq R}$) be the Hodge Laplacians for each dimension $r$ of a combinatorial complex $CC$ from an original distribution (resp. a generated combinatorial complex $\hat{CC}$). Then, the Hodge Laplacians distance between the two CCs is defined as:\par

$D_{\text{Hodge}}(CC,\hat{CC})=\frac{1}{R}\sum_{r=1}^{R}\underset{P_{\pi,r}\in \mathcal{P}_{r}}{\text{min }}d\left ( P_{\pi,r}^{T}H_{r}P_{\pi,r}, \hat{H}_{r} \right )$, where for all $r\in\llbracket 1,R\rrbracket$, $\mathcal{P}_{r}$ is the set of all the permutation matrices that swaps rank-$r$ cells when we swap nodes.
\end{definition}

However, the Hodge Laplacians distance is in practice not tractable due to the vast number of permutations. Thus, we propose another metric, the Hodge Laplacian spectrum, inspired by \cite{Thongprayoon_2023}, that could be considered as a proxy for the metric previously defined. It consists of calculating sets of eigenvalues that could then be compared together. The eigenvalues capturing the structural information, we could envision that the resulting vector encompasses the structure of the combinatorial complex too.\par

\begin{definition}[Hodge Laplacian Spectrum]
Let $CC=\left (\Omega_{0},\ldots\Omega_{R} \right )$ be a combinatorial complex of dimension $R\geq 2$. The Hodge Laplacian spectrum of the combinatorial complex $CC$, denoted $\text{Spec}_{\text{Hodge}}(CC)$, is defined by:\par

$\text{Spec}_{\text{Hodge}}(CC)=\left ( \text{Spec}_{\text{Hodge},r}(CC)\right )_{2\leq r\leq R}$, where for all $r\in\llbracket 2, R\rrbracket$, $\text{Spec}_{\text{Hodge},r}(CC)=\text{Spec}\left ( \mathcal{L}\left (\Omega_{r}\right ) \right )$ where $\mathcal{L}$ is the Hodge Laplacian operator defined in Subsection \ref{Preamble} and $\text{Spec}$ is the function that returns the vector of the eigenvalues of an endomorphism or a squared matrix, with their order of multiplicity.
\end{definition}

To compare two combinatorial complexes using the Hodge Laplacian spectrum, we can employ a distance measure such as Maximum Mean Discrepancy (MMD) (see Section \ref{Graphs_molecule_generation}) to compare the distributions of eigenvalues. This comparison can accommodate complexes of different dimensions by zero-padding the higher dimensions of the lower-dimensional complexes.\par

Finally, we introduce the rank-$r$ metric. It assesses the distribution of features or sizes of rank-$r$ cells within a combinatorial complex. This metric allows us to compare the sizes and properties of generated complexes with the original ones. The definition of the metric depends on whether feature information is attached to rank-$r$ cells as follows:\par

\begin{itemize}
\item If $f_{r}=1$, the metric computes the distribution of sizes of the rank-$r$ cells (how many nodes per cell).
\item If $f_{r}\neq 1$, the metric computes the distribution of features attached to rank-$r$ cells.
\end{itemize}

Mathematically, we can write the following definition:\\

\begin{definition}[Rank-$r$ metric]
Let $CC=\left (\Omega_{0},\ldots\Omega_{R} \right )$ be a combinatorial complex of dimension $R\geq 2$. The Rank-$r$ metric of the combinatorial complex $CC$, denoted $\beta_{r}(CC)$, is defined by:\par

$\bullet$ If $f_{r}=1$, $\beta_{r}(CC)=\left ( \beta_{r,i} \right )_{d_{min}^{r}\leq i \leq d_{max}^{r}}$, where $\forall i\in\llbracket d_{min}^{r}, d_{max}^{r}\rrbracket$, $\beta_{r,i}=\sum_{j=1}^{|\mathcal{X}_{r}|}\mathbb{1}_{\{ |x_{j}^{r}|=i\}}$.\par

The parameter $\beta_{r,i}$ represents the number of rank-$r$ cells of size $i$. $|\cdot|$ is the operator that returns the cardinal (number of nodes) of a cell.\par

$\bullet$ If $f_{r}\neq 1$, $\beta_{r}(CC)=\left ( \beta_{r,k} \right )_{1\leq k \leq f_{r}}$, where $\forall k\in\llbracket 1, f_{r}\rrbracket$, $\beta_{r,k}=\sum_{j=1}^{|\mathcal{X}_{r}|}\mathbb{1}_{\{ \Omega_{r}[:,j,k]\neq 0\}}$.\par

The parameter $\beta_{r,k}$ represents the number of rank-$r$ cells with a feature equal to $k$.\par
\end{definition}

In both cases, the metric provides valuable insights into the structural characteristics of the combinatorial complexes, enabling comparisons between generated and original distributions.\par

The theoretical aspect of this thesis being presented, we can now delve into our implementation.\par

%%%%%%%%%%%%%%%%%%%%%%%%%%%%%%%%%%%%

\chapter{Implementation}
\label{Implementation}

All the datasets employed in our experiments (Section \ref{Experiments}) can be easily and naturally elevated to combinatorial complexes of dimension 2 and can be efficiently represented numerically using the representation (Theorem \ref{theorem:representation} in Subsection \ref{Preamble}). Henceforth, throughout the remainder of this thesis, we will exclusively focus on combinatorial complexes of dimension 2.\par

Our models, layers, preprocessing scripts, metrics, plots, and diffusion framework have been meticulously implemented using the Python programming language \cite{van1995python}. The orbit metric \cite{jo2022scorebased} was developed in C++ \cite{ISO:1998:IIP} for optimal performance. In the subsequent section, Section \ref{Experiments}, we offer an in-depth account of our experiments. Following that, in Section \ref{library}, we present the software components developed for this thesis.\par

\section{Experiments}
\label{Experiments}

For our experiments, we carefully selected the datasets to have varying sizes and characteristics, for example, synthetic graphs, real-world graphs, social graphs or biochemical graphs, and to be able to compare the results with different existing models and approaches for the sub-problem that is graph generation. We also made sure that these datasets have been evaluated on other methods to benchmark our framework.\par

We compared our proposed method against several general graph generative models, each employing different architectural approaches. DeepGMG \cite{li2018learning} and GraphRNN \cite{you2018graphrnn} employ RNN-based architectures while GraphAF \cite{shi2020graphaf}, GraphDF \cite{luo2021graphdf} utilize flow-based architectures, and GRAPHARM \cite{kong2023autoregressive} adopts diffusion based architecture. These models are all autoregressive, generating graphs step by step. On the other hand, GraphVAE \cite{simonovsky2018graphvae}, GraphEBM \cite{liu2021graphebm}, GDSS \cite{jo2022scorebased}, EDP-GNN \cite{niu2020permutation} and SGGM+SLD \cite{yang2023scorebased} utilize VAE, and EBM, and score-based models respectively. GNF \cite{NEURIPS2019_1e44fdf9} and MoFlow \cite{Zang_2020} employ a flow-based model. The above models, as well as our framework CCSD, are all one-shot, generating the entire graph in one step.\par

We adopted three types of SDEs, VESDE, VPSDE, and sub-VP SDE, as introduced by Song et al. \cite{song2021scorebased}, for the diffusion processes of each component. Additionally, we employed either the PC sampler or the S4 solver to solve the system of SDEs. Further implementation details for each dataset and the associated SDEs can be found in Table \ref{tabular:hyperparameters_CCSD} and Table \ref{tabular:hyperparameters_CCSD_Base}. Similarly to GDSS, we created a script to compute the Frobenius norm of the Jacobian of our models to assess the complexity of learning partial scores, especially with respect to the higher-order matrices.\par

As the training can take several days for a single dataset, we only trained our models one time and sampled from them once. We acknowledge that this is a limitation of our work but it could be easily fixed with more time. For the GDSS baseline, we used the hyperparameters given by the original work.\par

\subsection{Molecule Generation}

For our molecule generation task, we utilized the QM9 dataset \cite{ramakrishnan_quantum_2014}, which comprises 133,885 molecules, each with a varying atom count ranging from 1 to 9. These molecules consist of Carbon (C), Fluorine (F), Oxygen (O), and Nitrogen (N) atoms, with implicit Hydrogen atoms. The bonds in these molecules can be single, double, or triple. 10\% of the dataset is set apart for testing purposes.\par

To prepare the molecules for processing, we converted each one into a graph representation. These graphs have node features denoted as $X\in \{0, 1\}^{N\times f_{0}}$ and an adjacency matrix represented as $A\in {0, 1, 2, 3}^{N×N}$. Here, $N$ signifies the maximum number of atoms found in any molecule in the QM9 dataset (which is 9 for QM9), and $f_{0}$ represents the number of possible bond types/entries for the adjacency matrix (4 for QM9). The entries in the adjacency matrix $A$ signify the types of bonds between atoms, such as single, double, or triple bonds. We lifted the graphs into CCs by transforming the ring into rank-2 cells (see Figure \ref{fig:molecule_lift}).\par

We followed a standard preprocessing procedure \cite{shi2020graphaf,luo2021graphdf} for these molecules. This involved kekulization of the molecules using the RDKit Python library \cite{Landrum2016RDKit2016_09_4} and removing hydrogen atoms. We also applied a valency correction method proposed by Zang and Wang \cite{Zang_2020}.\par

For our approach, we utilized the signal-to-noise ratio (SNR) and scale coefficient obtained through a grid search in GDSS \cite{jo2022scorebased} as the basis for our molecule generation. The optimization process in GDSS aimed to find the best FCD (Fréchet ChemNet Distance) value among those that achieved a novelty score exceeding 85\%. This procedure has been chosen by the authors as a low novelty value could potentially lead to low FCD and NSPDK MMD values.\par

After generating samples using the reverse diffusion process, we quantized the adjacency matrices to values in the set $\{0, 1, 2, 3\}$ by mapping them as follows:\par

\begin{itemize}
\item Values in the range $]-\infty, 0.5]$ were mapped to $0$.
\item Values in the range $[0.5, 1.5[$ were mapped to $1$.
\item Values in the range $[1.5, 2.5[$ were mapped to $2$.
\item Values in the range $[2.5, +\infty[$ were mapped to $3$.
\end{itemize}

For evaluation, we used the MMD distance to compare the distributions of the NSPDK statistic between the same number of generated and test molecules. We compared also the molecular metrics introduced in Section \ref{Graphs_molecule_generation}, such as validity, novelty, FCD, etc.

\subsection{Graph Generation}

We evaluated the performance of CCSD by examining the quality of generated samples on a diverse set of graph datasets, encompassing both synthetic and real-world graphs of varying sizes. Our evaluation datasets include:

\begin{itemize}
\item \textbf{Ego-small:} A collection of 200 small ego graphs, extracted from the larger Citeseer network dataset \cite{sen_collective_2008}.
\item \textbf{Community-small:} A set of 100 randomly generated community graphs, following the methodology introduced by Niu et al. \cite{niu2020permutation}.
\item \textbf{Enzymes small:} Comprising 35 graphs extracted from a pool of 587 protein graphs (original Enzymes dataset), where we selected the graphs with fewer than 12 nodes. These protein graphs represent the tertiary structures of enzymes from the BRENDA database \cite{schomburg}
\item \textbf{Grid small:} A collection of 100 standard 2D grid graphs with varying dimensions, ranging from 4x4 to 7x7 rows and columns (equivalent to 10x10 to 19x19 rows and columns in the original grid dataset).
\end{itemize}

The grid dataset generation and graph manipulation were carried out using the NetworkX Python library \cite{hagberg2008exploring}. We lifted the graphs into CCs by either applying a loop-based lift procedure like for the molecules or a path-based lift procedure (see Figure \ref{fig:path_lift}). More information about which lifting procedure is applied to which dataset can be found in the parameter table (Table \ref{tabular:hyperparameters_CCSD}).\par

To ensure a fair comparison, we followed the experimental and evaluation settings outlined by You et al. \cite{you2018graphrnn}, including the same train/test split. We employed the Maximum Mean Discrepancy (MMD) distance to compare the distributions of various graph statistics between the generated samples and the test graphs. The statistics we analyzed include degree distributions, clustering coefficients, and the occurrences of 4-node orbits (to capture higher-level motifs) \cite{Hocevar2014ACA}. We provided detailed definitions of these metrics in Section \ref{Graphs_molecule_generation}.\par

In line with previous work \cite{jo2022scorebased}, we used the Gaussian Earth Mover's Distance (EMD) kernel for computing MMDs, instead of the total variation (TV) distance employed in some other papers, like Liao et al. \cite{liao2020efficient}. This choice is made to avoid an indefinite kernel and undefined behaviour \cite{obray2022evaluation}.\par

For a fair evaluation of the generic graph generation task, we adhered to the standard settings established by existing works \cite{you2018graphrnn,NEURIPS2019_1e44fdf9,niu2020permutation}, which cover everything from node features to data splitting. Specifically, for CCSD, we initialized the node features using one-hot encoding based on the degrees of the nodes.\par

Our approach to graph generation also leverages the signal-to-noise ratio (SNR) and scale coefficient determined through an extensive grid search conducted in GDSS \cite{jo2022scorebased}. The optimization process that they used aims to find the optimal MMD value while minimizing the average of three key graph statistics: degree, clustering coefficient, and orbit. Additionally, we also incorporated an exponential moving average (EMA) \cite{NEURIPS2020_92c3b916} for larger graph datasets, such as Enzymes small and Grid small, to reduce the variance and enhance performance.\par

After generating the samples by simulating the reverse diffusion process, we quantized the entries of the adjacency matrices using the $\mathbb{1}_{\{x>0.5\}}$ operator to obtain an adjacency matrix with values in $\{0, 1\}$.\par

\subsection{Tanimoto similarity}
\label{Tanimoto}

The Tanimoto similarity metric plays a crucial role in molecule generation tasks, offering a quantitative means to gauge the structural likeness between newly generated molecules and those present in the training dataset \cite{tanimoto}. This metric relies on the Morgan fingerprints \cite{morgan_generation_1965}, a widely accepted molecular representation method. By leveraging these fingerprints' similarity, the Tanimoto similarity provides an effective way to assess how well the generated molecules align with the structural attributes of the training molecules. A higher Tanimoto similarity score signifies a closer structural resemblance, showcasing the model's ability to produce molecules that closely mirror the characteristics of the original dataset. The Tanimoto Similarity ($\mathcal{T}$) can be mathematically expressed as:\par

$\mathcal{T}(A, B) = \frac{|A \cap B|}{|A \cup B|}$,\par

where $A$ and $B$ are sets of Morgan fingerprints. The resulting similarity score ranges from 0 (indicating complete dissimilarity) to 1 (representing perfect similarity, where the sets of features are identical).\par

\subsection{Ablation Study \& Models}

We conducted experiments employing two different approaches. Firstly, for the partial score function with respect to the adjacency matrix, denoted as $s_{\theta_{1}}(CC_{t})\approx \nabla_{\Omega_{1,t}}\log\left ( p_{t}\left ( CC_{t} \right ) \right )$, we utilized the ScoreNetworkA\_CC model, which incorporates an attention system tailored for our specific generation task. This approach is called \textit{CCSD}. Conversely, we trained an alternative model called \textit{CCSD Base}, which employs the ScoreNetworkA\_Base\_CC model, consisting solely of MLP layers. These model descriptions were detailed in Section \ref{Models_and_layers}.\par

Our ablation study aims to compare the attention-based model with the \textit{vanilla} variant that relies solely on MLP layers. Furthermore, we evaluated the results of these two models alongside GDSS \cite{jo2022scorebased}, the graph generation framework based on stochastic differential equations that served as the foundation for our framework. To ensure a fair comparison, we fine-tuned the parameters of both \textit{CCSD} and \textit{CCSD Base} models to make them relatively similar to those used in GDSS. We also ensured that the total number of trainable parameters was roughly comparable, as illustrated in Table \ref{tabular:nb_parameters}. Additionally, we limited the depth of the MLP layers to minimize disparities in terms of inference speed.\par

\subsection{Parameters}
\label{parameters}

In Table \ref{tabular:nb_parameters}, we provided an overview of the number of parameters in each of the score networks for every dataset. The trainable parameters represent the cumulative count across all three score networks. Additionally, we reported the maximum percentage difference between the models with the fewest and most trainable parameters to assess that they do not differ by an order of magnitude.\par

\begin{table}
\centering
\begin{tabular}{ lcccc }
    \hline
     & GDSS & CCSD & CCSD Base & MAP (\%)$\downarrow$ \\
     \cline{2-5}
    Ego-small & 179806 & 217633 & 144470 & 33.62 \\
    Community-small & 206803 & 227023 & 222838 & 8.91 \\
    Enzymes-small & 386179 & 390410 & 375160 & 3.91 \\
    Grid-small & 373234 & 436536 & 280839 & 35.67 \\
    QM9 & 32389 & 42764 & 34626 & 24.26 \\
    \hline
\end{tabular}
\caption{\textbf{Number of parameters for each dataset.} We compared the number of parameters between \textit{GDSS}, \textit{CCSD}, \textit{CCSD Base}, and calculated the difference in percentage between the models with the fewest and most trainable parameters.}
\label{tabular:nb_parameters}
\end{table}

In Table \ref{tabular:hyperparameters_CCSD} and Table \ref{tabular:hyperparameters_CCSD_Base}, we present the hyperparameters of the two approaches for all the datasets.\par

\begin{table}
\scriptsize
\centering
\begin{tabular}{ llccccc } 
    \hline
     & Hyperparameters & Ego-small & Community-small & Enzymes-small & Grid-small & QM9 \\
    \hline
    \multirow{2}{3em}{General} & Batch Normalization & False & False & False & False & False \\
    & Block layer & GCN & GCN & GCN & GCN & GCN \\
    & Node init method & deg & deg & deg & deg & atom \\
    & Lift procedure & cycles & path\_based & cycles & path\_based & cycles \\
    & Lift kwargs & & 3 & - & 3 & - \\
    & $d_{min}$ & 3 & 3 & 3 & 3 & 3 \\
    & $d_{max}$ & 5 & 3 & 4 & 3 & 9 \\
    & Max node number & 18 & 20 & 12 & 49 & 9 \\
    & Min node value & 1 & 1 & 1 & 1 & 6 \\
    & Max node value & 1 & 1 & 1 & 1 & 9 \\
    & Max feature number & 17 & 10 & 10 & 5 & 4 \\
    & Min edge value & 1 & 1 & 1 & 1 & 1 \\
    & Max edge value & 1 & 1 & 1 & 1 & 3 \\
    \hline
    \multirow{2}{3em}{$s_{\theta_{0}}$} & Number of GCN layers & 2 & 3 & 5 & 4 & 2 \\
    & Hidden dimension & 12 & 32 & 32 & 24 & 10 \\
    \hline
    \multirow{6}{3em}{$s_{\theta_{1}}$} & Number of attention heads & 4 & 4 & 4 & 4 & 4 \\
    & Number of initial channels & 2 & 2 & 2 & 2 & 2 \\
    & Number of hidden channels & 5 & 8 & 8 & 6 & 8 \\
    & Number of final channels & 4 & 4 & 4 & 4 & 4 \\
    & Number of Block layers & 2 & 5 & 6 & 6 & 3 \\
    & Number layers in MLP & 2 & 2 & 2 & 2 & 3 \\
    & Hidden dimension & 12 & 32 & 32 & 24 & 10 \\
    & Hodge Block layer & HCN & HCN & HCN & HCN & HCN \\
    & Number of Hodge Block layers & 1 & 1 & 2 & 1 & 2 \\
    & Number layers in Hodge MLP & 1 & 1 & 2 & 1 & 1 \\
    & Hidden dimension in Hodge MLP & 4 & 4 & 4 & 4 & 4 \\
    & Number of hidden Hodge channels & 2 & 2 & 4 & 2 & 4 \\
    & Number of final Hodge channels & 2 & 2 & 2 & 2 & 2 \\
    & Number of Hodge attention head & 2 & 2 & 2 & 2 & 2 \\
    & Hodge attention dim & 4 & 4 & 4 & 4 & 4 \\
    \hline
    \multirow{3}{3em}{$s_{\theta_{2}}$} & Power higher-order & 2 & 2 & 2 & 2 & 2 \\
    & Number layers in MLP & 1 & 1 & 1 & 1 & 1 \\
    & Apply Hodge masks & True & True & True & True & True \\
    \hline
    \multirow{4}{3em}{SDE for $X$} & Type & VP & VP & VP & VP & VE \\
    & Number of sampling steps & 1000 & 1000 & 1000 & 1000 & 1000 \\
    & $\beta_{min}$ & 0.1 & 0.1 & 0.1 & 0.1 & 0.1 \\
    & $\beta_{max}$ & 1 & 1 & 1 & 1 & 1 \\
    \hline
    \multirow{4}{3em}{SDE for $A$} & Type & VP & VP & VE & VP & VE \\
    & Number of sampling steps & 1000 & 1000 & 1000 & 1000 & 1000 \\
    & $\beta_{min}$ & 0.1 & 0.1 & 0.2 & 0.2 & 0.1 \\
    & $\beta_{max}$ & 1 & 1 & 1 & 0.8 & 1 \\
    \hline
    \multirow{4}{3em}{SDE for $F$} & Type & VP & VP & VE & VP & VE \\
    & Number of sampling steps & 1000 & 1000 & 1000 & 1000 & 1000 \\
    & $\beta_{min}$ & 0.1 & 0.1 & 0.1 & 0.1 & 0.1 \\
    & $\beta_{max}$ & 1 & 1 & 1 & 1 & 1 \\
    \hline
    \multirow{3}{3em}{Solver} & Type & EM & EM + Langevin & S4 & Rev. + Langevin & Rev. + Langevin \\
    & SNR & - & 0.05 & 0.15 & 0.1 & 0.2 \\
    & Scale coefficient & - & 0.7 & 0.7 & 0.7 & 0.7\\
    \hline
    \multirow{6}{3em}{Train} & Optimizer & Adam & Adam & Adam & Adam & Adam \\
    & Learning rate & $1\times 10^{-2}$ & 1 $\times 10^{-2}$ & $1 \times 10^{-2}$ & $1 \times 10^{-2}$ & $5 \times 10^{-3}$ \\
    & Weight decay & $1 \times 10^{-4}$ & $1 \times 10^{-4}$ & $1 \times 10^{-4}$ & $1 \times 10^{-4}$ & $1 \times 10^{-4}$ \\
    & Batch size & 128 & 128 & 64 & 8 & 1024 \\
    & Number of epochs & 5000 & 5000 & 5000 & 5000 & 300 \\
    & EMA & - & - & 0.999 & 0.999 & - \\
    \hline
\end{tabular}
\caption{\textbf{Hyperparameters of CCSD} used in the generic graph generation tasks and the molecule generation tasks. We provide the hyperparameters of the score-based models ($s_{\theta_{0}}=\text{ScoreNetworkX}$, $s_{\theta_{1}}=\text{ScoreNetworkA\_CC}$, and $s_{\theta_{2}}=\text{ScoreNetworkF}$), the diffusion processes (SDE for $X=\Omega_{0}$, $A=\Omega_{1}$, and $F=\Omega_{2}$), the SDE solver, and the training. $d_{min}$ and $d_{max}$ refer to the constrained on the rank-2 matrices generated via $\Omega_{2}=F$.}
\label{tabular:hyperparameters_CCSD}
\end{table}

\begin{table}
\scriptsize
\centering
\begin{tabular}{ llccccc } 
    \hline
     & Hyperparameters & Ego-small & Community-small & Enzymes-small & Grid-small & QM9 \\
    \hline
    \multirow{2}{3em}{$s_{\theta_{0}}$} & Number of GCN layers & 2 & 3 & 5 & 4 & 2 \\
    & Hidden dimension & 32 & 32 & 32 & 24 & 10 \\
    \hline
    \multirow{6}{3em}{$s_{\theta_{1}}$} & Number of attention heads & 4 & 4 & 4 & 4 & 4 \\
    & Number of initial channels & 2 & 2 & 2 & 2 & 4 \\
    & Number of hidden channels & 6 & 8 & 8 & 6 & 8 \\
    & Number of final channels & 4 & 4 & 4 & 4 & 4 \\
    & Number of Block layers & 4 & 5 & 6 & 6 & 3 \\
    & Number layers in MLP & 2 & 2 & 2 & 2 & 3 \\
    & Hidden dimension & 32 & 32 & 32 & 24 & 10 \\
    & Hodge Block layer & HCN & HCN & HCN & HCN & HCN \\
    & Number of Hodge Block layers & 3 & 2 & 2 & 2 & 2 \\
    & Number layers in Hodge MLP & 2 & 2 & 2 & 1 & 2 \\
    & Hidden dimension in Hodge MLP & 4 & 4 & 8 & 4 & 8 \\
    & Number of hidden Hodge channels & 4 & 4 & 8 & 2 & 8 \\
    & Number of final Hodge channels & 6 & 4 & 8 & 4 & 6 \\
    & Hidden dimension in Hodge layers & 6 & 4 & 8 & 2 & 8\\
    \hline
    \multirow{3}{3em}{$s_{\theta_{2}}$} & Power higher-order & 2 & 2 & 2 & 1 & 2 \\
    & Number layers in MLP & 1 & 2 & 2 & 1 & 1 \\
    & Apply Hodge masks & True & True & True & True & True \\
    \hline
\end{tabular}
\caption{\textbf{Hyperparameters of CCSD Baseline} used in the generic graph generation tasks and the molecule generation tasks. This time, the score-based model for the rank-2 cells is $s_{\theta_{1}}=\text{ScoreNetworkA\_Base\_CC}$. The general parameters, the SDEs, the solver and the hyperparameters for training are not displayed as they are the same as CCSD.}
\label{tabular:hyperparameters_CCSD_Base}
\end{table}

\section{Introducing CCSD: A Python library}
\label{library}

In addition to the research outcomes presented above, we also delivered a publicly available \textbf{Python library, \textit{CCSD}}, designed to facilitate the replication of our research results and to allow researchers to extend our work conducted in this thesis. This comprehensive tool empowers users to train and sample combinatorial complexes using the CCSD approach or generate graphs and molecules employing the GDSS approach. The library offers extensive functionality, including the creation of detailed logs and model checkpoints during training and the generation of plots, animations, and object pickle files during sampling. The library is technical and yet easy to use.\par

For easy access and utilization, \textit{CCSD} is available on PyPi at the following link: \url{https://pypi.org/project/ccsd/}. Additionally, its source code and further information can be found on GitHub: \url{https://github.com/AdrienC21/CCSD}.\par

\begin{figure}[H]
\centering
\begin{minipage}{0.5\textwidth}
  \centering
  \includegraphics[width=0.8\linewidth]{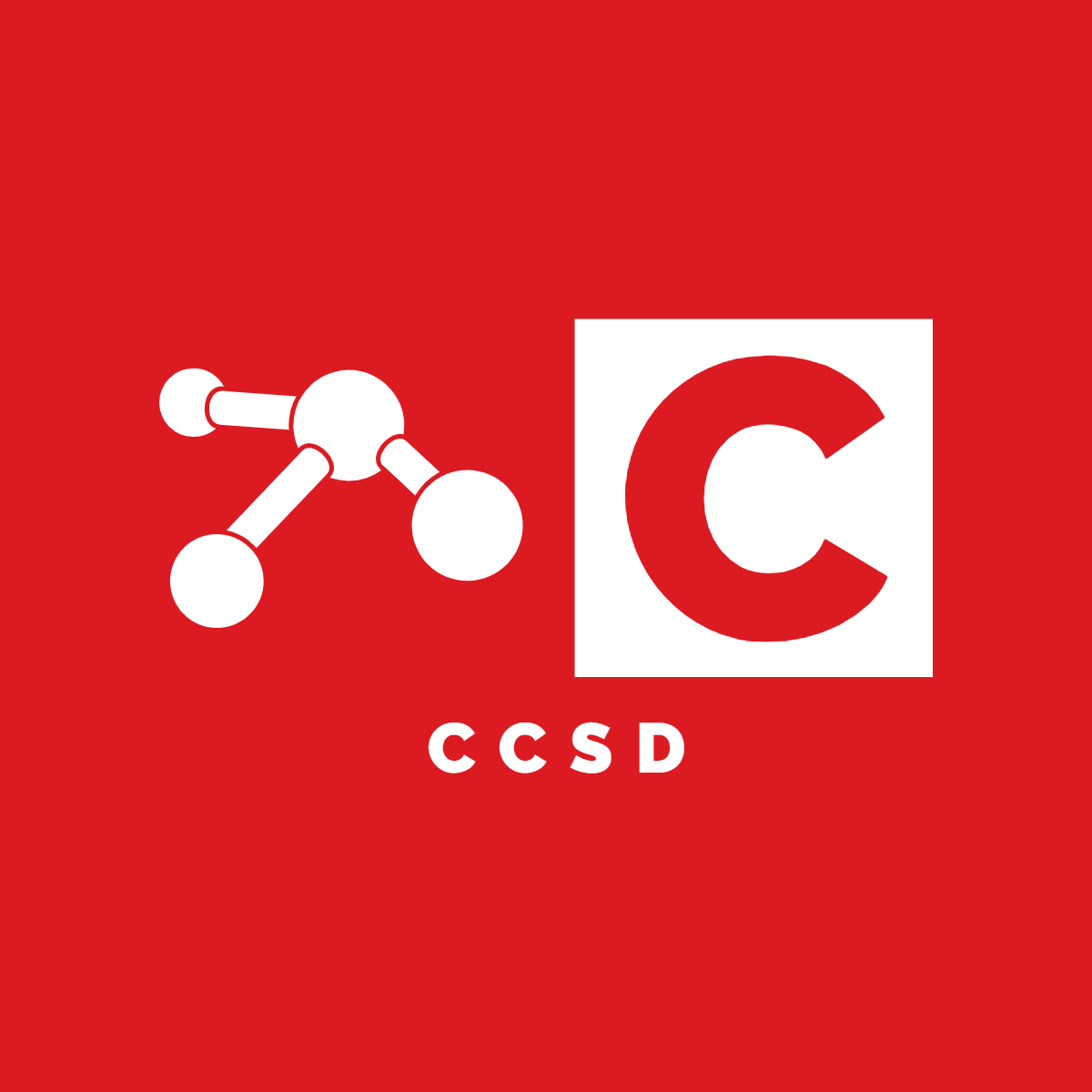}
  \captionof{figure}{CCSD Logo.}
  \label{fig:ccsd-logo}
\end{minipage}%
\begin{minipage}{0.5\textwidth}
  \centering
  \includegraphics[width=0.8\linewidth]{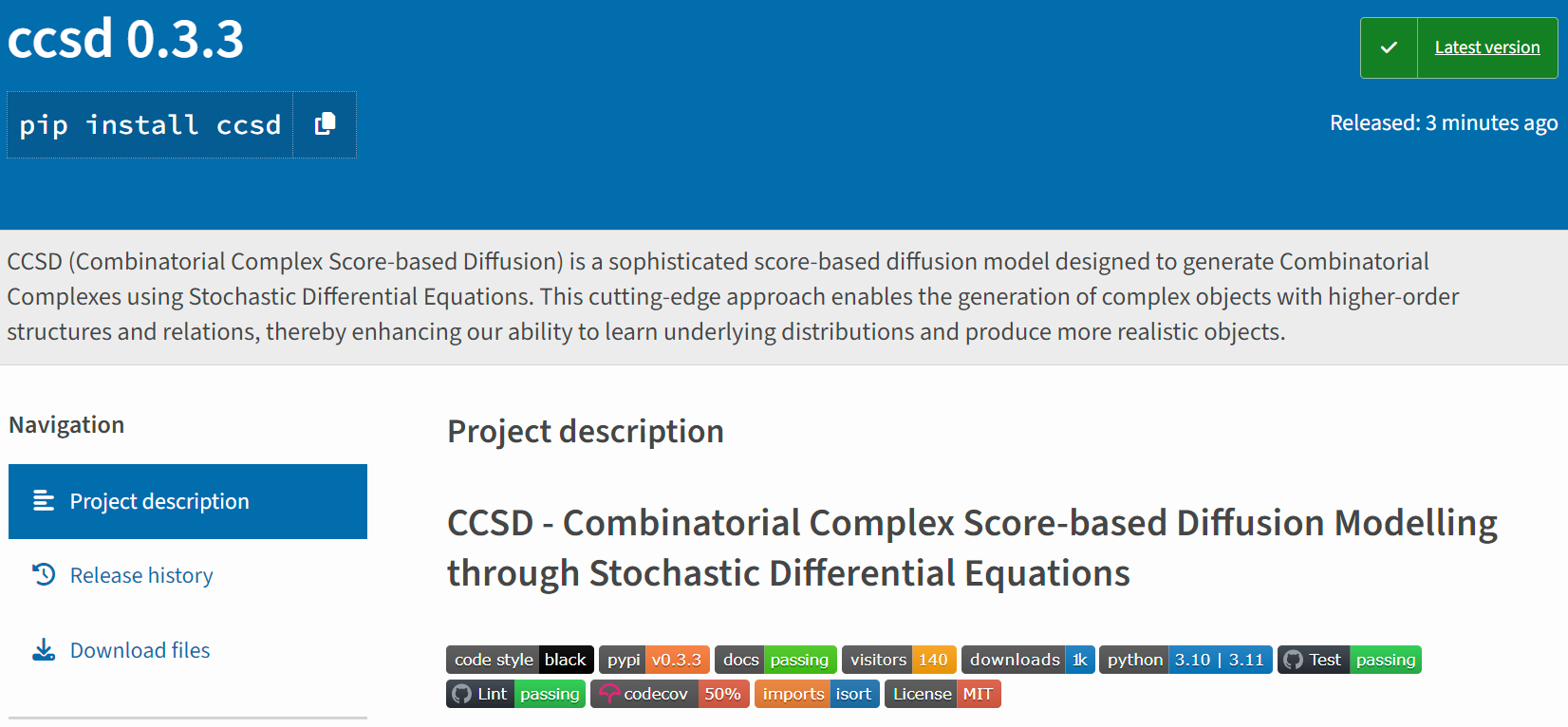}
  \captionof{figure}{PyPi page of the library.}
  \label{fig:pypi}
\end{minipage}
\end{figure}

To ensure the highest quality, our library has undergone meticulous linting, thorough documentation, and extensive unit test coverage. We have strived to offer a user-friendly interface, supported by comprehensive documentation (see Subsection \ref{documentation} below), and implemented a well-documented parser with a detailed log system (Figure \ref{fig:run_script}).\par

All researchers are more than welcome to contribute, whether it is through debugging, adding functionnalities or new score network models.\par

\begin{figure}[H]
\centering
\includegraphics[width=0.6\linewidth]{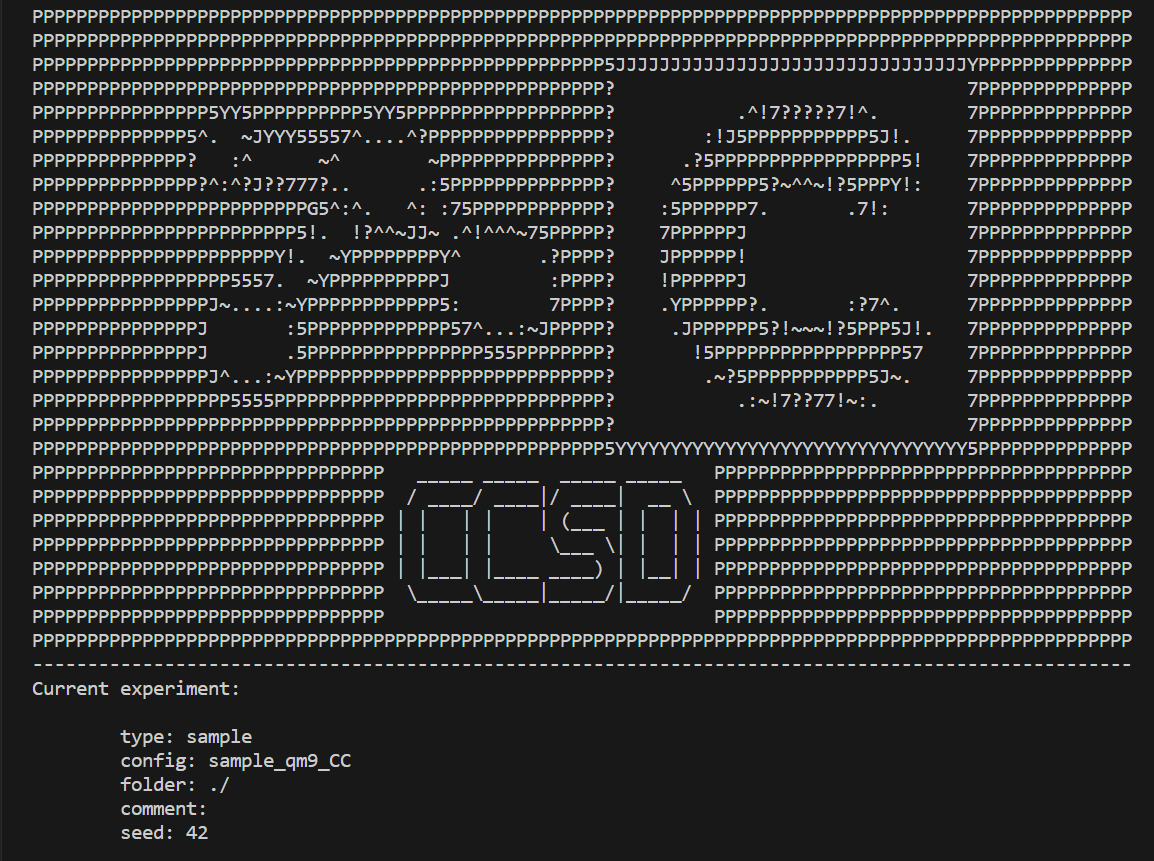}
\captionof{figure}{Logs at the start of a sampling procedure.}
\label{fig:run_script}
\end{figure}

\subsection{Documentation}
\label{documentation}

For complete information on \textit{CCSD}, including installation instructions and usage details, please refer to our comprehensive documentation available at: \url{https://ccsd.readthedocs.io/} (Figure \ref{fig:ccsd-documentation}). This documentation provides in-depth insights into the package, including functions, parameters, input and output types, and customization options. It is designed to assist users in quickly grasping the concepts and making the most of our framework. Further information can also be accessed on our GitHub page; \underline{repository:} AdrienC21/CCSD (Figure \ref{fig:ccsd-github}).\par

\begin{figure}[H]
\centering
\begin{minipage}{0.5\textwidth}
  \centering
  \includegraphics[width=0.9\linewidth]{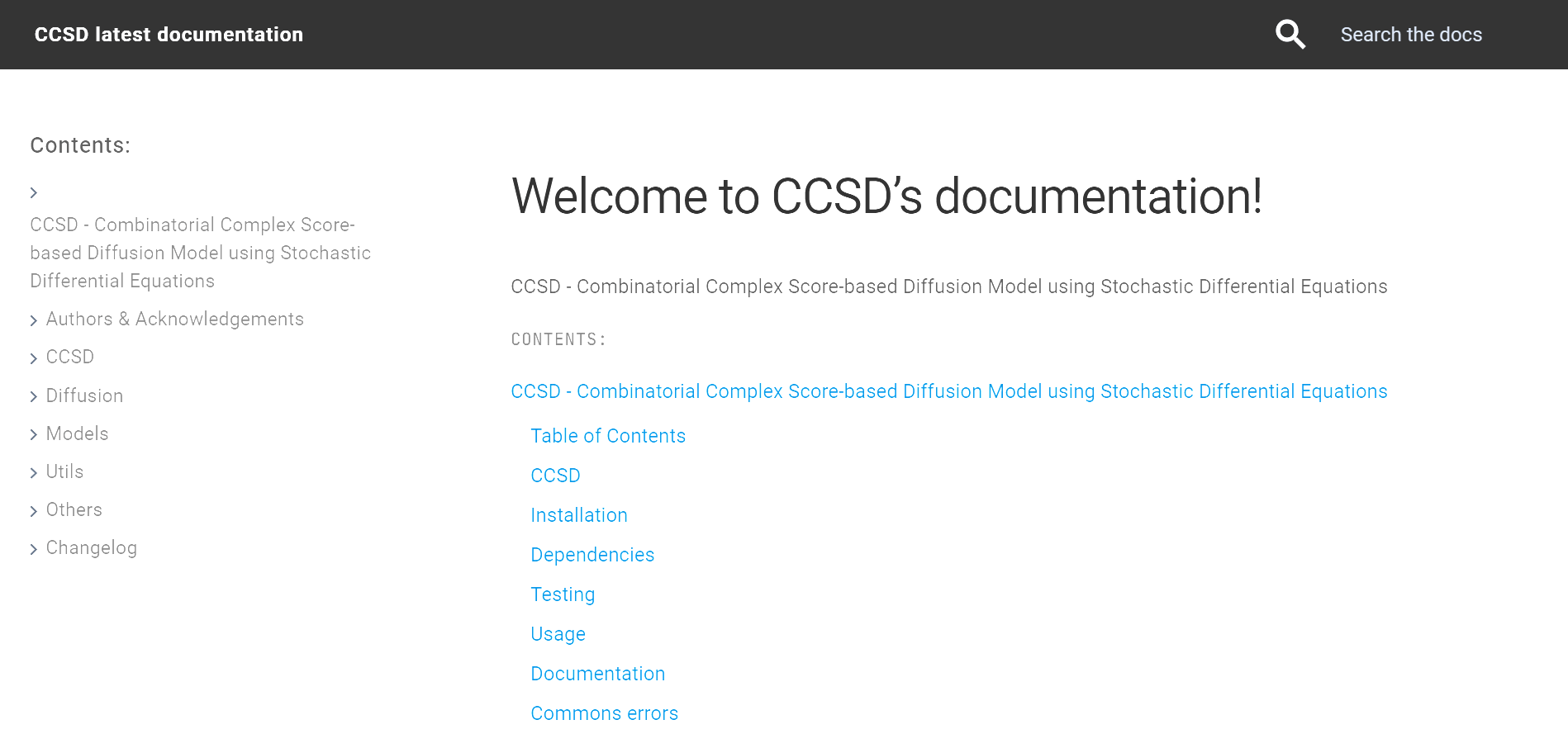}
  \captionof{figure}{CCSD online documentation.}
  \label{fig:ccsd-documentation}
\end{minipage}%
\begin{minipage}{0.5\textwidth}
  \centering
  \includegraphics[width=0.9\linewidth]{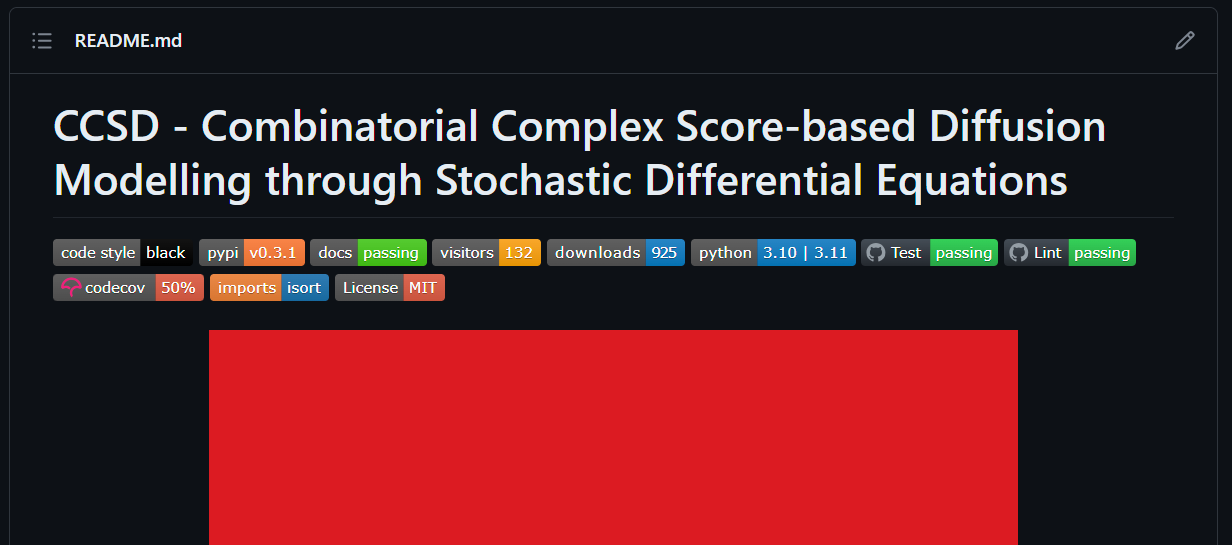}
  \captionof{figure}{GitHub page of the library.}
  \label{fig:ccsd-github}
\end{minipage}
\end{figure}

\subsection{Example 1: Generating molecules by sampling combinatorial complexes}

In this example, we demonstrate how to use the CCSD package to sample molecules from the learned QM9 dataset distribution. First, we import the package, define the experiment parameters, and then sample 10,000 molecules. The logs and SMILES representations of the generated molecules will be saved in a folder named \textit{logs\_sample}, while visual materials such as combinatorial complexes, graphs, molecules, and diffusion animations will be stored in a \textit{sample} folder.\par

Using a terminal, you can use the following bash command after cloning the repository (recommended):\par

\begin{lstlisting}[language=bash]
python main.py
    --type sample
    --config sample_qm9_CC
    --folder "./"
    --comment "Sample molecule through combinatorial complexes"
    --seed 42
\end{lstlisting}

Or, after installing the library using the command \textit{pip install ccsd}, you can run the experiment using Python:\par

\begin{lstlisting}[language=Python]
from ccsd.diffusion import CCSD


params = {
  "type": "sample",
  "config": "sample_qm9_CC",
  "folder": "./",  # assume we are at the source of our project
  "comment": "Sample molecule through combinatorial complexes",
  "seed": 42  # optional
}
diffusion_model = CCSD(**params)  # define the object
diffusion_model.run()  # run the experiment
\end{lstlisting}

\subsection{Example 2: Creating a combinatorial complex dataset}

In this example, we generate a combinatorial complex dataset based on the Community Small graph dataset.\par

After cloning the repository, run the following command:\par

\begin{lstlisting}[language=bash]
python ccsd/data/data_generators.py --dataset community_small --is_cc
\end{lstlisting}

%%%%%%%%%%%%%%%%%%%%%%%%%%%%%%%%%%%%
\chapter{Experiments Evaluation \& Results}
\label{Results_eval}

This chapter provides insights into the computational resources employed for our experiments (Section \ref{Computing_Resources}), a quantitative evaluation of our models across various datasets using the previously defined metrics (Section \ref{Evaluation}), and visualizations of the results as part of a more qualitative evaluation (Section \ref{Results}).\par

\section{Computing Resources}
\label{Computing_Resources}

We developed this thesis using PyTorch \cite{paszke2019pytorch} and trained our score network models on a Nvidia RTX 2080 Ti, Nvidia Tesla V100-SMX2-16GB, Nvidia Tesla V100-PCIE-32GB, or Nvidia L4-24GB GPUs. Sampling was performed on the latter three GPUs due to their larger memory capacity. Our code execution utilized OVH Public Cloud (\url{ovh.com/}) and Google Cloud resources (\url{https://console.cloud.google.com/}). We tracked losses using Weights \& Biases (\url{https://wandb.ai/}). All machines and virtual machines featured 8 CPU cores and 16GB of RAM. The IDE employed throughout the thesis was Visual Studio Code (VSCode).

\section{Quantitative Evaluation Results}
\label{Evaluation}

We trained our models, plotted the training curves (refer to Appendix \ref{learning_curves}), and performed sampling on various datasets to evaluate our models.\par

\begin{table}
\tiny
\centering
\begin{tabular}{ llccccccc } 
    \hline
     &  &  \multicolumn{7}{c}{QM9} \\
     \cline{3-9}
     &  &  \multicolumn{7}{c}{Real, $1\leq | V | \leq 9$} \\
    \cline{3-9}
     & Method &  Val. w/o corr. (\%)$\uparrow$  & NSPDK$\downarrow$ & FCD$\downarrow$ & Validity (\%)$\uparrow$ & Uniqueness (\%)$\uparrow$ & Novelty (\%)$\uparrow$ & Time (s)$\downarrow$ \\
    \hline
    \multirow{4}{3em}{Autoreg.} & GraphAF \cite{shi2020graphaf} & 67* & 0.020 & 5.268 & \textbf{100.00}* & 94.51* & 88.83* & $2.52e^{3}$ \\
     & GraphAF+FC & 74.43 & 0.021 & 5.625 & \textbf{100.00} & 88.64 & 86.59 & $2.55e^{3}$ \\
     & GraphDF \cite{luo2021graphdf} & 82.67* & 0.063 & 10.816 & \textbf{100.00}* & 97.62* & 98.10* & $5.35e^{4}$ \\
     & GraphDF+FC & 93.88 & 0.064 & 10.928 & \textbf{100.00} & 98.58 & \textbf{98.54} & $4.91e^{4}$ \\
    \hline
    \multirow{8}{3em}{One shot} & MoFlow \cite{Zang_2020} & 91.36 & 0.017 & 4.467 & \textbf{100.00} & 98.65 & 94.72 & \textbf{4.60} \\
    & EDP-GNN \cite{niu2020permutation} & \underline{47.52} & 0.005 & 2.680 & \textbf{100.00} & 99.25 & 86.58 & $4.40e^{3}$ \\
    & GraphEBM \cite{liu2021graphebm} & \underline{8.22} & 0.030 & 6.143 & \textbf{100.00}* & 97.90* & 97.01* & $3.71e^{1}$ \\
    & GDSS \cite{jo2022scorebased} & 95.72 & 0.003 & 3.096 & \textbf{100.00} & 98.4 & 86.10 & $8.4e^{1}$ \\
    & SGGM+SLD \cite{yang2023scorebased} & \textbf{97.35} & 0.004 & 2.593 & \textbf{100.00} & \textbf{99.41} & 97.49 & ? \\
    & GRAPHARM \cite{kong2023autoregressive} & 90.25 & 0.002 & \textbf{1.22} & ? & 95.62 & 70.39 & $1.52e^{1}$ \\
    \cline{2-9}
    & \textbf{CCSD} (Ours) & 92.74 & \textbf{0.002} & 2.682 & \textbf{100.00} & 98.01 & 77.78 & $5.2e^{3}$\\
    & \textbf{CCSD Base} (Ours) & 92.42 & 0.008 & 4.320 & \textbf{100.00} & 97.13 & 84.26 & $6.5e^{3}$ \\
    \hline
\end{tabular}
\caption{\textbf{Generation results on the QM9 dataset.} Results are taken from the same seed (42). The best results are highlighted in bold. Values denoted by * are taken from the respective original papers. Values denoted by ? are unknown. Other results are taken from Jo et al. \cite{jo2022scorebased} or, for our results and GDSS, have been retrained and calculated by ourselves. Val. w/o corr. denotes the Validity w/o correction metric, and values that do not exceed 50\% are underlined. Results are rounded to 3 or 4 digits.}
\label{tabular:result_mol}
\end{table}

\begin{table}
\tiny
\centering
\begin{tabular}{ llcccccccc } 
    \hline
     &  & \multicolumn{4}{c}{Ego-small} & \multicolumn{4}{c}{Community-small} \\
    \cline{3-10}
     &  & \multicolumn{4}{c}{Real, $4\leq | V | \leq 18$} & \multicolumn{4}{c}{Synthetic, $12\leq | V | \leq 19$} \\
    \cline{3-10}
     &  & Deg.$\downarrow$ & Clus.$\downarrow$ & Orbit$\downarrow$ & Avg.$\downarrow$ & Deg.$\downarrow$ & Clus.$\downarrow$ & Orbit$\downarrow$ & Avg.$\downarrow$ \\
    \hline
    \multirow{4}{3em}{Autoreg.} & DeepGMG \cite{li2018learning} & 0.040 & 0.100 & 0.020 & 0.053 & 0.220 & 0.950 & 0.400 & 0.523 \\
    & GraphRNN \cite{you2018graphrnn} & 0.090 & 0.220 & 0.003 & 0.104 & 0.080 & 0.120 & 0.040 & 0.080 \\
    & GraphAF \cite{shi2020graphaf} & 0.03 & 0.11 & \textbf{0.001} & 0.047 & 0.18 & 0.20 & 0.02 & 0.133 \\
    & GraphDF \cite{luo2021graphdf} & 0.04 & 0.13 & 0.01 & 0.060 & 0.06 & 0.12 & 0.03 & 0.070 \\
    \hline
    \multirow{8}{3em}{One shot} & GraphVAE \cite{simonovsky2018graphvae} & 0.130 & 0.170 & 0.050 & 0.117 & 0.350 & 0.980 & 0.540 & 0.623 \\
    & GNF$\dagger$ \cite{NEURIPS2019_1e44fdf9} & 0.030 & 0.100 & \textbf{0.001} & 0.044 & 0.200 & 0.200 & 0.110 & 0.170 \\
    & EDP-GNN \cite{niu2020permutation} & 0.052 & 0.093 & 0.007 & 0.051 & 0.053 & 0.144 & 0.026 & 0.074 \\
    & GDSS \cite{jo2022scorebased} & 0.021 & 0.024 & 0.007 & 0.017 & 0.077 & 0.064 & 0.013 & 0.051 \\
    & SGGM+SLD \cite{yang2023scorebased} & \textbf{0.014} & 0.019 & 0.007 & \textbf{0.013} & 0.035 & 0.071 & 0.006 & \textbf{0.037} \\
    & GRAPHARM \cite{kong2023autoregressive} & 0.019 & \textbf{0.017} & 0.010 & 0.015 & \textbf{0.034} & 0.082 & \textbf{0.004} & 0.04 \\
    \cline{2-10}
    & \textbf{CCSD} (Ours) & 0.030 & 0.023 & 0.018 & 0.024 & 0.114 & \textbf{0.063} & 0.065 & 0.081 \\
    & \textbf{CCSD Base} (Ours) & ? & ? & ? & ? & 0.053 & \textbf{0.052} & 0.040 & 0.048 \\
    \hline
\end{tabular}

\vspace{1cm}

\begin{tabular}{ llcccccccc } 
    \hline
     &  & \multicolumn{4}{c}{Enzymes-small} & \multicolumn{4}{c}{Grid-small} \\
    \cline{3-10}
     &  & \multicolumn{4}{c}{Real, $4\leq | V | \leq 12$} & \multicolumn{4}{c}{Synthetic, $4\leq | V | \leq 49$} \\
    \cline{3-10}
     &  & Deg.$\downarrow$ & Clus.$\downarrow$ & Orbit$\downarrow$ & Avg.$\downarrow$ & Deg.$\downarrow$ & Clus.$\downarrow$ & Orbit$\downarrow$ & Avg.$\downarrow$ \\
    \hline
    \multirow{3}{3em}{One shot} & GDSS & 0.133 & \textbf{0.147} & \textbf{0.008} & \textbf{0.096} & \textbf{0.013} & \textbf{0.041} & \textbf{0.008} & \textbf{0.021} \\
    \cline{2-10}
    & \textbf{CCSD} (Ours) & 0.234 & 0.155 & 0.013 &  0.134 & ? & ? & ? & ? \\
    & \textbf{CCSD Base} (Ours) & \textbf{0.129} & 0.329 & 0.030 & 0.163 & ? & ? & ? & ? \\
    \hline
\end{tabular}

\caption{\textbf{Generation results on the generic graph datasets.} Results are taken from the same seed (42). We report the MMD distances between the test datasets and generated graphs. The best results are highlighted in bold (the smaller the better). Values denoted by ? are unknown. In the case of our models, it means that we ran out of RAM during the sampling procedure. The results are taken from Jo et al. \cite{jo2022scorebased} or, for our results and GDSS, have been retrained and calculated by ourselves. $\dagger$ indicates unreproducible results. Results are rounded to 3 or 4 digits.}
\label{tabular:result_graph}
\end{table}

\begin{table}
\scriptsize
\centering
\begin{tabular}{ llcc } 
    \hline
     &  & Ego-small & Community-small \\
    \cline{3-4}
     &  & Real, $4\leq | V | \leq 18$ & Synthetic, $12\leq | V | \leq 19$\\
    \cline{3-4}
     &  & Rank-2$\downarrow$ & Hodge.$\downarrow$ \\
    \hline
    \multirow{3}{3em}{One shot} & GDSS & ? & ? \\
    \cline{2-4}
    & \textbf{CCSD} (Ours) & \textbf{0.314} & \textbf{0.668} \\
    & \textbf{CCSD Base} (Ours) & ? & 0.714 \\
    \hline
\end{tabular}

\vspace{1cm}

\begin{tabular}{ llcc } 
    \hline
     &  & \multicolumn{2}{c}{Enzymes-small} \\
    \cline{3-4}
     &  & \multicolumn{2}{c}{Real, $4\leq | V | \leq 12$}\\
    \cline{3-4}
     &  & Rank-2$\downarrow$ & Hodge.$\downarrow$\\
    \hline
    \multirow{3}{3em}{One shot} & GDSS & \textbf{0.017} & \textbf{0.558} \\
    \cline{2-4}
    & \textbf{CCSD} (Ours) & 0.361 & 0.905 \\
    & \textbf{CCSD Base} (Ours) & 0.345 & 0.783 \\
    \hline
\end{tabular}

\vspace{1cm}

\begin{tabular}{ llccccc } 
    \hline
     &  & \multicolumn{5}{c}{QM9} \\
    \cline{3-7}
     &  & \multicolumn{5}{c}{Real, $1\leq | V | \leq 9$}\\
    \cline{3-7}
     &  & Rank-0$\downarrow$ & Rank-1$\downarrow$ & Rank-2$\downarrow$ & Avg.$\downarrow$ & Hodge.$\downarrow$ \\
    \hline
    \multirow{3}{3em}{One shot} & GDSS & 0.004 & \textbf{0.002} & 1.049 & 0.352 & ? \\
    \cline{2-7}
    & \textbf{CCSD} (Ours) & 0.0004 & 0.0023 & 1.3912 & 0.4646 & ? \\
    & \textbf{CCSD Base} (Ours) & \textbf{0.0002} & 0.0076 & \textbf{0.0129} & \textbf{0.0069} & ? \\
    \hline
\end{tabular}

\caption{\textbf{Generation higher-order metrics results on all datasets.} Results are taken from the same seed (42). We report the MMD distances between the test datasets and generated objects. The best results are highlighted in bold (the smaller the better). Values denoted by ? are unknown. In the case of our models, it means that we ran out of RAM during the sampling procedure. All the metrics have been calculated by ourselves. - indicates that it is not relevant for this dataset as there are no features attached for the corresponding dimension. Results are rounded to 3 or 4 digits.}
\label{tabular:result_cc}
\end{table}

In the evaluation of our models on the QM9 dataset (Table \ref{tabular:result_mol}), CCSD demonstrates robust performance across various metrics, aligning closely with state-of-the-art methods. Notably, our method excels in the NSPDK MMD measure, indicating that it effectively generates molecular structures that closely resemble the original distribution. Moreover, our framework exhibits high Uniqueness and Novelty scores, suggesting its capability to also produce novel molecular graphs while maintaining diversity and avoiding excessive similarity between generated molecules.\par

For graph generation tasks (Table \ref{tabular:result_graph}), both CCSD and CCSD Base consistently yield promising results in both real and synthetic datasets. CCSD stands out across all datasets, achieving impressive metric scores and rivalling other methods. Specifically, for the Community small dataset, our two models achieve state-of-the-art performance on the Clustering metric, beating the two latest graph generation models recently published, SGGM+SLD and GRAPHARM, which dominate the leaderboard. These results suggest that our approaches effectively capture the inherent structural characteristics of graphs.\par

In the higher-order metrics assessment (Table \ref{tabular:result_cc}), our approach consistently performs well compared to GDSS. For the GDSS method, we generated graphs and molecules and then lifted them into combinatorial complexes to calculate these higher-order metrics for graph generation methods. This performance extends to the QM9 dataset, where CCSD Base attains exceptionally low values across all metrics, including rank-0, rank-1, rank-2, average, and the Hodge Laplacian spectrum metric. These results underscore CCSD's ability to generate objects that closely align with the ground truth distribution.\par

To measure the quality of our samples, we also employed the Tanimoto similarity, based on Morgan fingerprints obtained using the RDKit library \cite{Landrum2016RDKit2016_09_4} with 1024 bits and a radius of 2 as suggested in Jo et al. \cite{jo2022scorebased} (see Subsection \ref{Tanimoto} for a more detailed definition). We calculated the average Tanimoto similarity of the generated molecules compared to all the training molecules. As depicted in Table \ref{tabular:tanimoto}, CCSD has a higher similarity, thus demonstrating its capability to generate molecules that closely match the structural characteristics of training molecules, while other baseline models including GDSS tend to generate molecules that could deviate from the training distribution.\par

\begin{table}
\centering
\begin{tabular}{ cc }
    \multicolumn{2}{c}{Tanimoto Similarity $\uparrow$}\\
    \cline{2-2}
     \textbf{CCSD (Ours)} & \textbf{0.564} \\
     GDSS & 0.502\\
\end{tabular}
\caption{\textbf{Tanimoto similarity on the QM9 dataset.} We compared the average Tanimoto similarity across the generated molecules by comparing their fingerprints to the entire training dataset. The best results are highlighted in bold (the larger the better). Results are rounded to 3 digits.}
\label{tabular:tanimoto}
\end{table}

Overall, our framework excels in capturing target distributions, with noteworthy performance in the NSPDK MMD metric for molecule generation, surpassing other all the existing approaches. More importantly, our framework performs well even despite its generality, as it is capable of generating not only graphs and molecules but also higher-order topological structures. This broader scope sets CCSD apart from competing approaches, positioning it as a versatile and potent tool for diverse generative AI applications.\par

\section{Qualitative Results}
\label{Results}

In this section, we present visualizations of the generated combinatorial complexes for both generic graph generation tasks and molecule generation tasks. These visualizations include the underlying graphs of the generated combinatorial complexes.\par

\subsection{Molecule Generation}

In Figure \ref{fig:qm9_small_sample}, we showcase both the molecules generated by CCSD and a selection of original molecules from the QM9 dataset. Figure \ref{fig:cc_viz} offers a unique perspective by representing these molecules as hypergraphs, thus illustrating the rings depicted as rank-2 cells. A resemblance is evident between the generated molecules and the original distribution.\par

\begin{figure}[H]
\centering
\begin{minipage}{0.5\textwidth}
  \centering
  \includegraphics[width=0.9\linewidth]{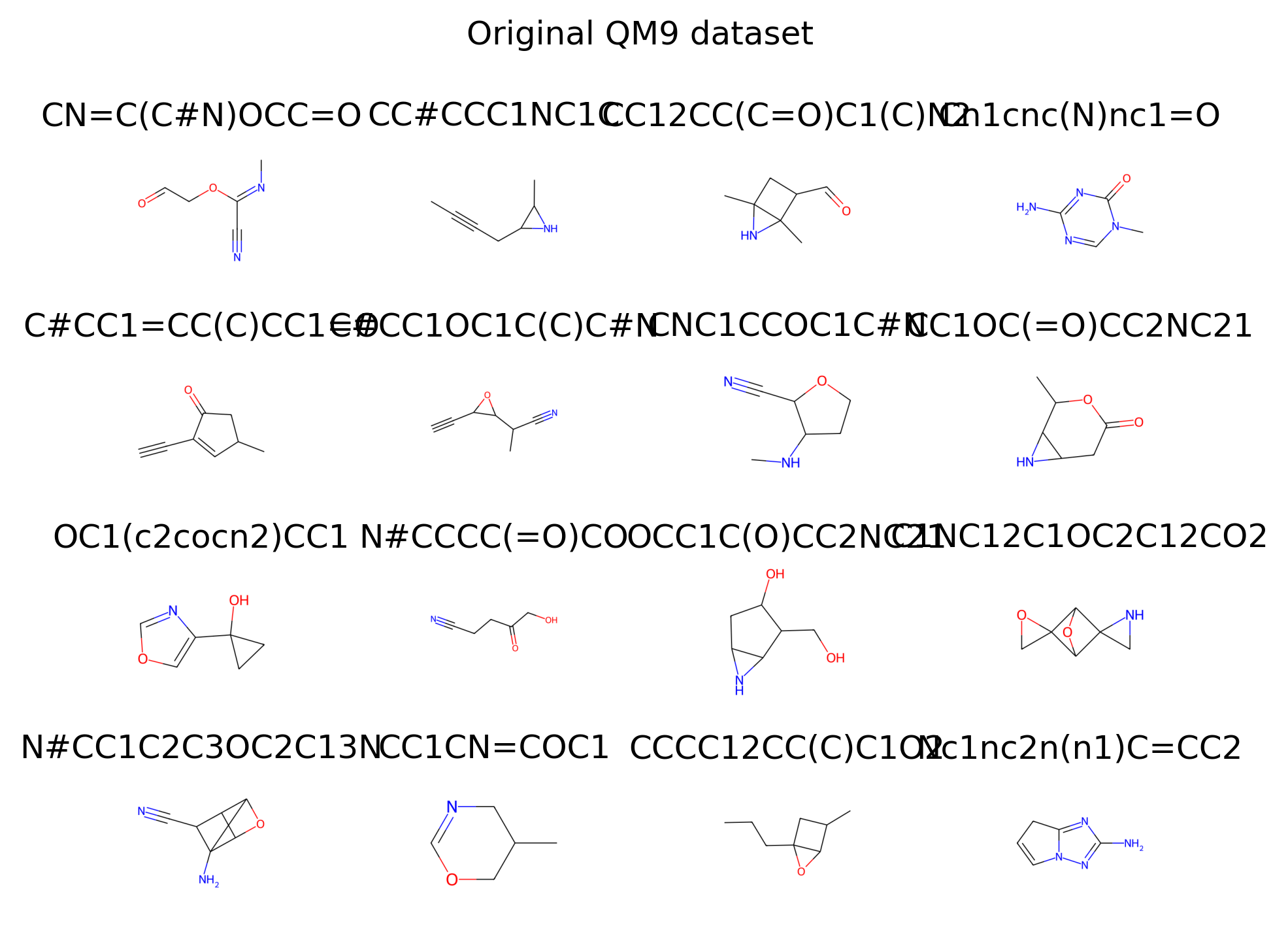}
  \captionof{figure}{Original QM9 molecular dataset.}
  \label{fig:qm9_small_original}
\end{minipage}%
\begin{minipage}{0.5\textwidth}
  \centering
  \includegraphics[width=0.9\linewidth]{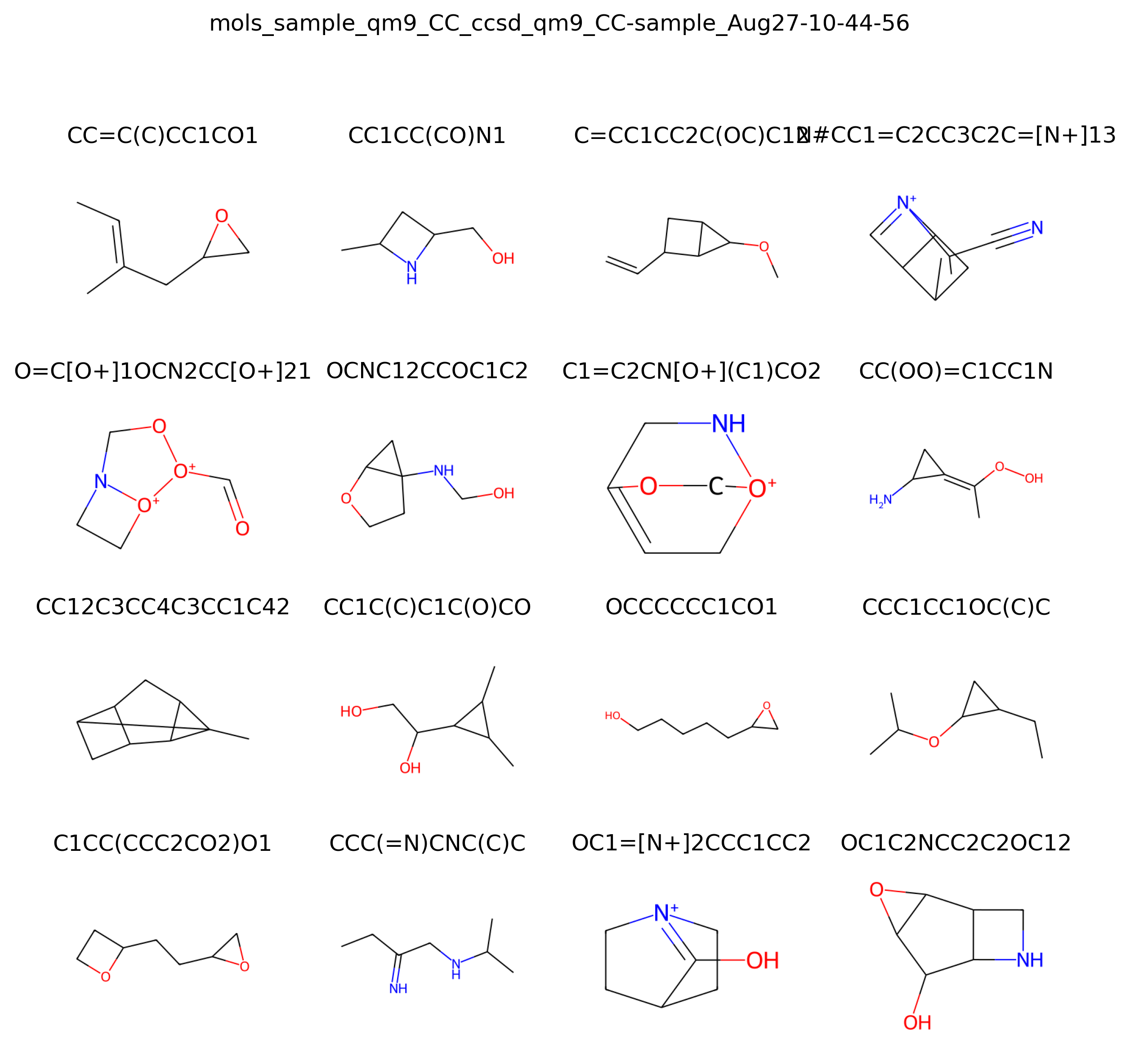}
  \captionof{figure}{Generated QM9 molecules using \textbf{CCSD}.}
  \label{fig:qm9_small_ccsd}
\end{minipage}
\caption{Visualization of the combinatorial complexes (represented as graphs) from the QM9 dataset and the generated objects of CCSD.}
\label{fig:qm9_small_sample}
\end{figure}

\begin{figure}[H]
\centering
\includegraphics[width=0.8\linewidth]{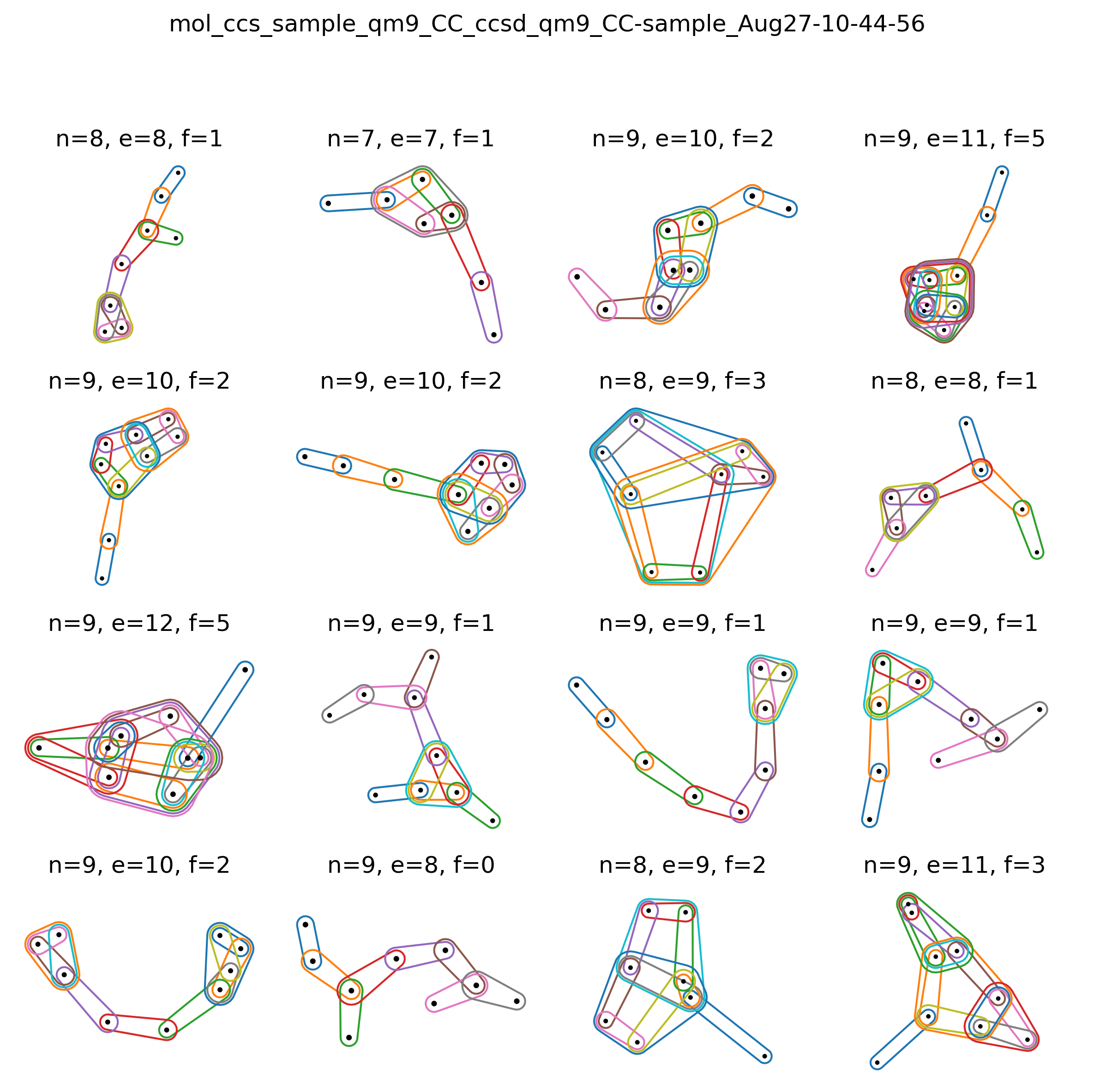}
\captionof{figure}{Visualization of the same combinatorial complexes generated via CCSD in Figure \ref{fig:qm9_small_sample}, represented as hypergraphs to visualize the rank-2 cells. The plot has been realised using the HyperNetX Python library \cite{doecode_22160}.}
\label{fig:cc_viz}
\end{figure}

\subsection{Generic Graph Generation}

Moving on to our generic graph generation tasks, we provide visualizations of the underlying graphs, drawn from both the training datasets and the generated combinatorial complexes produced by CCSD. The displayed graphs are randomly chosen from their respective datasets, accompanied by essential information, including the number of nodes ($n$, representing rank-0 cells), the number of edges ($e$, signifying rank-1 cells), and the count of faces ($f$, corresponding to rank-2 cells) for each combinatorial complex.\par

\begin{figure}[H]
\centering
\begin{minipage}{0.5\textwidth}
  \centering
  \includegraphics[width=0.9\linewidth]{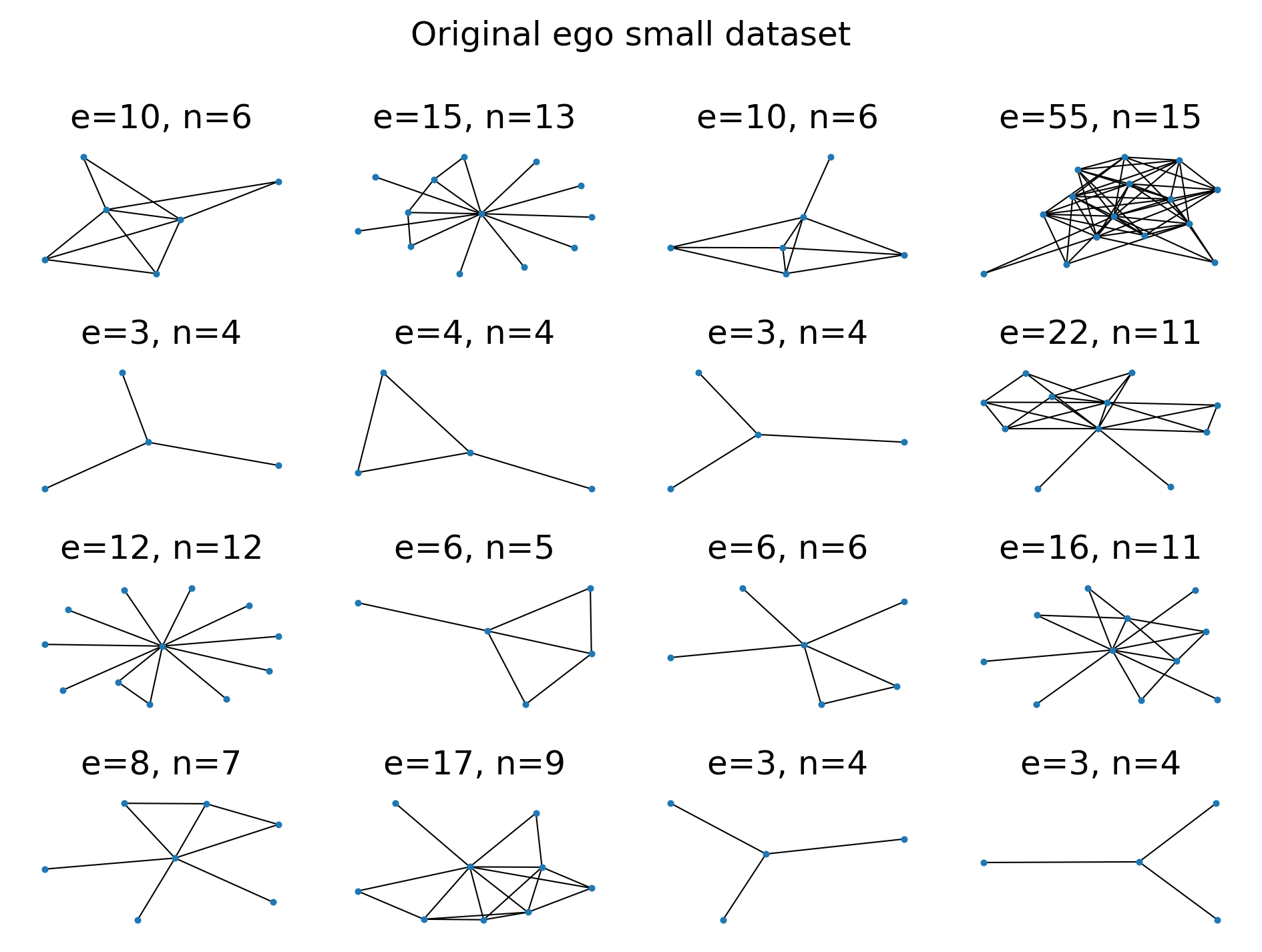}
  \captionof{figure}{Original Ego small graph dataset.}
  \label{fig:ego_small_original}
\end{minipage}%
\begin{minipage}{0.5\textwidth}
  \centering
  \includegraphics[width=0.9\linewidth]{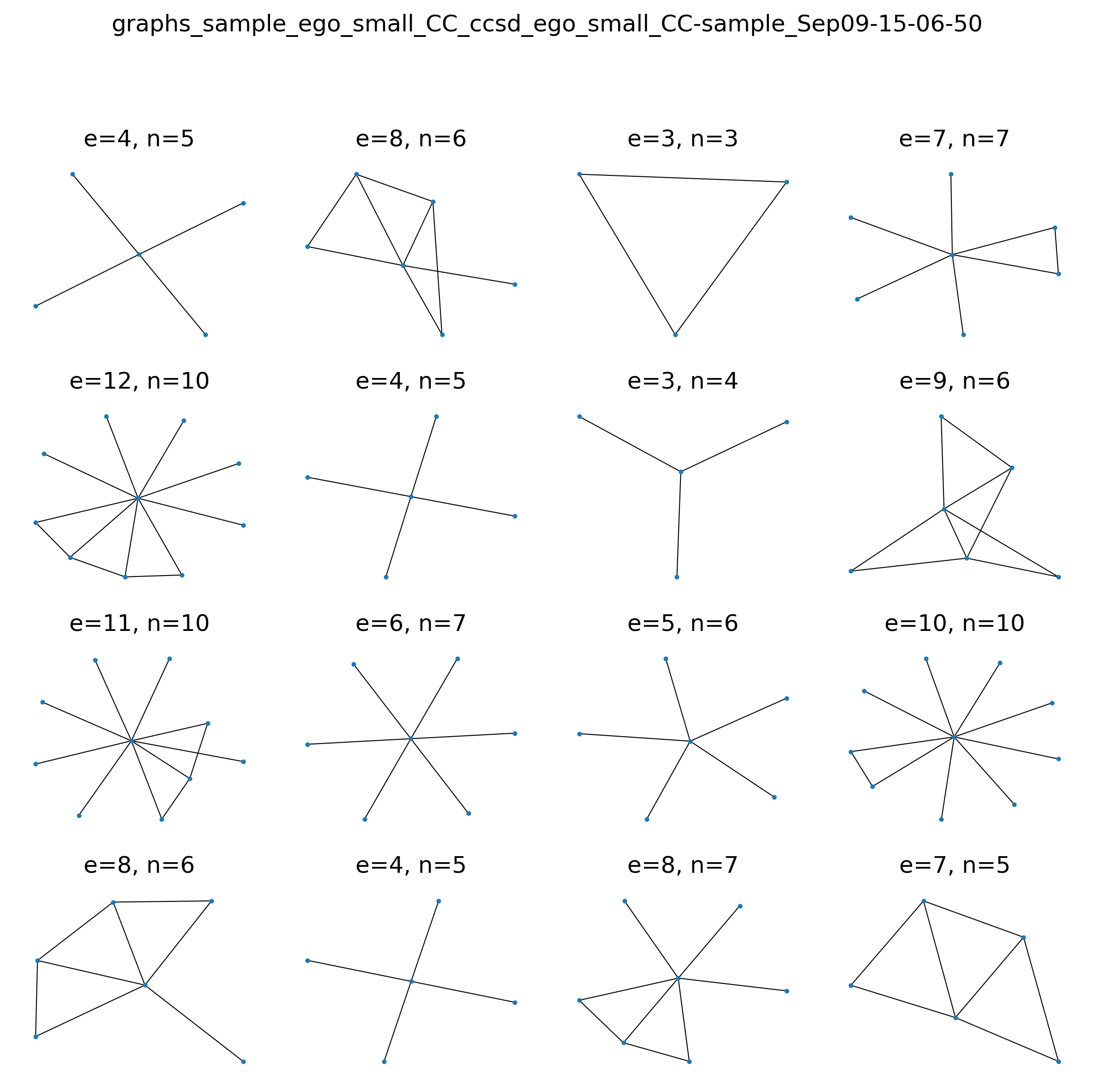}
  \captionof{figure}{Generated Ego small objects using \textbf{CCSD}.}
  \label{fig:ego_small_ccsd}
\end{minipage}
\caption{Visualization of the combinatorial complexes (represented as graphs) from the Ego small dataset and the generated objects of CCSD.}
\label{fig:ego_small_sample}
\end{figure}

\begin{figure}[H]
\centering
\begin{minipage}{0.5\textwidth}
  \centering
  \includegraphics[width=0.9\linewidth]{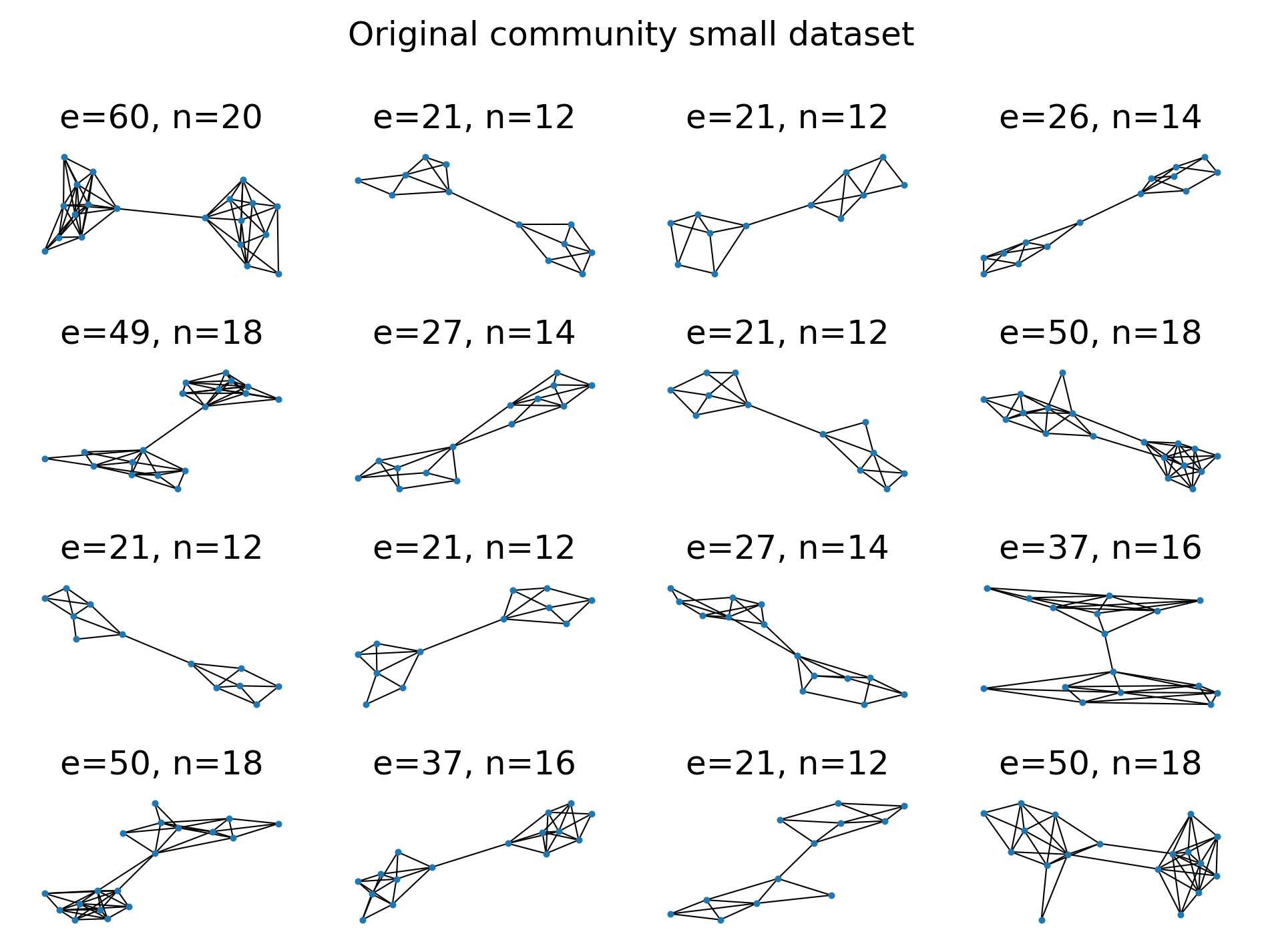}
  \captionof{figure}{Original Community small graph dataset.}
  \label{fig:community_small_original}
\end{minipage}%
\begin{minipage}{0.5\textwidth}
  \centering
  \includegraphics[width=0.9\linewidth]{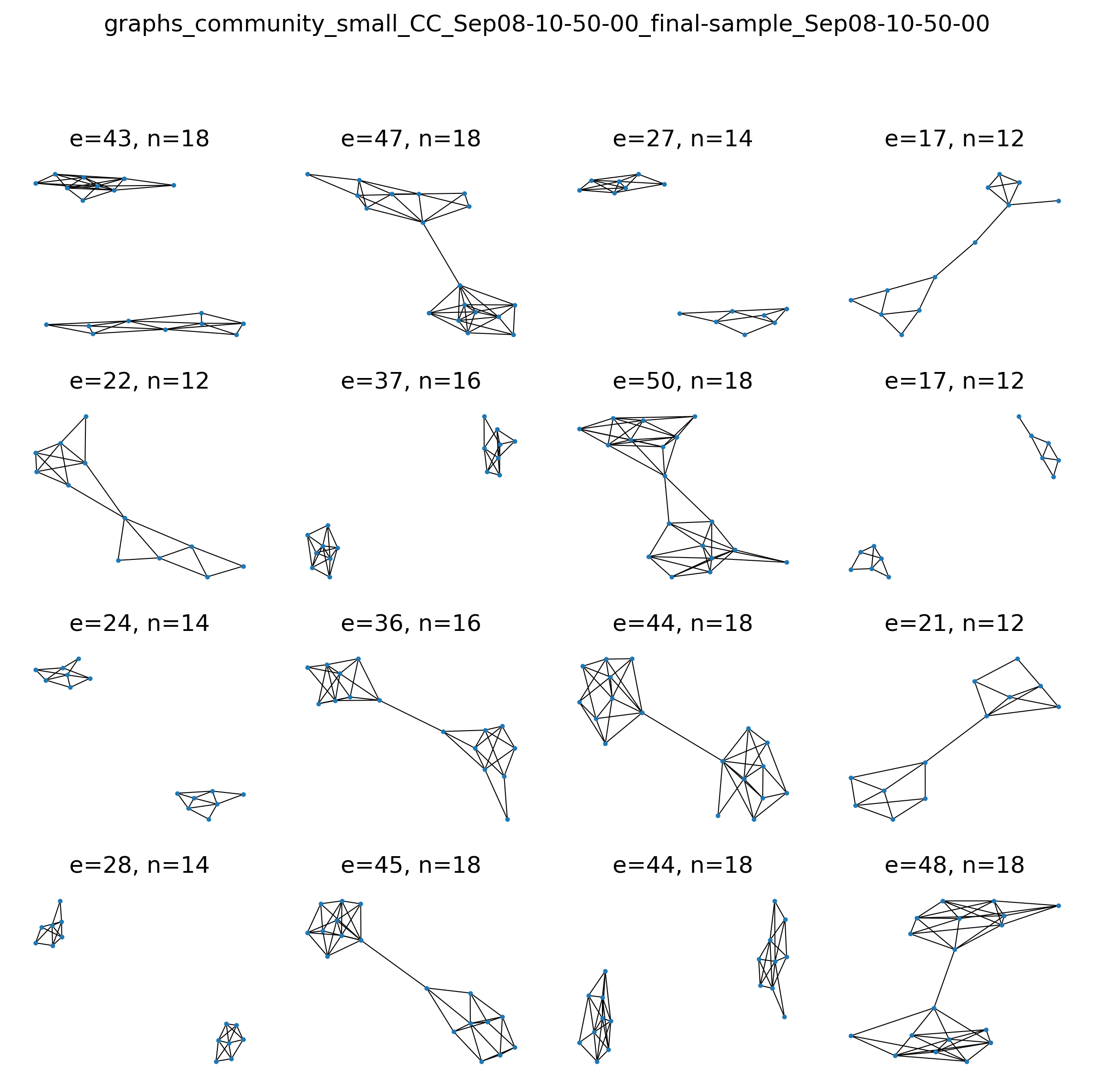}
  \captionof{figure}{Generated Community small objects using CCSD.}
  \label{fig:community_small_ccsd}
\end{minipage}
\caption{Visualization of the combinatorial complexes (represented as graphs) from the Community small dataset and the generated objects of CCSD.}
\label{fig:community_small_sample}
\end{figure}
% \captionof{figure}{Generated Community small objects using \textbf{CCSD}.}

\begin{figure}[H]
\centering
\begin{minipage}{0.5\textwidth}
  \centering
  \includegraphics[width=0.9\linewidth]{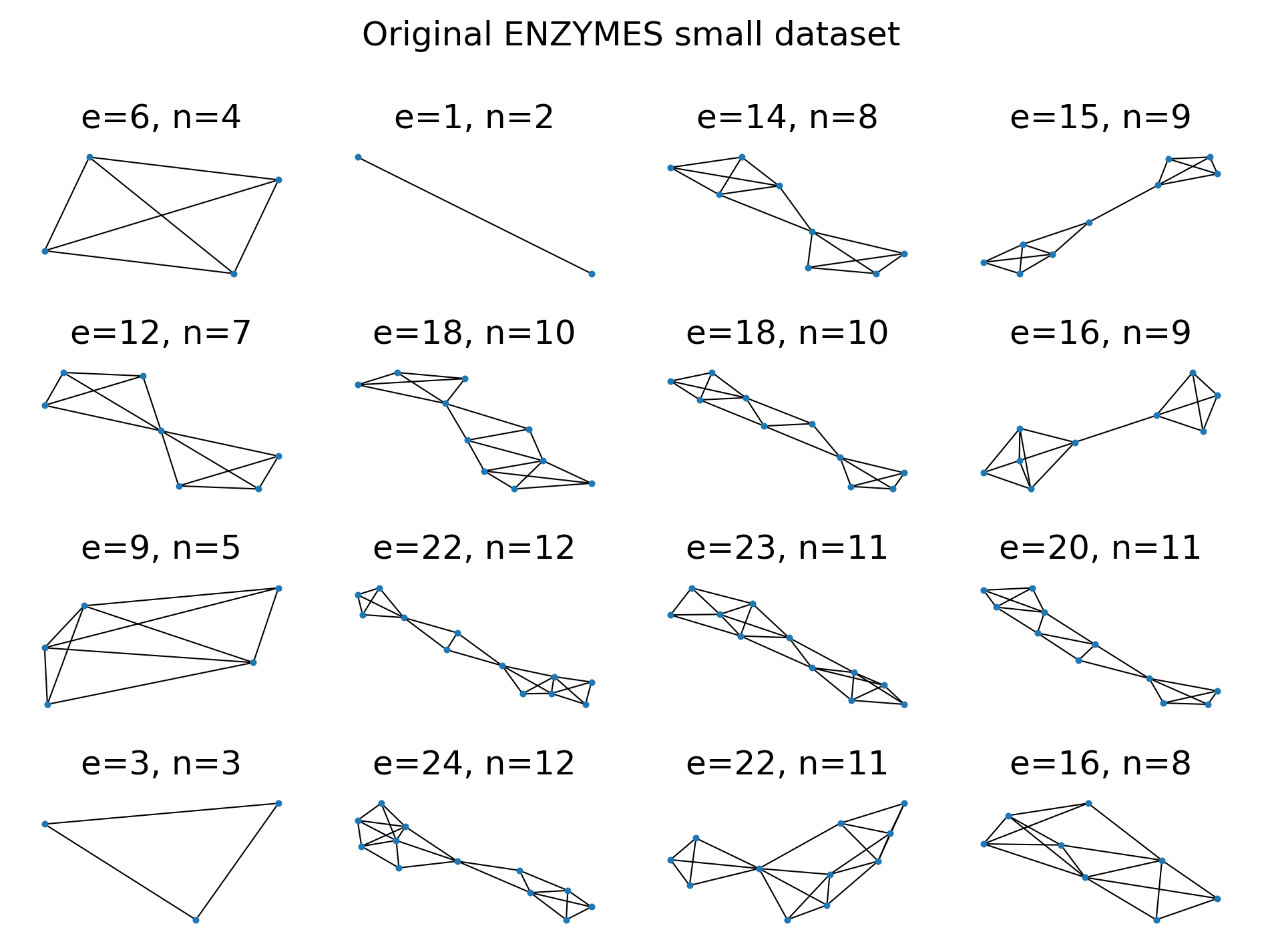}
  \captionof{figure}{Original Enzymes small graph dataset.}
  \label{fig:enzymes_small_original}
\end{minipage}%
\begin{minipage}{0.5\textwidth}
  \centering
  \includegraphics[width=0.9\linewidth]{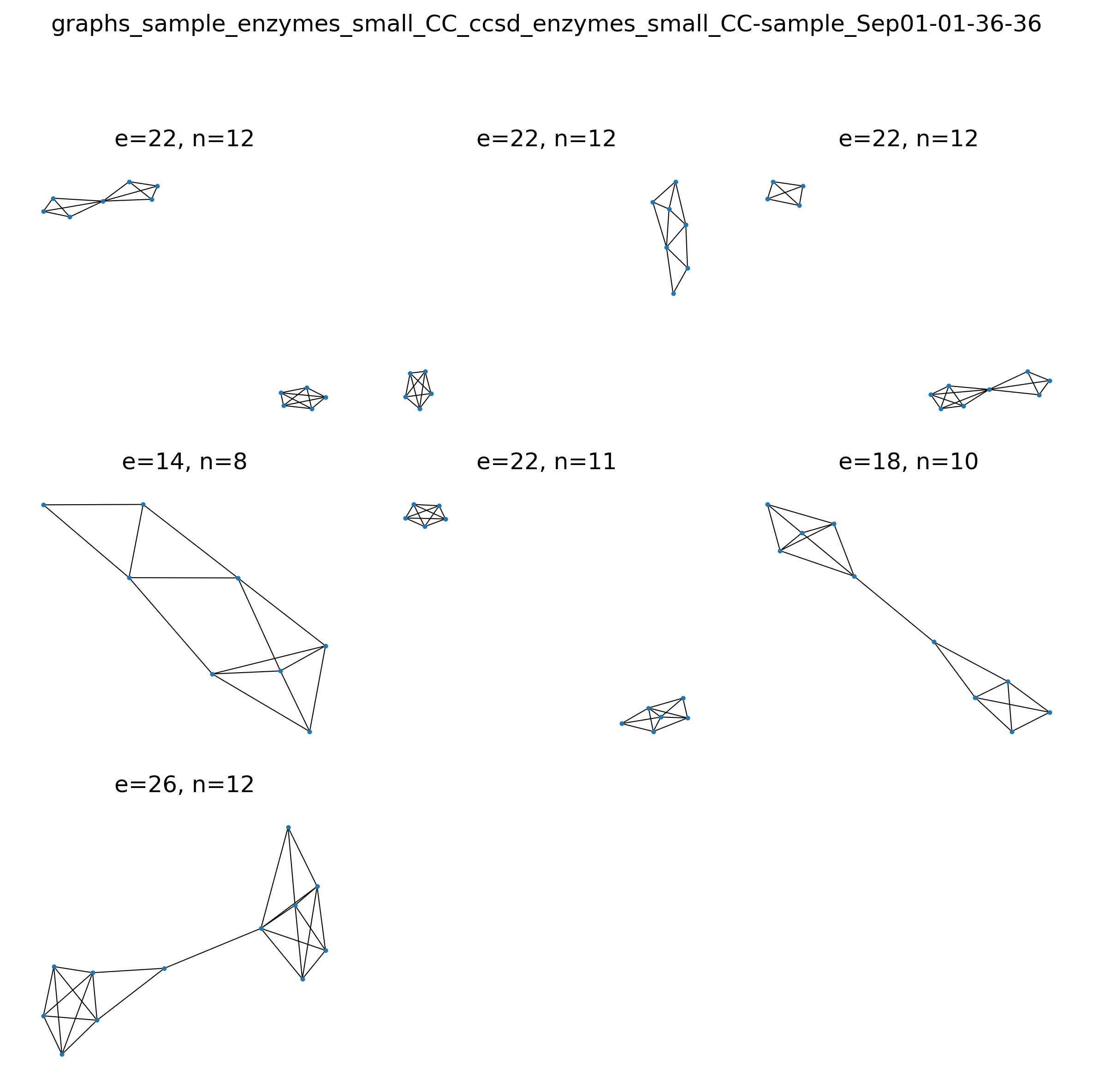}
  \captionof{figure}{Generated Enzymes small objects using CCSD.}
  \label{fig:enzymes_small_ccsd}
\end{minipage}
\caption{Visualization of the combinatorial complexes (represented as graphs) from the Enzymes small dataset and the generated objects of CCSD.}
\label{fig:enzymes_small_sample}
\end{figure}
% \captionof{figure}{Generated Enzymes small objects using \textbf{CCSD}.}

Likewise, it is clear that the graphs extracted from the generated combinatorial complexes are similar to the original graph distribution.\par

% \begin{figure}[H]
% \centering
% \begin{minipage}{0.5\textwidth}
%   \centering
%   \includegraphics[width=0.9\linewidth]{figures/original_grid_small.png}
%   \captionof{figure}{Original Grid small graph dataset.}
%   \label{fig:grid_small_original}
% \end{minipage}%
% \begin{minipage}{0.5\textwidth}
%   \centering
%   \includegraphics[width=0.9\linewidth]{figures/grid_small_sample.png}
%   \captionof{figure}{Generated Grid small objects using \textbf{CCSD}.}
%   \label{fig:grid_small_ccsd}
% \end{minipage}
% \caption{Visualization of the combinatorial complexes (represented as graphs) from the Grid small dataset and the generated objects of CCSD.}
% \label{fig:grid_small_sample}
% \end{figure}

\begin{figure}[H]
\centering
\includegraphics[width=0.6\linewidth]{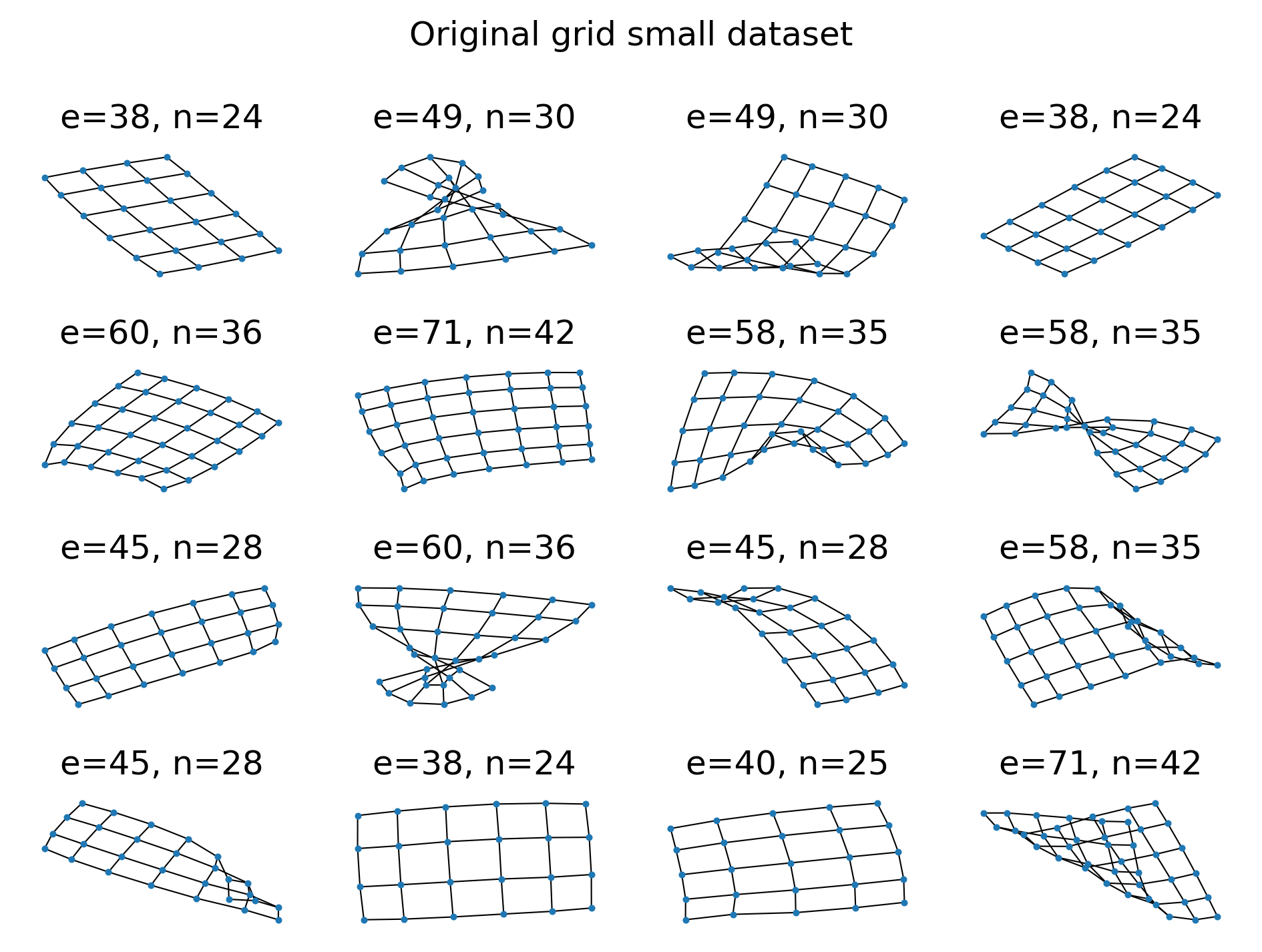}
\captionof{figure}{Original Grid small graph dataset.}
\label{fig:grid_small_original}
\end{figure}

%%%%%%%%%%%%%%%%%%%%%%%%%%%%%%%%%%%%
\chapter{Conclusion}
\label{Conclusion_chap}

In this concluding chapter, we address several key aspects of our work. We begin by providing a concise summary of our achievements in Section \ref{achievements}. Then, we present the ethical considerations that have guided our research in Section \ref{Ethical}. Following that, we recognize and discuss the limitations of our work in Section \ref{limitation}. Lastly, we offer valuable insights and directions for future research in Section \ref{future}.\par

\section{Summary of Achievements}
\label{achievements}

Our work, despite its few limitations, stands as a significant contribution to the fields of generative AI and topological deep learning. The key achievements of this endeavour are the following. Firstly, we have successfully designed and implemented CCSD, a score-based diffusion model tailored for generating combinatorial complexes. This innovative approach demonstrates promising outcomes across diverse datasets, encompassing tasks such as molecule and graph generation. Additionally, we have introduced novel objects, layers, and score neural network models that extend the horizons of generative AI by enabling the generation of more abstract structures than previously explored, thus pioneering Generative Topological Deep Learning. Furthermore, we have devised new metrics specifically designed to evaluate the quality of generated combinatorial complexes. Lastly, we have consolidated these accomplishments into a robust Python library, poised to catalyze advancements in the field.\par

\section{Ethical Considerations}
\label{Ethical}

In terms of ethical considerations, our work is conducted in alignment with established ethical guidelines and principles. It centres on the development of a diffusion model for generating combinatorial complexes and poses no direct ethical concerns. Our research solely aims at contributing to the advancement of scientific knowledge and does not involve human or animal subjects. Furthermore, the use of datasets does not raise any data-sharing issues. We have diligently credited the authors of external arguments and sources of inspiration, and we have acknowledged code that is not our own or has been inspired by others' work. Lastly, a strong emphasis has been placed on clarity, both in the thesis and the code, to facilitate reproducibility.\par

\section{Limitations}
\label{limitation}

Our work harnesses diffusion models, a potent tool that also comes with its set of limitations \cite{luo_understanding_2022}. Firstly, one cannot dismiss the fact that the assumption of iteratively denoising random noise, central to diffusion models, probably does not align with how humans model their environment or analyze it. Therefore, the approach of generating samples from noise and then denoising iteratively might not fully encapsulate the intricate structures and complexities inherent in real-world data.\par

In the context of Variational Diffusion Models, the interpretability of latent representations is also a notable limitation. Unlike Variational Autoencoders, where the encoder is optimized to yield structured latent spaces with the KL Divergence term, VDM employs predetermined linear Gaussian encoders at each time step. Consequently, intermediate latent representations in VDM remain as noisy variants of the original input, impeding the model's capacity to capture meaningful and interpretable latent structures. More specific to our score-based approach, interpreting and visualizing score functions related to graphs or other complex objects pose considerable challenges \cite{NEURIPS2019_f3a4ff48, you2018graphrnn}, let alone partial score functions. Another constraint inherent to MHVAE and, consequently, VDM, pertains to the latent space dimensionality. By confining latent representations to the original input's dimensionality, the model's potential to grasp higher-level abstractions, learn meaningful and compressed representation, and generate diverse and expressive samples may be compromised.\par

From a computational perspective, sampling from a diffusion model is expensive, thus diminishing the practicality and scalability of the proposed approach. Specifically, for graphs data, as the number of nodes grows, the cost of generating a new graph through diffusion increases dramatically as highlighted in some papers such as Vignac et al. \cite{vignac_digress_2023}. More specific to our work, the large search space associated with the higher-order incidence matrices of combinatorial complexes exacerbates this already-known limitation.\par

Speaking of applications, most of the existing graph datasets primarily focus on 2D data and/or artificial data, potentially inadequately representing meaningful tasks such as learning from molecular datasets. This dearth of representative graph generation datasets hinders the evaluation and application of diffusion models in this domain \cite{zhang_survey_2023}. Moreover, to the best of our knowledge, there is no natural combinatorial complexes dataset available, making our approach more difficult to benchmark.\par

Lastly, the intricate, irregular structures of graphs, characterized by varying numbers of nodes and edges and heterogeneous properties, present hurdles in designing effective diffusion models capable of capturing their dynamics \cite{veličković2018graph}. Accommodating the idiosyncratic structure of graphs within the diffusion process calls for innovative and complex techniques.\par

\section{Future Work}
\label{future}

Looking to the future, several avenues for further research emerge. Firstly, the exploration of new application domains that can leverage our combinatorial complex modelling framework, such as mesh generation, could lead to promising results. This includes exploring applications based on what we defined along our framework: the conditional sampling of CCs (Subsection \ref{conditional_sampling}), imputation of CCs (Subsection \ref{imputation}), and the generation with alternative training objectives through penalization/regularization (Subsection \ref{penalization}). Extensive research efforts could also be dedicated to conceiving new score network models, accompanied by rigorous mathematical investigations to ensure their permutation and equivariant properties. Furthermore, the development of additional evaluation metrics tailored for generative topological deep learning could substantially enhance the field's progress. From an implementation perspective, optimizing the pipeline to address the sparsity of incidence matrices could also greatly increase the applicability of our approach as it would increase scalability and potentially result in accelerated computations and reduced RAM utilization.\par

%%%%%%%%%%%%%%%%%%%%%%%%%%%%%%%%%%%%
%% bibliography
% \bibliography{bibliography}
\printbibliography[heading=bibintoc,title={Bibliography}]

@misc{papillon2023architectures,
      title={Architectures of Topological Deep Learning: A Survey on Topological Neural Networks}, 
      author={Mathilde Papillon and Sophia Sanborn and Mustafa Hajij and Nina Miolane},
      year={2023},
      eprint={2304.10031},
      archivePrefix={arXiv},
      primaryClass={cs.LG}
}

@article{hajij_topological_2023,
	title = {Topological {Deep} {Learning}: {Going} {Beyond} {Graph} {Data}},
	shorttitle = {Topological {Deep} {Learning}},
	author = {Hajij, Mustafa and Zamzmi, Ghada and Papamarkou, Theodore and Miolane, Nina and Guzmán-Sáenz, Aldo and Natesan Ramamurthy, Karthikeyan and Birdal, Tolga and Dey, Tamal and Mukherjee, Soham and Samaga, Shreyas and Livesay, Neal and Walters, Robin and Rosen, Paul and Schaub, Michael},
	month = apr,
	year = {2023},
}

@misc{kingma2022autoencoding,
      title={Auto-Encoding Variational Bayes}, 
      author={Diederik P Kingma and Max Welling},
      year={2022},
      eprint={1312.6114},
      archivePrefix={arXiv},
      primaryClass={stat.ML}
}

@inproceedings{NIPS2016_ddeebdee,
 author = {Kingma, Durk P and Salimans, Tim and Jozefowicz, Rafal and Chen, Xi and Sutskever, Ilya and Welling, Max},
 booktitle = {Advances in Neural Information Processing Systems},
 editor = {D. Lee and M. Sugiyama and U. Luxburg and I. Guyon and R. Garnett},
 pages = {},
 publisher = {Curran Associates, Inc.},
 title = {Improved Variational Inference with Inverse Autoregressive Flow},
 volume = {29},
 year = {2016}
}

@inproceedings{NIPS2016_6ae07dcb,
 author = {Sønderby, Casper Kaae and Raiko, Tapani and Maaløe, Lars and Sønderby, Søren Kaae and Winther, Ole},
 booktitle = {Advances in Neural Information Processing Systems},
 editor = {D. Lee and M. Sugiyama and U. Luxburg and I. Guyon and R. Garnett},
 pages = {},
 publisher = {Curran Associates, Inc.},
 title = {Ladder Variational Autoencoders},
 volume = {29},
 year = {2016}
}

@inproceedings{NEURIPS2020_4c5bcfec,
 author = {Ho, Jonathan and Jain, Ajay and Abbeel, Pieter},
 booktitle = {Advances in Neural Information Processing Systems},
 editor = {H. Larochelle and M. Ranzato and R. Hadsell and M.F. Balcan and H. Lin},
 pages = {6840--6851},
 publisher = {Curran Associates, Inc.},
 title = {Denoising Diffusion Probabilistic Models},
 volume = {33},
 year = {2020}
}

@inproceedings{NEURIPS2021_b578f2a5,
 author = {Kingma, Diederik and Salimans, Tim and Poole, Ben and Ho, Jonathan},
 booktitle = {Advances in Neural Information Processing Systems},
 editor = {M. Ranzato and A. Beygelzimer and Y. Dauphin and P.S. Liang and J. Wortman Vaughan},
 pages = {21696--21707},
 publisher = {Curran Associates, Inc.},
 title = {Variational Diffusion Models},
 volume = {34},
 year = {2021}
}

@inbook{928a56b7d6f1473e930f282a0c4b534e,
title = "A tutorial on energy-based learning",
author = "Yann Lecun and Sumit Chopra and Raia Hadsell and Ranzato, {Marc Aurelio} and Huang, {Fu Jie}",
year = "2006",
language = "English (US)",
editor = "G. Bakir and T. Hofman and B. Scholkopt and A. Smola and B. Taskar",
booktitle = "Predicting structured data",
publisher = "MIT Press",
}

@misc{song2021train,
      title={How to Train Your Energy-Based Models}, 
      author={Yang Song and Diederik P. Kingma},
      year={2021},
      eprint={2101.03288},
      archivePrefix={arXiv},
      primaryClass={cs.LG}
}

@misc{vignac_digress_2023,
	title = {{DiGress}: {Discrete} {Denoising} diffusion for graph generation},
	shorttitle = {{DiGress}},
	doi = {10.48550/arXiv.2209.14734},
	publisher = {arXiv},
	author = {Vignac, Clement and Krawczuk, Igor and Siraudin, Antoine and Wang, Bohan and Cevher, Volkan and Frossard, Pascal},
	month = feb,
	year = {2023},
}

@misc{luo_understanding_2022,
	title = {Understanding {Diffusion} {Models}: {A} {Unified} {Perspective}},
	shorttitle = {Understanding {Diffusion} {Models}},
	doi = {10.48550/arXiv.2208.11970},
	publisher = {arXiv},
	author = {Luo, Calvin},
	month = aug,
	year = {2022},
}

@misc{mcallester_mathematics_2023,
	title = {On the {Mathematics} of {Diffusion} {Models}},
	doi = {10.48550/arXiv.2301.11108},
	publisher = {arXiv},
	author = {McAllester, David},
	month = mar,
	year = {2023},
	note = {arXiv:2301.11108 [cs, math]},
}

@misc{jing_torsional_2023,
	title = {Torsional {Diffusion} for {Molecular} {Conformer} {Generation}},
	doi = {10.48550/arXiv.2206.01729},
	publisher = {arXiv},
	author = {Jing, Bowen and Corso, Gabriele and Chang, Jeffrey and Barzilay, Regina and Jaakkola, Tommi},
	month = feb,
	year = {2023},
}

@misc{velickovic_everything_2023,
	title = {Everything is {Connected}: {Graph} {Neural} {Networks}},
	shorttitle = {Everything is {Connected}},
	doi = {10.48550/arXiv.2301.08210},
	publisher = {arXiv},
	author = {Veličković, Petar},
	month = jan,
	year = {2023},
}

@article{zhang_survey_2023,
	title = {A {Survey} on {Graph} {Diffusion} {Models}: {Generative} {AI} in {Science} for {Molecule}, {Protein} and {Material}},
	shorttitle = {A {Survey} on {Graph} {Diffusion} {Models}},
	doi = {10.13140/RG.2.2.26493.64480},
	author = {Zhang, Mengchun and Qamar, Maryam and Kang, Taegoo and Jung, Yuna and Zhang, Chenshuang and Bae, Sung-Ho and Zhang, Chaoning},
	year = {2023},
	note = {arXiv:2304.01565 [cs]},
}

@misc{xu2022geodiff,
      title={GeoDiff: a Geometric Diffusion Model for Molecular Conformation Generation}, 
      author={Minkai Xu and Lantao Yu and Yang Song and Chence Shi and Stefano Ermon and Jian Tang},
      year={2022},
      eprint={2203.02923},
      archivePrefix={arXiv},
      primaryClass={cs.LG}
}

@inproceedings{NEURIPS2019_f3a4ff48,
 author = {Ingraham, John and Garg, Vikas and Barzilay, Regina and Jaakkola, Tommi},
 booktitle = {Advances in Neural Information Processing Systems},
 editor = {H. Wallach and H. Larochelle and A. Beygelzimer and F. d Alche-Buc and E. Fox and R. Garnett},
 pages = {},
 publisher = {Curran Associates, Inc.},
 title = {Generative Models for Graph-Based Protein Design},
 volume = {32},
 year = {2019}
}

@misc{you2018graphrnn,
      title={GraphRNN: Generating Realistic Graphs with Deep Auto-regressive Models}, 
      author={Jiaxuan You and Rex Ying and Xiang Ren and William L. Hamilton and Jure Leskovec},
      year={2018},
      eprint={1802.08773},
      archivePrefix={arXiv},
      primaryClass={cs.LG}
}

@misc{popova2019molecularrnn,
      title={MolecularRNN: Generating realistic molecular graphs with optimized properties}, 
      author={Mariya Popova and Mykhailo Shvets and Junier Oliva and Olexandr Isayev},
      year={2019},
      eprint={1905.13372},
      archivePrefix={arXiv},
      primaryClass={cs.LG}
}

@misc{simonovsky2018graphvae,
      title={GraphVAE: Towards Generation of Small Graphs Using Variational Autoencoders}, 
      author={Martin Simonovsky and Nikos Komodakis},
      year={2018},
      eprint={1802.03480},
      archivePrefix={arXiv},
      primaryClass={cs.LG}
}

@misc{decao2022molgan,
      title={MolGAN: An implicit generative model for small molecular graphs}, 
      author={Nicola De Cao and Thomas Kipf},
      year={2022},
      eprint={1805.11973},
      archivePrefix={arXiv},
      primaryClass={stat.ML}
}

@misc{huang2022graphgdp,
      title={GraphGDP: Generative Diffusion Processes for Permutation Invariant Graph Generation}, 
      author={Han Huang and Leilei Sun and Bowen Du and Yanjie Fu and Weifeng Lv},
      year={2022},
      eprint={2212.01842},
      archivePrefix={arXiv},
      primaryClass={cs.LG}
}

@misc{haefeli2022diffusion,
      title={Diffusion Models for Graphs Benefit From Discrete State Spaces}, 
      author={Kilian Konstantin Haefeli and Karolis Martinkus and Nathanaël Perraudin and Roger Wattenhofer},
      year={2022},
      eprint={2210.01549},
      archivePrefix={arXiv},
      primaryClass={cs.LG}
}

@misc{luo2022fast,
      title={Fast Graph Generation via Spectral Diffusion}, 
      author={Tianze Luo and Zhanfeng Mo and Sinno Jialin Pan},
      year={2022},
      eprint={2211.08892},
      archivePrefix={arXiv},
      primaryClass={cs.LG}
}

@article{walters_applications_2021,
	title = {Applications of {Deep} {Learning} in {Molecule} {Generation} and {Molecular} {Property} {Prediction}},
	volume = {54},
	issn = {0001-4842},
	url = {https://doi.org/10.1021/acs.accounts.0c00699},
	doi = {10.1021/acs.accounts.0c00699},
	number = {2},
	journal = {Accounts of Chemical Research},
	author = {Walters, W. Patrick and Barzilay, Regina},
	month = jan,
	year = {2021},
	note = {Publisher: American Chemical Society},
	pages = {263--270},
}

@misc{atz2021geometric,
      title={Geometric Deep Learning on Molecular Representations}, 
      author={Kenneth Atz and Francesca Grisoni and Gisbert Schneider},
      year={2021},
      eprint={2107.12375},
      archivePrefix={arXiv},
      primaryClass={physics.chem-ph}
}

@article{acsmolpharmaceut8b00930,
    author = {Zhavoronkov, Alex},
    title = {Artificial Intelligence for Drug Discovery, Biomarker Development, and Generation of Novel Chemistry},
    journal = {Molecular Pharmaceutics},
    volume = {15},
    number = {10},
    pages = {4311-4313},
    year = {2018},
    doi = {10.1021/acs.molpharmaceut.8b00930},
    
    URL = { 
            https://doi.org/10.1021/acs.molpharmaceut.8b00930
        
    },
    eprint = { 
            https://doi.org/10.1021/acs.molpharmaceut.8b00930
        
    }
}

@misc{bodnar2022weisfeiler,
      title={Weisfeiler and Lehman Go Cellular: CW Networks}, 
      author={Cristian Bodnar and Fabrizio Frasca and Nina Otter and Yu Guang Wang and Pietro Liò and Guido Montúfar and Michael Bronstein},
      year={2022},
      eprint={2106.12575},
      archivePrefix={arXiv},
      primaryClass={cs.LG}
}

@misc{veličković2018graph,
      title={Graph Attention Networks}, 
      author={Petar Veličković and Guillem Cucurull and Arantxa Casanova and Adriana Romero and Pietro Liò and Yoshua Bengio},
      year={2018},
      eprint={1710.10903},
      archivePrefix={arXiv},
      primaryClass={stat.ML}
}

@misc{saharia2022photorealistic,
      title={Photorealistic Text-to-Image Diffusion Models with Deep Language Understanding}, 
      author={Chitwan Saharia and William Chan and Saurabh Saxena and Lala Li and Jay Whang and Emily Denton and Seyed Kamyar Seyed Ghasemipour and Burcu Karagol Ayan and S. Sara Mahdavi and Rapha Gontijo Lopes and Tim Salimans and Jonathan Ho and David J Fleet and Mohammad Norouzi},
      year={2022},
      eprint={2205.11487},
      archivePrefix={arXiv},
      primaryClass={cs.CV}
}

@misc{ramesh2021zeroshot,
      title={Zero-Shot Text-to-Image Generation}, 
      author={Aditya Ramesh and Mikhail Pavlov and Gabriel Goh and Scott Gray and Chelsea Voss and Alec Radford and Mark Chen and Ilya Sutskever},
      year={2021},
      eprint={2102.12092},
      archivePrefix={arXiv},
      primaryClass={cs.CV}
}

@misc{rombach2022highresolution,
      title={High-Resolution Image Synthesis with Latent Diffusion Models}, 
      author={Robin Rombach and Andreas Blattmann and Dominik Lorenz and Patrick Esser and Bjorn Ommer},
      year={2022},
      eprint={2112.10752},
      archivePrefix={arXiv},
      primaryClass={cs.CV}
}

@misc{kotelnikov2022tabddpm,
      title={TabDDPM: Modelling Tabular Data with Diffusion Models}, 
      author={Akim Kotelnikov and Dmitry Baranchuk and Ivan Rubachev and Artem Babenko},
      year={2022},
      eprint={2209.15421},
      archivePrefix={arXiv},
      primaryClass={cs.LG}
}

@misc{chen2023sampling,
      title={Sampling is as easy as learning the score: theory for diffusion models with minimal data assumptions}, 
      author={Sitan Chen and Sinho Chewi and Jerry Li and Yuanzhi Li and Adil Salim and Anru R. Zhang},
      year={2023},
      eprint={2209.11215},
      archivePrefix={arXiv},
      primaryClass={cs.LG}
}

@misc{bronstein2021geometric,
      title={Geometric Deep Learning: Grids, Groups, Graphs, Geodesics, and Gauges}, 
      author={Michael M. Bronstein and Joan Bruna and Taco Cohen and Petar Veličković},
      year={2021},
      eprint={2104.13478},
      archivePrefix={arXiv},
      primaryClass={cs.LG}
}

@inproceedings{Do_2020,
	doi = {10.1145/3394486.3403060},
	url = {https://dl.acm.org/doi/10.1145/3394486.3403060},
	year = 2020,
	publisher = {{ACM}
},
	author = {Manh Tuan Do and Se-eun Yoon and Bryan Hooi and Kijung Shin},
	title = {Structural Patterns and Generative Models of Real-world Hypergraphs},
	booktitle = {Proceedings of the 26th {ACM} {SIGKDD} International Conference on Knowledge Discovery \& Data Mining}
}

@article{tweedie,
    title = {Tweedie’s Formula and Selection Bias},
    volume = {106},
    url = {http://www.jstor.org/stable/23239562},
    number = {496},
    journal = {Journal of the American Statistical Association},
    author = {Efron, B.},
    year = {2011},
    pages = {1602--1614},
}

@misc{sohldickstein2015deep,
      title={Deep Unsupervised Learning using Nonequilibrium Thermodynamics}, 
      author={Jascha Sohl-Dickstein and Eric A. Weiss and Niru Maheswaranathan and Surya Ganguli},
      year={2015},
      eprint={1503.03585},
      archivePrefix={arXiv},
      primaryClass={cs.LG}
}

@misc{ho2020denoising,
      title={Denoising Diffusion Probabilistic Models}, 
      author={Jonathan Ho and Ajay Jain and Pieter Abbeel},
      year={2020},
      eprint={2006.11239},
      archivePrefix={arXiv},
      primaryClass={cs.LG}
}

@inproceedings{NEURIPS2019_3001ef25,
 author = {Song, Yang and Ermon, Stefano},
 booktitle = {Advances in Neural Information Processing Systems},
 editor = {H. Wallach and H. Larochelle and A. Beygelzimer and F. d Alche-Buc and E. Fox and R. Garnett},
 pages = {},
 publisher = {Curran Associates, Inc.},
 title = {Generative Modeling by Estimating Gradients of the Data Distribution},
 url = {https://proceedings.neurips.cc/paper_files/paper/2019/file/3001ef257407d5a371a96dcd947c7d93-Paper.pdf},
 volume = {32},
 year = {2019}
}

@misc{niu2020permutation,
      title={Permutation Invariant Graph Generation via Score-Based Generative Modeling}, 
      author={Chenhao Niu and Yang Song and Jiaming Song and Shengjia Zhao and Aditya Grover and Stefano Ermon},
      year={2020},
      eprint={2003.00638},
      archivePrefix={arXiv},
      primaryClass={cs.LG}
}

@misc{shi2021learning,
      title={Learning Gradient Fields for Molecular Conformation Generation}, 
      author={Chence Shi and Shitong Luo and Minkai Xu and Jian Tang},
      year={2021},
      eprint={2105.03902},
      archivePrefix={arXiv},
      primaryClass={cs.LG}
}

@misc{kong2023autoregressive,
      title={Autoregressive Diffusion Model for Graph Generation}, 
      author={Lingkai Kong and Jiaming Cui and Haotian Sun and Yuchen Zhuang and B. Aditya Prakash and Chao Zhang},
      year={2023},
      eprint={2307.08849},
      archivePrefix={arXiv},
      primaryClass={cs.AI}
}

@misc{
    yang2023scorebased,
    title={Score-Based Graph Generative Modeling with Self-Guided Latent Diffusion},
    author={Ling Yang and Zhilong Zhang and Wentao Zhang and Shenda Hong},
    year={2023},
    url={https://openreview.net/forum?id=AykEgQNPJEK}
}

@misc{song2021scorebased,
      title={Score-Based Generative Modeling through Stochastic Differential Equations}, 
      author={Yang Song and Jascha Sohl-Dickstein and Diederik P. Kingma and Abhishek Kumar and Stefano Ermon and Ben Poole},
      year={2021},
      eprint={2011.13456},
      archivePrefix={arXiv},
      primaryClass={cs.LG}
}

@misc{jo2022scorebased,
      title={Score-based Generative Modeling of Graphs via the System of Stochastic Differential Equations}, 
      author={Jaehyeong Jo and Seul Lee and Sung Ju Hwang},
      year={2022},
      eprint={2202.02514},
      archivePrefix={arXiv},
      primaryClass={cs.LG}
}

@article{segre,
    author = {Segré, Thomas},
    year = {2022},
    month = {06},
    pages = {},
    title = {A summary of the major contributions in score-based generative modeling},
    doi = {10.13140/RG.2.2.27162.72649}
}

@article{Thongprayoon_2023,
	doi = {10.1109/access.2023.3268030},
  
	url = {https://doi.org/10.1109%2Faccess.2023.3268030},
  
	year = 2023,
	publisher = {Institute of Electrical and Electronics Engineers ({IEEE})},
  
	volume = {11},
  
	pages = {41426--41443},
  
	author = {Chanon Thongprayoon and Lorenzo Livi and Naoki Masuda},
  
	title = {Embedding and Trajectories of Temporal Networks},
  
	journal = {{IEEE} Access}
}

@misc{kipf2017semisupervised,
      title={Semi-Supervised Classification with Graph Convolutional Networks}, 
      author={Thomas N. Kipf and Max Welling},
      year={2017},
      eprint={1609.02907},
      archivePrefix={arXiv},
      primaryClass={cs.LG}
}

@article{Hinton1989ConnectionistLP,
  title={Connectionist Learning Procedures},
  author={Geoffrey E. Hinton},
  journal={Artif. Intell.},
  year={1989},
  volume={40},
  pages={185-234},
  url={https://api.semanticscholar.org/CorpusID:7840452}
}

@article{Sanger1958ThePA,
  title={The perceptron: a probabilistic model for information storage and organization in the brain.},
  author={Terence Sanger and Pallavi N. Baljekar},
  journal={Psychological review},
  year={1958},
  volume={65 6},
  pages={
          386-408
        },
  url={https://api.semanticscholar.org/CorpusID:12781225}
}

@misc{ioffe2015batch,
      title={Batch Normalization: Accelerating Deep Network Training by Reducing Internal Covariate Shift}, 
      author={Sergey Ioffe and Christian Szegedy},
      year={2015},
      eprint={1502.03167},
      archivePrefix={arXiv},
      primaryClass={cs.LG}
}

@misc{baek2021accurate,
      title={Accurate Learning of Graph Representations with Graph Multiset Pooling}, 
      author={Jinheon Baek and Minki Kang and Sung Ju Hwang},
      year={2021},
      eprint={2102.11533},
      archivePrefix={arXiv},
      primaryClass={cs.LG}
}

@article{doi:10.1137/18M1223101,
    author = {Lim, Lek-Heng},
    title = {Hodge Laplacians on Graphs},
    journal = {SIAM Review},
    volume = {62},
    number = {3},
    pages = {685-715},
    year = {2020},
    doi = {10.1137/18M1223101},
    
    URL = { 
        
            https://doi.org/10.1137/18M1223101
        
        
    
    },
    eprint = { 
        
            https://doi.org/10.1137/18M1223101
        
        
    
    }
    ,
}

@misc{shi2020graphaf,
      title={GraphAF: a Flow-based Autoregressive Model for Molecular Graph Generation}, 
      author={Chence Shi and Minkai Xu and Zhaocheng Zhu and Weinan Zhang and Ming Zhang and Jian Tang},
      year={2020},
      eprint={2001.09382},
      archivePrefix={arXiv},
      primaryClass={cs.LG}
}

@misc{luo2021graphdf,
      title={GraphDF: A Discrete Flow Model for Molecular Graph Generation}, 
      author={Youzhi Luo and Keqiang Yan and Shuiwang Ji},
      year={2021},
      eprint={2102.01189},
      archivePrefix={arXiv},
      primaryClass={cs.LG}
}

@inproceedings{Zang_2020,
	doi = {10.1145/3394486.3403104},
	url = {https://doi.org/10.1145%2F3394486.3403104},
	year = 2020,
	publisher = {{ACM}
},
	author = {Chengxi Zang and Fei Wang},
	title = {{MoFlow}: An Invertible Flow Model for Generating Molecular Graphs},
	booktitle = {Proceedings of the 26th {ACM} {SIGKDD} International Conference on Knowledge Discovery Data Mining}
}

@misc{liu2021graphebm,
      title={GraphEBM: Molecular Graph Generation with Energy-Based Models}, 
      author={Meng Liu and Keqiang Yan and Bora Oztekin and Shuiwang Ji},
      year={2021},
      eprint={2102.00546},
      archivePrefix={arXiv},
      primaryClass={cs.LG}
}

@misc{paszke2019pytorch,
      title={PyTorch: An Imperative Style, High-Performance Deep Learning Library}, 
      author={Adam Paszke and Sam Gross and Francisco Massa and Adam Lerer and James Bradbury and Gregory Chanan and Trevor Killeen and Zeming Lin and Natalia Gimelshein and Luca Antiga and Alban Desmaison and Andreas Köpf and Edward Yang and Zach DeVito and Martin Raison and Alykhan Tejani and Sasank Chilamkurthy and Benoit Steiner and Lu Fang and Junjie Bai and Soumith Chintala},
      year={2019},
      eprint={1912.01703},
      archivePrefix={arXiv},
      primaryClass={cs.LG}
}

@article{ramakrishnan_quantum_2014,
	title = {Quantum chemistry structures and properties of 134 kilo molecules},
	volume = {1},
	issn = {2052-4463},
	url = {https://doi.org/10.1038/sdata.2014.22},
	doi = {10.1038/sdata.2014.22},
	pages = {140022},
	number = {1},
	journaltitle = {Scientific Data},
	shortjournal = {Scientific Data},
	author = {Ramakrishnan, Raghunathan and Dral, Pavlo O. and Rupp, Matthias and von Lilienfeld, O. Anatole},
	date = {2014-08-05},
}

@misc{preuer2018frechet,
      title={Frechet ChemNet Distance: A metric for generative models for molecules in drug discovery}, 
      author={Kristina Preuer and Philipp Renz and Thomas Unterthiner and Sepp Hochreiter and Günter Klambauer},
      year={2018},
      eprint={1803.09518},
      archivePrefix={arXiv},
      primaryClass={cs.LG}
}

@article{Landrum2016RDKit2016_09_4,
  author = {Landrum, Greg},
  title = {RDKit: Open-Source Cheminformatics Software},
  url = {https://github.com/rdkit/rdkit/releases/tag/Release_2016_09_4},
  year = 2016
}

@inproceedings{Costa2010FastNS,
  title={Fast Neighborhood Subgraph Pairwise Distance Kernel},
  author={Fabrizio Costa and Kurt De Grave},
  booktitle={International Conference on Machine Learning},
  year={2010},
  url={https://api.semanticscholar.org/CorpusID:16262476}
}

@article{anderson,
  author={Anderson, Brian D.O.},
  title={{Reverse-time diffusion equation models}},
  journal={Stochastic Processes and their Applications},
  year=1982,
  volume={12},
  number={3},
  pages={313-326},
  doi={},
  url={https://ideas.repec.org/a/eee/spapps/v12y1982i3p313-326.html}
}

@misc{keriven2019universal,
      title={Universal Invariant and Equivariant Graph Neural Networks}, 
      author={Nicolas Keriven and Gabriel Peyré},
      year={2019},
      eprint={1905.04943},
      archivePrefix={arXiv},
      primaryClass={cs.LG}
}

@article{sen_collective_2008,
	title = {Collective Classification in Network Data},
	volume = {29},
	url = {https://ojs.aaai.org/aimagazine/index.php/aimagazine/article/view/2157},
	doi = {10.1609/aimag.v29i3.2157},
	pages = {93},
	number = {3},
	journaltitle = {{AI} Magazine},
	author = {Sen, Prithviraj and Namata, Galileo and Bilgic, Mustafa and Getoor, Lise and Galligher, Brian and Eliassi-Rad, Tina},
	date = {2008-09},
}

@article{schomburg,
author = {Schomburg, Ida and Jäde, Antje and Ebeling, Christian and Gremse, Marion and Heldt, Christian and Huhn, Gregor and Schomburg, Dietmar},
year = {2004},
month = {01},
pages = {D431-3},
title = {BRENDA, the enzyme database: Updates and major new developments},
volume = {32},
journal = {Nucleic acids research},
doi = {10.1093/nar/gkh081}
}

@techreport{hagberg2008exploring,
  title={Exploring network structure, dynamics, and function using NetworkX},
  author={Hagberg, Aric and Swart, Pieter and S Chult, Daniel},
  year={2008},
  institution={Los Alamos National Lab.(LANL), Los Alamos, NM (United States)}
}

@misc{liao2020efficient,
      title={Efficient Graph Generation with Graph Recurrent Attention Networks}, 
      author={Renjie Liao and Yujia Li and Yang Song and Shenlong Wang and Charlie Nash and William L. Hamilton and David Duvenaud and Raquel Urtasun and Richard S. Zemel},
      year={2020},
      eprint={1910.00760},
      archivePrefix={arXiv},
      primaryClass={cs.LG}
}

@misc{obray2022evaluation,
      title={Evaluation Metrics for Graph Generative Models: Problems, Pitfalls, and Practical Solutions}, 
      author={Leslie O'Bray and Max Horn and Bastian Rieck and Karsten Borgwardt},
      year={2022},
      eprint={2106.01098},
      archivePrefix={arXiv},
      primaryClass={cs.LG}
}

@misc{li2018learning,
      title={Learning Deep Generative Models of Graphs}, 
      author={Yujia Li and Oriol Vinyals and Chris Dyer and Razvan Pascanu and Peter Battaglia},
      year={2018},
      eprint={1803.03324},
      archivePrefix={arXiv},
      primaryClass={cs.LG}
}

@inproceedings{NEURIPS2019_1e44fdf9,
 author = {Liu, Jenny and Kumar, Aviral and Ba, Jimmy and Kiros, Jamie and Swersky, Kevin},
 booktitle = {Advances in Neural Information Processing Systems},
 editor = {H. Wallach and H. Larochelle and A. Beygelzimer and F. d Alche-Buc and E. Fox and R. Garnett},
 pages = {},
 publisher = {Curran Associates, Inc.},
 title = {Graph Normalizing Flows},
 volume = {32},
 year = {2019}
}

@misc{li2018multiobjective,
      title={Multi-Objective De Novo Drug Design with Conditional Graph Generative Model}, 
      author={Yibo Li and Liangren Zhang and Zhenming Liu},
      year={2018},
      eprint={1801.07299},
      archivePrefix={arXiv},
      primaryClass={q-bio.QM}
}

@inproceedings{NEURIPS2020_92c3b916,
 author = {Song, Yang and Ermon, Stefano},
 booktitle = {Advances in Neural Information Processing Systems},
 editor = {H. Larochelle and M. Ranzato and R. Hadsell and M.F. Balcan and H. Lin},
 pages = {12438--12448},
 publisher = {Curran Associates, Inc.},
 title = {Improved Techniques for Training Score-Based Generative Models},
 volume = {33},
 year = {2020}
}

@book{särkkä2019applied,
  title={Applied Stochastic Differential Equations},
  author={Sarkka, S. and Solin, A.},
  isbn={9781316510087},
  lccn={2018026584},
  series={Institute of Mathematical Statistics Textbooks},
  year={2019},
  publisher={Cambridge University Press}
}

@misc{doecode_22160,
title = {pnnl/HyperNetX},
author = {Praggastis, Brenda and Arendt, Dustin and Joslyn, Cliff and Purvine, Emilie and Aksoy, Sinan and Monson, Kyle},
abstractNote = {HyperNetX is a Python library of classes, algorithms, and generators for working with multi-way and nested relational data. Given a network of nodes and relationships, the base class HyperNetX.Entity is used to instantiate individual data elements along with their relationships to other elements of the network. By exploiting Python's dictionary data structure we capture both multi-way and nested data structures with constant time lookup. The base class HyperNetX.Hypergraph instantiates a network of this data into a hypergraph structure. Many of the common metrics and features of graphs, such as diameter, connectedness, and clustering coefficient, have analogous definitions with respect to hypergraphs, which are computed using algorithms in the HyperNetX library. },
}

@article{Schaub_2021,
	doi = {10.1016/j.sigpro.2021.108149},
	url = {https://doi.org/10.1016%2Fj.sigpro.2021.108149},
	year = 2021,
	publisher = {Elsevier {BV}
},
	volume = {187},
	pages = {108149},
	author = {Michael T. Schaub and Yu Zhu and Jean-Baptiste Seby and T. Mitchell Roddenberry and Santiago Segarra},
	title = {Signal processing on higher-order networks: Livin' on the edge... and beyond},
	journal = {Signal Processing}
}

@article{Barbarossa_2020,
	doi = {10.1109/tsp.2020.2981920},
	url = {https://doi.org/10.1109%2Ftsp.2020.2981920},
	year = 2020,
	publisher = {Institute of Electrical and Electronics Engineers ({IEEE})},
	volume = {68},
	pages = {2992--3007},
	author = {Sergio Barbarossa and Stefania Sardellitti},
	title = {Topological Signal Processing Over Simplicial Complexes},
	journal = {{IEEE} Transactions on Signal Processing}
}

@misc{cai2020learning,
      title={Learning Gradient Fields for Shape Generation}, 
      author={Ruojin Cai and Guandao Yang and Hadar Averbuch-Elor and Zekun Hao and Serge Belongie and Noah Snavely and Bharath Hariharan},
      year={2020},
      eprint={2008.06520},
      archivePrefix={arXiv},
      primaryClass={cs.CV}
}

@misc{chen2020wavegrad,
      title={WaveGrad: Estimating Gradients for Waveform Generation}, 
      author={Nanxin Chen and Yu Zhang and Heiga Zen and Ron J. Weiss and Mohammad Norouzi and William Chan},
      year={2020},
      eprint={2009.00713},
      archivePrefix={arXiv},
      primaryClass={eess.AS}
}

@misc{kong2021diffwave,
      title={DiffWave: A Versatile Diffusion Model for Audio Synthesis}, 
      author={Zhifeng Kong and Wei Ping and Jiaji Huang and Kexin Zhao and Bryan Catanzaro},
      year={2021},
      eprint={2009.09761},
      archivePrefix={arXiv},
      primaryClass={eess.AS}
}

@inproceedings{NEURIPS2019_378a063b,
 author = {Du, Yilun and Mordatch, Igor},
 booktitle = {Advances in Neural Information Processing Systems},
 editor = {H. Wallach and H. Larochelle and A. Beygelzimer and F. d Alche-Buc and E. Fox and R. Garnett},
 pages = {},
 publisher = {Curran Associates, Inc.},
 title = {Implicit Generation and Modeling with Energy Based Models},
 volume = {32},
 year = {2019}
}

@misc{bordes2017learning,
      title={Learning to Generate Samples from Noise through Infusion Training}, 
      author={Florian Bordes and Sina Honari and Pascal Vincent},
      year={2017},
      eprint={1703.06975},
      archivePrefix={arXiv},
      primaryClass={stat.ML}
}

@misc{goyal2017variational,
      title={Variational Walkback: Learning a Transition Operator as a Stochastic Recurrent Net}, 
      author={Anirudh Goyal and Nan Rosemary Ke and Surya Ganguli and Yoshua Bengio},
      year={2017},
      eprint={1711.02282},
      archivePrefix={arXiv},
      primaryClass={stat.ML}
}

@software{Carrel_CCSD_-_Combinatorial_2023,
author = {Carrel, Adrien},
month = jul,
title = {{CCSD - Combinatorial Complex Score-based Diffusion model using stochastic differential equations.}},
url = {https://github.com/AdrienC21/CCSD},
version = {1.0.0},
year = {2023}
}

@article{irwin_zinc_2012,
	title = {{ZINC}: A Free Tool to Discover Chemistry for Biology},
	volume = {52},
	issn = {1549-9596},
	url = {https://doi.org/10.1021/ci3001277},
	doi = {10.1021/ci3001277},
	pages = {1757--1768},
	number = {7},
	journaltitle = {Journal of Chemical Information and Modeling},
	shortjournal = {J. Chem. Inf. Model.},
	author = {Irwin, John J. and Sterling, Teague and Mysinger, Michael M. and Bolstad, Erin S. and Coleman, Ryan G.},
	date = {2012-07-23},
	note = {Publisher: American Chemical Society},
}

@article{Hocevar2014ACA,
  title={A combinatorial approach to graphlet counting},
  author={Tomaz Hocevar and Janez Demsar},
  journal={Bioinformatics},
  year={2014},
  volume={30 4},
  pages={
          559-65
        },
  url={https://api.semanticscholar.org/CorpusID:33092354}
}

@article{JMLR:v6:hyvarinen05a,
  author  = {Aapo Hyvarinen},
  title   = {Estimation of Non-Normalized Statistical Models by Score Matching},
  journal = {Journal of Machine Learning Research},
  year    = {2005},
  volume  = {6},
  number  = {24},
  pages   = {695--709},
  url     = {http://jmlr.org/papers/v6/hyvarinen05a.html}
}

@article{Vincent2011ACB,
  title={A Connection Between Score Matching and Denoising Autoencoders},
  author={Pascal Vincent},
  journal={Neural Computation},
  year={2011},
  volume={23},
  pages={1661-1674},
  url={https://api.semanticscholar.org/CorpusID:5560643}
}

@InProceedings{pmlr-v115-song20a,
  title = 	 {Sliced Score Matching: A Scalable Approach to Density and Score Estimation},
  author =       {Song, Yang and Garg, Sahaj and Shi, Jiaxin and Ermon, Stefano},
  booktitle = 	 {Proceedings of The 35th Uncertainty in Artificial Intelligence Conference},
  pages = 	 {574--584},
  year = 	 {2020},
  editor = 	 {Adams, Ryan P. and Gogate, Vibhav},
  volume = 	 {115},
  series = 	 {Proceedings of Machine Learning Research},
  publisher =    {PMLR},
  url = 	 {https://proceedings.mlr.press/v115/song20a.html}
}

@misc{pang2020efficient,
      title={Efficient Learning of Generative Models via Finite-Difference Score Matching}, 
      author={Tianyu Pang and Kun Xu and Chongxuan Li and Yang Song and Stefano Ermon and Jun Zhu},
      year={2020},
      eprint={2007.03317},
      archivePrefix={arXiv},
      primaryClass={cs.LG}
}

@misc{chen2019neural,
      title={Neural Ordinary Differential Equations}, 
      author={Ricky T. Q. Chen and Yulia Rubanova and Jesse Bettencourt and David Duvenaud},
      year={2019},
      eprint={1806.07366},
      archivePrefix={arXiv},
      primaryClass={cs.LG}
}

@misc{roeder2020linear,
      title={On Linear Identifiability of Learned Representations}, 
      author={Geoffrey Roeder and Luke Metz and Diederik P. Kingma},
      year={2020},
      eprint={2007.00810},
      archivePrefix={arXiv},
      primaryClass={stat.ML}
}

@book{van1995python,
  title={Python reference manual},
  author={Van Rossum, Guido and Drake Jr, Fred L},
  year={1995},
  publisher={Centrum voor Wiskunde en Informatica Amsterdam}
}

@Book{ISO:1998:IIP,
  author =       "{ISO}",
  title =        "{ISO IEC 14882:1998}: {Programming} languages
                 --- {C++}",
  publisher =    pub-ISO,
  address =      pub-ISO:adr,
  pages =        "732",
  day =          "1",
  month =        sep,
  year =         "1998",
  bibdate =      "Tue Dec 12 06:45:55 2000",
  acknowledgement = ack-nhfb,
}

@article{tanimoto,
author = {Bajusz, Dávid and Rácz, Anita and Héberger, Károly},
year = {2015},
month = {05},
pages = {},
title = {Why is Tanimoto index an appropriate choice for fingerprint-based similarity calculations?},
volume = {7},
journal = {Journal of Cheminformatics},
doi = {10.1186/s13321-015-0069-3}
}

@article{morgan_generation_1965,
	title = {The Generation of a Unique Machine Description for Chemical Structures-A Technique Developed at Chemical Abstracts Service.},
	volume = {5},
	issn = {0021-9576},
	url = {https://doi.org/10.1021/c160017a018},
	doi = {10.1021/c160017a018},
	pages = {107--113},
	number = {2},
	journaltitle = {Journal of Chemical Documentation},
	shortjournal = {J. Chem. Doc.},
	author = {Morgan, H. L.},
	date = {1965-05-01},
	note = {Publisher: American Chemical Society},
}

@software{Carrel_Optimal_transport_applied_2019,
author = {Carrel, Adrien},
month = aug,
title = {{Optimal transport applied to color transportation in image processing.}},
url = {https://github.com/AdrienC21/optimal-transport-color-transportation},
version = {1.0.0},
year = {2019}
}

@book{monge_memoire_1781,
	location = {Paris},
	title = {Mémoire sur la théorie des déblais et des remblais},
	pagetotal = {666},
	publisher = {De l'Imprimerie Royale},
	author = {Monge, Gaspard},
	date = {1781},
	note = {{OCLC}: 51928110},
}

@book{oksendal_stochastic_2003,
	location = {Berlin, Heidelberg},
	title = {Stochastic Differential Equations},
	url = {http://link.springer.com/10.1007/978-3-642-14394-6},
	series = {Universitext},
	publisher = {Springer},
	author = {Øksendal, Bernt},
	urldate = {2023-09-10},
	date = {2003},
	doi = {10.1007/978-3-642-14394-6},}

@Inbook{Revuz1999,
      author="Revuz, Daniel
      and Yor, Marc",
      title="Martingales",
      bookTitle="Continuous Martingales and Brownian Motion",
      year="1999",
      publisher="Springer Berlin Heidelberg",
      address="Berlin, Heidelberg",
      pages="51--77",
      doi="10.1007/978-3-662-06400-9_3",
      url="https://doi.org/10.1007/978-3-662-06400-9_3"
}

@article{hovcevar2014combinatorial,
  title={A combinatorial approach to graphlet counting},
  author={Hosevar, Tomaz and Demsar, Janez},
  journal={Bioinformatics},
  volume={30},
  number={4},
  pages={559--565},
  year={2014},
  publisher={Oxford University Press}
}

@article{10.1093/bioinformatics/btx758,
    author = {Melckenbeeck, Ine and Audenaert, Pieter and Colle, Didier and Pickavet, Mario},
    title = "{Efficiently counting all orbits of graphlets of any order in a graph using autogenerated equations}",
    journal = {Bioinformatics},
    volume = {34},
    number = {8},
    pages = {1372-1380},
    year = {2017},
    month = {11},
    issn = {1367-4803},
    doi = {10.1093/bioinformatics/btx758},
    url = {https://doi.org/10.1093/bioinformatics/btx758},
    eprint = {https://academic.oup.com/bioinformatics/article-pdf/34/8/1372/48915964/bioinformatics\_34\_8\_1372.pdf},
}

@article{prvzulj2004modeling,
  title={Modeling interactome: scale-free or geometric?},
  author={Przulj, Natasa and Corneil, Derek G and Jurisica, Igor},
  journal={Bioinformatics},
  volume={20},
  number={18},
  pages={3508--3515},
  year={2004},
  publisher={Oxford University Press}
}

%%%%%%%%%%%%%%%%%%%%%%%%%%%%%%%%%%%%
\appendix

\chapter{Appendix}

% \textcolor{red}{TODO: OPTIONAL animation}

\section{Learning curves}
\label{learning_curves}

This section includes the learning curves we obtained for each dataset, presenting the training and testing losses for each partial score function. As a reminder, the loss is measured as the $L2$ distance between the predicted score and the actual score.

\subsection{QM9}

\begin{figure}[H]
\centering
\begin{minipage}{0.33\textwidth}
  \centering
  \includegraphics[width=1\linewidth]{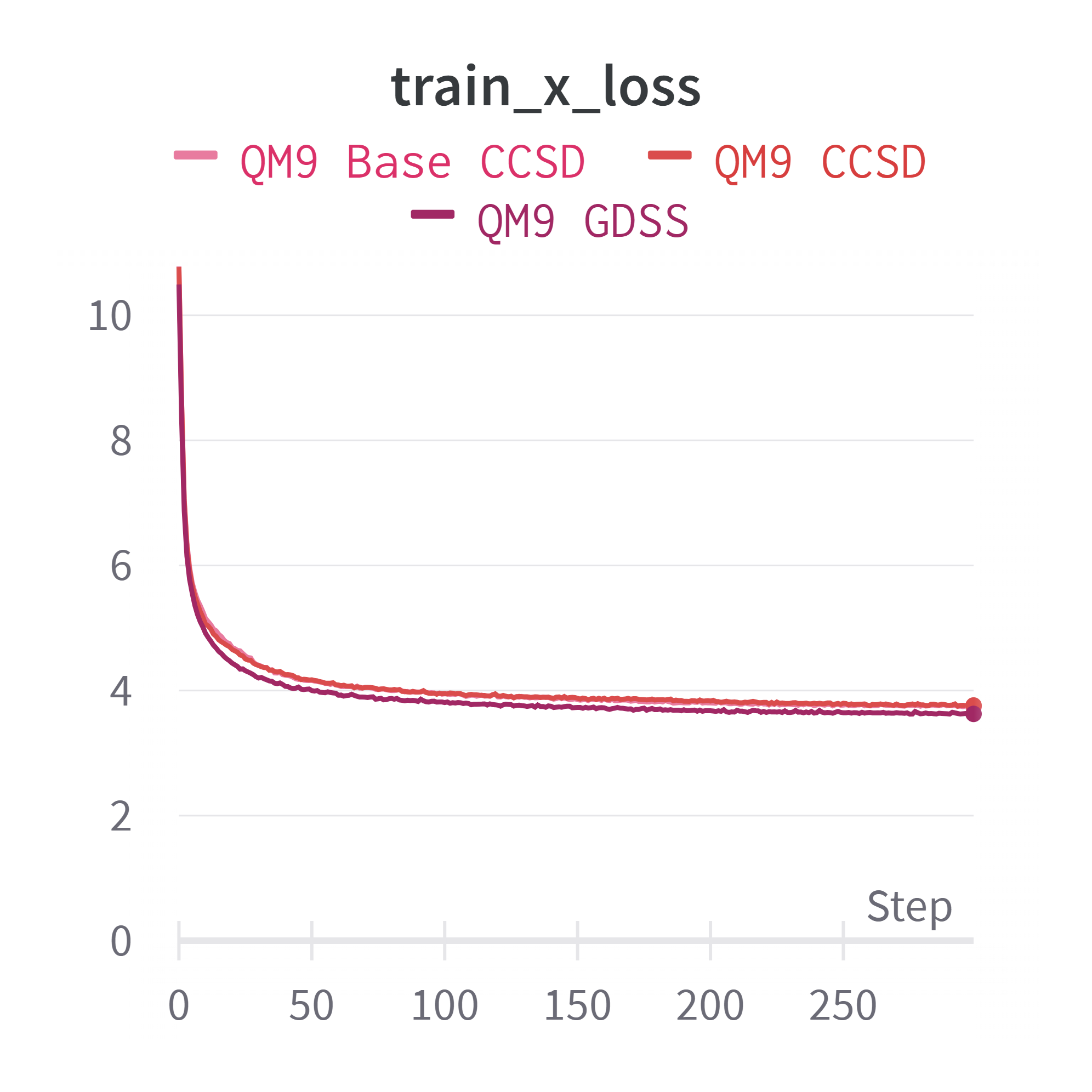}
\end{minipage}%
\begin{minipage}{0.33\textwidth}
  \centering
  \includegraphics[width=1\linewidth]{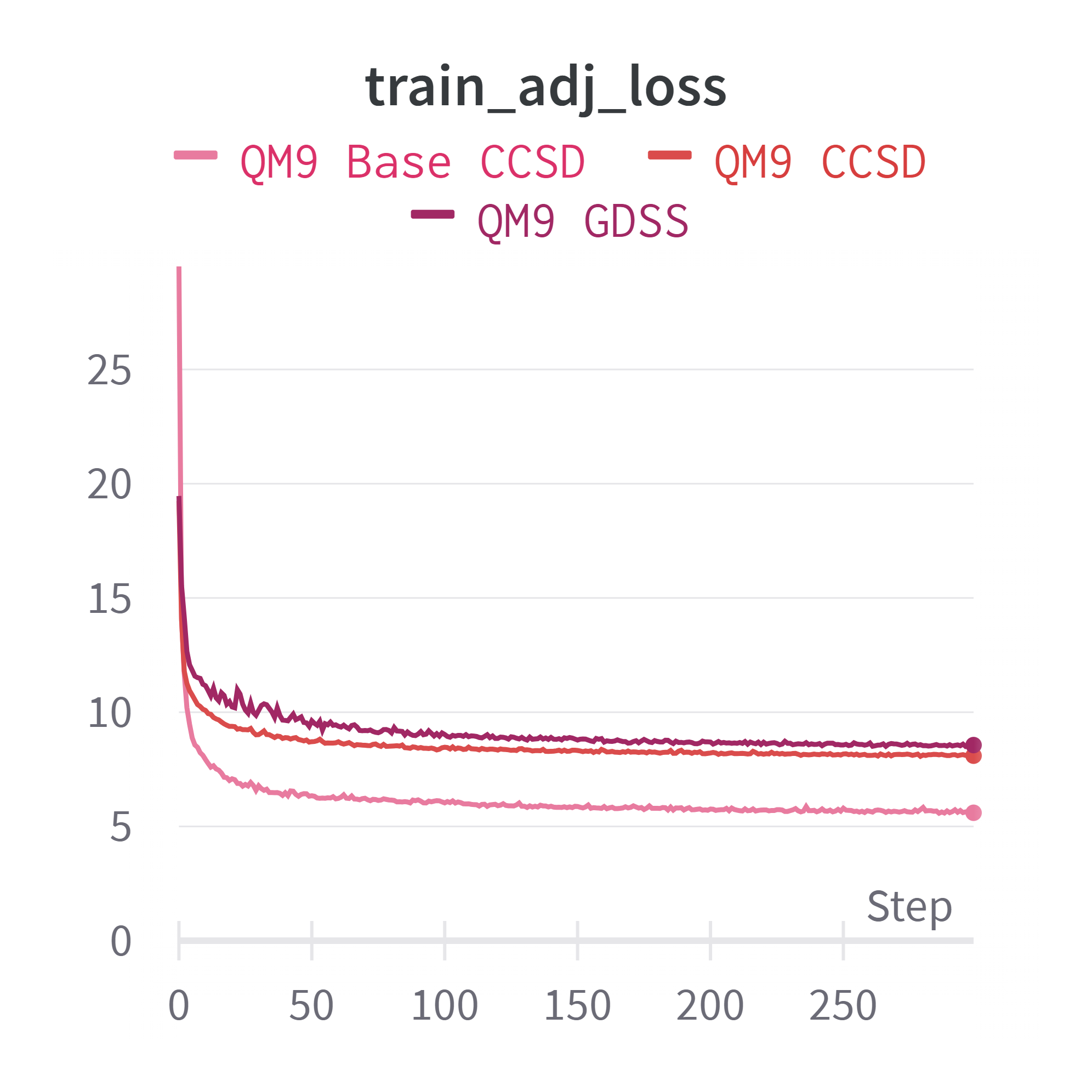}
\end{minipage}%
\begin{minipage}{0.33\textwidth}
  \centering
  \includegraphics[width=1\linewidth]{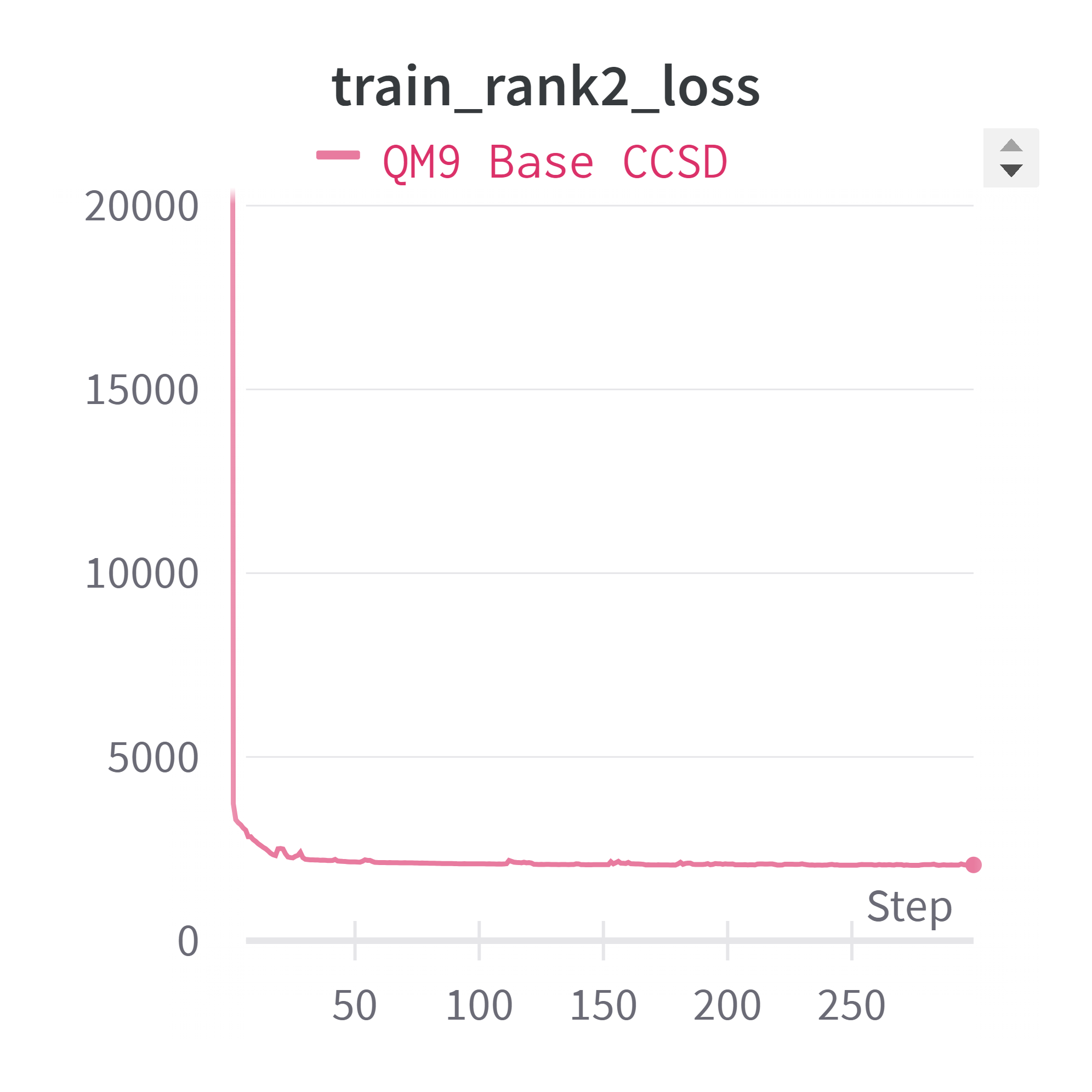}
\end{minipage}
\captionof{figure}{Train losses QM9.}
\label{fig:train_qm9}
\end{figure}

\begin{figure}[H]
\centering
\begin{minipage}{0.33\textwidth}
  \centering
  \includegraphics[width=1\linewidth]{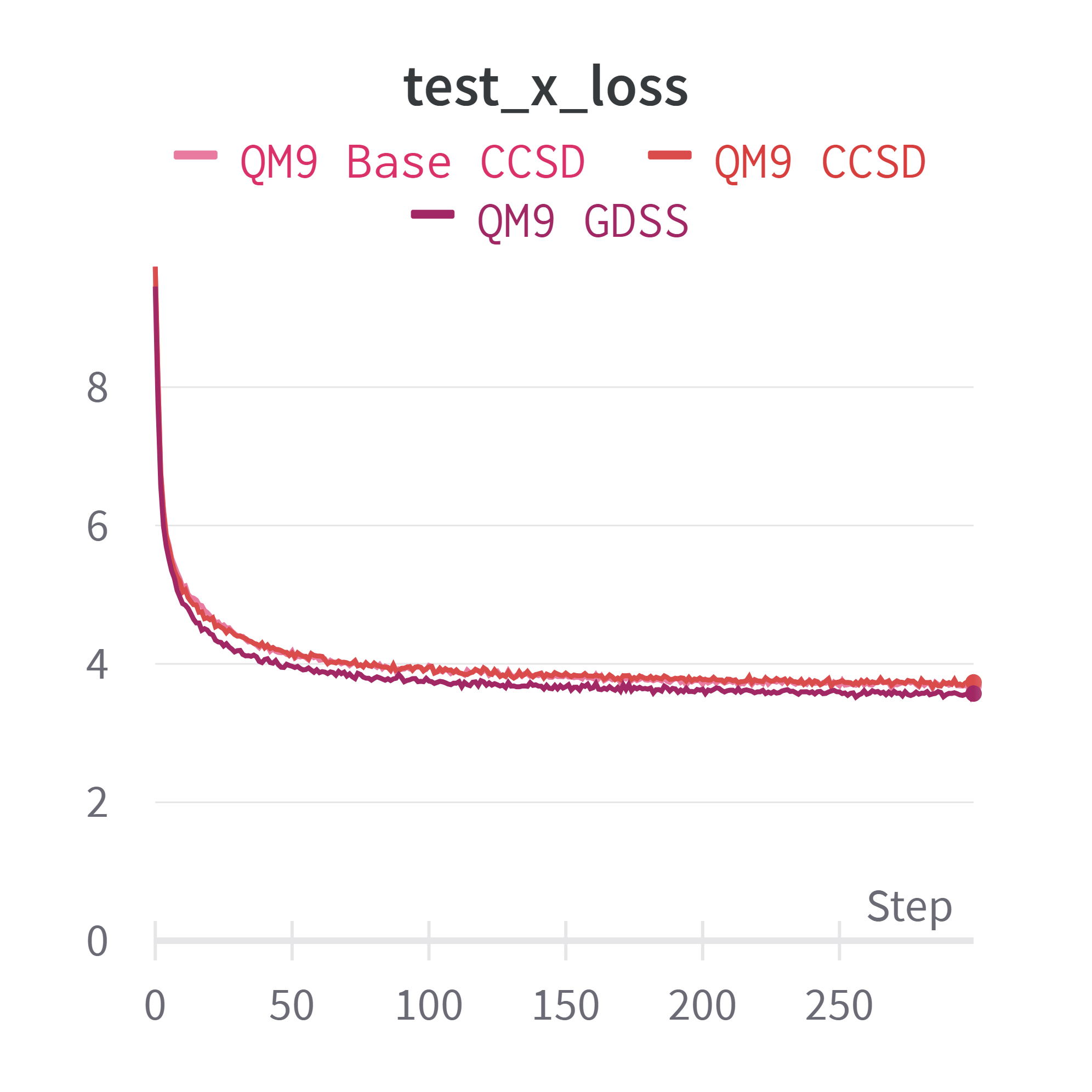}
\end{minipage}%
\begin{minipage}{0.33\textwidth}
  \centering
  \includegraphics[width=1\linewidth]{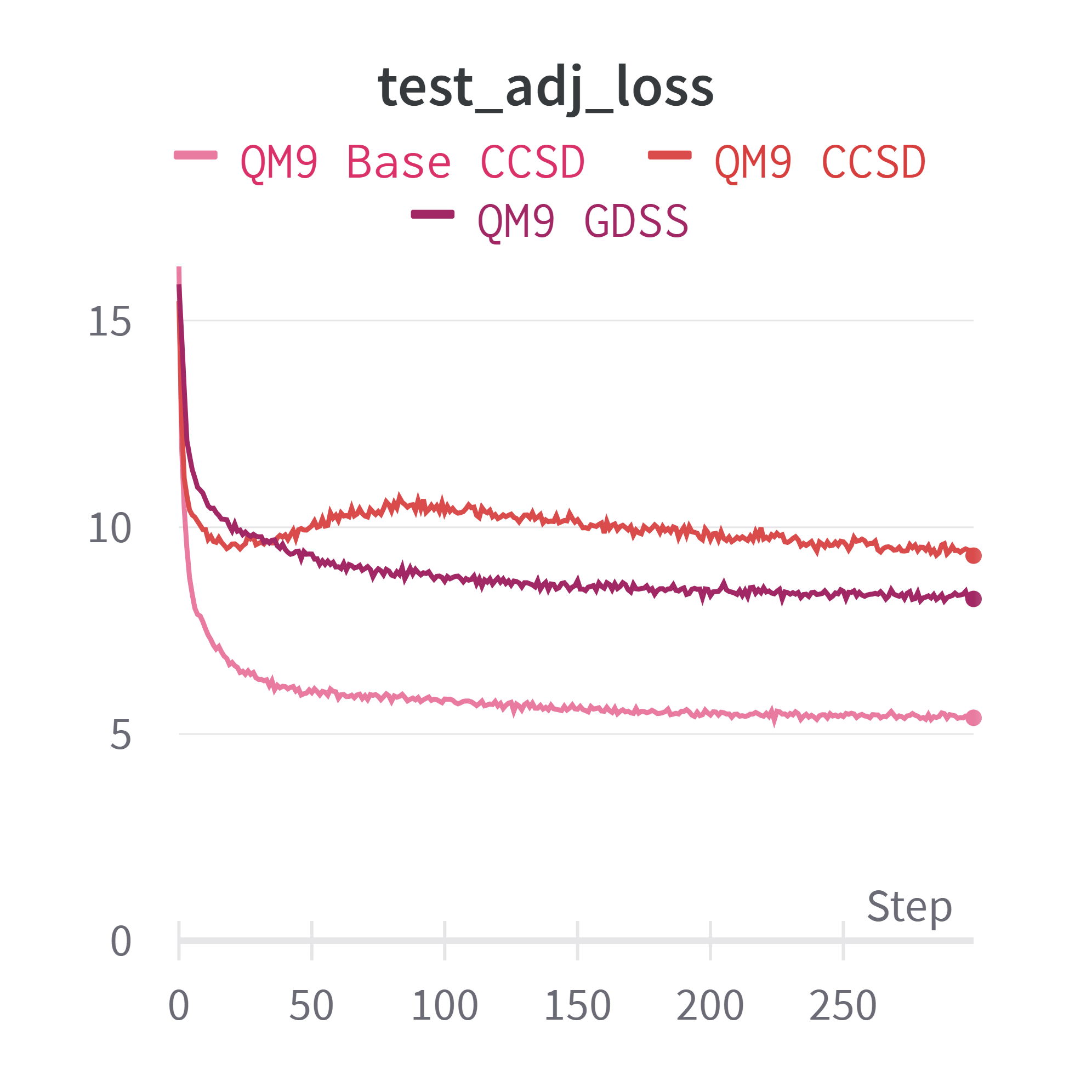}
\end{minipage}%
\begin{minipage}{0.33\textwidth}
  \centering
  \includegraphics[width=1\linewidth]{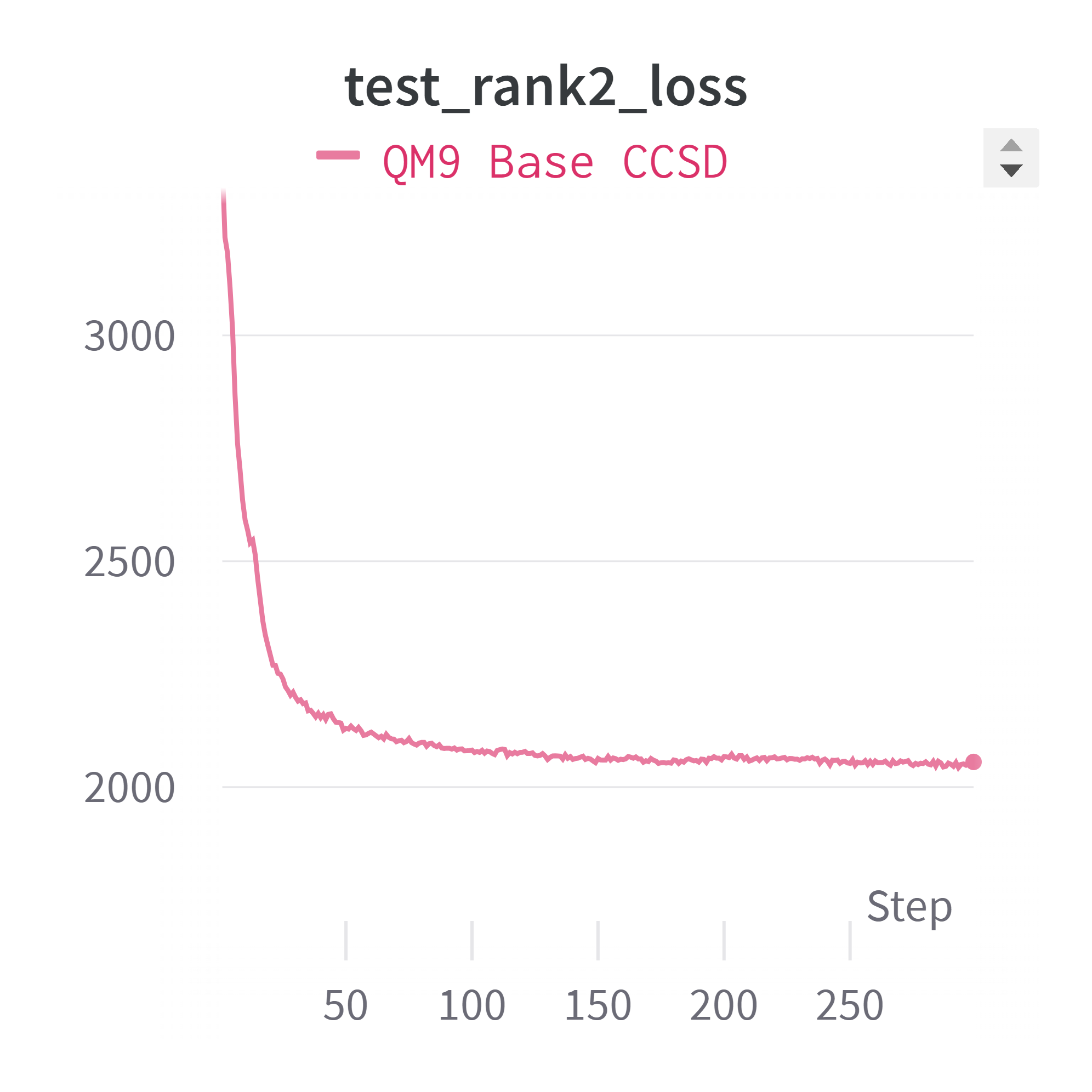}
\end{minipage}
\captionof{figure}{Test losses QM9.}
\label{fig:test_qm9}
\end{figure}

\subsection{Ego small}

\begin{figure}[H]
\centering
\begin{minipage}{0.33\textwidth}
  \centering
  \includegraphics[width=1\linewidth]{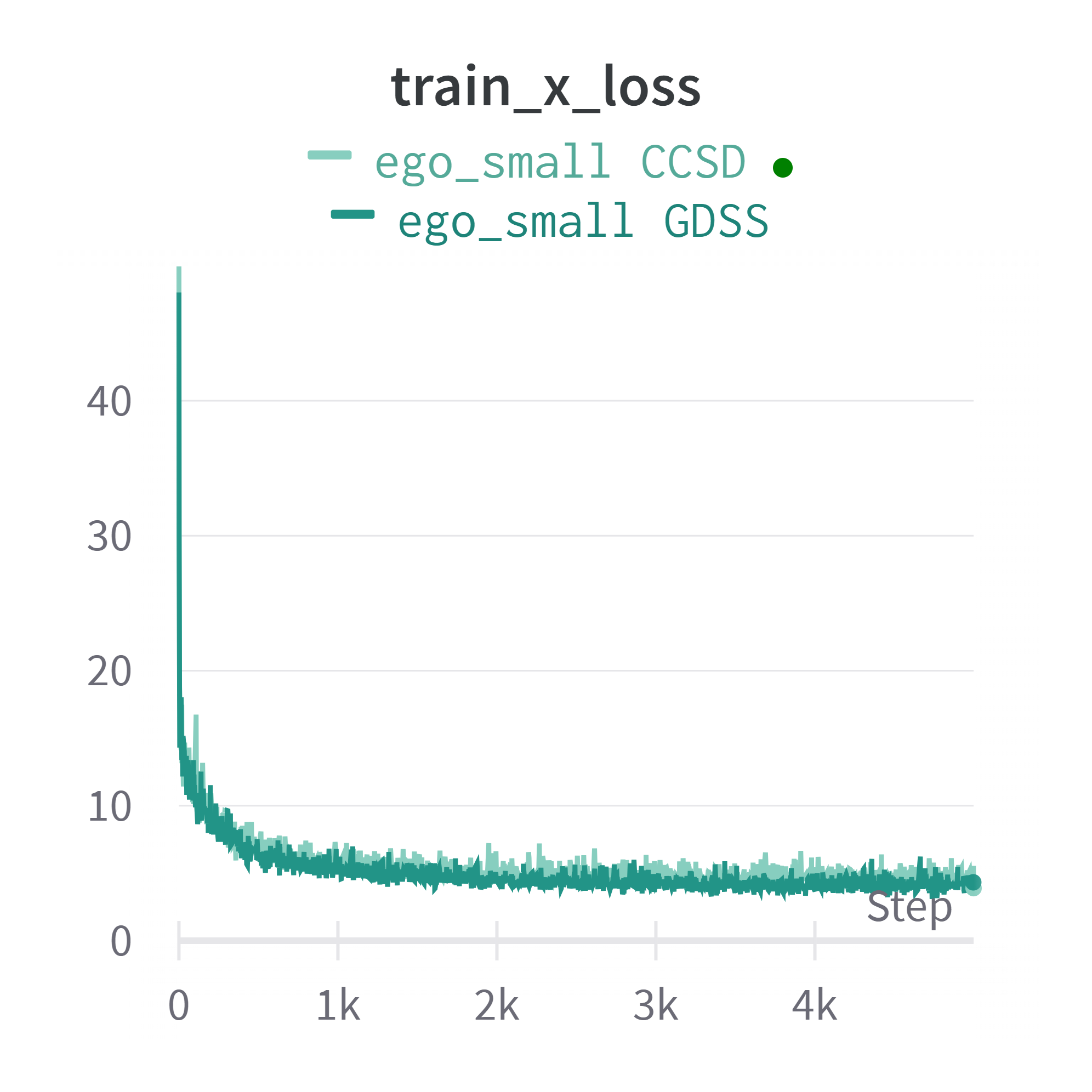}
\end{minipage}%
\begin{minipage}{0.33\textwidth}
  \centering
  \includegraphics[width=1\linewidth]{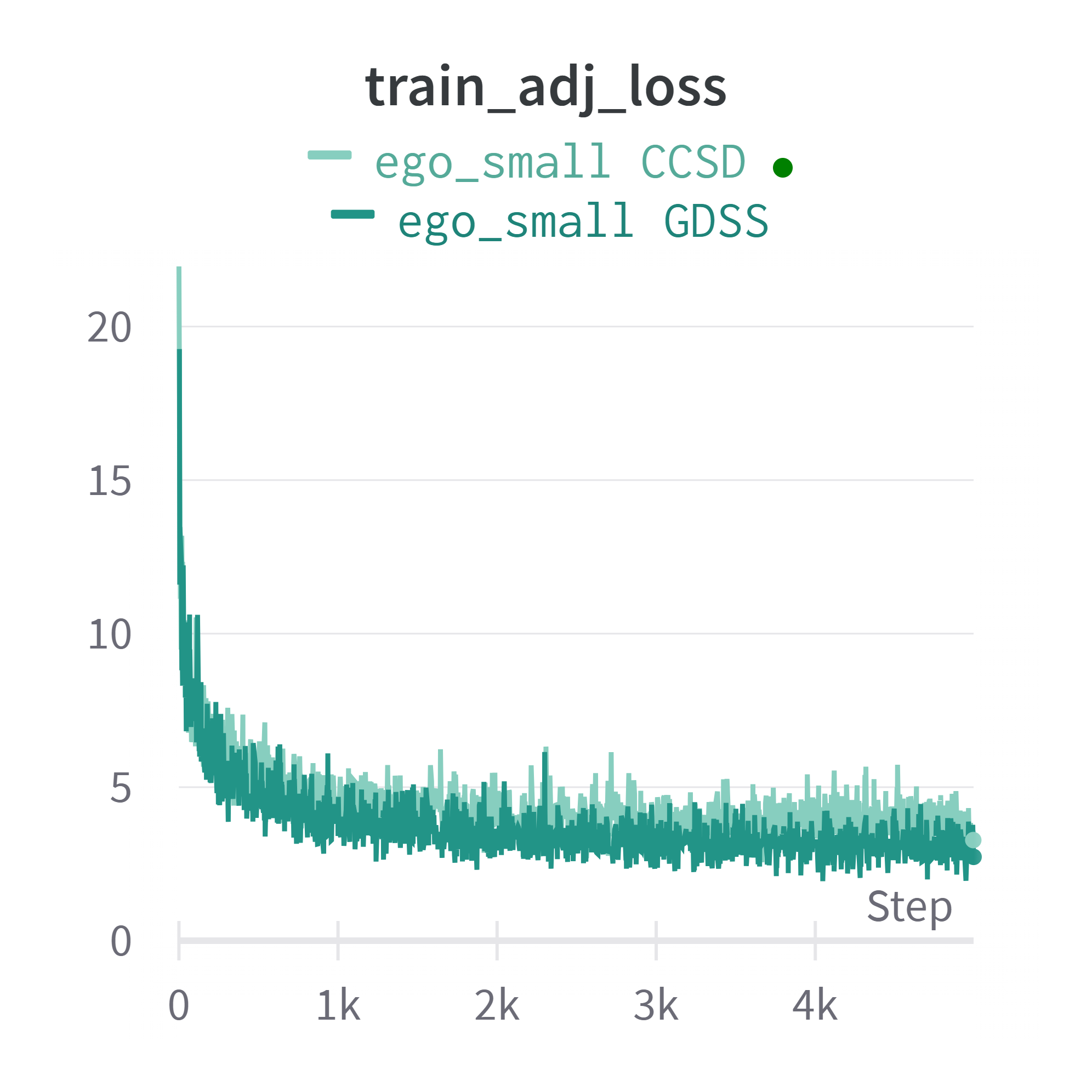}
\end{minipage}%
\begin{minipage}{0.33\textwidth}
  \centering
  \includegraphics[width=1\linewidth]{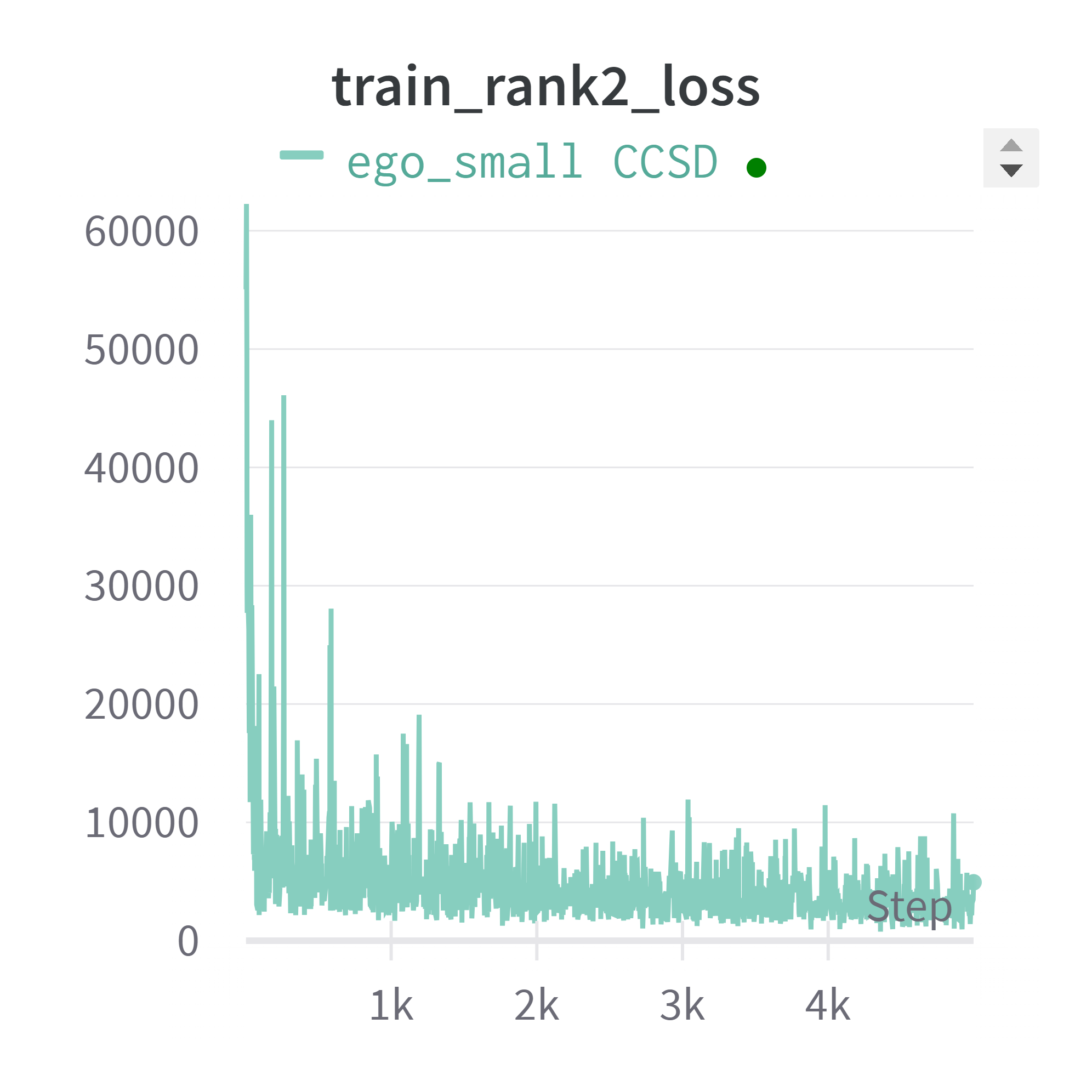}
\end{minipage}
\captionof{figure}{Train losses Ego small.}
\label{fig:train_ego}
\end{figure}

\begin{figure}[H]
\centering
\begin{minipage}{0.33\textwidth}
  \centering
  \includegraphics[width=1\linewidth]{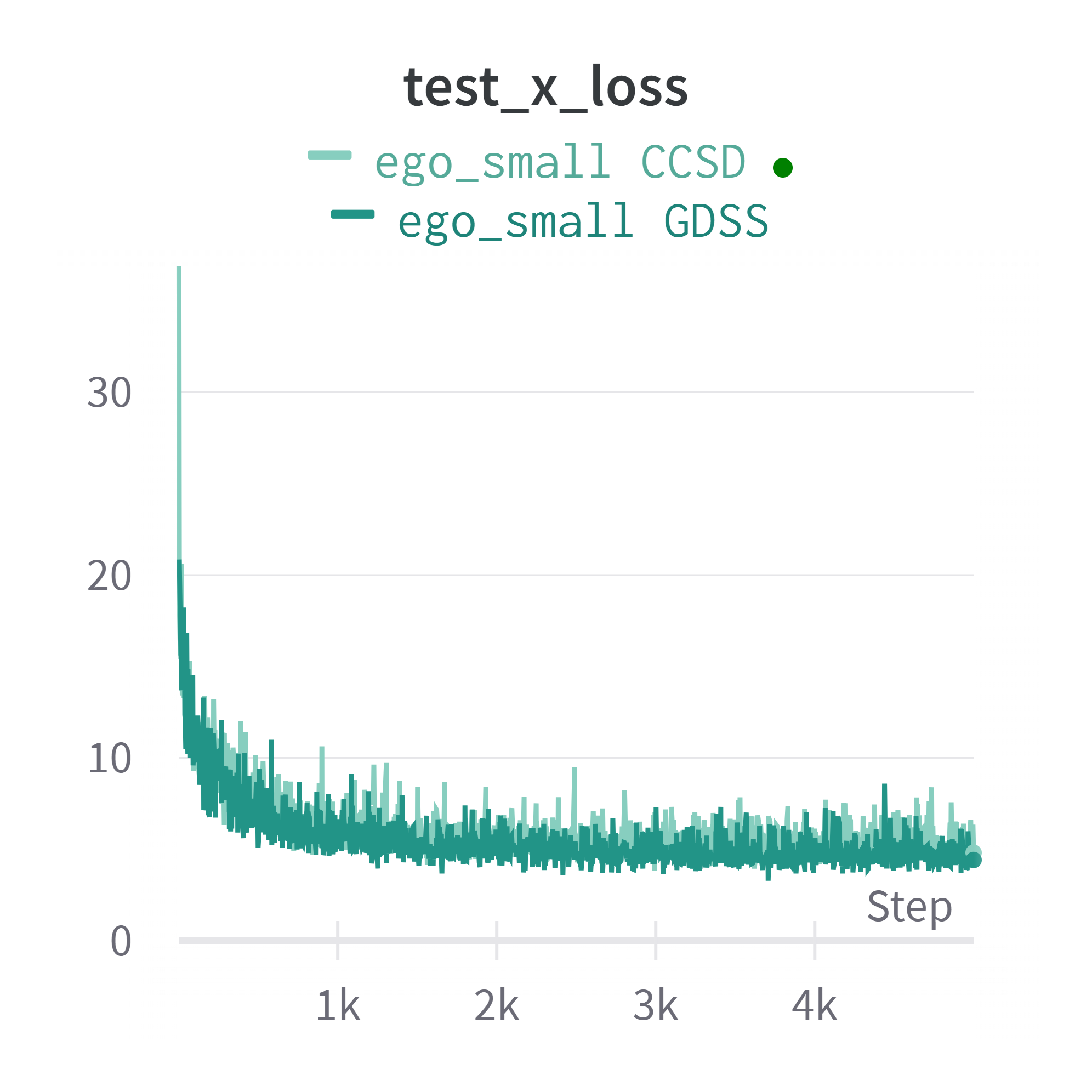}
\end{minipage}%
\begin{minipage}{0.33\textwidth}
  \centering
  \includegraphics[width=1\linewidth]{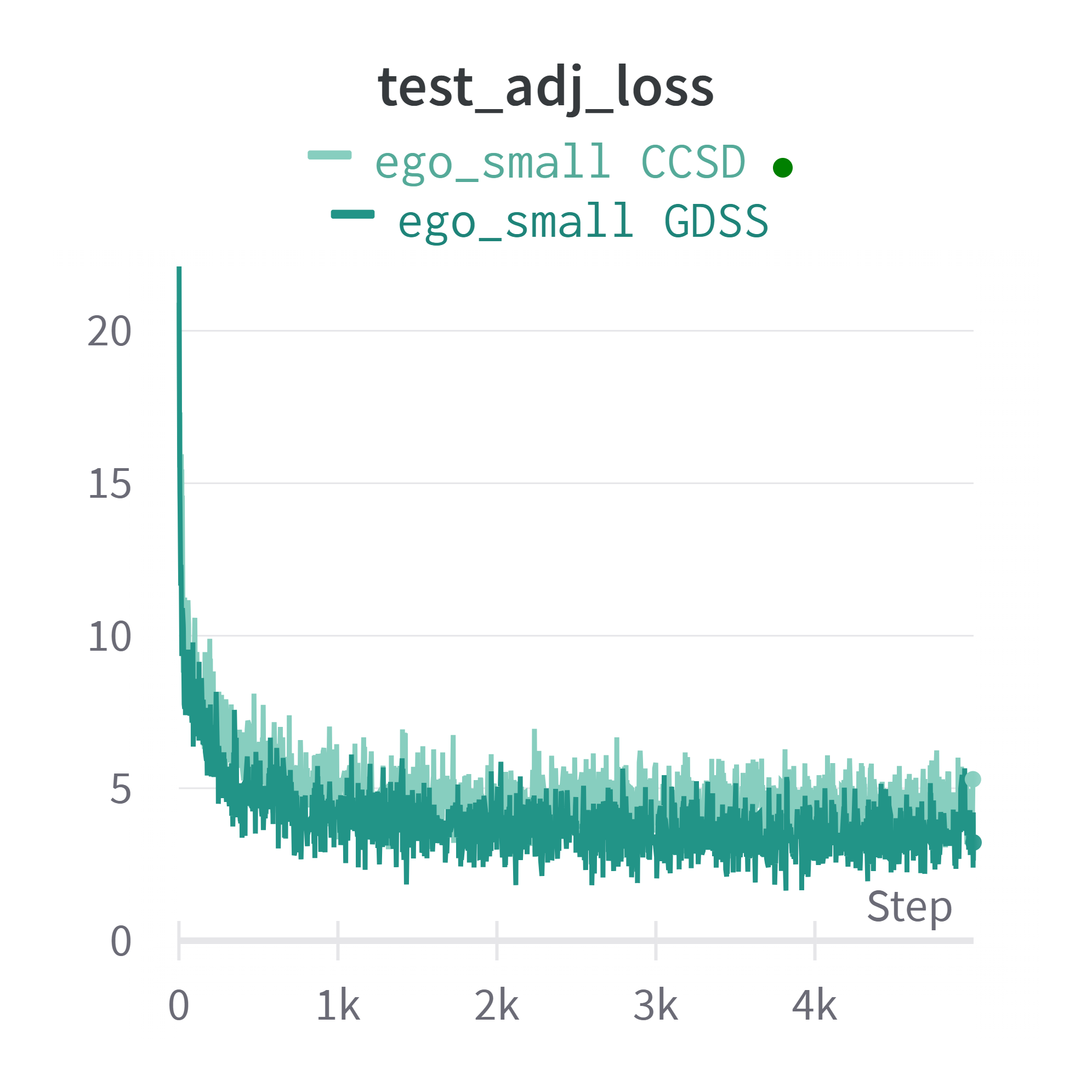}
\end{minipage}%
\begin{minipage}{0.33\textwidth}
  \centering
  \includegraphics[width=1\linewidth]{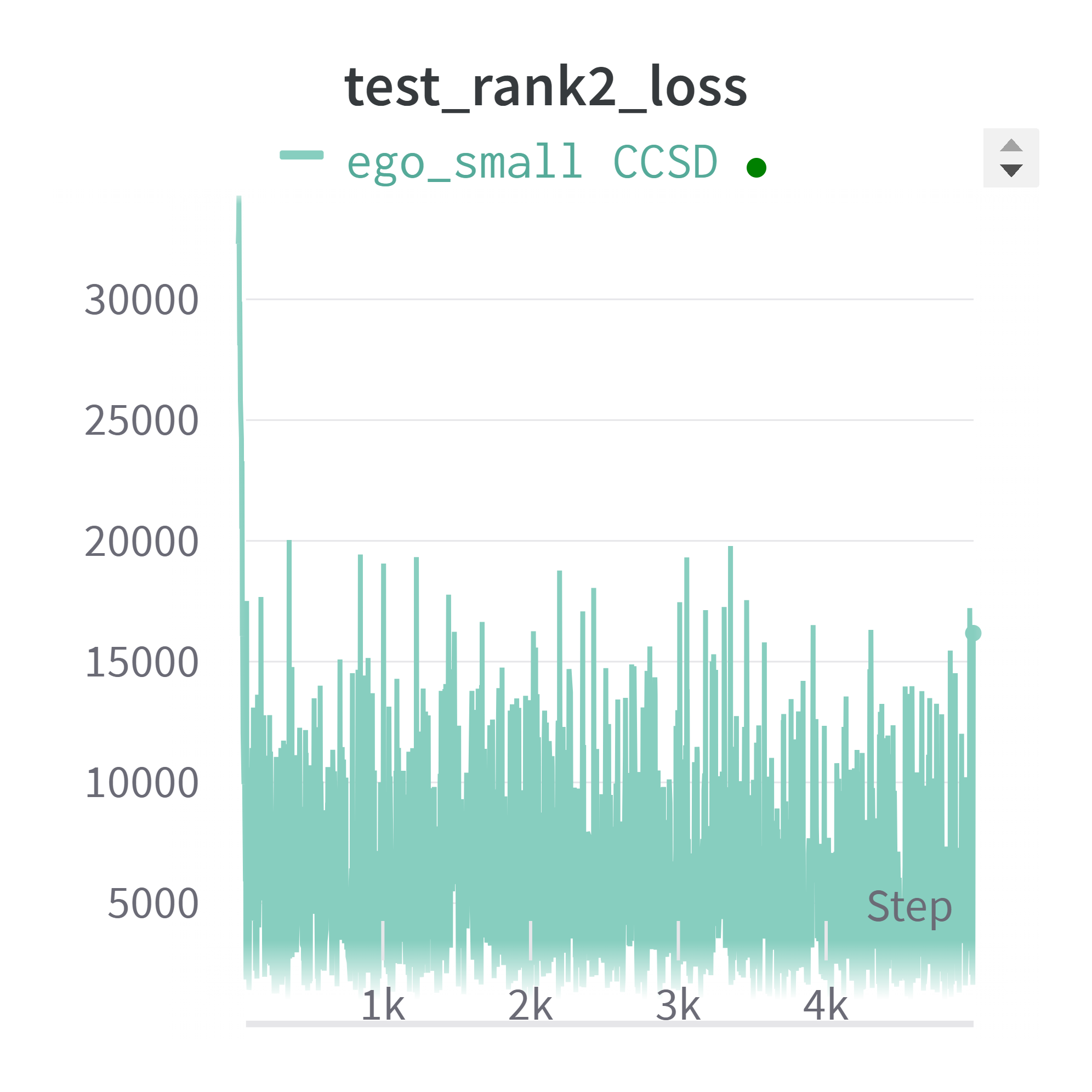}
\end{minipage}
\captionof{figure}{Test losses Ego small.}
\label{fig:test_ego}
\end{figure}

\subsection{Community small}

\begin{figure}[H]
\centering
\begin{minipage}{0.33\textwidth}
  \centering
  \includegraphics[width=1\linewidth]{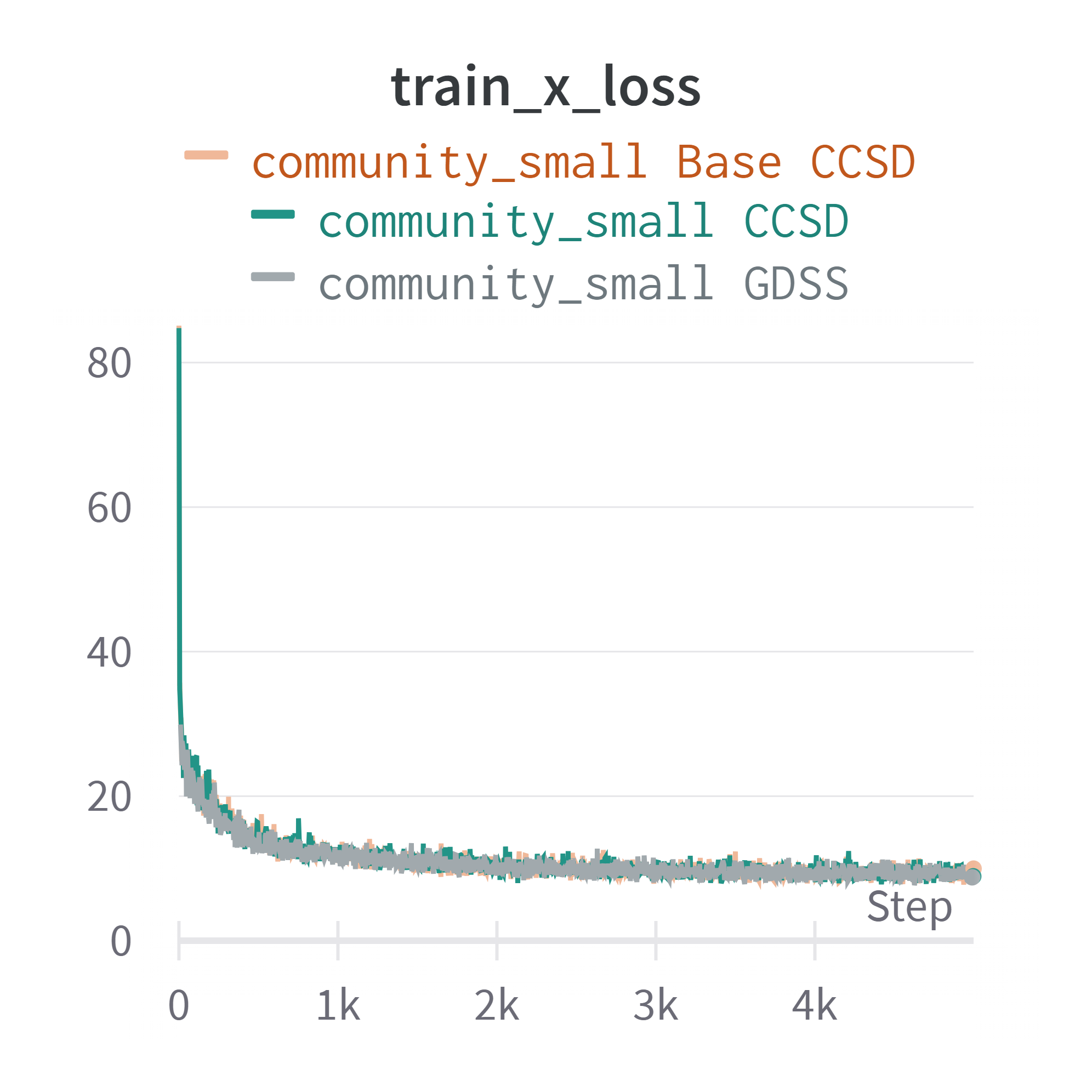}
\end{minipage}%
\begin{minipage}{0.33\textwidth}
  \centering
  \includegraphics[width=1\linewidth]{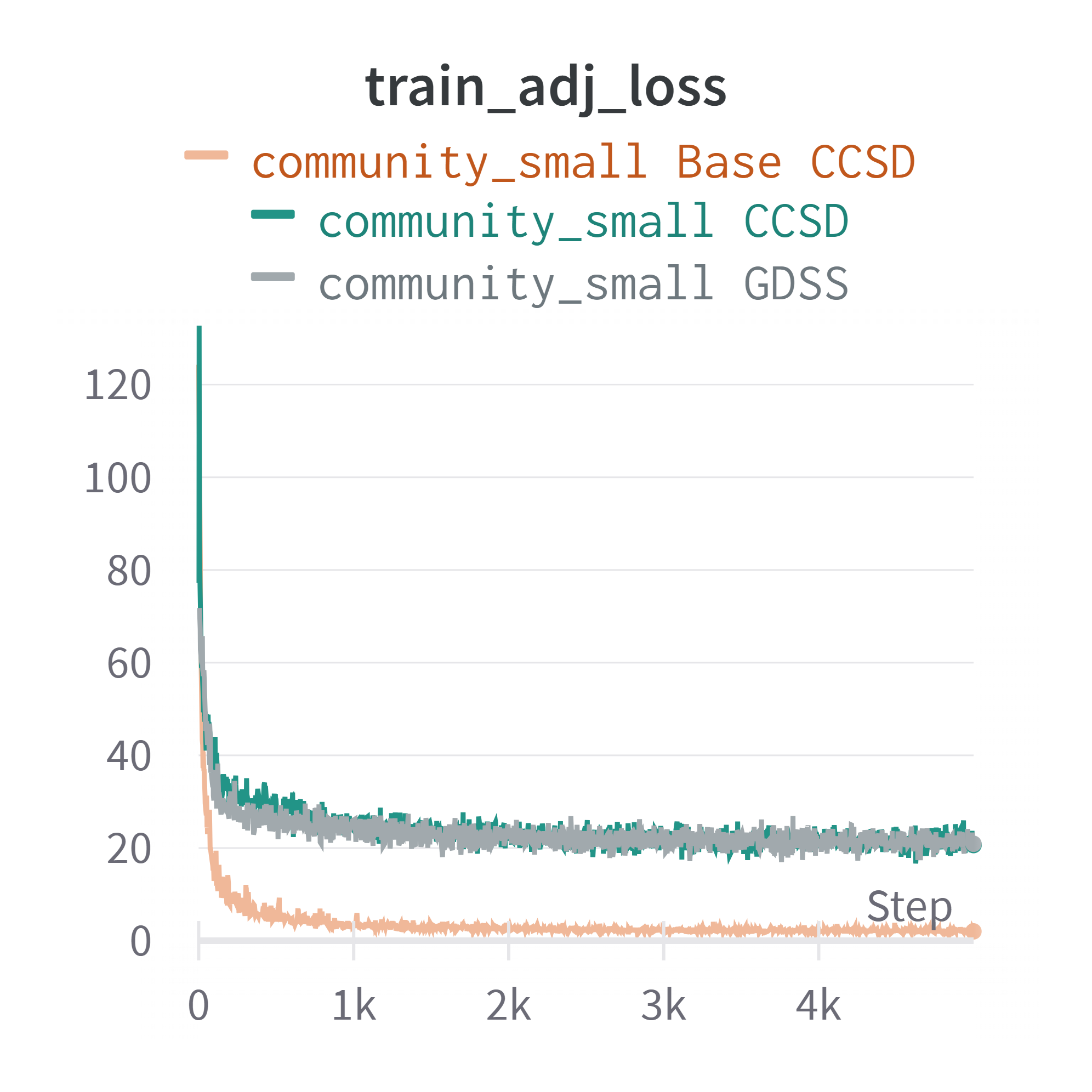}
\end{minipage}%
\begin{minipage}{0.33\textwidth}
  \centering
  \includegraphics[width=1\linewidth]{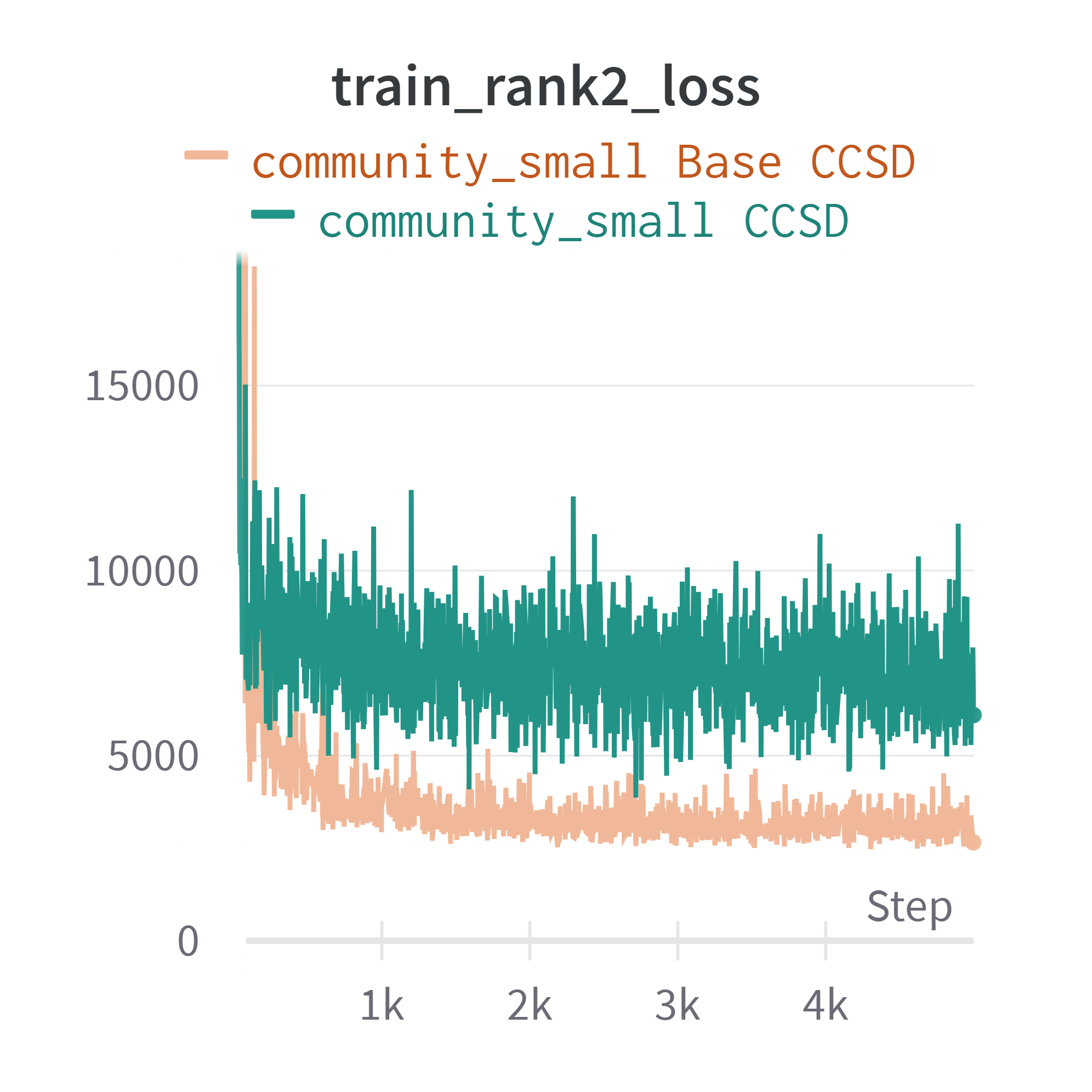}
\end{minipage}
\captionof{figure}{Train losses Community small.}
\label{fig:train_commu}
\end{figure}

\begin{figure}[H]
\centering
\begin{minipage}{0.33\textwidth}
  \centering
  \includegraphics[width=1\linewidth]{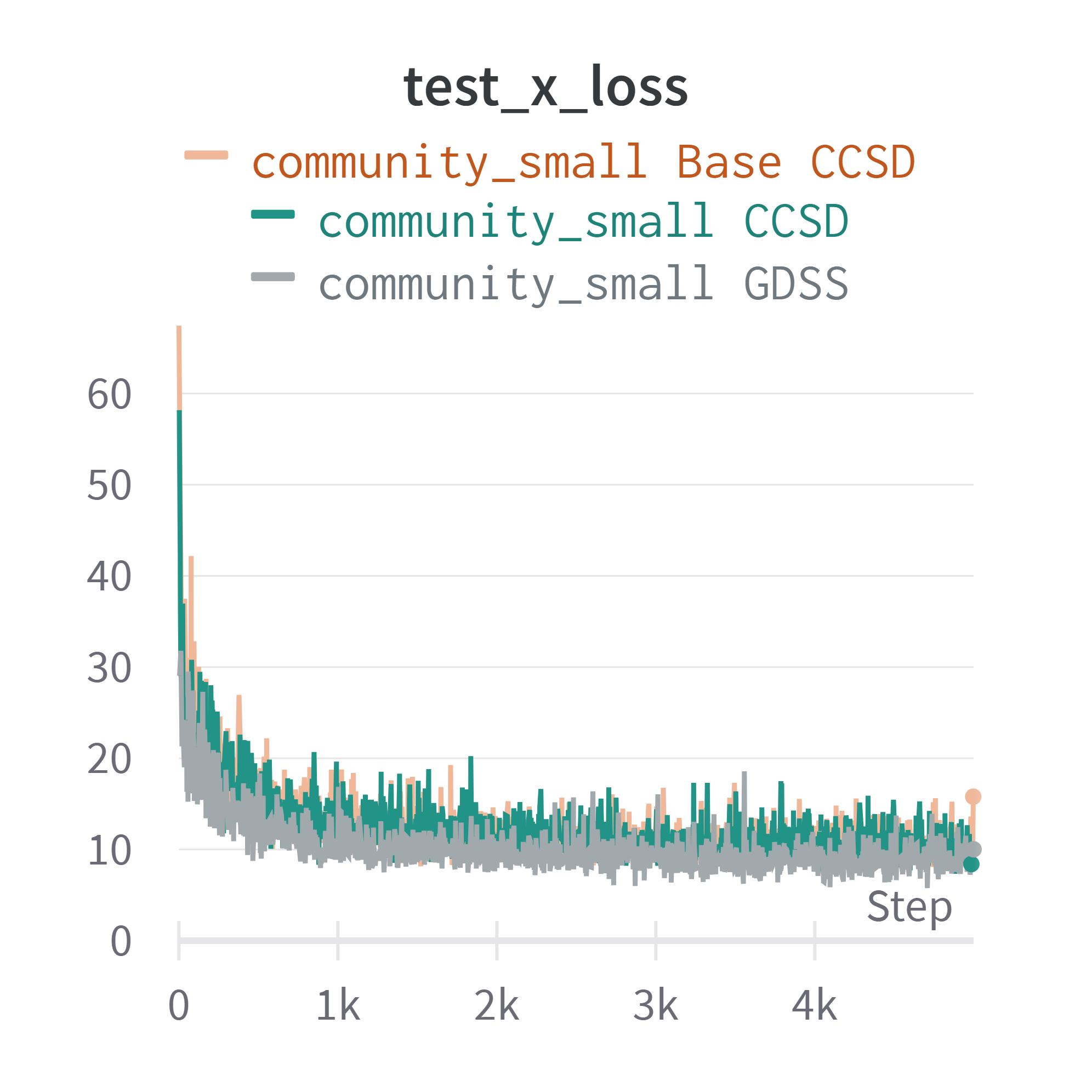}
\end{minipage}%
\begin{minipage}{0.33\textwidth}
  \centering
  \includegraphics[width=1\linewidth]{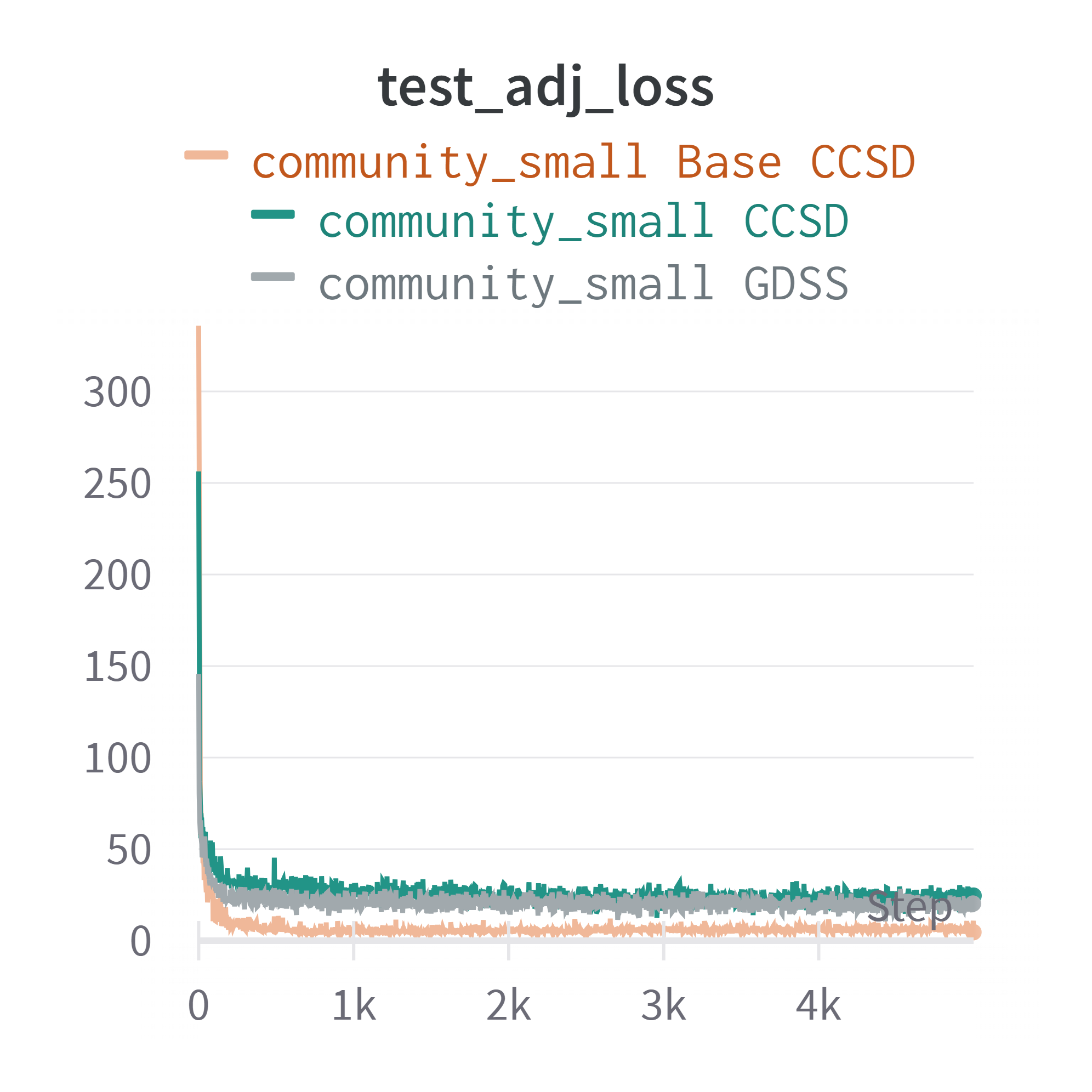}
\end{minipage}%
\begin{minipage}{0.33\textwidth}
  \centering
  \includegraphics[width=1\linewidth]{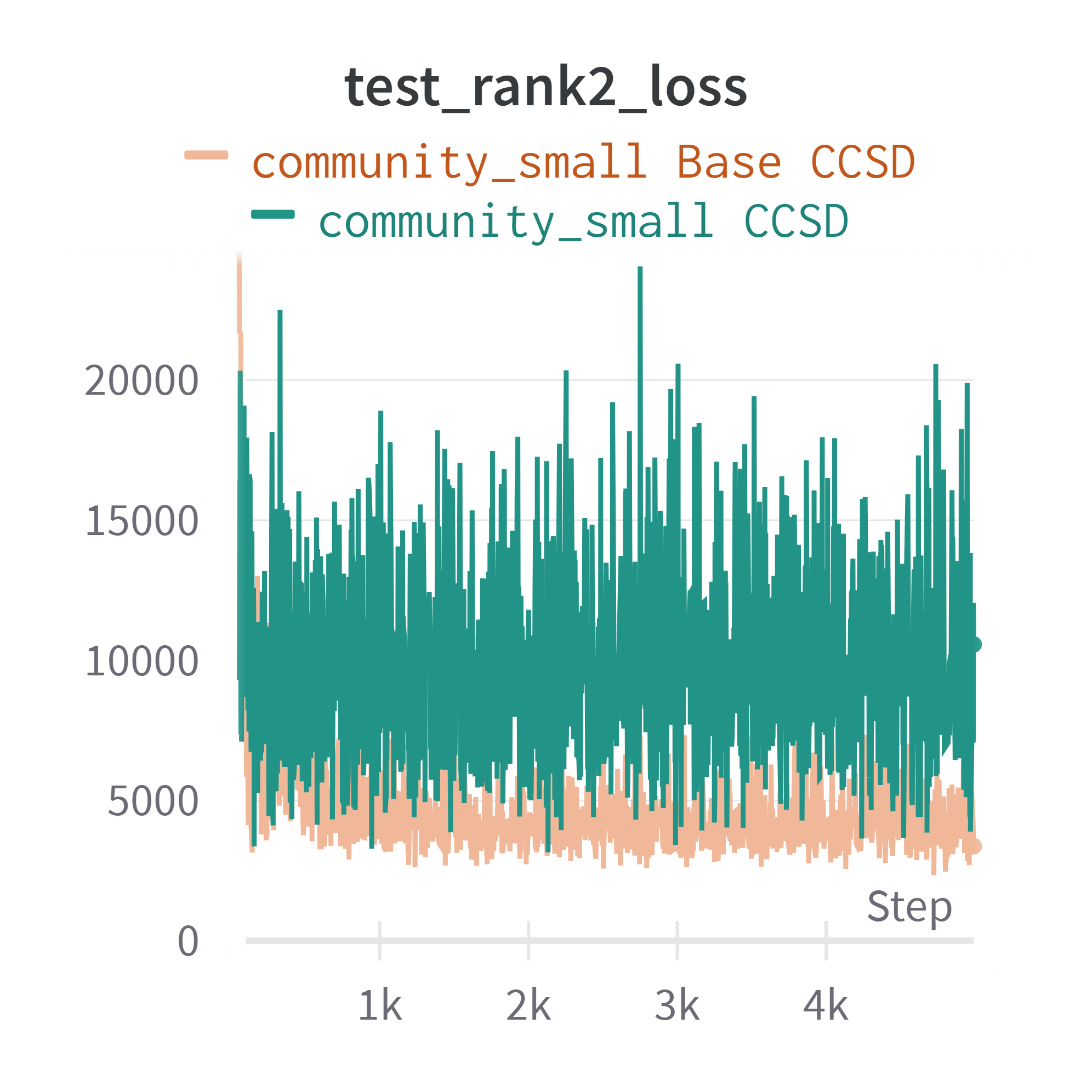}
\end{minipage}
\captionof{figure}{Test losses Community small.}
\label{fig:test_commu}
\end{figure}

\subsection{Enzymes small}

\begin{figure}[H]
\centering
\begin{minipage}{0.33\textwidth}
  \centering
  \includegraphics[width=1\linewidth]{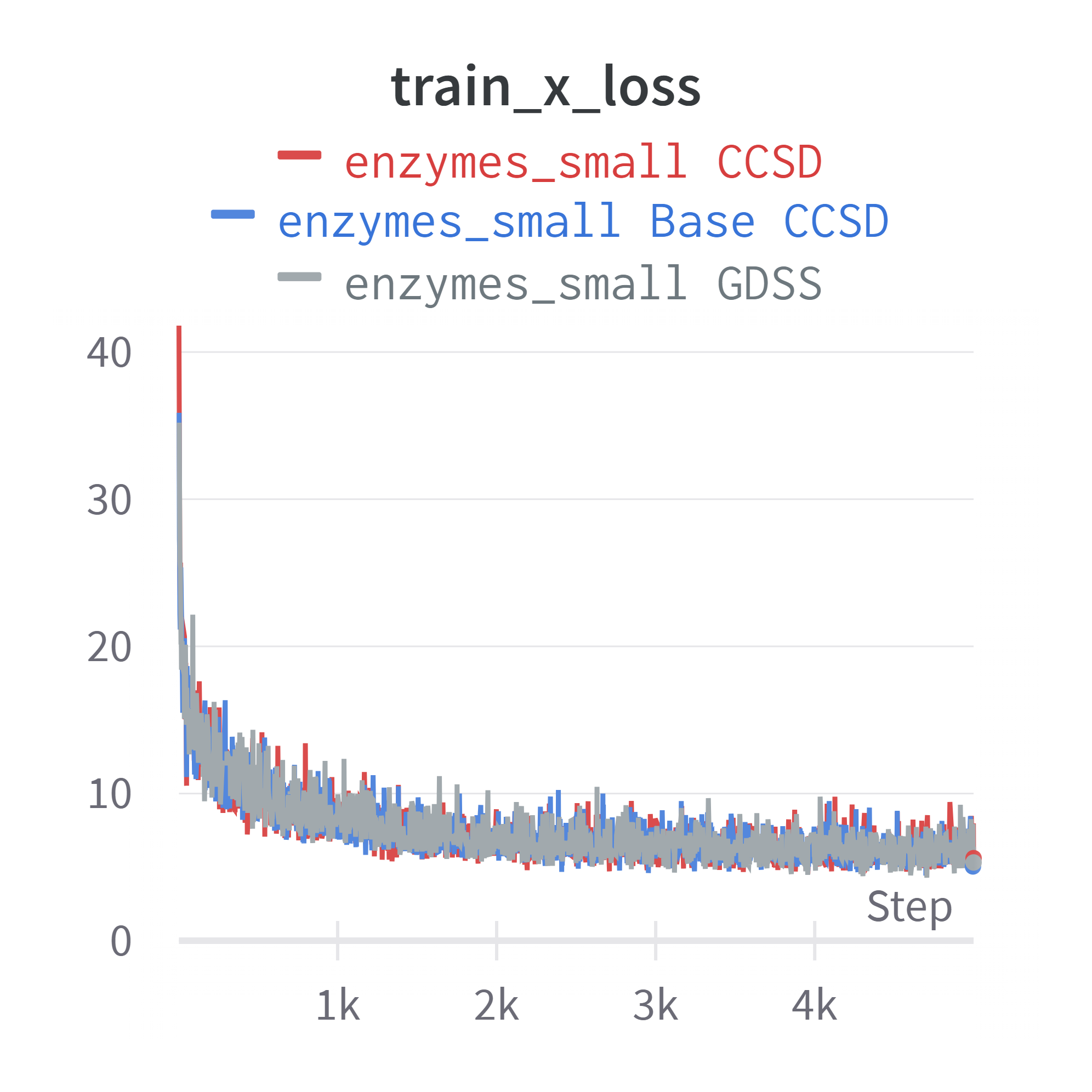}
\end{minipage}%
\begin{minipage}{0.33\textwidth}
  \centering
  \includegraphics[width=1\linewidth]{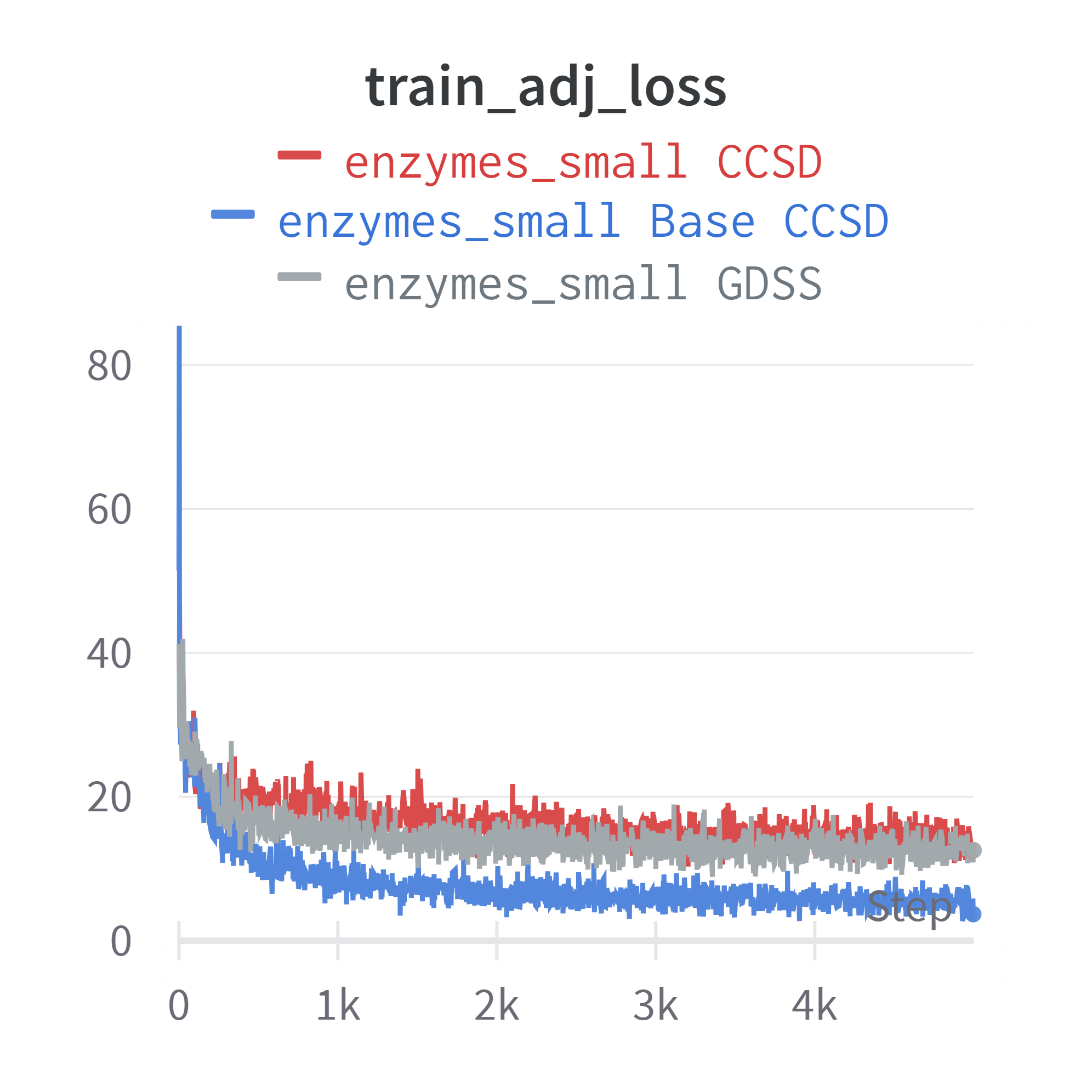}
\end{minipage}%
\begin{minipage}{0.33\textwidth}
  \centering
  \includegraphics[width=1\linewidth]{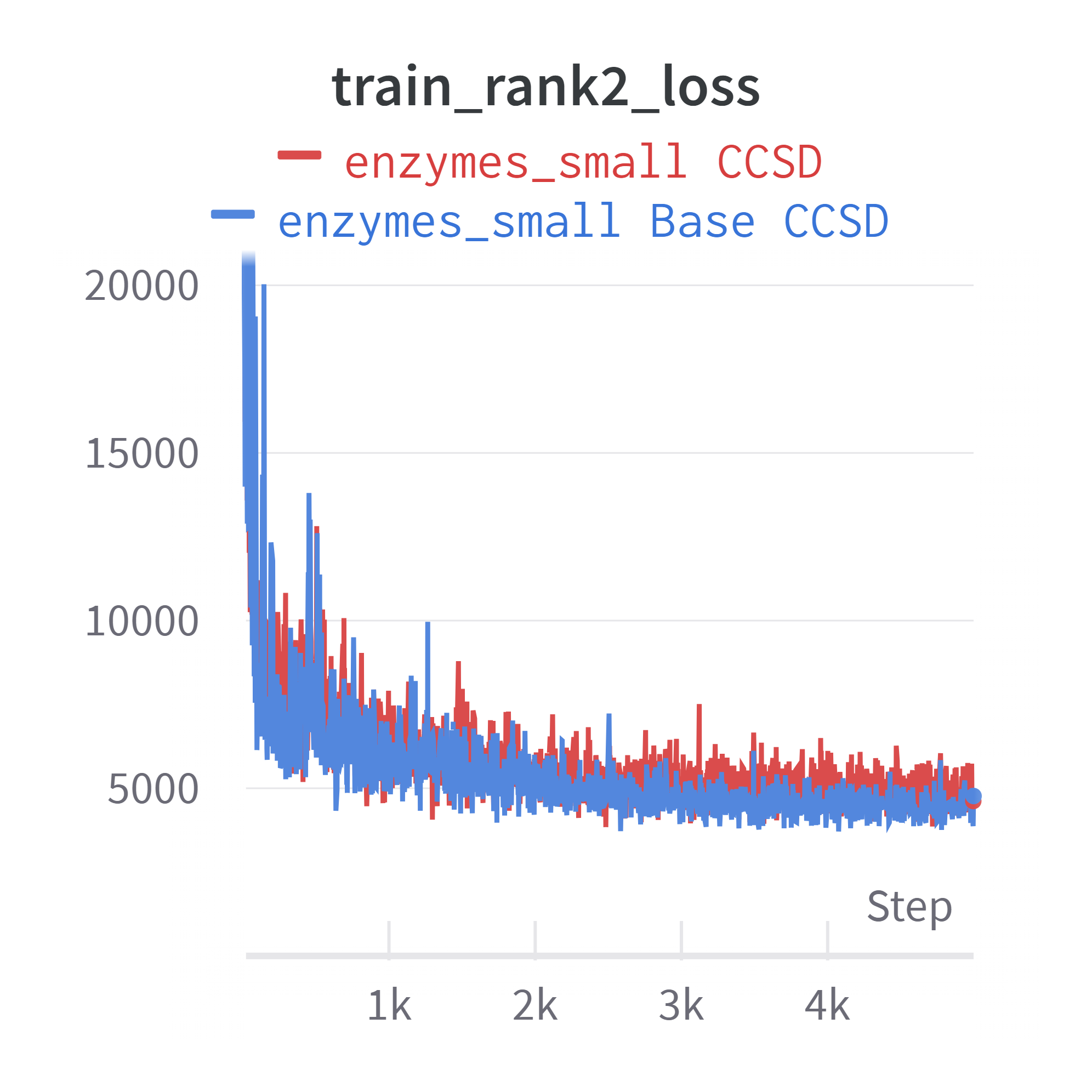}
\end{minipage}
\captionof{figure}{Train losses Enzymes small.}
\label{fig:train_enzymes}
\end{figure}

\begin{figure}[H]
\centering
\begin{minipage}{0.33\textwidth}
  \centering
  \includegraphics[width=1\linewidth]{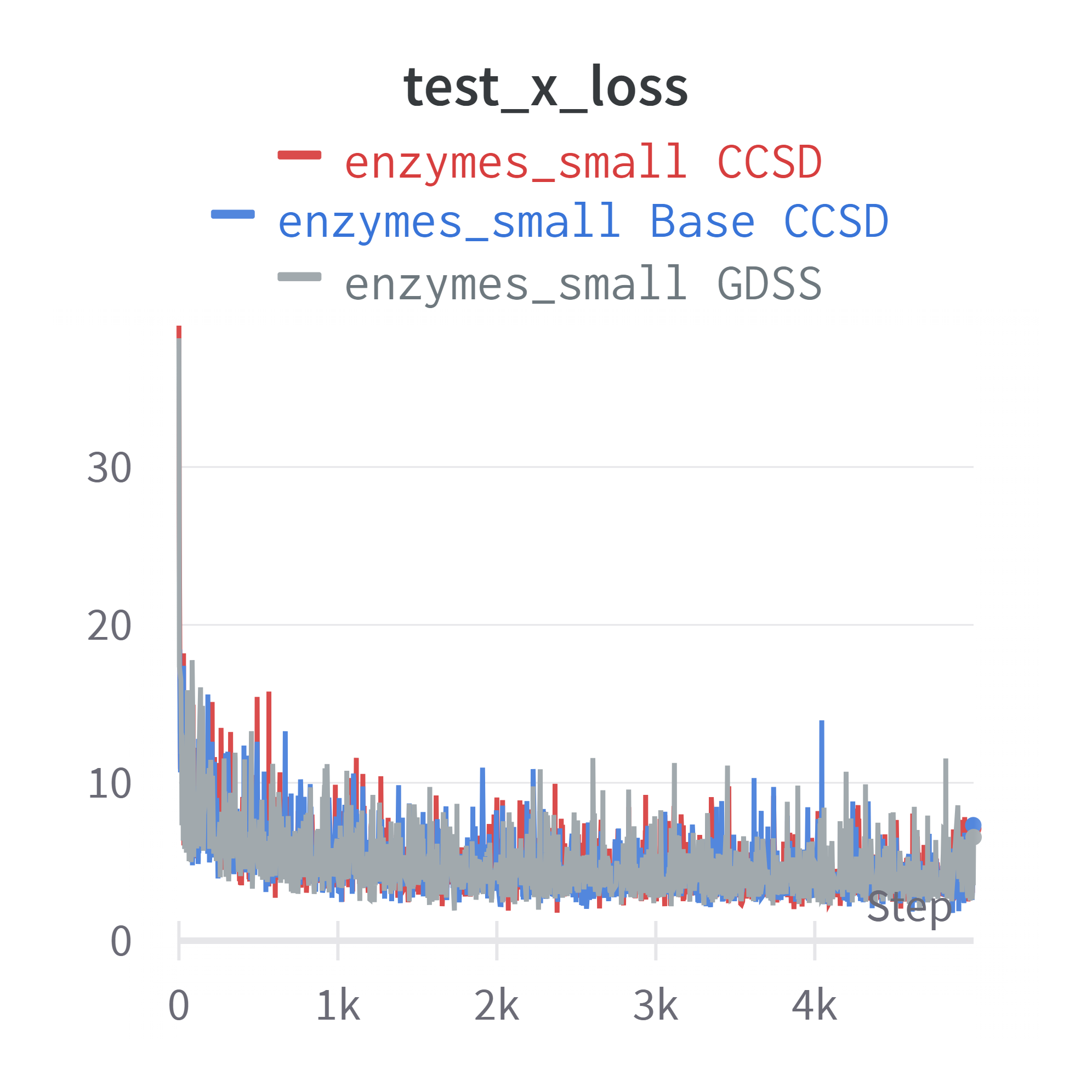}
\end{minipage}%
\begin{minipage}{0.33\textwidth}
  \centering
  \includegraphics[width=1\linewidth]{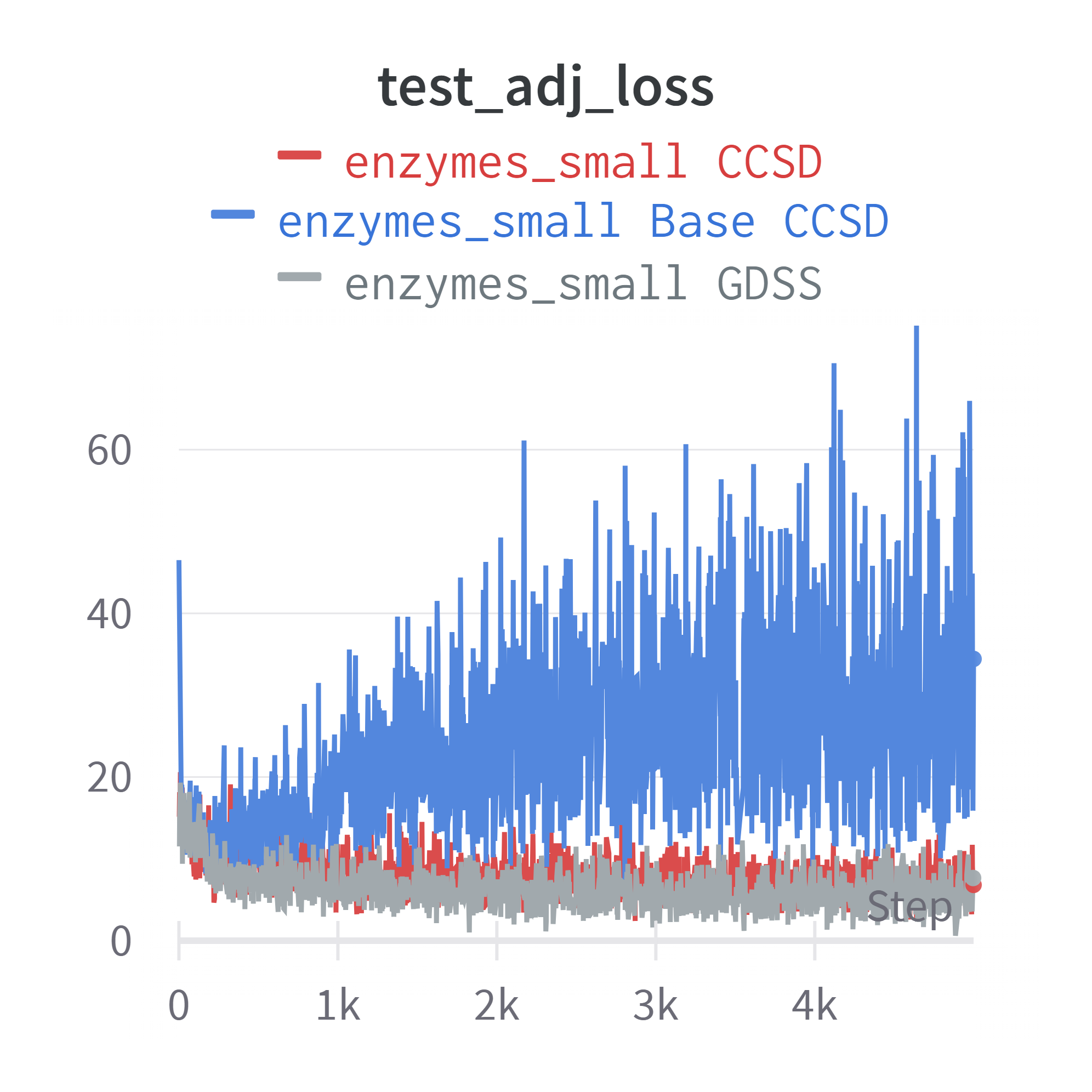}
\end{minipage}%
\begin{minipage}{0.33\textwidth}
  \centering
  \includegraphics[width=1\linewidth]{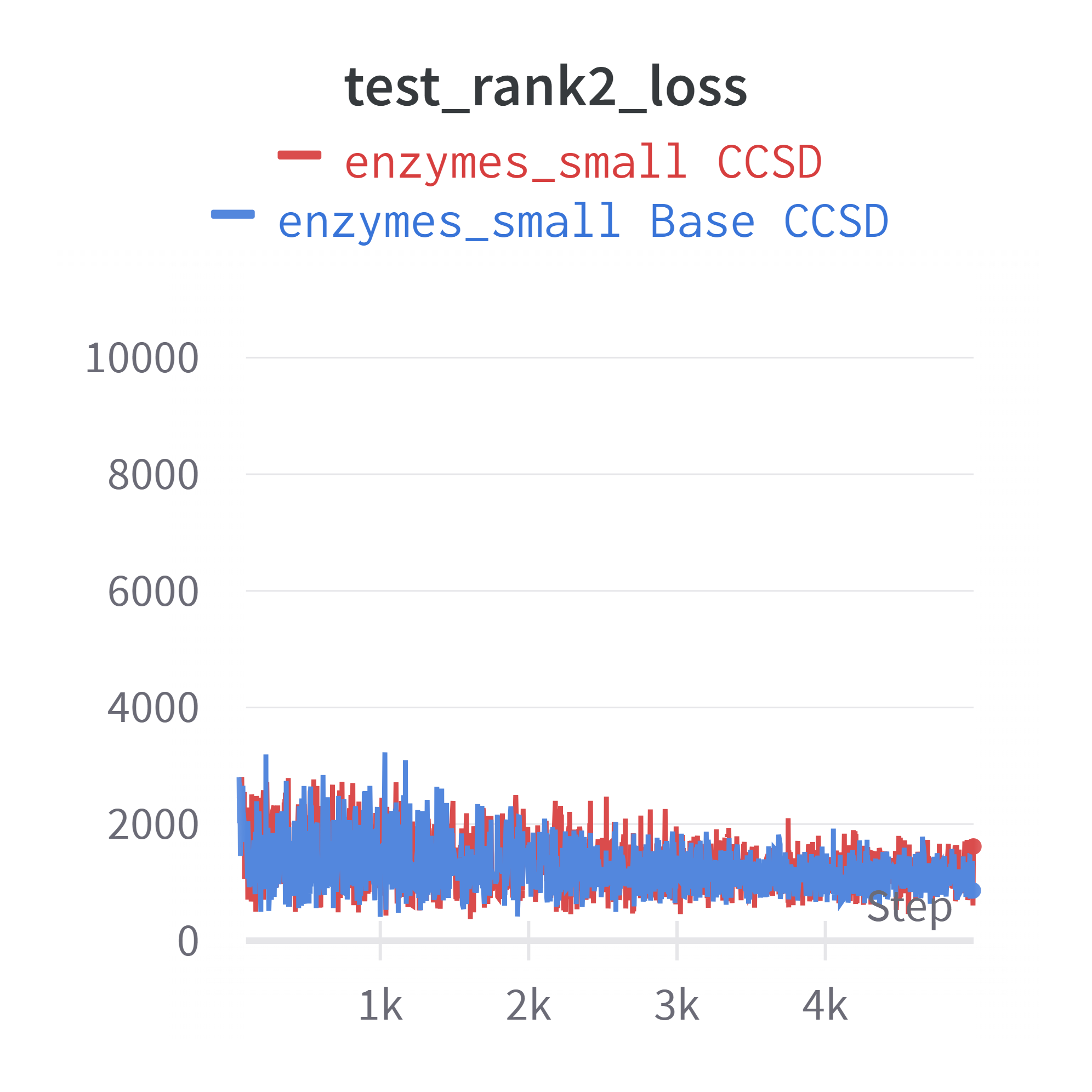}
\end{minipage}
\captionof{figure}{Test losses Enzymes small.}
\label{fig:test_enzymes}
\end{figure}

\subsection{Grid small}

\begin{figure}[H]
\centering
\begin{minipage}{0.33\textwidth}
  \centering
  \includegraphics[width=1\linewidth]{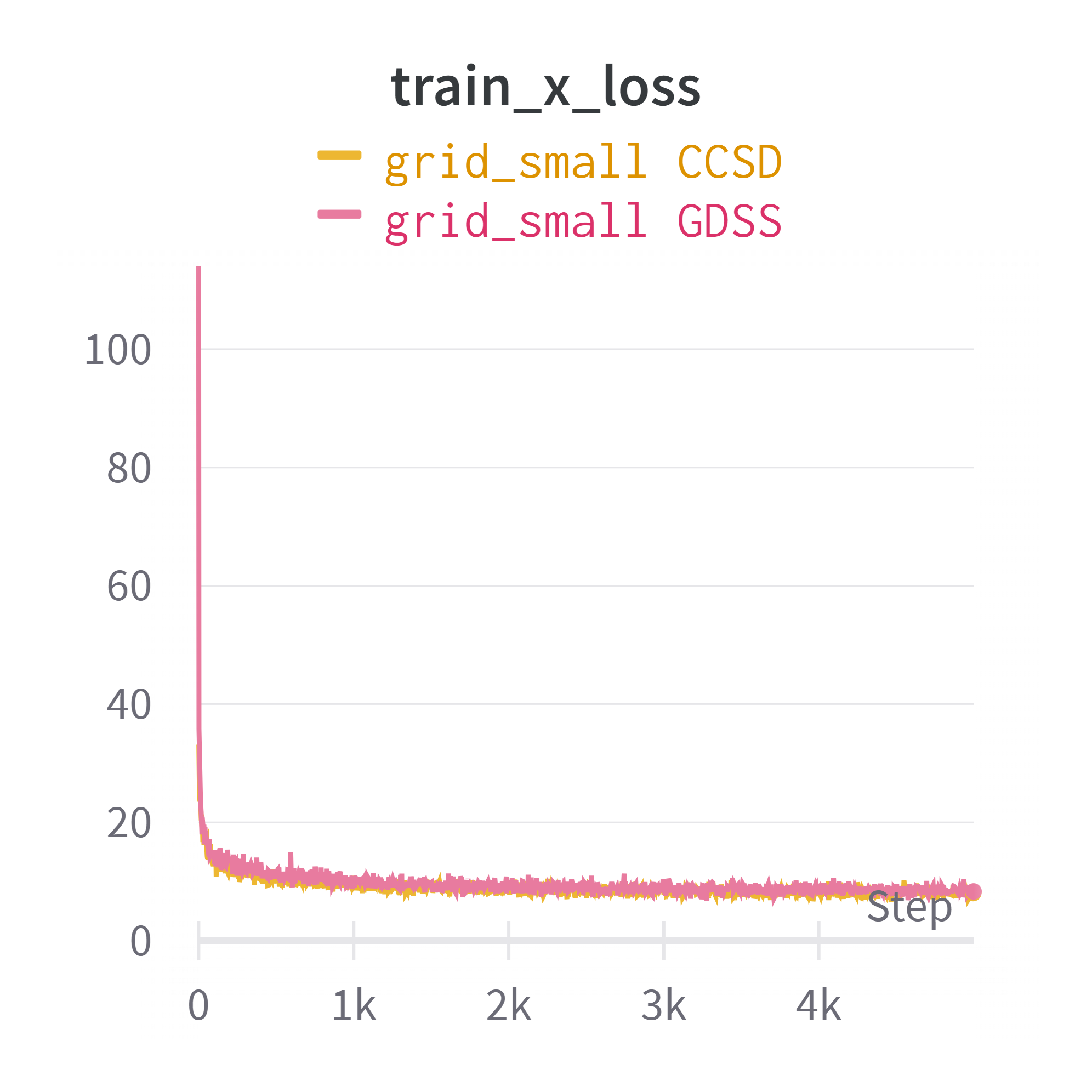}
\end{minipage}%
\begin{minipage}{0.33\textwidth}
  \centering
  \includegraphics[width=1\linewidth]{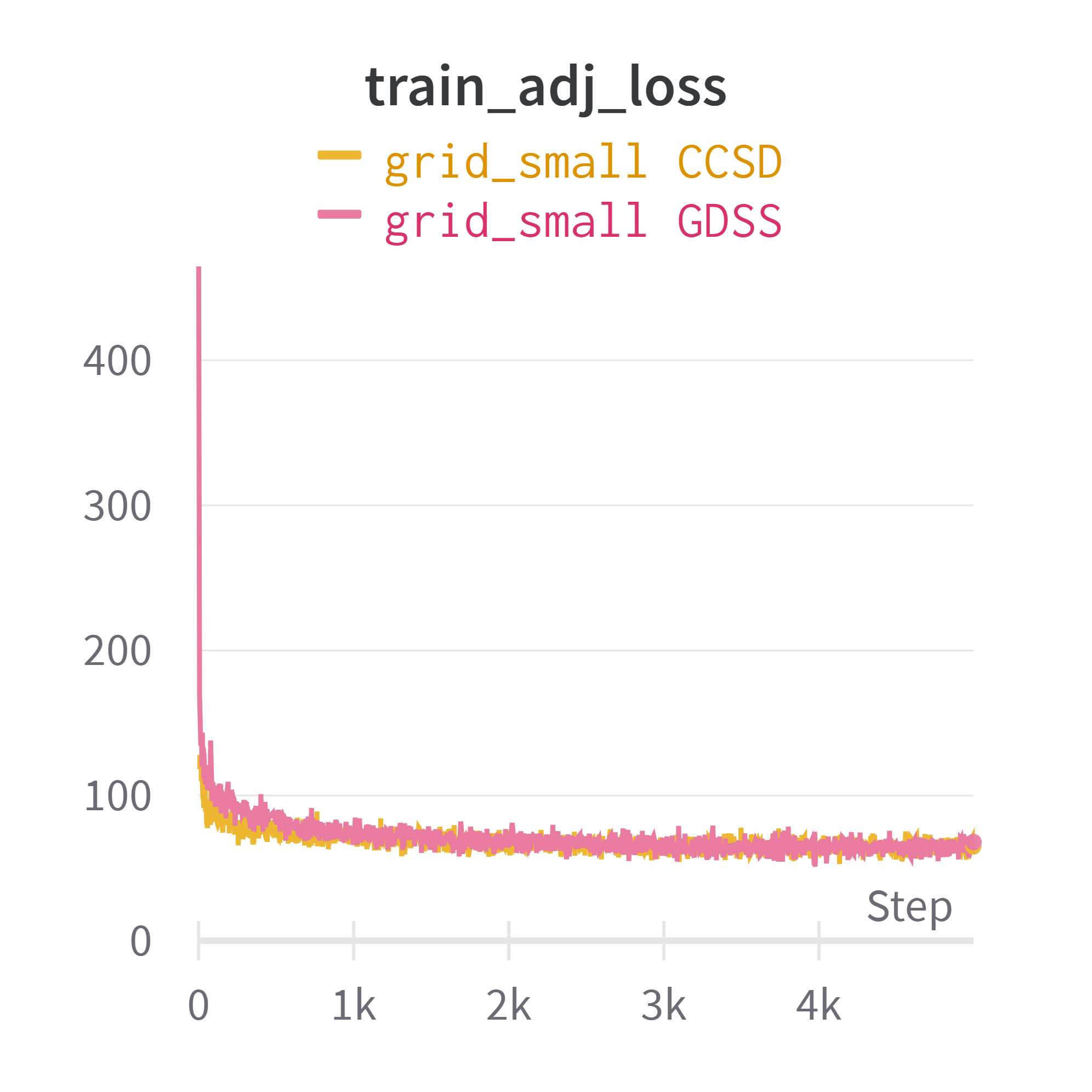}
\end{minipage}%
\begin{minipage}{0.33\textwidth}
  \centering
  \includegraphics[width=1\linewidth]{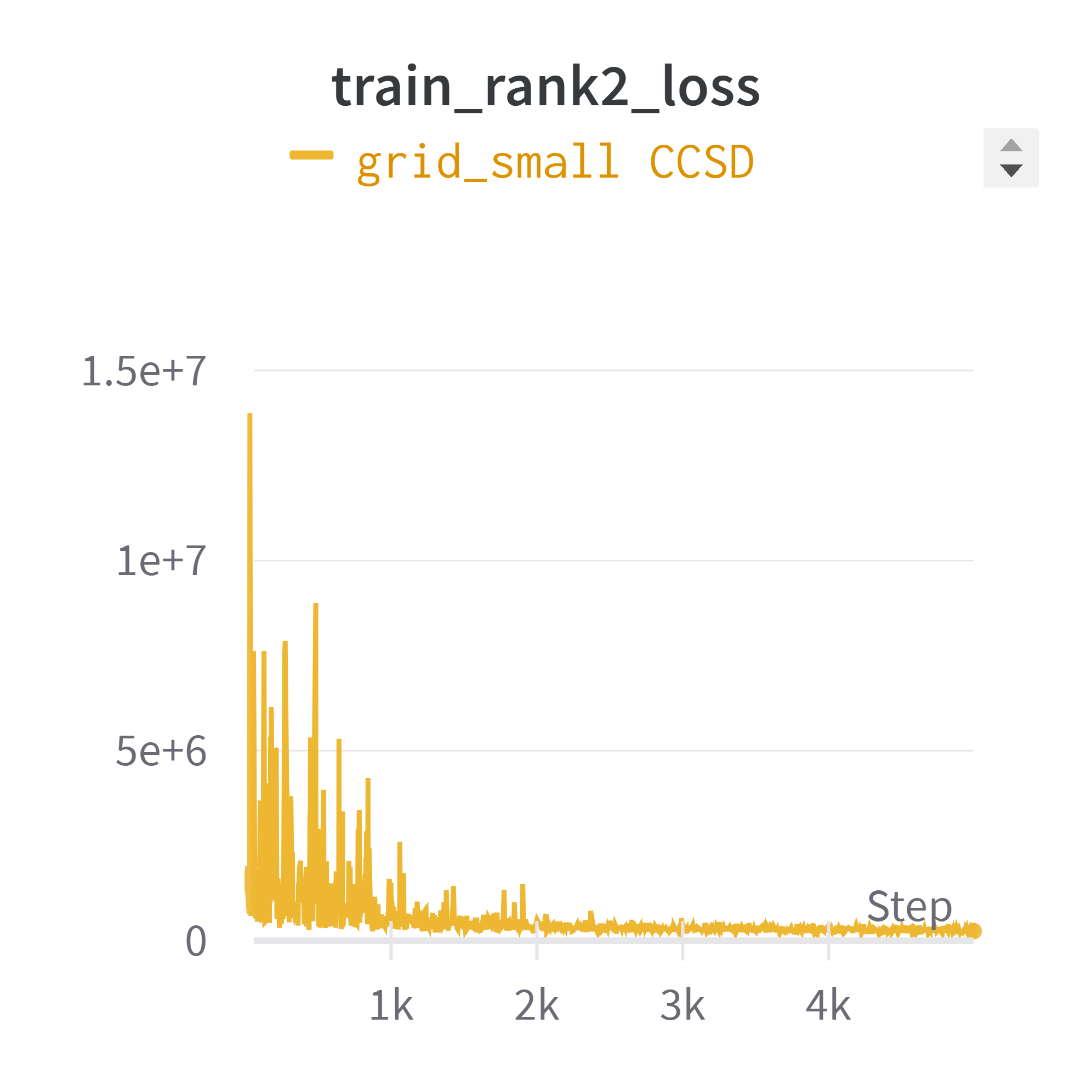}
\end{minipage}
\captionof{figure}{Train losses Grid small.}
\label{fig:train_grid}
\end{figure}

\begin{figure}[H]
\centering
\begin{minipage}{0.33\textwidth}
  \centering
  \includegraphics[width=1\linewidth]{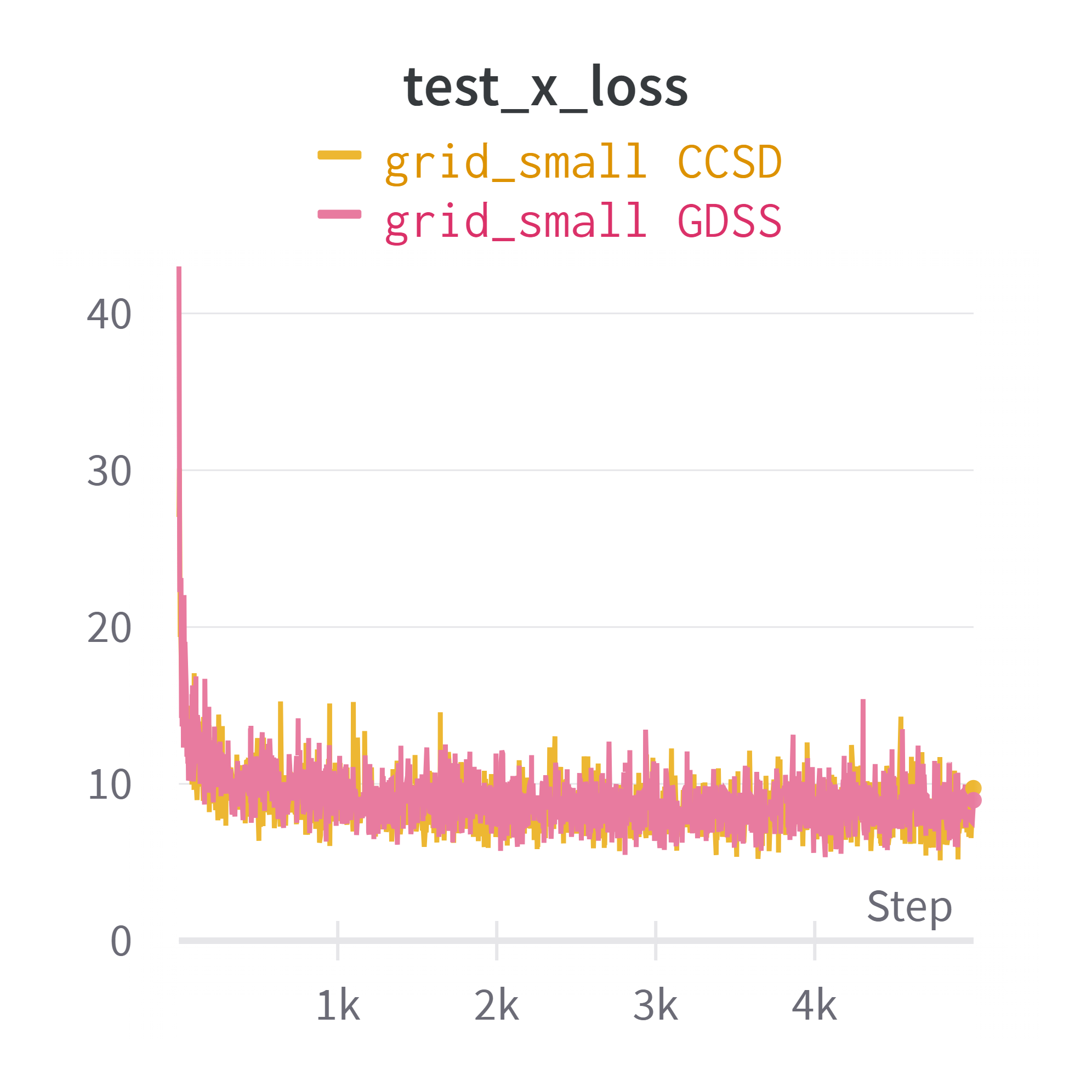}
\end{minipage}%
\begin{minipage}{0.33\textwidth}
  \centering
  \includegraphics[width=1\linewidth]{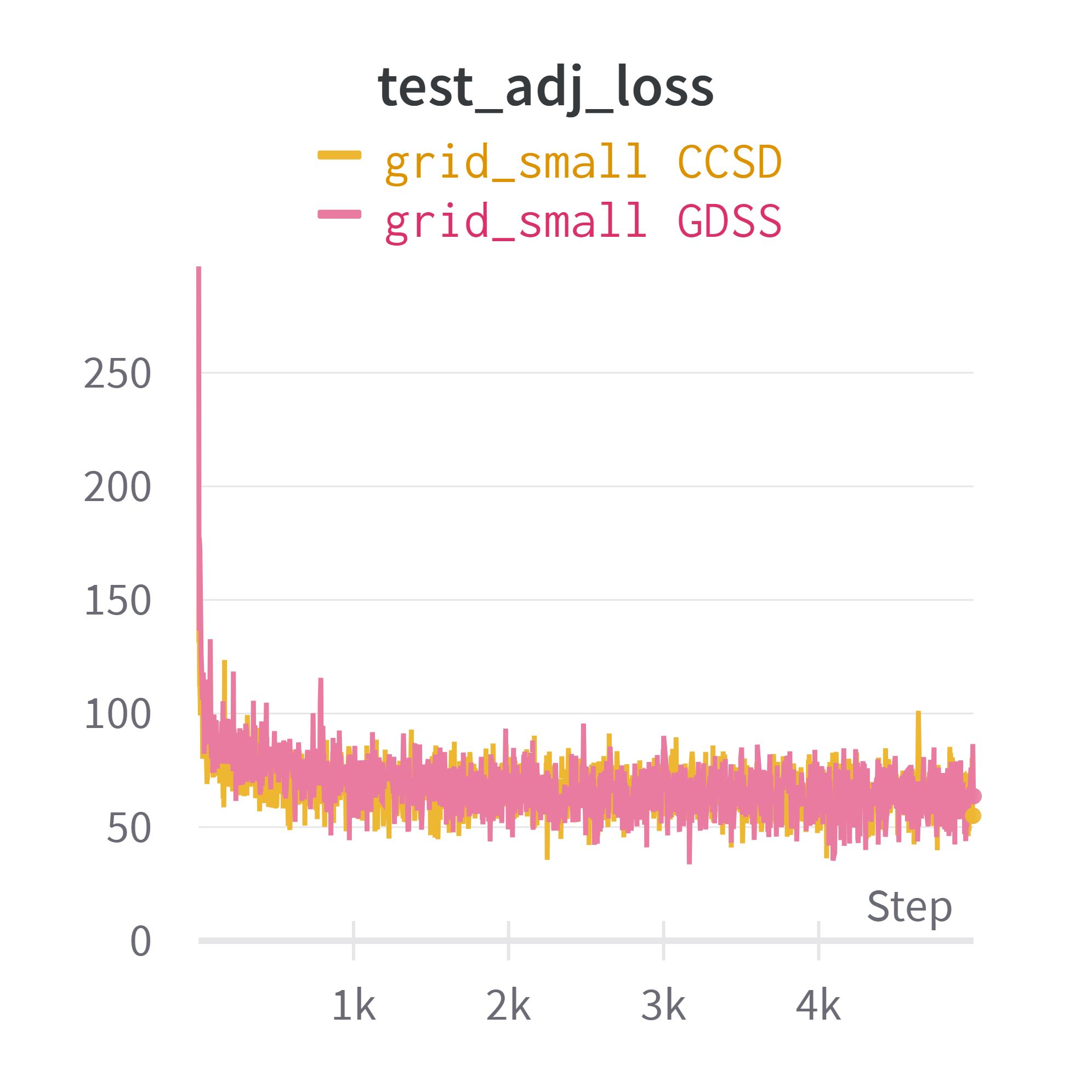}
\end{minipage}%
\begin{minipage}{0.33\textwidth}
  \centering
  \includegraphics[width=1\linewidth]{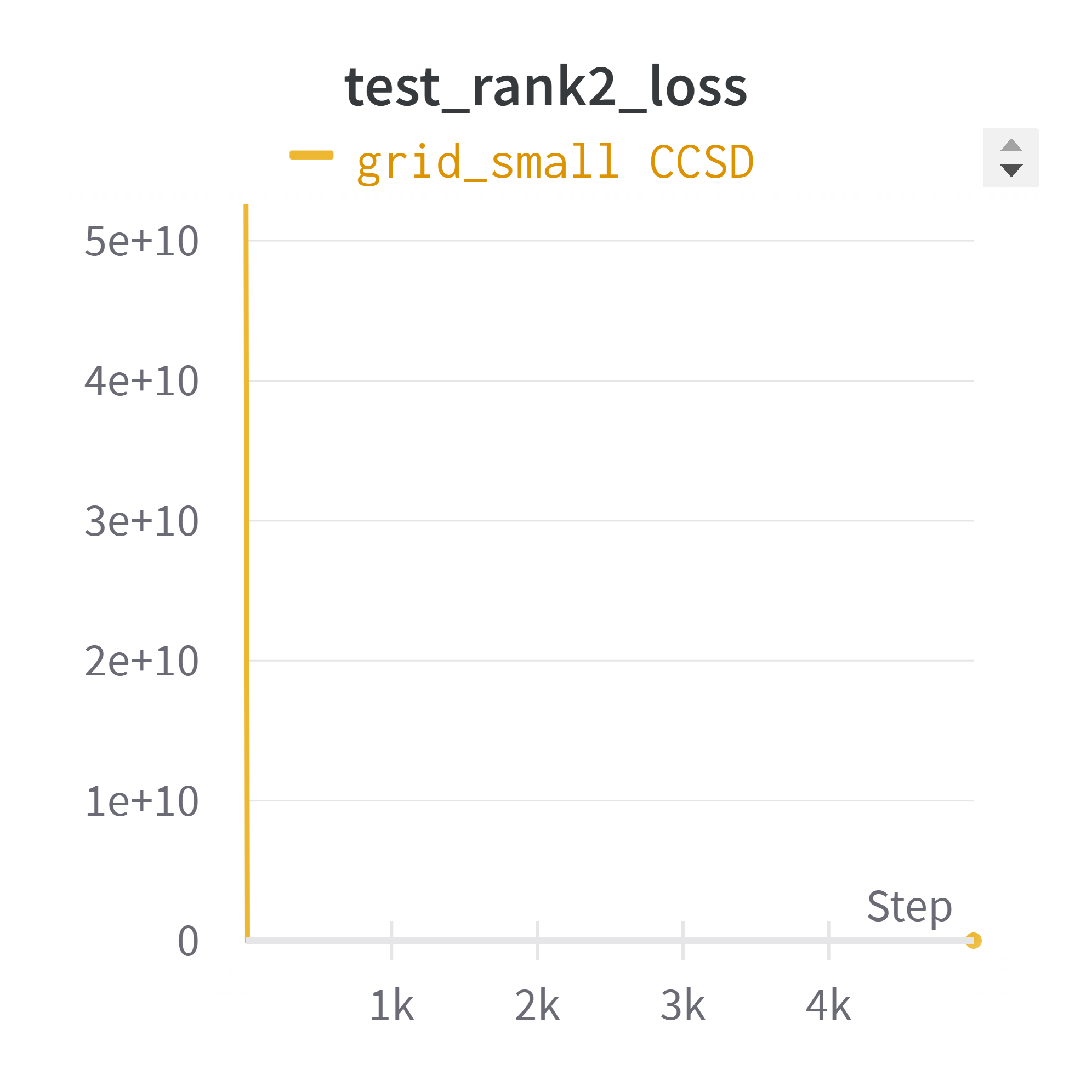}
\end{minipage}
\captionof{figure}{Test losses Grid small.}
\label{fig:test_grid}
\end{figure}

\end{document}